\documentclass{article}

\usepackage[utf8]{inputenc} %
\usepackage[T1]{fontenc}    %
\usepackage{url}            %
\usepackage{booktabs}       %
\usepackage{amsfonts}       %
\usepackage{nicefrac}       %
\usepackage{microtype}      %

\usepackage{tocloft}            %

\usepackage{enumitem}

\usepackage{breakcites}

\usepackage{etoolbox}
\usepackage{comment}
\newtoggle{estimationlb}
\toggletrue{estimationlb}
\newtoggle{draft}
\togglefalse{draft}

\newcommand{\draft}[1]{\iftoggle{draft}{#1}{}}

\usepackage{mathrsfs}

\usepackage{algorithm}
\usepackage{verbatim}
\usepackage[noend]{algpseudocode}
\newcommand{\multiline}[1]{\parbox[t]{\dimexpr\linewidth-\algorithmicindent}{#1}}

\usepackage{multicol}

\usepackage{colortbl}

\usepackage{setspace}

\usepackage{transparent}

\usepackage{inconsolata}
\usepackage[scaled=.90]{helvet}
\usepackage{xspace}

\usepackage{pifont}
\newcommand{\cmark}{\ding{51}}%
\newcommand{\xmark}{\ding{55}}%

\usepackage[letterpaper, left=1in, right=1in, top=1in, bottom=1in]{geometry}

\PassOptionsToPackage{hypertexnames=false}{hyperref}  %
\usepackage{parskip}

\usepackage[dvipsnames]{xcolor}
\usepackage[colorlinks=true, linkcolor=blue!70!black, citecolor=blue!70!black,urlcolor=black,breaklinks=true]{hyperref}
\usepackage{microtype}
\usepackage{hhline}

\makeatletter
\newcommand{\neutralize}[1]{\expandafter\let\csname c@#1\endcsname\count@}
\makeatother

\newenvironment{thmmod}[2]
  {\renewcommand{\thetheorem}{\ref*{#1}#2}%
   \neutralize{theorem}\phantomsection
   \begin{theorem}}
  {\end{theorem}}

    \newenvironment{propmod}[2]
  {\renewcommand{\theproposition}{\ref*{#1}#2}%
   \neutralize{proposition}\phantomsection
   \begin{proposition}}
  {\end{proposition}}

\usepackage{algorithm}

\usepackage{natbib}
\bibpunct{(}{)}{;}{a}{,}{,}

\usepackage{amsthm}
\usepackage{mathtools}
\usepackage{amsmath}
\usepackage{bbm}
\usepackage{amsfonts}
\usepackage{amssymb}

\usepackage{xpatch}

\usepackage{thmtools}
\usepackage{thm-restate}
\declaretheorem[name=Theorem,parent=section]{theorem}
\declaretheorem[name=Lemma,parent=section]{lemma}
\declaretheorem[name=Assumption, parent=section]{assumption}
\declaretheorem[name=Condition, parent=section]{condition}
\declaretheorem[qed=$\triangleleft$,name=Example,style=definition, parent=section]{example}
\declaretheorem[name=Remark,style=definition, parent=section]{remark}
\declaretheorem[name=Proposition, parent=section]{proposition}
\declaretheorem[name=Fact, parent=section]{fact}

\makeatletter
  \renewenvironment{proof}[1][Proof]%
  {%
   \par\noindent{\bfseries\upshape {#1.}\ }%
  }%
  {\qed\newline}
  \makeatother

\theoremstyle{definition}  %

\newtheorem{corollary}{Corollary}[section]

\theoremstyle{plain}
\newtheorem{definition}{Definition}[section]

\xpatchcmd{\proof}{\itshape}{\normalfont\proofnameformat}{}{}
\newcommand{\proofnameformat}{\bfseries}

\usepackage[nameinlink,capitalize]{cleveref}

\newcommand{\pref}[1]{\cref{#1}}
\newcommand{\pfref}[1]{Proof of \pref{#1}}

\crefformat{equation}{#2(#1)#3}
\Crefformat{equation}{#2(#1)#3}

\Crefformat{figure}{#2Figure #1#3}
\Crefformat{assumption}{#2Assumption #1#3}

\usepackage{crossreftools}
\newcommand{\preft}[1]{\crtcref{#1}}

\usepackage{xparse}

\ExplSyntaxOn
\DeclareDocumentCommand{\XDeclarePairedDelimiter}{mm}
 {
  \__egreg_delimiter_clear_keys: %
  \keys_set:nn { egreg/delimiters } { #2 }
  \use:x %
   {
    \exp_not:n {\NewDocumentCommand{#1}{sO{}m} }
     {
      \exp_not:n { \IfBooleanTF{##1} }
       {
        \exp_not:N \egreg_paired_delimiter_expand:nnnn
         { \exp_not:V \l_egreg_delimiter_left_tl }
         { \exp_not:V \l_egreg_delimiter_right_tl }
         { \exp_not:n { ##3 } }
         { \exp_not:V \l_egreg_delimiter_subscript_tl }
       }
       {
        \exp_not:N \egreg_paired_delimiter_fixed:nnnnn 
         { \exp_not:n { ##2 } }
         { \exp_not:V \l_egreg_delimiter_left_tl }
         { \exp_not:V \l_egreg_delimiter_right_tl }
         { \exp_not:n { ##3 } }
         { \exp_not:V \l_egreg_delimiter_subscript_tl }
       }
     }
   }
 }

\keys_define:nn { egreg/delimiters }
 {
  left      .tl_set:N = \l_egreg_delimiter_left_tl,
  right     .tl_set:N = \l_egreg_delimiter_right_tl,
  subscript .tl_set:N = \l_egreg_delimiter_subscript_tl,
 }

\cs_new_protected:Npn \__egreg_delimiter_clear_keys:
 {
  \keys_set:nn { egreg/delimiters } { left=.,right=.,subscript={} }
 }

\cs_new_protected:Npn \egreg_paired_delimiter_expand:nnnn #1 #2 #3 #4
 {%
  \mathopen{}
  \mathclose\c_group_begin_token
   \left#1
   #3
   \group_insert_after:N \c_group_end_token
   \right#2
   \tl_if_empty:nF {#4} { \c_math_subscript_token {#4} }
 }
\cs_new_protected:Npn \egreg_paired_delimiter_fixed:nnnnn #1 #2 #3 #4 #5
 {
  \mathopen{#1#2}#4\mathclose{#1#3}
  \tl_if_empty:nF {#5} { \c_math_subscript_token {#5} }
 }
\ExplSyntaxOff

\XDeclarePairedDelimiter{\supnorm}{
  left=\lVert,
  right=\rVert,
  subscript=\infty
  }
\DeclarePairedDelimiter{\abs}{\lvert}{\rvert} %
\DeclarePairedDelimiter{\brk}{[}{]}
\DeclarePairedDelimiter{\crl}{\{}{\}}
\DeclarePairedDelimiter{\prn}{(}{)}
\DeclarePairedDelimiter{\nrm}{\|}{\|}
\DeclarePairedDelimiter{\tri}{\langle}{\rangle}

\DeclarePairedDelimiter{\ceil}{\lceil}{\rceil}
\DeclarePairedDelimiter{\floor}{\lfloor}{\rfloor}

\let\P\undefined
\DeclareMathOperator{\En}{\mathbb{E}}

\DeclareMathOperator{\P}{P}

\DeclareMathOperator*{\maximize}{maximize} %

\DeclareMathOperator*{\argmin}{arg\,min} %
\DeclareMathOperator*{\argmax}{arg\,max}

\newcommand{\mb}[1]{\boldsymbol{#1}}
\newcommand{\wt}[1]{\widetilde{#1}}
\newcommand{\wh}[1]{\widehat{#1}}
\newcommand{\wb}[1]{\widebar{#1}}

\def\ddefloop#1{\ifx\ddefloop#1\else\ddef{#1}\expandafter\ddefloop\fi}
\def\ddef#1{\expandafter\def\csname bb#1\endcsname{\ensuremath{\mathbb{#1}}}}
\ddefloop ABCDEFGHIJKLMNOPQRSTUVWXYZ\ddefloop
\def\ddefloop#1{\ifx\ddefloop#1\else\ddef{#1}\expandafter\ddefloop\fi}
\def\ddef#1{\expandafter\def\csname b#1\endcsname{\ensuremath{\mathbf{#1}}}}
\ddefloop ABCDEFGHIJKLMNOPQRSTUVWXYZ\ddefloop
\def\ddef#1{\expandafter\def\csname sf#1\endcsname{\ensuremath{\mathsf{#1}}}}
\ddefloop ABCDEFGHIJKLMNOPQRSTUVWXYZ\ddefloop
\def\ddef#1{\expandafter\def\csname c#1\endcsname{\ensuremath{\mathcal{#1}}}}
\ddefloop ABCDEFGHIJKLMNOPQRSTUVWXYZ\ddefloop
\def\ddef#1{\expandafter\def\csname h#1\endcsname{\ensuremath{\widehat{#1}}}}
\ddefloop ABCDEFGHIJKLMNOPQRSTUVWXYZ\ddefloop
\def\ddef#1{\expandafter\def\csname hc#1\endcsname{\ensuremath{\widehat{\mathcal{#1}}}}}
\ddefloop ABCDEFGHIJKLMNOPQRSTUVWXYZ\ddefloop
\def\ddef#1{\expandafter\def\csname t#1\endcsname{\ensuremath{\widetilde{#1}}}}
\ddefloop ABCDEFGHIJKLMNOPQRSTUVWXYZ\ddefloop
\def\ddef#1{\expandafter\def\csname tc#1\endcsname{\ensuremath{\widetilde{\mathcal{#1}}}}}
\ddefloop ABCDEFGHIJKLMNOPQRSTUVWXYZ\ddefloop
\def\ddefloop#1{\ifx\ddefloop#1\else\ddef{#1}\expandafter\ddefloop\fi}
\def\ddef#1{\expandafter\def\csname scr#1\endcsname{\ensuremath{\mathscr{#1}}}}
\ddefloop ABCDEFGHIJKLMNOPQRSTUVWXYZ\ddefloop

\newcommand{\ls}{\ell}
\newcommand{\ind}{\mathbbm{1}}    %

\newcommand{\eps}{\epsilon}
\newcommand{\veps}{\varepsilon}

\newcommand{\ldef}{\vcentcolon=}
\newcommand{\rdef}{=\vcentcolon}

\newcommand{\Mtil}{\wt{M}}

\newcommand{\fmtil}{f\sups{\Mtil}}

\newcommand{\pmbar}{p_{\sMbar}}
\newcommand{\pmdist}{p\subs{M}}
\newcommand{\PM}{\bbP\sups{M}}
\newcommand{\PMbar}{\bbP\sups{\Mbar}}

\newcommand{\delm}{\Delta_{M}}
\newcommand{\delmbar}{\Delta_{\Mbar}}
\newcommand{\gammat}{\frac{\gamma}{T}}

\newcommand{\Dbi}[2]{D_{\mathrm{bi}}\prn*{#1\dmid#2}}
\newcommand{\Dbishort}{D_{\mathrm{bi}}}
\newcommand{\cQm}{\cQ_{\cM}}
\newcommand{\cQmh}[1][h]{\cQ_{\cM;#1}}
\newcommand{\cTm}[1][M]{\cT\sups{#1}}
\newcommand{\cTmstar}[1][\Mstar]{\cT\sups{#1}}

\newcommand{\Confone}{R_1}
\newcommand{\Conftwo}{R_2}
\newcommand{\compbi}{\compgen[\mathrm{bi}]}
\newcommand{\compbidual}{\compgendual[\mathrm{bi}]}

\newcommand{\bloss}{\cL}

\newcommand{\cGmbar}{\cG\sups{\Mbar}}
\newcommand{\cGmtil}{\cG\sups{\Mtil}}
\newcommand{\gmmbar}[1][M]{g\sups{#1;\Mbar}}
\newcommand{\gmmtil}[1][M]{g\sups{#1;\Mtil}}
\newcommand{\Lbr}{L_{\mathrm{br}}}
\newcommand{\Lbrfull}{L_{\mathrm{br}}(\cM)}
\newcommand{\Lbrfulls}{L^{2}_{\mathrm{br}}(\cM)}

\newcommand{\mum}[1][M]{\mu\sups{#1}}
\newcommand{\wm}[1][M]{w\sups{#1}}

\newcommand{\AlgEsth}[1][h]{\mathrm{\mathbf{Alg}}_{\mathsf{Est};#1}}
\newcommand{\EstHelh}[1][h]{\mathrm{\mathbf{Est}}_{\mathsf{H};#1}}

\newcommand{\thetambar}[1][\Mbar]{\theta(#1)}

\newcommand{\Copt}{C_{\mathrm{opt}}}

\newcommand{\Calpha}{C_{\alpha}}

\newcommand{\dimbi}{d_{\mathrm{bi}}}
\newcommand{\Lbi}{L_{\mathrm{bi}}}
\newcommand{\dimbifull}{d_{\mathrm{bi}}(\cM)}
\newcommand{\dimbifulls}[1][2]{d^{#1}_{\mathrm{bi}}(\cM)}
\newcommand{\Lbifull}{L_{\mathrm{bi}}(\cM)}
\newcommand{\Lbifulls}{L^{2}_{\mathrm{bi}}(\cM)}

\newcommand{\piest}{\pi^{\mathrm{est}}}
\newcommand{\piestm}[1][M]{\pi^{\mathrm{est}}_{\sss{#1}}}

\newcommand{\lestm}[1][M]{\ell^{\mathrm{est}}_{\sss{#1}}}

\newcommand{\qbar}{\bar{q}}
\newcommand{\qopt}{q^{\mathrm{opt}}}

\newcommand{\pialpham}[1][M]{\pi^{\alpha}_{\sss{#1}}}

\newcommand{\con}{x}
\newcommand{\Cspace}{\cX}

\newcommand{\cMx}[1][x]{\cM|_{#1}}
\newcommand{\cMhatx}[1][x]{\cMhat|_{#1}}

\newcommand{\cpol}{\mb{\pi}}
\newcommand{\cpolstar}{\cpol^{\star}}
\newcommand{\cpolm}[1][M]{\cpol\subs{#1}}

\newcommand{\RegCDM}{\Reg_{\mathsf{DM}}}
\newcommand{\EstCD}{\mathrm{\mathbf{Est}}_{\mathsf{D}}}

\newcommand{\relu}{\mathsf{relu}}

\newcommand{\sdis}{\mb{\theta}}

\newcommand{\Delm}{\Delta\subs{M}}

\newcommand{\Star}{\mathfrak{s}}
\newcommand{\StarCheck}{\check{\Star}}
\newcommand{\StarDim}{\Star}

\newcommand{\El}{\mathfrak{e}}
\newcommand{\ElCheck}{\check{\El}}
\newcommand{\ElDim}{\El}

\newcommand{\errm}[1][M]{\mathrm{err}\sups{#1}}

\newcommand{\hardfamily}{$(\alpha,\beta,\delta)$-family\xspace}
\newcommand{\hardu}{u}
\newcommand{\hardv}{v}

\newcommand{\igwtext}{inverse gap weighting\xspace}
\newcommand{\pcigw}{\textsf{PC-IGW}\xspace}
\newcommand{\pcigwb}{\textsf{PC-IGW.Bilinear}\xspace}

\newcommand{\igwspanner}{\textsf{IGW-Spanner}\xspace}
\newcommand{\igwargmax}{\textsf{IGW-ArgMax}\xspace}

\newcommand{\oracle}{\mathrm{\mathbf{Alg}}_{\mathsf{Plan}}\xspace}

\newcommand{\I}{\mathrm{Err}_{\mathrm{I}}}
\newcommand{\II}{\mathrm{Err}_{\mathrm{II}}}
\newcommand{\III}{\mathrm{Err}_{\mathrm{III}}}

\newcommand{\MinimaxReg}{\mathfrak{M}(\cM,T)}
\newcommand{\MinimaxRegBayes}{\underbar{\mathfrak{M}}(\cM,T)}

\renewcommand{\emptyset}{\varnothing}

\newcommand{\NullObs}{\crl{\emptyset}}

\newcommand{\filt}{\mathscr{F}}

\newcommand{\hist}{\mathcal{H}}

\newcommand{\Asig}{\mathscr{P}}
\newcommand{\Rsig}{\mathscr{R}}
\newcommand{\Osig}{\mathscr{O}}

\newcommand{\Hspace}{\Omega}
\newcommand{\Hsig}{\filt}

\newcommand{\dom}{\nu}

\newcommand{\densm}[1][M]{m^{\sss{#1}}}

\newcommand{\abscont}{V(\cM)}
\newcommand{\abscontp}{V(\cM')}

\newcommand{\Ct}{C(T)}

\newcommand{\Vmbar}{V^{\sMbar}}

\newcommand{\vepsg}{\veps_{\gamma}}

\newcommand{\vepsupg}{\bar{\veps}_{\gamma}}
\newcommand{\vepslowg}{\ubar{\veps}_{\gamma}}

\newcommand{\gameval}[1]{\cV_{\gamma}\sups{#1}}

\newcommand{\Framework}{Decision Making with Structured Observations\xspace}
\newcommand{\FrameworkShort}{DMSO\xspace}

\newcommand{\learner}{learner\xspace}

\newcommand{\Cloc}{c_{\ell}}

\newcommand{\act}{\pi}
\newcommand{\Act}{\Pi}

\newcommand{\ActSize}{\abs{\Pi}}

\newcommand{\actstar}{\pistar}

\newcommand{\obs}{o}
\newcommand{\Obs}{\mathcal{\cO}}
\newcommand{\ObsSpace}{\mathcal{\cO}}
\newcommand{\Ospace}{\mathcal{\cO}}

\newcommand{\RewardSpace}{\cR}
\newcommand{\RSpace}{\RewardSpace}
\newcommand{\Rspace}{\RewardSpace}

\newcommand{\cMhat}{\wh{\cM}}

\newcommand{\compbasic}{\mathsf{dec}}
\newcommand{\comp}[1][\gamma]{\mathsf{dec}_{#1}}

\newcommand{\compb}[1][\gamma]{\underbar{\mathsf{dec}}_{#1}}
\newcommand{\compdual}[1][\gamma]{\compb[#1]}

\newcommand{\compKL}{\compgen[\mathrm{KL}]}
\newcommand{\compKLdual}{\compgendual[\mathrm{KL}]}

\newcommand{\compSq}{\compgen[\mathrm{Sq}]}

\newcommand{\compSqdual}{\compgendual[\mathrm{Sq}]}

\newcommand{\compH}{\compgen[\mathrm{H}]}

\newcommand{\complocshort}{\comp^{\mathrm{loc}}}
\newcommand{\comploc}[2][\gamma]{\mathsf{dec}_{#1,#2}}

\newcommand{\compthmone}{\comploc{\vepslowg}(\cM)}

\newcommand{\CompText}{Decision-Estimation Coefficient\xspace}
\newcommand{\CompAbbrev}{DEC\xspace}
\newcommand{\CompShort}{\CompAbbrev}

\newcommand{\compgen}[1][D]{\comp^{#1}}
\newcommand{\compgendual}[1][D]{\compb^{#1}}
\newcommand{\compd}[1][D]{\comp^{#1}}

\newcommand{\est}{\mathsf{est}}

\newcommand{\AlgText}{Estimation-to-Decisions\xspace}
\newcommand{\mainalg}{\textsf{E\protect\scalebox{1.04}{2}D}\xspace}
\newcommand{\mainalgsection}{\texorpdfstring{\mainalg}{\textsf{E2D}\xspace}}%

\newcommand{\mainalgB}{\textsf{E\protect\scalebox{1.04}{2}D.Bayes}\xspace}

\newcommand{\M}[1]{^{{\scriptscriptstyle M}}}  %
\newcommand{\sM}{\sss{M}}
\newcommand{\sMbar}{\sss{\Mbar}}
 
\newcommand{\sups}[1]{^{{\scriptscriptstyle#1}}}
\newcommand{\subs}[1]{_{{\scriptscriptstyle#1}}}

\newcommand{\sss}[1]{{\scriptscriptstyle#1}}

\newcommand{\Ens}[2]{\En^{\sss{#1},#2}}
\newcommand{\Em}[1][M]{\En^{\sss{#1}}}
\newcommand{\Embar}{\En\sups{\Mbar}}
\newcommand{\Enm}[2]{\En^{\sss{#1},#2}}

\newcommand{\fmhatt}{f\sups{\Mhat\ind{t}}}

\newcommand{\fm}[1][M]{f\sups{#1}}
\newcommand{\pim}[1][M]{\pi_{\sss{#1}}}
\newcommand{\pimh}[1][h]{\pi_{\sss{M,#1}}}

\newcommand{\gm}[1][M]{g\sups{#1}}
\newcommand{\gmbar}[1][\Mbar]{g\sups{#1}}

\newcommand{\cFm}{\cF_{\cM}}
\newcommand{\Pim}{\Pi_{\cM}}
\newcommand{\cMf}{\cM_{\cF}}

\newcommand{\fmbar}{f\sups{\Mbar}}
\newcommand{\pimbar}{\pi\subs{\Mbar}}

\newcommand{\fmstar}{f\sups{\Mstar}}
\newcommand{\pimstar}{\pi\subs{\Mstar}}

\newcommand{\fstar}{f^{\star}}
\newcommand{\pistar}{\pi^{\star}}

\newcommand{\pihat}{\wh{\pi}}

\newcommand{\cMloc}[1][\veps]{\cM_{#1}}
\newcommand{\cMinf}[1][\veps]{\cM^{\infty}_{#1}}

\newcommand{\Mbar}{\wb{M}}

\newcommand{\fmi}{f\sups{M_i}}
\newcommand{\pimi}{\pi\subs{M_i}}

\newcommand{\Rm}[1][M]{R\sups{#1}}
\newcommand{\Pm}[1][M]{P\sups{#1}}
\newcommand{\Pmstar}[1][\Mstar]{P\sups{#1}}

\newcommand{\Pmbar}{P\sups{\Mbar}}
\newcommand{\Rmbar}{R\sups{\Mbar}}

\newcommand{\Prm}[2]{\bbP^{\sss{#1},#2}}

\newcommand{\PiRNS}{\Pi_{\mathrm{RNS}}} %
\newcommand{\PiGen}{\Pi_{\mathrm{RNS}}} %

\newcommand{\Qmistar}[1][i]{Q^{#1,\star}}
\newcommand{\Vmistar}[1][i]{V^{#1,\star}}

\newcommand{\Qmstar}[1][M]{Q^{\sss{#1},\star}}
\newcommand{\Vmstar}[1][M]{V^{\sss{#1},\star}}
\newcommand{\Qmpi}[1][\pi]{Q^{\sss{M},#1}}
\newcommand{\Vmpi}[1][\pi]{V^{\sss{M},#1}}

\newcommand{\dmpi}[1][\pi]{d^{\sss{M},#1}}

\newcommand{\dm}[2]{d^{\sss{#1},#2}}
\newcommand{\dbar}{\bar{d}}

\newcommand{\Reg}{\mathrm{\mathbf{Reg}}}

\newcommand{\EstHel}{\mathrm{\mathbf{Est}}_{\mathsf{H}}}

\newcommand{\EstProbHel}{\mathrm{\mathbf{Est}}_{\mathsf{H}}(T,\delta)}
\newcommand{\EstH}{\EstHel}

\newcommand{\EstD}{\mathrm{\mathbf{Est}}_{\mathsf{D}}}

\newcommand{\EstSq}{\mathrm{\mathbf{Est}}_{\mathsf{Sq}}}

\newcommand{\RegDM}{\Reg_{\mathsf{DM}}}

  \newcommand{\AlgEst}{\mathrm{\mathbf{Alg}}_{\mathsf{Est}}}

  \newcommand{\logloss}{\ls_{\mathrm{\log}}}

  \newcommand{\RegLog}{\Reg_{\mathsf{KL}}}

    \newcommand{\MComp}{\mathsf{est}(\cM,T)}

    \newcommand{\GComp}{\mathsf{est}(\cG,T)}

    \newcommand{\QComp}{\mathsf{est}(\cQm,T)}
    \newcommand{\Qcomp}{\mathsf{est}(\cQm,T)}

    \newcommand{\PComph}[1][h]{\mathsf{est}(\cP_{#1},T)}
    
    \newcommand{\ActComp}{\mathsf{est}(\Pim,T)}

    \newcommand{\ActCov}[1][\veps]{\cN(\Pim,#1)}

\newcommand{\Mhat}{\wh{M}}
\newcommand{\Mstar}{M^{\star}}

  \newcommand{\bbPl}{\bbP_{\lambda}}

  \newcommand{\ghat}{\wh{g}}

\newcommand{\tens}{\otimes}

\newcommand{\Rbar}{\wb{R}}

\newcommand{\Pbar}{\widebar{P}}

\newcommand{\optionone}{\textsc{Option I}\xspace}
\newcommand{\optiontwo}{\textsc{Option II}\xspace}

\newcommand{\algcomment}[1]{\textcolor{blue!70!black}{\small{\texttt{\textbf{//\hspace{2pt}#1}}}}}
\newcommand{\algcommentlight}[1]{\textcolor{blue!70!black}{\transparent{0.5}\small{\texttt{\textbf{//\hspace{2pt}#1}}}}}

\newcommand{\Holder}{H{\"o}lder\xspace}

\newcommand{\midsem}{\,;}

\newcommand{\trn}{\top}
\newcommand{\pinv}{\dagger}

\newcommand{\psdgeq}{\succeq}

\newcommand{\approxleq}{\lesssim}
\newcommand{\approxgeq}{\gtrsim}

\newcommand{\fhat}{\wh{f}}

\newcommand{\fbar}{\bar{f}}

\renewcommand{\ind}[1]{^{{\scriptscriptstyle(#1)}}}

\newcommand{\bigoh}{O}
\newcommand{\bigoht}{\wt{O}}
\newcommand{\bigom}{\Omega}
\newcommand{\bigomt}{\wt{\Omega}}
\newcommand{\bigthetat}{\wt{\Theta}}

\newcommand{\indic}{\mathbb{I}}

\newcommand{\poly}{\mathrm{poly}}
\newcommand{\polylog}{\mathrm{polylog}}

\newcommand{\Dsq}[2]{D_{\mathrm{Sq}}\prn*{#1,#2}}
\newcommand{\Dsqshort}{D_{\mathrm{Sq}}}

\newcommand{\kl}[2]{D_{\mathrm{KL}}\prn*{#1\,\|\,#2}}
\newcommand{\Dkl}[2]{D_{\mathrm{KL}}\prn*{#1\,\|\,#2}}
\newcommand{\Dklshort}{D_{\mathrm{KL}}}
\newcommand{\Dhel}[2]{D_{\mathrm{H}}\prn*{#1,#2}}
\newcommand{\Dgen}[2]{D\prn*{#1\dmid{}#2}}

\newcommand{\Dhels}[2]{D^{2}_{\mathrm{H}}\prn*{#1,#2}}
\newcommand{\Dhelshort}{D^{2}_{\mathrm{H}}}

\newcommand{\Dtv}[2]{D_{\mathrm{TV}}\prn*{#1,#2}}
\newcommand{\Dtvs}[2]{D^2_{\mathrm{TV}}\prn*{#1,#2}}

\newcommand{\DhelsX}[3]{D^{2}_{\mathrm{H}}\prn[#1]{#2,#3}}

\newcommand{\Ber}{\mathrm{Ber}}
\newcommand{\Rad}{\mathrm{Rad}}
\newcommand{\dmid}{\;\|\;}

\newcommand{\conv}{\mathrm{co}}
\newcommand{\diam}{\mathrm{diam}}

\newcommand{\squarecb}{\textsf{SquareCB}\xspace}

\newcommand{\Qstar}{Q^{\star}}

\newcommand{\unif}{\mathrm{unif}}

\newcommand{\starhull}{\mathrm{star}}

\newcommand{\Phat}{\wh{P}}

\newcommand{\piunif}{\pi_{\mathrm{unif}}}

\newcommand{\supp}{\mathrm{supp}}

\newcommand{\astar}{a^{\star}}

\newcommand{\mathand}{\quad\text{and}\quad}
\newcommand{\mathwhere}{\quad\text{where}\quad}

\def\multiset#1#2{\ensuremath{\left(\kern-.3em\left(\genfrac{}{}{0pt}{}{#1}{#2}\right)\kern-.3em\right)}}

\newcommand{\iid}{i.i.d.\xspace}

\renewcommand{\ls}{\ell}

\newcommand{\glcomp}{c}

\newcommand{\InfB}{\mathcal{I}_{\mathrm{B}}}
\newcommand{\InfF}{\mathcal{I}_{\mathrm{F}}}

\renewcommand{\emptyset}{\varnothing}

\newcommand{\phat}{\wh{p}}

\newcommand{\Meps}{M_{\veps}}

 \let\underbar\undefined
\makeatletter
\let\save@mathaccent\mathaccent
\newcommand*\if@single[3]{%
  \setbox0\hbox{${\mathaccent"0362{#1}}^H$}%
  \setbox2\hbox{${\mathaccent"0362{\kern0pt#1}}^H$}%
  \ifdim\ht0=\ht2 #3\else #2\fi
  }
\newcommand*\rel@kern[1]{\kern#1\dimexpr\macc@kerna}
\newcommand*\widebar[1]{\@ifnextchar^{{\wide@bar{#1}{0}}}{\wide@bar{#1}{1}}}
\newcommand*\underbar[1]{\@ifnextchar_{{\under@bar{#1}{0}}}{\under@bar{#1}{1}}}
\newcommand*\wide@bar[2]{\if@single{#1}{\wide@bar@{#1}{#2}{1}}{\wide@bar@{#1}{#2}{2}}}
\newcommand*\under@bar[2]{\if@single{#1}{\under@bar@{#1}{#2}{1}}{\under@bar@{#1}{#2}{2}}}
\newcommand*\wide@bar@[3]{%
  \begingroup
  \def\mathaccent##1##2{%
    \let\mathaccent\save@mathaccent
    \if#32 \let\macc@nucleus\first@char \fi
    \setbox\z@\hbox{$\macc@style{\macc@nucleus}_{}$}%
    \setbox\tw@\hbox{$\macc@style{\macc@nucleus}{}_{}$}%
    \dimen@\wd\tw@
    \advance\dimen@-\wd\z@
    \divide\dimen@ 3
    \@tempdima\wd\tw@
    \advance\@tempdima-\scriptspace
    \divide\@tempdima 10
    \advance\dimen@-\@tempdima
    \ifdim\dimen@>\z@ \dimen@0pt\fi
    \rel@kern{0.6}\kern-\dimen@
    \if#31
      \overline{\rel@kern{-0.6}\kern\dimen@\macc@nucleus\rel@kern{0.4}\kern\dimen@}%
      \advance\dimen@0.4\dimexpr\macc@kerna
      \let\final@kern#2%
      \ifdim\dimen@<\z@ \let\final@kern1\fi
      \if\final@kern1 \kern-\dimen@\fi
    \else
      \overline{\rel@kern{-0.6}\kern\dimen@#1}%
    \fi
  }%
  \macc@depth\@ne
  \let\math@bgroup\@empty \let\math@egroup\macc@set@skewchar
  \mathsurround\z@ \frozen@everymath{\mathgroup\macc@group\relax}%
  \macc@set@skewchar\relax
  \let\mathaccentV\macc@nested@a
  \if#31
    \macc@nested@a\relax111{#1}%
  \else
    \def\gobble@till@marker##1\endmarker{}%
    \futurelet\first@char\gobble@till@marker#1\endmarker
    \ifcat\noexpand\first@char A\else
      \def\first@char{}%
    \fi
    \macc@nested@a\relax111{\first@char}%
  \fi
  \endgroup
}
\newcommand*\under@bar@[3]{%
  \begingroup
  \def\mathaccent##1##2{%
    \let\mathaccent\save@mathaccent
    \if#32 \let\macc@nucleus\first@char \fi
    \setbox\z@\hbox{$\macc@style{\macc@nucleus}_{}$}%
    \setbox\tw@\hbox{$\macc@style{\macc@nucleus}{}_{}$}%
    \dimen@\wd\tw@
    \advance\dimen@-\wd\z@
    \divide\dimen@ 3
    \@tempdima\wd\tw@
    \advance\@tempdima-\scriptspace
    \divide\@tempdima 10
    \advance\dimen@-\@tempdima
    \ifdim\dimen@>\z@ \dimen@0pt\fi
    \rel@kern{0.6}\kern-\dimen@
    \if#31
      \underline{\rel@kern{-0.6}\kern\dimen@\macc@nucleus\rel@kern{0.4}\kern\dimen@}%
      \advance\dimen@0.4\dimexpr\macc@kerna
      \let\final@kern#2%
      \ifdim\dimen@<\z@ \let\final@kern1\fi
      \if\final@kern1 \kern-\dimen@\fi
    \else
      \underline{\rel@kern{-0.6}\kern\dimen@#1}%
    \fi
  }%
  \macc@depth\@ne
  \let\math@bgroup\@empty \let\math@egroup\macc@set@skewchar
  \mathsurround\z@ \frozen@everymath{\mathgroup\macc@group\relax}%
  \macc@set@skewchar\relax
  \let\mathaccentV\macc@nested@a
  \if#31
    \macc@nested@a\relax111{#1}%
  \else
    \def\gobble@till@marker##1\endmarker{}%
    \futurelet\first@char\gobble@till@marker#1\endmarker
    \ifcat\noexpand\first@char A\else
      \def\first@char{}%
    \fi
    \macc@nested@a\relax111{\first@char}%
  \fi
  \endgroup
}
\makeatother

\usepackage[suppress]{color-edits}
 \addauthor{df}{ForestGreen}
 \addauthor{dk}{magenta}
 \addauthor{sr}{red}
 \addauthor{jq}{yellow!60!black}
 \draft{

 \usepackage[textsize=tiny]{todonotes}

 }

\makeatletter
\let\OldStatex\Statex
\renewcommand{\Statex}[1][3]{%
  \setlength\@tempdima{\algorithmicindent}%
  \OldStatex\hskip\dimexpr#1\@tempdima\relax}
\makeatother

 \usepackage{accents}
 \newcommand{\ubar}[1]{\underaccent{\bar}{#1}}

\let\oldparagraph\paragraph
\renewcommand{\paragraph}[1]{\oldparagraph{#1.}}

\title{The Statistical Complexity of Interactive Decision Making}

  \author{%
    Dylan J. Foster\\%
{\small\texttt{dylanfoster@microsoft.com}}
\and
Sham M. Kakade\\%
{\small\texttt{sham@seas.harvard.edu}}
\and
Jian Qian\\%
{\small\texttt{jianqian@mit.edu}}
\and
Alexander Rakhlin\\%
{\small\texttt{rakhlin@mit.edu}}
}

\date{}

\begin{document}
\maketitle

\begin{abstract}

A fundamental challenge in interactive learning and decision making, ranging from bandit problems to reinforcement learning, is to provide sample-efficient, adaptive learning algorithms that achieve near-optimal regret. This question is analogous to the classical problem of optimal (supervised) statistical learning, where there are well-known complexity measures (e.g., VC dimension and Rademacher complexity) that govern the statistical complexity of learning.
However, characterizing the statistical complexity of interactive learning is substantially more challenging due to the adaptive nature of the problem. The main result of this work provides a complexity measure, 
the \emph{\CompText}, that is proven to be both \emph{necessary} and \emph{sufficient} for sample-efficient interactive learning.
In particular, we provide:
\begin{itemize}
\item a lower bound on the optimal regret for \emph{any} interactive decision making problem, establishing the \CompText as a fundamental limit.
\item a unified algorithm design principle, \emph{\AlgText} (\mainalg), which transforms any algorithm for supervised estimation into an online algorithm for decision making. \mainalg attains a regret bound that matches our lower bound up to dependence on a notion of estimation performance, thereby achieving optimal sample-efficient learning as characterized by the \CompText.
\end{itemize}
Taken together, these results constitute a theory of learnability for interactive decision making.  When applied to reinforcement learning settings, the \CompText recovers essentially all existing hardness results and lower bounds.
More broadly, the approach can be viewed as a decision-theoretic analogue of the classical Le Cam theory of statistical estimation; it also unifies a number of existing approaches---both Bayesian and frequentist.

%

%
%
%
%

 \end{abstract}

{
\setlength\cftbeforesecskip{.8em} %
  
\hypersetup{linkcolor=blue!60!black}
\tableofcontents
}
\addtocontents{toc}{\protect\setcounter{tocdepth}{2}}

\clearpage

\section{Introduction}
\label{sec:intro}

Over the last decade, algorithms for data-driven decision making (in
particular, contextual bandits and reinforcement learning) have
achieved impressive empirical results in application domains ranging from online personalization
\citep{agarwal2016making,tewari2017ads}  to game-playing
\citep{mnih2015human, silver2016mastering}, robotics
\citep{kober2013reinforcement,lillicrap2015continuous}, and dialogue
systems \citep{li2016deep}. Algorithm design and sample complexity for
data-driven decision making have a relatively complete theory for
problems with small state and action spaces or short horizon. However,
many of the most compelling applications necessitate long-term planning in
high-dimensional spaces, where function approximation is required, and where exploration,
generalization, and sample efficiency remain major challenges. For
real-world problems, where data is limited, it is critical that we bridge this gap and develop 
sample-efficient decision making methods capable of exploring large
state and action spaces.

With a focus on reinforcement learning, a growing body
of research identifies specific settings in which sample-efficient
interactive decision making is possible, typically under conditions that
control the interplay between system dynamics and function approximation
\citep{dean2020sample,yang2019sample,jin2020provably,modi2020sample,ayoub2020model,krishnamurthy2016pac,du2019latent,li2009unifying,dong2019provably,zhou2021nearly}.
While these results highlight a
number of important special cases (e.g., linear function approximation), it is desirable from a practical perspective to develop 
algorithms that accommodate \emph{general-purpose} function
approximation, similar to what one expects in supervised,
statistical learning.  This leads to significant
research challenges:

 \begin{enumerate}
     \item \emph{Sample complexity and fundamental limits.} In statistical learning, the
 classical Vapnik-Chervonenkis (VC) theory provides complexity
 measures (e.g., VC dimension and Rademacher complexity) that upper bound the number of samples required to achieve a desired accuracy level, as well as fundamental
 limits. In data-driven
 decision making, we lack general tools that can be systematically
 applied to understand sample complexity for new problem domains.
As a result, it is
often far from obvious whether existing algorithms are optimal, or to what extent they can be improved. %
  \item \emph{Algorithmic principles.} Can we design general-purpose
    algorithms for data-driven decision making that can take any class
    of models or policies as input and produce accurate decisions out of the box?  In statistical
 learning we have
 universal algorithmic principles such as empirical risk
 minimization (taking the function or model that best fits the data) that can be applied to any problem. In data-driven decision making, algorithm design
 has largely proceeded on a case-by-case basis, and developing
 reliable, provable algorithms for even the simplest models often
 requires non-trivial mathematical insights.
\end{enumerate}
Toward a general theory, a recent line of research proposes structural conditions that attempt to unify existing approaches
      to sample-efficient reinforcement learning
      \citep{russo2013eluder,jiang2017contextual,sun2019model,wang2020provably,du2021bilinear,jin2021bellman}.
      These conditions are not known to be necessary for learning, and they do not recover existing
      hardness results \citep{weisz2021exponential,wang2021exponential}.
      In fact,
      we do not yet have a unified understanding of sample complexity and algorithm
      design even for the basic
      bandit problem (that is, reinforcement learning with horizon one)
      when the action space is structured and high-dimensional.

      We address issues (1) and (2) through a
      two-pronged approach. We introduce a general framework for
      interactive, online decision making,
      \emph{\Framework}, which subsumes structured (high-dimensional) bandits,
      reinforcement learning, partially
      observed Markov decision processes (POMDPs), and
      beyond. We provide a new complexity measure, the
      \emph{\CompText} (\CompAbbrev) and show that it is a
      fundamental limit that lower bounds the sample complexity for
      any interactive decision making problem.
      We complement this result with a
      universal algorithm design principle, \emph{\AlgText}
      (\mainalg), which achieves optimal sample complexity as
      characterized by the \CompText whenever a notion of
        ``estimation complexity'' for the problem under consideration is bounded.
      Together, these results provide the first theory of
      learnability for interactive decision making.

\subsection{Framework: \Framework}
We consider a general framework for interactive decision making,
which we refer to as \emph{\Framework} (\FrameworkShort). The protocol proceeds in $T$
rounds, where for each round $t=1,\ldots,T$:
\begin{enumerate}
\item The \learner selects a \emph{decision} $\act\ind{t}\in\Act$,
  where $\Act$ is the \emph{decision space}.
  \item Nature selects a \emph{reward} $r\ind{t}\in\RewardSpace$ and
  \emph{observation} $\obs\ind{t}\in\ObsSpace$ based on the decision,
  where $\RSpace\subseteq\bbR$ is the \emph{reward space} and $\ObsSpace$ is the \emph{observation space}. The reward and
  observation are then observed by the learner. \looseness=-1
\end{enumerate}

We focus on a stochastic variant of the \FrameworkShort framework: At each timestep, the pair
$(r\ind{t}, \obs\ind{t})$ is drawn independently from an unknown
distribution $\Mstar(\pi\ind{t})$, where
$\Mstar:\Pi\to\Delta(\cR\times\cO)$ is a \emph{model} that maps
decisions to distributions over outcomes.
To facilitate the use of learning and function approximation,
we assume the \learner has access to a \emph{model class} $\cM$ that
attempts to capture the model
$\Mstar$.
Depending on the
problem domain, $\cM$ might consist of linear models, neural networks,
random forests, or
other complex function approximators. We make the following standard realizability
assumption
\citep{agarwal2012contextual,foster2020beyond,du2021bilinear},
which asserts that $\cM$ is flexible enough to express the true model.%
\begin{assumption}[Realizability]
  \label{ass:realizability}
  The model class $\cM$ contains the true model $\Mstar$.
\end{assumption}
For a model $M\in\cM$, let $\Enm{M}{\pi}\brk*{\cdot}$ denote the expectation under
$(r,\obs)\sim{}M(\pi)$.
Further, let $\fm(\pi)\ldef{}\Enm{M}{\pi}\brk*{r}$ denote the
mean reward function and $\pim\ldef{}\argmax_{\act\in\Act}\fm(\act)$ denote the decision with
the greatest expected reward. Finally, we define $\Pim\ldef\crl*{\pim\mid{}M\in\cM}$ as the
induced set of (potentially) optimal decisions and define $\cFm\ldef{}\crl*{\fm\mid{}M\in\cM}$ as the induced class of mean reward functions.

We evaluate the \learner's performance in terms of \emph{regret} to the optimal
decision for $\Mstar$:
\begin{equation}
  \label{eq:regret}
  \RegDM
  \ldef \sum_{t=1}^{T}\En_{\act\ind{t}\sim{}p\ind{t}}\brk*{\fstar(\pistar) - \fstar(\act\ind{t})},
\end{equation}
where we abbreviate $\fstar=\fmstar$ and
$\pistar=\pimstar$, and where $p\ind{t}\in\Delta(\Act)$ is the
learner's distribution over decisions at round $t$.\footnote{
  Our results are most conveniently stated in terms of this notion of
  regret, but immediately extend to the empirical regret $\sum_{t=1}^{T}r\ind{t}(\pistar)-r\ind{t}(\act\ind{t})$
  via standard concentration arguments.
  }

In spite of the apparent simplicity, the \FrameworkShort framework is
general enough to capture most online decision making problems. We focus on two special
cases: structured bandits and reinforcement learning.
\begin{example}[Structured bandits]
  \label{ex:bandit}
When there are no observations
(i.e., $\ObsSpace=\crl*{\emptyset}$), the \FrameworkShort framework is equivalent to the well-known
\emph{structured bandit} problem
\citep{combes2017minimal,degenne2020structure,jun2020crush}. Here,
adopting standard terminology, $\act\ind{t}$ is referred to as an
\emph{action} or \emph{arm} (rather
than a decision), and $\Act$
is referred to as the \emph{action space}.
The structured bandit problem is simplest special case of our
framework, but is extremely rich, and includes the following
well-studied problems:
\begin{itemize}
  \item The classical finite-armed bandit problem with $A$ actions, where
    $\Act=\crl{1,\ldots,A}$ and $\cFm=\bbR^{A}$ \citep{lai1985asymptotically,burnetas1996optimal,kaufmann2016complexity,garivier2019explore}.
  \item The linear bandit problem, in which $\Act\subseteq\bbR^{d}$ and $\cFm$ is a class of linear
    functions \citep{abe1999associative,auer2002finite,dani2008stochastic,chu2011contextual,abbasi2011improved}.
  \item Bandit convex optimization (or, zeroth order optimization), where $\Act\subseteq\bbR^{d}$ and $\cFm$
    is a class of concave\footnote{Here, $\cFm$ is concave rather
      than convex because we
      work with rewards rather than losses.} functions \citep{kleinberg2004nearly,flaxman2005online,agarwal2013stochastic,bubeck2016multi,bubeck2017kernel,lattimore2020improved}.
  \item Nonparametric bandits, where $\cFm$ is a class of Lipschitz or
    \Holder functions over a metric space \citep{kleinberg2004nearly,auer2007improved,kleinberg2019bandits,bubeck2011x,magureanu2014lipschitz,combes2017minimal}.
\end{itemize}
Despite significant research effort, there is no general theory
characterizing what properties of the action space and model class
determine the minimax rates for the structured bandit problem. Outside of the asymptotic regime
\citep{combes2017minimal,degenne2020structure,jun2020crush}, progress has proceeded case by case.

\end{example}
A more challenging setting is reinforcement learning (with function
approximation). 
\begin{example}[Online reinforcement learning]
  \label{ex:rl}
  We consider an episodic finite-horizon reinforcement
  learning setting. With $H$ denoting the horizon, each model $M\in\cM$ specifies a non-stationary Markov decision process
  $M=\crl*{\crl{\cS_h}_{h=1}^{H}, \cA, \crl{\Pm_h}_{h=1}^{H}, \crl{\Rm_h}_{h=1}^{H},
    d_1}$, where $\cS_h$ is the state space for layer $h$, $\cA$ is the action space,
  $\Pm_h:\cS_h\times\cA\to\Delta(\cS_{h+1})$ is the probability transition
  kernel for layer $h$, $\Rm_h:\cS_h\times\cA\to\Delta(\bbR)$ is
  the reward distribution for layer $h$, and $d_1\in\Delta(\cS_1)$ is the initial
  state distribution. We allow
 the reward distribution and transition
  kernel to vary across models in $\cM$, but assume for simplicity that the initial
  state distribution is fixed.\looseness=-1

  For a fixed MDP $M\in\cM$, each episode proceeds under the
  following protocol.
  At the beginning of the episode, the learner selects a
  randomized, non-stationary \emph{policy}
  $\pi=(\pi_1,\ldots,\pi_H)\in\PiRNS$, where $\pi_h:\cS_h\to\Delta(\cA)$ and
  $\PiRNS$ denotes the set of all such policies. The episode
  then evolves through the following process, beginning from
  $s_1\sim{}d_1$. For $h=1,\ldots,H$:
  \begin{itemize}
  \item $a_h\sim\pi_h(s_h)$.
  \item $r_h\sim\Rm(s_h,a_h)$ and $s_{h+1}\sim{}P\sups{M}(\cdot\mid{}s_h,a_h)$.
  \end{itemize}
  For notational convenience, we take $s_{H+1}$ to be a deterministic terminal
state. The value for a policy $\pi$ under $M$ is given by
  $\fm(\pi)\ldef\Ens{M}{\pi}\brk[\big]{\sum_{h=1}^{H}r_h}$, where
  $\Ens{M}{\pi}\brk{\cdot}$ denotes expectation under the process above.

In the online reinforcement learning setting, we interact with an
unknown MDP $\Mstar\in\cM$ for $T$ episodes. We now show how this
setting can be captured in the DMSO framework, where we take
the reward $r\ind{t}$ to be the cumulative reward in the episode and
the observation $o\ind{t}$ to be the observed trajectory. In particular,
for each episode
$t=1,\ldots,T$, the learner selects a policy $\pi\ind{t}\in\PiRNS$. The
policy is executed in the MDP $\Mstar$, and the learner observes the resulting
trajectory
$\tau\ind{t}=(s_1\ind{t},a_1\ind{t},r_1\ind{t}),\ldots,(s_H\ind{t},a_H\ind{t},r_H\ind{t})$. Their
goal is to minimize the total regret
$\sum_{t=1}^{T}\En_{\act\ind{t}\sim{}p\ind{t}}\brk*{\fmstar(\pistar)
  - \fmstar(\pi\ind{t})}$ against the optimal policy for $\Mstar$.
This setting immediately falls into the \FrameworkShort framework
by taking $\Act=\PiRNS$, $r\ind{t}=\sum_{h=1}^{H}r_h\ind{t}$, and
$\obs\ind{t}=\tau\ind{t}$.

Online reinforcement learning has received extensive attention, in
terms of both algorithms and sample complexity bounds
for specific model classes (e.g.,
\citet{dean2020sample,yang2019sample,jin2020provably,modi2020sample,ayoub2020model,krishnamurthy2016pac,du2019latent,li2009unifying,dong2019provably}),
as well as general structural results
\citep{jiang2017contextual,sun2019model,wang2020provably,du2021bilinear,jin2021bellman}. However,
there is currently no unified understanding of what properties of the
class of MDPs $\cM$ determine the optimal regret.

  This formulation, where the model class $\cM$ is taken as a
  given, may at first seem tailored to \emph{model-based} reinforcement
  learning, where the underlying transition dynamics are explicitly
  modeled. However, we can also express \emph{model-free}
  reinforcement learning problems---where we instead take a class of candidate value functions
  $\cQ$ as given and avoid directly learning the dynamics---by choosing $\cM$ to be the set of all MDPs that
  are consistent with the
  given value function class.
  As we show in \pref{sec:rl}, this perspective suffices to recover
standard lower bounds for 
    model-free RL, though our upper bounds are generally only tight
    for model-based settings.

\end{example}
\paragraph{Further examples}
Beyond these examples, the \FrameworkShort framework captures
many other canonical decision making problems, including (stochastic) contextual bandits, reinforcement
learning in partially observable Markov decision processes (POMDPs) 
\citep{even2005reinforcement,jin2020sample}
and reinforcement learning with a
generative model \citep{azar2013minimax,yang2019sample,agarwal2020model}.

\subsection{The \CompText}

We introduce a new complexity measure, the \emph{\CompText} (\CompShort), defined
for a model class $\cM$ and nominal model $\Mbar$ as
    \begin{equation}
  \label{eq:comp}
  \comp(\cM,\Mbar) =
  \inf_{p\in\Delta(\Act)}\sup_{M\in\cM}\En_{\act\sim{}p}\biggl[\hspace{1pt}\underbrace{\fm(\pim)-\fm(\pi)}_{\text{regret
    of decision}}
    -\gamma\cdot\hspace{-2pt}\underbrace{\Dhels{M(\act)}{\Mbar(\act)}}_{\text{estimation
      error for obs.}}\hspace{-2pt}
    \biggr],
  \end{equation}
where $\gamma>0$ is a scale parameter and
$\Dhels{\bbP}{\bbQ}=\int\prn{\sqrt{d\bbP}-\sqrt{d\bbQ}}^{2}$ is the
Hellinger distance; we further define
$\comp(\cM)=\sup_{\Mbar\in\cM}\comp(\cM,\Mbar)$.

The \CompText is the value of a
game in which the learner (represented by the min player) aims to find
a distribution over decisions such that for a worst-case problem
instance (represented by the max player), the \emph{regret} of their
decision is controlled by the \emph{estimation error} relative to the nominal
model. Conceptually, $\Mbar$ should be thought of as a guess for the
true model, and the learner (the min player) aims to---in the face of an unknown
environment (the max player)---optimally balance the regret of their
decision with the amount information they acquire. With enough
information, the learner can confirm or rule out their guess $\Mbar$, and
scale parameter $\gamma$ controls how much regret they are willing to
incur to do this.

\begin{table*}[t]
\centering\resizebox{.6\columnwidth}{!}{
\begin{tabular}{ | c | c | c |}
\hline
  Setting & \CompShort Lower Bound & Tight? \\
\hline 
  Multi-Armed Bandit   & $\sqrt{AT}$ & \cmark \\
  \hline 
Multi-Armed Bandit w/ gap  &  $A/\Delta$ & \cmark\\
  \hline
  Linear Bandit &  $\sqrt{dT}$ & ~~~~~~~~~~\xmark~~($d\sqrt{T}$) \\
    \hline
  Lipschitz Bandit  & $T^{\frac{d+1}{d+2}}$ & \cmark\\
  \hline
  ReLU Bandit  & $2^{d}$ & \cmark\\
  \hline
  Tabular RL  & $\sqrt{HSAT}$  & \cmark\\
  \hline
  Linear MDP  & $\sqrt{dT}$ & ~~~~~~~~~~\xmark~~($d\sqrt{T}$) \\
  \hline
    RL w/ linear $\Qstar$ &  $2^{d}$ & \cmark \\
  \hline
  Deterministic RL w/ linear $\Qstar$ &  $d$ & \cmark\\
  \hline
\end{tabular}
}
\caption{
Lower bounds for bandits and reinforcement learning recovered by the
\CompText, where $A=\text{\#actions}$, $\Delta=\text{gap}$,
$d=\text{feature dim.}$, $H=\text{episode horizon}$, and
$S=\text{\#states}$. See \pref{sec:examples,sec:bandit,sec:rl} for
details. The ``Tight?'' column indicates whether the lower bound is
tight up to logarithmic factors and dependence on episode horizon, with the
optimal rate stated for \xmark{} cases.
}
\label{table:lower}
\end{table*}

Let us give some intuition as to how the \CompText leads to both
\emph{upper} and \emph{lower} bounds on the optimal regret for
decision making. On the algorithmic side (upper bounds), the
connection between decision making (where the learner's decisions influence what feedback is collected) and estimation (where
data is collected passively) may not seem apparent a-priori. Here, the
power of the \CompText is that it---\emph{by definition}---provides a
bridge. One can select decisions
by building an estimate for the model using all of the observations
collected so far, then sampling from the distribution $p$ that
solves \pref{eq:comp} with the estimated model plugged in for
$\Mbar$. Boundedness of the \CompShort implies that at every round, any learner
using this strategy either enjoys small regret or acquires information, with their
total regret controlled by the cumulative error for (online/sequential) estimation. This
strategy generalizes notable prior approaches---Bayesian
\citep{russo2014learning,russo2018learning} and frequentist \citep{abe1999associative,foster2020beyond}.\looseness=-1

Of course, the perspective above is only useful if the \CompShort is
indeed bounded, which itself is not immediately apparent. However, we show that the \CompShort is not only bounded
(e.g., for the finite-armed bandit problem where $\Act=\crl{1,\ldots,A}$, we have
$\comp(\cM)\propto\frac{A}{\gamma}$), but---turning our focus to lower
bounds---quantitatively necessary for sample-efficient
learning. Developing lower bounds for a setting as general as the \FrameworkShort framework is challenging because
any complexity measure needs to capture both i) simple problems like
the multi-armed bandit, where the mean rewards serve as a sufficient
statistic, and ii) problems with rich, structured feedback (e.g., reinforcement learning), where
observations, or even structure in the noise itself, can provide 
non-trivial information about the underlying problem
instance. Here, the information-theoretic nature of the \CompShort
plays a key role. In particular, taking a dual perspective, the
\CompShort can be seen to certify a lower bound on the price (in terms
of regret) that \emph{any algorithm} must pay to obtain enough
information to distinguish a reference model from the least favorable alternative.

\subsection{Contributions}

Our main results show that the \CompText captures the
statistical complexity of interactive decision making. %
We provide:
\begin{itemize}
\item \textbf{A new fundamental limit and lower bound.} We establish (\pref{thm:lower_main,thm:lower_main_expectation}) that the \CompText is a fundamental
  limit for interactive decision making: For any class of
  models, any algorithm must have
  \begin{equation}
    \label{eq:lower_informal}
    \RegDM \approxgeq \max_{\gamma>0}\min\crl*{\complocshort(\cM)\cdot{}T, \gamma},
  \end{equation}
  where $\complocshort(\cM)$ is a certain
  localized variant of the \CompShort (cf. \pref{sec:main_lower}).
  This lower bound holds for
  \emph{any model class}, with no assumption on the structure. When
  applied to bandits and reinforcement learning, it recovers standard minimax lower bounds for
  canonical problems classes (\pref{sec:bandit,sec:rl}), as well as
  strong impossibility results for sample-efficient reinforcement
  learning with linear function approximation
  \citep{weisz2021exponential}; see \pref{table:lower} for
  highlights. In addition, the lower bound has an appealing conceptual
  interpretation as a decision making analogue of
  classical modulus of continuity techniques in statistical estimation
  \citep{donoho1987geometrizing,donoho1991geometrizingii,donoho1991geometrizingiii}.

\item \textbf{A unified meta-algorithm.} We provide
  (\pref{thm:upper_main}) a unified meta-algorithm, \AlgText{}
  (\mainalg), which achieves the decision making lower bound in
  \pref{eq:lower_informal} whenever estimation is possible. We show that for any model class $\cM$,
  the regret of \mainalg is bounded as
  \begin{equation}
    \label{eq:upper_informal}
    \RegDM \approxleq \min_{\gamma>0} \max\crl*{\complocshort(\cM)\cdot{}T, \gamma\cdot{}\est(\cM)},
  \end{equation}
  where $\est(\cM)$ is a certain classical notion of
  statistical estimation complexity for the class $\cM$ (for example,
  $\est(\cM)=\log\abs{\cM}$ for finite classes).
  \mainalg is a universal reduction which transforms any algorithm for online
  estimation with the model class $\cM$ into an algorithm for decision
  making. The algorithm is computationally efficient
  whenever the minimax program \pref{eq:comp} can be solved
  efficiently. This result recovers classical and contemporary sample complexity
guarantees for bandits and reinforcement learning, and leads to new algorithms for various classes of interest.

\end{itemize}
Together, these results provide a theory of learnability for
interactive decision making: Whenever $\cM$ has non-trivial online estimation
complexity $\est(\cM)$, sublinear regret for decision making is possible \emph{if and
only if} $\comp(\cM)$ decays sufficiently quickly as $\gamma\to\infty$.

\paragraph{Additional contributions}
Notable additional features of our results include:
\begin{itemize}
\item %
  \emph{Recovering existing frameworks.}
  Various works have proposed structural conditions that enable sample-efficient reinforcement learning,
      including Bellman Rank \citep{jiang2017contextual}, Witness
      Rank \citep{sun2019model}, (Bellman-) Eluder Dimension
      \citep{russo2013eluder,wang2020provably,jin2021bellman}, and
      Bilinear Classes \citep{du2021bilinear}. We show that the
      \CompShort subsumes each of these conditions and,
      importantly, provides the first such necessary
      condition.\footnote{While the \CompShort itself can be bounded in
        terms of each of these complexity measures, the regret bounds
        we provide are primarily limited to model-based settings. See section \pref{sec:rl} for detailed discussion.}
\item \emph{New perspectives on existing algorithms.} Our results tie
  together and generalize disparate
  algorithmic approaches across the literature on bandits and reinforcement
  learning, most notably posterior sampling and information-directed sampling
  and the information ratio
  \citep{russo2014learning,russo2018learning}, and inverse gap weighting
  approaches to bandits and contextual bandits
  \citep{abe1999associative,foster2020beyond,foster2020adapting}.
\item \emph{Seamlessly incorporating contextual information.} Our
  algorithms immediately extend to provide regret bounds for a
  \emph{contextual} variant of the \FrameworkShort framework in which
  the learner observes side information in the form a context (or, covariate)
  $x\ind{t}$ before making their decision at each round
  (\pref{sec:contextual}). This leads to new oracle-efficient
  algorithms for contextual bandits with large action spaces
  (generalizing \cite{foster2020beyond,foster2020adapting}) and
  contextual Markov decision processes \citep{abbasi2014online,modi2018markov,dann2019policy,modi2020no}.

\end{itemize}

  \subsection{Organization}
The first part of the paper (\pref{sec:framework,sec:algorithm})
  presents central results and tools, with preliminaries in
  \pref{sec:prelims}.
  \begin{itemize}
  \item \pref{sec:framework} contains our main results: a universal lower bound
    for interactive decision making based on the
\CompText, and a nearly-matching upper bound via the \AlgText{} (\mainalg)
meta-algorithm. Combining these results, we provide a characterization
of learnability in the \FrameworkShort framework. This section also
includes extensive discussion and interpretation of the main results.
\item Building on this development, \pref{sec:algorithm} provides a toolkit for deriving
    regret bounds using the \mainalg algorithm, including
    background on online estimation, general-purpose regret bounds, a Bayesian
    variant of the \mainalg algorithm, and  other extensions.
  \end{itemize}
 In the second part of this paper
 (\pref{sec:examples,sec:bandit,sec:rl}), we apply our general tools
 to give efficient algorithms and regret bounds for specific examples of
  interest:
  \begin{itemize}
  \item \pref{sec:examples} (Illustrative Examples) serves as a
    warm-up, and contains detailed examples for the basic multi-armed
    bandit problem and tabular reinforcement learning. For both
    settings, we carefully show how to upper and lower bound the
    \CompText and derive efficient algorithms. These examples
    highlight a number of basic techniques which find use in later
    sections.
  \item \pref{sec:bandit} (Application to Bandits) applies our
    techniques to the structured bandit problem. We recover efficient
    algorithms and lower bounds for well-known problems (linear
    bandits, convex bandits, eluder dimension, and more), then derive
    new guarantees based on a parameter called the \emph{star
      number}.\looseness=-1
  \item \pref{sec:rl} (Application to Reinforcement Learning) applies
    our techniques to reinforcement learning with function
    approximation. Highlights here include new algorithms
    for reinforcement learning with bilinear classes, and lower bounds
    for learning with linearly realizable value functions.

  \end{itemize}
  We conclude with a contextual extension of the \mainalg algorithm
  (\pref{sec:contextual}), further related work (\pref{sec:related}),
  and discussion (\pref{sec:discussion}). All proofs are deferred to
  the appendix unless otherwise stated.

\section{Preliminaries}
\label{sec:prelims}

\paragraph{Probability spaces}
We briefly formalize the probability
spaces associated with the \FrameworkShort framework. Decisions are
associated with a measurable space $(\Act,\Asig)$, rewards are
associated with the space $(\Rspace,\Rsig)$, and observations are
associated with the space $(\Obs,\Osig)$.  Let $\hist\ind{t}
=(\act\ind{1},r\ind{1},\obs\ind{1}),\ldots,(\act\ind{t},r\ind{t},\obs\ind{t})$
denote the history up to time $t$. We define
\[
\Hspace\ind{t}=\prod_{i=1}^{t}(\Act\times\Rspace\times\Obs),\mathand\Hsig\ind{t}=\bigotimes_{i=1}^{t}(\Asig\otimes{}\Rsig\otimes\Osig)
\]
so that $\hist\ind{t}$ is associated with the space $(\Hspace\ind{t},\Hsig\ind{t})$.

Recall that for measurable spaces $(\cX,\scrX)$ and $(\cY,\scrY)$ a
probability kernel $P(\cdot\mid{}\cdot)$ from $(\cX,\scrX)$ to $(\cY,\scrY)$ has the property that 1) For
all $x\in\cX$, $P(\cdot\mid{}x)$ is a probability measure, and 2) for
all $Y\in\scrY$, $x\mapsto{}P(Y\mid{}x)$ is measurable. If
$P(\cdot\mid{}x)$ is a $\sigma$-finite measure rather than a
probability measure, we instead call $P$ a kernel.

Formally, a model $M=M(\cdot,\cdot\mid\cdot)\in\cM$ is a probability
kernel from $(\Act,\Asig)$ to $(\Rspace\times\Obs,\Rsig\otimes\Osig)$. An \emph{algorithm} for horizon $T$ is specified by a sequence of
probability kernels $p\ind{1},\ldots,p\ind{T}$, where $p\ind{t}(\cdot\mid\cdot)$ is a
probability kernel from $(\Hspace\ind{t-1},\Hsig\ind{t-1})$ to
$(\Act,\Asig)$. We let $\bbP^{\sss{M,}p}$ denote the law of $\hist\ind{T}$
under the process
$\act\ind{t}\sim{}p\ind{t}(\cdot\mid{}\hist\ind{t-1})$,
$(r\ind{t},\obs\ind{t})\sim{}M(\cdot,\cdot\mid{}\act\ind{t})$. We
adopt the shorthand $M(\act) =  M(\cdot,\cdot\mid\act)$ throughout.

We assume that there
exists a common dominating kernel
$\dom(\act)=\dom(\cdot,\cdot\mid{}\act)$ from $(\Act,\Asig)$ to
$(\Rspace\times\Obs,\Rsig\otimes\Osig)$ such that $M(\act)\ll{}\dom(\act)$ for all $M\in\cM$, $\act\in\Act$.
For each $\act\in\Act$ we let $m^{\sss{M}}(\cdot,\cdot\mid{}\pi)$
denote the density of $M(\cdot,\cdot\mid{}\pi)$ with respect to
$\nu(\cdot,\cdot\mid{}\act)$. This assumption facilitates the use of
density estimation for our upper bounds, but is not required by our
lower bounds.

\paragraph{Divergences}

  For probability distributions $\bbP$ and $\bbQ$ over a measurable space
  $(\Omega,\filt)$ with a common dominating measure, we define the total variation distance as
  \[
    \Dtv{\bbP}{\bbQ}=\sup_{A\in\filt}\abs{\bbP(A)-\bbQ(A)}
    = \frac{1}{2}\int\abs{d\bbP-d\bbQ}.
  \]
  The (squared) Hellinger distance is defined as\footnote{If
    $\nu$ is a common dominating measure, then
    $\Dhels{\bbP}{\bbQ}=\int\prn[\Big]{\sqrt{\frac{d\bbP}{d\nu}}-\sqrt{\frac{d\bbQ}{d\nu}}}^{2}d\nu$,
  where $\frac{d\bbP}{d\nu}$ and $\frac{d\bbQ}{d\nu}$ are
  Radon-Nikodym derivatives. The notation in \pref{eq:hellinger} (and for other divergences) reflects that this quantity is
  invariant under the choice of $\nu$.}
\begin{equation}
  \label{eq:hellinger}
    \Dhels{\bbP}{\bbQ}=\int\prn*{\sqrt{d\bbP}-\sqrt{d\bbQ}}^{2},
  \end{equation}
  and the Kullback-Leibler divergence is defined as
  \[
    \Dkl{\bbP}{\bbQ} =\left\{
    \begin{array}{ll}
\int\log\prn[\big]{
      \frac{d\bbP}{d\bbQ}
      }d\bbP,\quad{}&\bbP\ll\bbQ,\\
      +\infty,\quad&\text{otherwise.}
    \end{array}\right.
\]

\paragraph{Additional reinforcement learning notation}
For a given model $M$ and policy $\pi$, we define the state-action
value function and state value functions via
\[
Q_h^{\sss{M},\pi}(s,a)=\En^{\sss{M},\pi}\brk*{\sum_{h'=h}^{H}r_{h'}\mid{}s_h=s,
  a_h=a},
\mathand V_h^{\sss{M},\pi}(s)=\En^{\sss{M},\pi}\brk*{\sum_{h'=h}^{H}r_{h'}\mid{}s_h=s}.
\]
Note that the policy $\pim=\argmax_{\pi}\fm(\pi)$ may not be uniquely
defined. We assume without loss of generality that $\pim=(\pi\subs{M,1},\ldots,\pi\subs{M,H})$ solves
$\pimh(s)=\argmax_{a\in\cA}\Qmpi[\pim]_h(s,a)$ for all $s\in\cS$, $h\in\brk{H}$. We adopt the convention that $V_{H+1}^{\sss{M},\pi}(s)=0$ and use $\Qmstar_h(s,a)\ldef\Qmpi[\pim]_h(s,a)$
and $\Vmstar_h(s,a)=\Vmpi[\pim]_h(s,a)$ to denote the optimal value functions.

Let $\bbP^{\sss{M},\pi}\prn{\cdot}$ denote the law of a trajectory evolving
under MDP $M$ and policy $\pi$. We define state occupancy
measures via
$d^{\sss{M},\pi}_h(s)=\bbP^{\sss{M},\pi}(s_h=s)$ and state-action occupancy
measures via $d^{\sss{M},\pi}_h(s,a)=\bbP^{\sss{M},\pi}(s_h=s,a_h=a)$. Note that we
have $d^{\sss{M},\pi}_1(s)=d_1(s)$ for all $M$ and $\pi$.

We let $\PiRNS$ denote the collection of randomized non-stationary
policies. For $\pi=(\pi_1,\ldots,\pi_H)\in\PiRNS$, $\pi_h(s,a)$
denotes the probability that action $a$ is selected in state $s$, so
that $\pi_h(s)\ldef\pi_h(s,\cdot)\in\Delta(\cA)$.
For a pair of policies $\pi$, $\pi'$, we define $\pi\circ_h\pi'$ as the policy that selects actions
$a_1,\ldots,a_{h-1}$ using $\pi$ and selects all subsequent actions
using $\pi'$. 

For a transition distribution $P(\cdot\mid{}\cdot,\cdot)$,
we define $\brk{Pf}(s,a) =\En_{s'\sim{}P(s,a)}\brk{f(s')}$.

    \paragraph{Further notation}
    We adopt non-asymptotic big-oh notation: For functions
	$f,g:\cX\to\bbR_{+}$, we write $f=\bigoh(g)$ (resp. $f=\bigom(g)$) if there exists some constant
	$C>0$ such that $f(x)\leq{}Cg(x)$ (resp. $f(x)\geq{}Cg(x)$)
        for all $x\in\cX$. We write $f=\bigoht(g)$ if
        $f=\bigoh(g\cdot\mathrm{polylog}(T))$, $f=\bigomt(g)$ if $f=\bigom(g/\polylog(T))$, and
        $f=\bigthetat(g)$ if $f=\bigoht(g)$ and $f=\bigomt(g)$. %
	We use $f\propto g$ as shorthand for $f=\bigthetat(g)$. We use
        $\approxleq$ only in informal statements to highlight salient elements of an inequality.

	For a vector $x\in\bbR^{d}$, we let $\nrm*{x}_{2}$ denote the euclidean
	norm and $\nrm*{x}_{\infty}$ denote the element-wise $\ell_{\infty}$
	norm. We let $\mathbf{0}$ denote the all-zeros vector, with
        dimension clear from context.
	For an integer $n\in\bbN$, we let $[n]$ denote the set
        $\{1,\dots,n\}$. For a set $S$, we let
        $\unif(S)$ denote the uniform distribution over all the
        elements in $S$. For a set $\cX$, we let
        $\Delta(\cX)$ denote the set of all Radon probability measures
        over $\cX$. We let $\conv(\cX)$ denote the set of all finitely
        supported convex combinations of elements in $\cX$, and let
        $\starhull(\cX,x)=\bigcup_{x'\in\cX}\conv(\crl{x',x})$ denote
        the star hull. We use the convention
        $a\wedge{}b=\min\crl{a,b}$ and $a\vee{}b=\max\crl{a,b}$. 
        
        For a symmetric matrix $A\in\bbR^{d\times{}d}$, $A^{\pinv}$ denotes the pseudoinverse, and when $A\psdgeq{}0$, $\nrm{x}_{A}^{2}\ldef{}\tri{x,Ax}$.

        For a parameter $\mu\in\brk*{0,1}$, $\Ber(\mu)$ denotes a
        Bernoulli random variable with mean $\mu$. Likewise, for
        $\mu\in\brk{-1,+1}$, $\Rad(\mu)$ denotes a (biased) Rademacher
        random variable with mean $\mu$, which has $\bbP(+1) =
  \frac{1+\mu}{2}$ and $\bbP(-1)=\frac{1-\mu}{2}$. For an element
  $x$ in a measurable space, we let $\delta_{x}$ denote the delta distribution on $x$.

        \subsection{Minimax Regret: Frequentist and Bayesian}
        \label{sec:minimax_swap}

The main goal of this paper is to characterize what properties of the model class $\cM$ determine the \emph{minimax regret}, defined via
\begin{equation}
  \label{eq:minimax_regret}
  \MinimaxReg = \inf_{p\ind{1},\ldots,p\ind{T}}\sup_{\Mstar\in\cM}\Enm{\Mstar}{p}\brk*{\RegDM(T)},
\end{equation}
where we write $\RegDM(T)$ to make the dependence on $T$ explicit, and recall that
$p\ind{t}=p\ind{t}(\cdot\mid\hist\ind{t-1})$. While our focus is on
this frequentist notion of regret, we establish certain results by considering the \emph{Bayesian} setting in which the
underlying model $\Mstar$ is drawn from a prior $\mu\in\Delta(\cM)$
that is known to the learner.\footnote{Unless otherwise specified, we
  assume that $\cM$ is equipped with the discrete topology, so that
  $\Delta(\cM)$ contains the space of finitely supported probability measures.}
We define the worst-case Bayesian regret via
\begin{equation}
  \label{eq:minimax_regret_bayes}
  \MinimaxRegBayes = \sup_{\mu\in\Delta(\cM)}\inf_{p\ind{1},\ldots,p\ind{T}}\En_{\Mstar\sim\mu}\Enm{\Mstar}{p}\brk*{\RegDM(T)}.
\end{equation}
Building on a long line of research in online learning and bandits
\citep{abernethy2009stochastic,rakhlin2010online,bubeck2015bandit,bubeck2016multi,lattimore2019information},
one can show that under mild technical conditions
(essentially, whenever the prerequisites for Sion's minimax theorem
are satisfied), the minimax frequentist regret and Bayesian regret
coincide. This allows us to deduce existence of algorithms for the
frequentist setting by exhibiting algorithms for the Bayesian setting,
which is---in some cases---a simpler task.
In particular, we have the following result.
\begin{restatable}{proposition}{minimaxregret}
  \label{prop:minimax_swap}
  Suppose that $\Act$ is finite and $\Rspace$ is bounded, and that for
  all $M\in\cM$, $\bigcup_{\act\in\Act}\supp(M(\act))$ is countable. Then we have
  \begin{equation}
    \label{eq:minimax_swap}
    \MinimaxReg = \MinimaxRegBayes.
  \end{equation}
\end{restatable}
We emphasize that the condition in \pref{prop:minimax_swap} was chosen for
simplicity and concreteness. We expect that the same conclusion can be
proven under far weaker conditions, though we note that minimax
theorems for learning with partial information are somewhat more subtle
than for classical online learning; see, e.g.,
\cite{lattimore2019information}.
We refer ahead to \pref{sec:dual} for additional discussion of
connections between the frequentist and Bayesian setting, and for an
analogous Bayesian counterpart to the \CompText.

\section{A Theory of Learnability for Interactive Decision Making}
\label{sec:framework}
In this section we provide our main results: a universal lower bound
for interactive decision making based on the
\CompText, and a nearly-matching upper bound via the \AlgText{}
meta-algorithm.
The section is organized as follows:
\begin{itemize}
  \item In \pref{sec:main_lower}, we state our main lower bound on
    regret. We then walk through basic examples that give intuition as
    to how the \CompText captures the complexity of interactive
    decision making, and show how the lower bound behaves for each example.
\item In \pref{sec:main_upper}, we provide our main upper bound on
  regret and introduce the \mainalg algorithm, highlighting how
  these results complement the lower bound.
\item Building on these results, in \pref{sec:learnability} we
provide a characterization for learnability (i.e., achievability of
non-trivial regret) based on the \CompText.
\item \pref{sec:main_bayes} gives the last major result of the section:
  a refinement of the main upper bound with weaker dependence
  on model estimation complexity.
\end{itemize}
We conclude the section by discussing gaps between the upper
and lower bounds, opportunities for improvement, and subsequent work
(\pref{sec:main_discussion}). All proofs are deferred to
\pref{app:framework}.

\subsection{Lower Bound: The \CompText is a Fundamental Limit}
\label{sec:main_lower}

We now present the main lower bound based on the \CompText. The result
is proven using a decision-theoretic adaptation of the \emph{local minimax} method
\citep{donoho1987geometrizing,donoho1991geometrizingii,donoho1991geometrizingiii};
the proof uses the \CompShort to show that, for any nominal model
$\Mbar\in\cM$, any algorithm must either experience large regret on
$\Mbar$ or on a worst-case alternative model in the neighborhood of
$\Mbar$.  To state the result, we formalize the notion of a
neighborhood via a \emph{localized} restriction of the class $\cM$.
\begin{definition}[Localized model class]
  For a model class $\cM$ and reference model $\Mbar\in\cM$, we define
  \begin{equation}
    \label{eq:localized}
    \cMloc[\veps](\Mbar) = \crl*{
      M\in\cM: \fmbar(\pimbar) \geq{} \fm(\pim) - \veps
    }
  \end{equation}
  as the neighborhood of $\Mbar$ at radius $\veps$.
\end{definition}
The localized model class is tailored to decision making and asserts that no alternative should
have mean reward substantially larger than that of the nominal model,
but otherwise does not place any restriction on the mean reward
function or observation distribution.

Our lower bound takes the
worst case over all possible nominal models. We define the shorthand
\begin{equation}
\comp(\cM)=\sup_{\Mbar\in\cM}\comp(\cM,\Mbar) \mathand
\comploc{\veps}(\cM)=\sup_{\Mbar\in\cM}\comp(\cMloc[\veps](\Mbar),\Mbar).\label{eq:comp_short}
\end{equation}
where $\comp(\cM,\Mbar)$ is defined in \pref{eq:comp}.
We let $\abscont $ denote a bound on the density ratio over all 
models. Precisely, we define 
$\abscont=\sup_{M,M'\in\cM}\sup_{\act\in\Act}\sup_{A\in{}\Rsig\otimes\Osig}\crl*{\frac{M(A\mid\act)}{M'(A\mid{}\act)}}\vee{}e$; 
finiteness of this quantity is not necessary, but leads to tighter
guarantees.\footnote{We do not require that $\abscont<\infty$, but
  whenever this holds we can remove certain $\log(T)$ factors in
  \pref{thm:lower_main,thm:lower_main_expectation} for tighter results.}\looseness=-1

The main lower bound is as follows.
%
%

%

\begin{restatable}[Main lower bound]{theorem}{lowermain}
  \label{thm:lower_main}
  Consider a model class $\cM$ with
  $\cF_{\cM}\subseteq(\Act\to\brk*{0,1})$, and let $\delta\in(0,1)$
  and $T\in\bbN$ be fixed. Define $\Ct\ldef{}2^{15}\log(2T\wedge{}\abscont)$ and $\vepslowg{}\ldef
  \Ct^{-1}\frac{\gamma}{T}$. Then for any algorithm, there exists a
  model in $\cM$ for which
      \begin{equation}
    \label{eq:lower_main}
    \RegDM\geq{}
(6\Ct)^{-1}\cdot\max_{\gamma>\sqrt{\Ct{}T}}\min\crl[\Big]{\prn*{\compthmone-15\delta}\cdot{}T,\;
  \gamma}
\end{equation}
with probability at least $\delta/2$.
\end{restatable}

\pref{thm:lower_main} shows that for any model class, any
algorithm has
$\RegDM\approxgeq{}\max_{\gamma\approxgeq\sqrt{T}}\min\crl[\big]{\compthmone\cdot{}T,
  \gamma}$ with moderate probability. The result serves as a converse
to high-probability upper bounds we establish in the sequel, though in
general it does not lead to optimal lower bounds on \emph{expected}
regret. Our next result is a variant of \pref{thm:lower_main} that provides stronger lower bounds on expected
regret using a more restrictive notion of localization.
Define $\gm(\pi) \ldef \fm(\pim) - \fm(\pi)$ and
  \begin{equation}
    \label{eq:localized_infinity}
    \cMinf[\veps](\Mbar) = \crl*{
      M\in\cM: \abs*{\gm(\pi) - \gmbar(\pi)}\leq\veps\;\;\forall{}\pi\in\Pi
    }.
  \end{equation}
\begin{restatable}[Main lower bound---in-expectation version]{theorem}{lowermainexpectation}
  \label{thm:lower_main_expectation}
  Consider a model class $\cM$ with
  $\cF_{\cM}\subseteq(\Act\to\brk*{0,1})$. Let $T\in\bbN$ be fixed and
  define $\Ct\ldef{}2^{11}\log(2T\wedge{}\abscont)$ and $\vepslowg{}\ldef
  \Ct^{-1}\frac{\gamma}{T}$. Then for any algorithm, there exists a
  model in $\cM$ for which
      \begin{equation}
    \label{eq:lower_main_expectation}
    \En\brk*{\RegDM}\geq{}
    6^{-1}\cdot\max_{\gamma>0}\sup_{\Mbar\in\cM}\comp(\cMinf[\vepslowg](\Mbar),\Mbar)\cdot{}T.%
\end{equation}
\end{restatable}
Despite using a stronger $L_{\infty}$ notion of
localization, this result suffice to derive the lower bounds in
\pref{table:lower}. In general, any lower bound based on the
  \CompShort must make use of localization, as there exist
  classes for which the global DEC is larger than the minimax rate. See
  \pref{sec:main_discussion} for discussion.

\subsubsection{Lower Bound and \CompText: Basic Properties and Intuition}
\label{sec:intuition}
\newcommand{\Mo}{M_{\cO}}%
\newcommand{\Mr}{M_{\cR}}%
\newcommand{\Mbaro}{\Mbar_{\cO}}%
\newcommand{\Mbarr}{\Mbar_{\cR}}%
\newcommand{\Dbar}{\wb{D}}%

To build intuition, we take this opportunity to highlight some basic properties
of the \CompText and how these properties influence the behavior of our lower bounds. To keep the examples we consider as simple as possible, we only
sketch the proof details. We refer ahead to
\pref{sec:examples,sec:bandit,sec:rl} for examples with complete
calculations. Additional technical aspects of the lower bound are
discussed at the end of the section (\pref{sec:main_discussion}).%

\paragraph{Capturing the complexity of decision making with
  reward-based feedback}
As a starting point, we consider the behavior of the \CompText
and the lower bounds in \pref{thm:lower_main,thm:lower_main_expectation} for bandit-type problems
(of varying degrees of structure) in
which only reward-based feedback is available. Here, a basic property
of the \CompText is that it reflects the amount of
information that any given decision reveals about other possible
decisions, thereby acting as a measure of intrinsic dimension. To
illustrate this point, we focus on two problems at opposite ends of
the bandit spectrum: Finite-armed bandits and bandits with full information.

\begin{example}[Finite-armed bandits]
  \label{ex:bandit_lower}
  For the classical multi-armed bandit problem, we have $\Act=\brk*{A}$, $\Rspace=\brk{0,1}$, and $\Obs=\NullObs$, and
the model class $\cM=\crl*{M:M(\act)\in\Delta(\Rspace)}$ consists of
all possible distributions over $\Rspace$. Let us sketch a lower bound
on the \CompText for this setting. Fix $\Delta\in(0,\nicefrac{1}{2})$ and define a
sub-family of models $\cM'=\crl*{M_1,\ldots,M_A}\cup\crl*{\Mbar}$ via
$M_i(\act)\ldef{}\Ber(\nicefrac{1}{2}+\Delta\indic\crl{\act=i})$ and
$\Mbar(\act)\ldef\Ber(\nicefrac{1}{2})$. In other words, each model
$M_i$ has $A-1$ arms with suboptimality gap $\Delta$ relative to the
optimal arm $i$, while the model $\Mbar$ has identical rewards for all arms.
One can verify that
$\Dhels{M_i(\act)}{\Mbar(\act)}\approxleq{}\Delta^{2}\cdot{}\indic\crl{\act=i}$
(cf. \pref{lem:divergence_bernoulli})
and that
$\fmi(\pimi)-\fmi(\pi)\geq{}\Delta\cdot{}\indic\crl{\act\neq{}i}$. In
other words, for each model, only a single decision reveals
information about the model's identity and leads to low regret. As a result, we calculate
that
\begin{align}
  \comp(\cM',\Mbar)
    &\geq{} \inf_{p\in\Delta(\Act)}\max_{i\in\brk{A}}\En_{\act\sim{}p}\brk*{
      \fmi(\pimi) - \fmi(\act)
      -\gamma\cdot\Dhels{M_i(\act)}{\Mbar(\act)}
      }
\nonumber      \\
  &\geq{} \inf_{p\in\Delta(\Act)}\En_{i\sim\unif(\brk{A})}\En_{\act\sim{}p}\brk*{
    \fmi(\pimi) - \fmi(\act) -\gamma\cdot\Dhels{M_i(\act)}{\Mbar(\act)}
    }
  \approxgeq{}
    \Delta-\gamma\cdot\frac{\Delta^{2}}{A}.
\label{eq:example_follow}
\end{align}
Crucially, because the decisions that reveal information are
disjoint, the negative contribution from the Hellinger divergence
scales as $A^{-1}$. By choosing $\Delta\propto{}\frac{A}{\gamma}$, we
conclude that $\comp(\cM',\Mbar)\approxgeq{}\frac{A}{\gamma}$, and
with more care, we show in \pref{sec:examples} that this argument in
fact implies that (i) $\cM'\subseteq\cMinf[\vepslowg](\Mbar)$ whenever
$\gamma\approxgeq\sqrt{AT}$, and (ii) $\abscontp=\bigoh(1)$. Plugging these
values into \pref{thm:lower_main_expectation}, the lower bound in
\pref{eq:lower_main_expectation} yields
\[
  \En\brk*{\RegDM} \approxgeq{}
  \max_{\gamma\approxgeq\sqrt{AT}}\crl*{\frac{A}{\gamma}} = \bigom\prn[\big]{\sqrt{AT}},
\]
which matches the well-known minimax rate
\citep{audibert2009minimax}. In \pref{sec:examples} we provide a matching upper bound of the
form $\comp(\cM)\leq\frac{A}{\gamma}$, which certifies that this lower
bound on the \CompShort is optimal.
\end{example}
The finite-armed bandit setting is completely unstructured in
the sense that choosing a given arm $\act$ reveals no information about
the others. The next example, \emph{bandits
with full information}, lies at the opposite extreme. Here, each arm
reveals the rewards for all other arms, completely removing the need
for exploration.
\begin{example}[Bandits with full information]
  \label{ex:full_info_lower}
  Consider a full-information variant of the bandit
  setting. We have $\Act=\brk{A}$ and $\cR=\brk{0,1}$, and for a given decision $\act$ we
  observe a reward $r$ as in \pref{ex:bandit_lower}, but also receive an observation
  $o = (r(\pi'))_{\pi'\in\brk{A}}$ consisting of (counterfactual) rewards for every action.

  For a given model $M$, let $\Mr(\pi)$ denote the distribution over
  the reward for $\pi$. Then for any decision $\pi$, since all rewards are observed,
the data processing inequality implies that for all $\pi'$,
  \begin{equation}
    \label{eq:hellinger_full_info}
    \Dhels{M(\pi)}{\Mbar(\pi)}\geq{} \Dhels{\Mr(\pi')}{\Mbarr(\pi')}.
  \end{equation}
Using this property, we can show that
  $\comploc{\vepslowg}(\cM)\propto\frac{1}{\gamma}$. We sketch a proof
  of the upper bound; we find this to be more illuminating than the lower
  bound for this example because it directly highlights how extra
  information helps to shrink the \CompShort.

  For a given model $\Mbar$ we choose $p=\delta_{\pimbar}$ and bound
  $\En_{\act\sim{}p}\brk*{\fm(\pim)-\fm(\act)}$ by
  \begin{align*}
    \fm(\pim)-\fm(\pimbar)
    & \leq{}     \fm(\pim)-\fm(\pimbar) +     \fmbar(\pimbar)-\fmbar(\pim)\\
    &\leq{}
    2\cdot{}\max_{\act\in\crl{\pim,\pimbar}}\abs{\fm(\act)-\fmbar(\act)} \\
    &\leq{}  2\cdot{}\max_{\act\in\crl{\pim,\pimbar}}\Dhel{\Mr(\act)}{\Mbarr(\act)}.
  \end{align*}
  We then use the AM-GM inequality, which implies that for any $\gamma>0$,
  \begin{align*}
    \max_{\act\in\crl{\pim,\pimbar}}\Dhels{\Mr(\act)}{\Mbarr(\act)}
    \approxleq{}
      \gamma\cdot{}\max_{\act\in\crl{\pim,\pimbar}}\Dhels{\Mr(\act)}{\Mbarr(\act)}
    + \frac{1}{\gamma}
    \leq{}
    \gamma\cdot{}\Dhels{M(\pimbar)}{\Mbar(\pimbar)}
    + \frac{1}{\gamma},
  \end{align*}
  where the final inequality uses \pref{eq:hellinger_full_info}. This
  certifies that for all $M\in\cM$, the choice for $p$ above
  satisfies \[\En_{\act\sim{}p}\brk*{\fm(\pim)-\fm(\pi)-\gamma\cdot\Dhels{M(\act)}{\Mbar(\act)}}\approxleq{}\frac{1}{\gamma},\]
  so we have $\comp(\cM,\Mbar)\approxleq{}\frac{1}{\gamma}$. A matching
  lower bound follows by adapting
  \pref{ex:bandit_lower}. Comparing to
  the finite-armed bandit, we see that the \CompShort for this example is
  independent of $A$, which reflects the information
  sharing. The final lower bound from \pref{thm:lower_main_expectation} takes the form
  \[
    \En\brk*{\RegDM} \geq \bigom\prn[\big]{\sqrt{T}}.
  \]
\end{example}
In \pref{sec:bandit} we consider structured bandit problems
that lie between the extremes above, including linear and convex
bandits. These problems typically have
$\comp(\cM)\propto\frac{\mathsf{dim}}{\gamma}$, where---generalizing
\pref{ex:bandit_lower,ex:full_info_lower}---$\mathsf{dim}$ is a
quantity that reflects the intrinsic problem dimension. In general,
however, the \CompShort can exhibit slower decay as a function of
$\gamma$, particularly for nonparametric problems where the optimal
regret is large than $\sqrt{T}$.

%
%

%
%

%

%
%
%

%
%

\pref{ex:bandit_lower,ex:full_info_lower} can be thought of as bandit
problems with ``problem-instance agnostic'' noise, in the sense
that the mean rewards serve as a sufficient statistic, while the additive
noise (in the reward distribution) is otherwise uninformative about the problem instance.
However, any complexity
measure which claims to characterize the complexity of decision making
must account for more general (and possibly corner-case) problems, where seemingly
irrelevant properties of the noise distribution can encode non-trivial
information about the problem instance itself. The following example
highlights one such problem.

\begin{example}[Bandits: hiding information in lower order bits]
  \label{ex:lower_bits}
  Consider a family of multi-armed bandit models with
  $\Act=\brk{A}$. Following \pref{ex:bandit_lower}, fix
  $\Delta\in(0,\nicefrac{1}{2})$ and define $\cM=\crl*{M_1,\ldots,M_A}\cup\crl*{\Mbar}$ via
$M_i(\act)\ldef{}\cN(\nicefrac{1}{2}+\Delta\indic\crl{\act=i}, 1)$ and
$\Mbar(\act)\ldef{}\cN(\nicefrac{1}{2}, 1)$. As defined, this family has the same
lower bound $\comp(\cM,\Mbar)\geq{}\frac{A}{\gamma}$ as in
\pref{ex:bandit_lower}, because $M_i$ and $\Mbar$ can only be
distinguished by choosing $\act=i$, and the Hellinger distance for
this arm scales as $\Delta^{2}$. Now, suppose we modify each model
$M_i$ such that $M_i(\act)$ is unchanged for $\pi\neq{}i$, but for $\pi=i$, the reward $r$ is generated by sampling
$r'\sim{}\cN(\nicefrac{1}{2}+\Delta, 1)$, then
setting bits $N,\dots,N+\ceil{\log_2(A)}$ in the reward's
infinite-precision binary representation to encode index $i$, where
$N\in\bbN$ is a parameter. In other words, if decision $i$ is chosen
in model $M_i(\act)$, then the lower order bits of the infinite precision representation of the observed reward will
reveal that $i$ is in fact the optimal arm. Taking $N$ large makes the mean reward
$\fmi(i)$ arbitrarily close to $\nicefrac{1}{2}+\Delta$,
but we have
\[
  \Dhels{M_i(\act)}{\Mbar(\act)} \propto{}\indic\crl{\act=i}
\]
for all $\act$ when $A$ is sufficiently large. That is,
compared to \pref{ex:bandit_lower}, the Hellinger distance does not
scale with the gap $\Delta$ due to the auxiliary information encoded
in bits $N,\dots,N+\ceil{\log_2(A)}$.\footnote{As $N$ and $A$ grow, the
  probability of observing any given string in the bits
  $N,\dots,N+\ceil{\log_2(A)}$ becomes arbitrarily small under
  $\Mbar$. This causes the total variation distance (and consequently
  Hellinger distance) to become constant.} Choosing $\Delta=1/2$ and
adapting the argument in \pref{ex:bandit_lower} (where we introduce
the expectation over $i\sim\unif(\brk{A})$ as in \pref{eq:example_follow}),
we compute that 
\[
  \comp(\cM_{1/2}(\Mbar),\Mbar) \approxgeq{} 1 - \gamma\cdot\frac{1}{A}
  \approxgeq{}\indic\crl{\gamma\leq{}A/2},
\]
and \pref{thm:lower_main_expectation} yields
\[
\En\brk*{\RegDM} \geq \bigomt(\min\crl{A,T}).
\]
This is seen to be optimal, since
we can detect the index of the underlying instance with constant
probability by enumerating over the arms a single time.%
\end{example}

Note that for this example, working with Hellinger distance (or
another information-theoretic divergence) is critical. A weaker divergence
such as the squared distance
$\Dsq{M_i(\act)}{\Mbar(\act)}\ldef{}(\fmi(\act)-\fmbar(\act))^2$  would miss the
information hiding in the lower order bits. Further examples with a
similar information-theoretic flavor include noiseless multi-armed
bandits, as well as ``cheating code'' problems of the type considered in
\cite{jun2020crush}, where a basic multi-armed bandit problem instance
is equipped with auxiliary arms that---when queried---directly reveal the
identity of the instance.

\paragraph{From reward-based feedback to general observations}
Generalizing \pref{ex:lower_bits}, one can create structured bandit problems in
which arbitrary auxiliary information is hidden in the reward
signal.
The \FrameworkShort framework makes auxiliary
information explicit via the observations in the process
$(r,o)\sim{}M(\act)$, which we find to be more clear conceptually. Our last example considers a setting where accounting for such observations is crucial.
\begin{example}[Importance of observations for reinforcement learning]
In the tabular (finite state/action) reinforcement learning setting,
the model class $\cM$ is the collection
of all non-stationary MDPs with state space $\cS=\brk{S}$, action space
$\cA=\brk{A}$, and horizon $H$; we have $\Act=\PiRNS$ (recall that
$\PiRNS$ is the set of randomized non-stationary policies). For a given MDP $M$, when
$(r,o)\sim{}M(\pi)$, the reward signal $r$ is the total reward for an episode in which
the policy $\pi$ is executed, and the observation $o$ is the resulting
trajectory. Here, if only reward information were available, sample
complexity $2^{\bigom(H)}$ would be unavoidable. However, because trajectories
are observed, the problem admits polynomial sample complexity. We show (\pref{sec:examples}) that
\[
  \comploc{\vepslowg}(\cM) \propto{} \frac{\poly(S,A,H)}{\gamma},
\]
whenever $\gamma\geq{}\sqrt{\poly(S,A,H)\cdot{}T}$, which leads to $\sqrt{\poly(S,A,H)\cdot{}T}$ upper and lower bounds on regret.
\end{example}
Similar discussion applies to essentially all non-trivial
reinforcement learning problems, and highlights the necessity of
learning from observations. Because the
Decision-Estimation Coefficient is able to account for bandit
problems with auxiliary information in the reward signal,
it incorporates additional observations seamlessly.\looseness=-1

\paragraph{Additional information-theoretic considerations}
We conclude by briefly discussing some additional
information-theoretic properties of the \CompText. 

\begin{example}[Filtering irrelevant information]
  Adding observations that are unrelated to the model under
  consideration never changes the value of the \CompText. In more
  detail, consider a model class $\cM$ with observation space $\cO_1$,
  and consider a class of conditional distributions $\cD$ over a
  secondary observation space $\cO_2$, where each $D\in\cD$ has the
  form $D(\pi)\in\Delta(\cO_2)$. For $M\in\cM$ and $D\in\cD$, let
  $(M\otimes{}D)(\pi)$ be the model that, given $\act\in\Act$, samples
  $(r,o_1)\sim{}M(\pi)$ and $o_2\sim{}D(\pi)$, then emits
  $(r,(o_1,o_2))$. Set
  \[
    \cM\otimes\cD=\crl*{M\otimes{}D\mid{}M\in\cM,D\in\cD}.
  \]
  Then for all $\Mbar\in\cM$ and $\Dbar\in\cD$,
  \[
    \comp(\cM\otimes\cD,\Mbar\otimes\Dbar) = \comp(\cM,\Mbar).
  \]%
  This can be seen to hold by restricting the supremum in
  \pref{eq:comp} to range over models of the form $M\otimes\Dbar$.
\end{example}

\begin{example}[Data processing]
  Passing observations through a channel never decreases the
  \CompText. Consider a class of models $\cM$ with observation space $\cO$. Let
$\rho:\cO\to\cO'$ be given, and define $\rho\circ{}M$ to be the model
that, given decision $\pi$, samples $(r,o)\sim{}M(\pi)$, then emits
$(r,\rho(o))$. Let
$\rho\circ\cM\ldef{}\crl*{\rho\circ{}M\mid{}M\in\cM}$. Then for all $\Mbar\in\cM$, we have
\[
  \comp(\cM,\Mbar) \leq{} \comp(\rho\circ{}\cM,\rho\circ\Mbar).
\]%
This is an immediate consequence of the data processing
inequality for Hellinger distance, which implies that $\Dhels{\prn[\big]{\rho\circ{}M}(\act)}{\prn[\big]{\rho\circ\Mbar}(\act)}\leq{}\Dhels{M(\act)}{\Mbar(\act)}$.
\end{example}

All of the properties we have discussed up to this point are essential for any complexity measure that aims to
characterize the statistical complexity of general decision making
problems. With this in mind, we proceed to sketch the proof of
\pref{thm:lower_main}; the proof for \pref{thm:lower_main_expectation}
follows similar reasoning.

\subsubsection{\preft{thm:lower_main}: Proof Sketch}
  The key idea behind \pref{thm:lower_main} is to show that for any algorithm $p$
  and nominal model $\Mbar$, the \CompText certifies a lower bound on the price (in terms
of regret) that the algorithm must pay to obtain enough information to
distinguish the nominal model from an adversarially chosen alternative. In
particular, we will prove that for any $\gamma\approxgeq{}\sqrt{T}$
(recall the constraints on $\gamma$ in \pref{thm:lower_main})
and $\Mbar\in\cM$, for any $\veps>0$ sufficiently small (as a function
of $\gamma$ and $T$), 
there exists a
  model $M\in\cMloc[\veps](\Mbar)$ such that
  \begin{equation}
    \label{eq:lower_sketch0}
    \RegDM\approxgeq{}
    \min\crl[\big]{\prn*{\comp(\cMloc[\veps](\Mbar),\Mbar)-\delta}\cdot{}T, \gamma}
  \end{equation}
  with probability at least $\Omega(\delta)$. From here, the result in
  \pref{eq:lower_main} follows by maximizing over $\Mbar$ and $\gamma$.

  Let $T\in\bbN$, $\gamma>0$, and $\veps>0$ be fixed and consider an algorithm $p
  = \crl*{p\ind{t}(\cdot\mid\cdot)}_{t=1}^{T}$. Recall that
  $\Enm{M}{p}\brk{\cdot}$ denotes the expectation over the history $\hist\ind{T}$
when $M$ is the underlying problem instance and $p$ is the
algorithm. For each model $M$, let 
\[
\pmdist\ldef\Enm{M}{p}\brk*{\frac{1}{T}\sum_{t=1}^{T}p\ind{t}(\cdot\mid{}\hist\ind{t-1})}
\]
denote the average (conditional) distribution over decisions under $M$. Fix
a nominal model $\Mbar$. Since $\pmbar\in\Delta(\Act)$, the definition
of $\comp(\cMloc[\veps](\Mbar),\Mbar)$ guarantees that
\begin{align*}
  \sup_{M\in\cM_{\veps}(\Mbar)}
\En_{\act\sim\pmbar}\brk*{\fm(\pim) - \fm(\act)
  -\gamma\cdot\Dhels{M(\act)}{\Mbar(\act)}}
  \geq{} \comp(\cMloc[\veps](\Mbar),\Mbar).
\end{align*}
In particular, if $M\in\cMloc[\veps](\Mbar)$ attains the supremum
above, then by rearranging we have
\begin{align}
  \En_{\act\sim\pmbar}\brk*{\fm(\pim) - \fm(\act)}
  \geq{} \gamma\cdot \En_{\act\sim\pmbar}\brk*{\Dhels{M(\act)}{\Mbar(\act)}}
  + \comp(\cMloc[\veps](\Mbar),\Mbar).
  \label{eq:lower_sketch1}
\end{align}
Note that due to the equality
\begin{equation}
  \frac{1}{T}\cdot{}\Enm{M}{p}\brk*{\RegDM}=
  \En_{\act\sim{}\pmdist}\brk*{\fm(\pim) - \fm(\act)} %
  \label{eq:lower_sketch2},
\end{equation}
the result would be established if the left-hand side of
\pref{eq:lower_sketch1} were to be replaced by
$\En_{\act\sim\pmdist}\brk*{\fm(\pim) - \fm(\act)}$ (recalling that $\Dhels{M(\act)}{\Mbar(\act)}$ is non-negative).
The majority of the technical effort behind the proof is to use a change
of measure argument to show that as long as the localization radius
$\veps$ is sufficiently small as a function of $\gamma$ and $T$,
\begin{equation}
  \En_{\act\sim\pmbar}\brk*{\fm(\pim) -
    \fm(\act)}\approxleq{}\En_{\act\sim{}\pmdist}\brk*{\fm(\pim) -
    \fm(\act)}
  + \gamma\cdot
  \En_{\act\sim\pmbar}\brk*{\Dhels{M(\act)}{\Mbar(\act)}}.
    \label{eq:lower_sketch3}
\end{equation}
Combining this with \pref{eq:lower_sketch1} and
\pref{eq:lower_sketch2} yields
\begin{align*}
\Enm{M}{p}\brk*{\RegDM} &\approxgeq{}
\prn*{\En_{\act\sim\pmbar}\brk*{\fm(\pim) - \fm(\act)}
- \gamma\cdot \En_{\act\sim\pmbar}\brk*{\Dhels{M(\act)}{\Mbar(\act)}}}\cdot{}T\\
&\approxgeq{} \comp(\cMloc[\veps](\Mbar),\Mbar)\cdot{}T.
\end{align*}
The final lower bound scales with
$\min\crl*{\comp(\cMloc[\veps](\Mbar),\Mbar)\cdot{}T,\gamma}$
(suppressing dependence on $\delta$) because the
change of measure argument above requires that
$\RegDM\approxleq{}\gamma$ with high probability.
We deduce
\pref{eq:lower_sketch0} by optimizing over $\Mbar$ and $\gamma$.

For \pref{thm:lower_main_expectation}, the main difference from the proof
above is that the
$L_{\infty}$ notion of localization in \pref{eq:localized_infinity}
allows for a stronger change of measure argument.

    \subsection{\mainalgsection: A Unified Meta-Algorithm for
    Interactive Decision Making}
  \label{sec:main_upper}
With a fundamental limit based on the \CompText established, we now state the upper
bound on regret that serves as its counterpart. Our main regret bound is achieved by the \emph{\AlgText}
  meta-algorithm (\mainalg), described in \pref{alg:main}. The
  algorithm is based on the primitive of an \emph{online estimation
    oracle}, which generalizes the notion of online regression
oracle used in closely related work on the contextual bandit problem \citep{foster2020beyond,foster2020adapting,foster2021efficient}.\looseness=-1

An online estimation oracle, which we denote by $\AlgEst$, is an algorithm that attempts to estimate the underlying model $\Mstar$ from
data in a sequential fashion (typically via density estimation), and can be thought of as an (implicit)
input to \mainalg. At each round $t$, given the
observations and rewards collected so far, the estimation oracle produces an
estimate
\[
\Mhat\ind{t}=\AlgEst\ind{t}\prn*{ \crl*{(\act\ind{i},
r\ind{i},\obs\ind{i})}_{i=1}^{t-1} }
\]
for the true model $\Mstar$. Using this estimate, the most basic variant
of \mainalg (\optionone in \pref{alg:main}) proceeds by computing
the distribution $p\ind{t}$ that achieves the value
$\comp(\cM,\Mhat\ind{t})$ for the \CompText. That is, defining
\begin{equation}
  \label{eq:game_value}
  \gameval{\Mhat}(p, M) =
  \En_{\act\sim{}p}\biggl[\fm(\pim)-\fm(\pi)
    -\gamma\cdot\DhelsX{\big}{M(\act)}{\Mhat(\act)}
    \biggr],
\end{equation}
we set
\begin{equation}
  \label{eq:algorithm_argmin}
  p\ind{t} = \argmin_{p\in\Delta(\Act)}\sup_{M\in\cM}  \gameval{\Mhat\ind{t}}(p, M).
  \end{equation}
  \mainalg then samples the decision $\act\ind{t}$ from this
  distribution and moves on to the next round. This approach, where we
  map an estimate to a decision in a simple per-round fashion, can be
  viewed as a generalization of the \squarecb algorithm of
  \cite{foster2020beyond}.

 \optionone of \mainalg leads to regret bounds that scale with
 $\comp(\cM)$.  We achieve a localized regret bound that
  matches the quantity $\comploc{\veps}(\cM)$ appearing in
  \pref{thm:lower_main} using a more sophisticated variant
  of \mainalg (\optiontwo in \pref{alg:main}) augments the approach
  above by first restricting to a Hellinger confidence
  set $\cM\ind{t}\subseteq\cM$ based
  on the estimators $\crl[\big]{\Mhat\ind{i}}_{i=1}^{t-1}$, then
  solving the minimization problem in \pref{eq:algorithm_argmin} using
  the restricted set $\cM\ind{t}$ rather than the whole model
  class. We state both variants (\optionone and \optiontwo) because
  the former is simpler and more efficient, yet suffices for
  essentially all applications we consider. \optiontwo---while
  important for matching \pref{thm:lower_main} as closely as
  possible---is mainly of theoretical interest.

We show that \mainalg serves as a universal reduction from decision
making to estimation. Whenever the online estimation algorithm $\AlgEst$ can
accurately estimate the true model (with respect to Hellinger
distance), \mainalg enjoys low regret for decision making, with the
bound on regret determined by the \CompText. The process
bears some similarity to the certainty equivalence principle in
control, which also provides a conceptual separation of estimation and
decision making. The main difference is that rather than simply selecting the
optimal decision for the estimated model (which would have poor exploration),
we solve the minimax program in \pref{eq:algorithm_argmin} with the
estimated model to balance exploration and exploitation.

\subsubsection{Finite Model Classes}

While our most
general guarantees for \mainalg hold for any choice of the estimator
$\AlgEst$, in this section we state specialized guarantees---first for
finite, then general model classes---that allow for easy
comparison with our lower bounds (\pref{thm:lower_main,thm:lower_main_expectation}); general
results are deferred to \pref{sec:algorithm}. To state the guarantees in the simplest form possible, we make the
following basic regularity assumption.
\begin{assumption}
  \label{ass:localization_constant}
  There exists a constant $\Cloc>1$ such that for all models $\Mbar$ and all
  $\gamma>0$, $\veps>0$, \[\comp(\cMloc[(\Cloc\cdot\veps)](\Mbar),\Mbar)\leq{}\Cloc\cdot{}\comp(\cMloc[\veps](\Mbar),\Mbar).\]
\end{assumption}
Note that if $\comp(\cMloc[\veps](\Mbar),\Mbar)\propto\veps^{\rho}$ for
$\rho\leq{}1$, this is assumption is satisfied for all $\Cloc>1$. 
\begin{restatable}[Main upper bound---finite class version]{theorem}{uppermain}
\label{thm:upper_main}
Fix $\delta\in(0,1)$. Assume that
$\Rspace\subseteq\brk*{0,1}$ and 
\pref{ass:localization_constant} holds. Define
$C=\bigoh(\Cloc^{2}\log_{\Cloc}(T))$ and $\vepsupg\ldef{}48\prn[\big]{\frac{\gamma}{T}\log(\abs{\cM}/\delta) + \sup_{\Mbar\in\conv(\cM)}\comp(\cM,\Mbar)
+ \gamma^{-1}}$. Then \pref{alg:main},
with an
appropriate choice for parameters and estimation oracle, ensures
that with probability at least $1-\delta$,
\begin{equation}
  \label{eq:upper_main}
    \RegDM
  \leq{}
  C\cdot{}\min_{\gamma>0}\max\crl[\bigg]{\sup_{\Mbar\in\conv(\cM)}\comp(\cMloc[\vepsupg](\Mbar),\Mbar)\cdot{}T,\;
    \gamma\cdot{}\log(\abs{\cM}/\delta)}.
\end{equation}

\end{restatable}

\begin{algorithm}[tp]
    \setstretch{1.3}
     \begin{algorithmic}[1]
       \State \textbf{parameters}:
       \Statex[1] Online estimation oracle $\AlgEst$.
       \Statex[1] Exploration parameter $\gamma>0$.
       \Statex[1] Confidence radius $R>0$. \algcommentlight{\optiontwo only.}
  \For{$t=1, 2, \cdots, T$}
  \State Compute estimate $\Mhat\ind{t} = \AlgEst\ind{t}\prn[\Big]{ \crl*{(\act\ind{i},
r\ind{i},\obs\ind{i})}_{i=1}^{t-1} }$.
\State\textcolor{blue!70!black}{\optionone:} 
\State\hspace{\algorithmicindent}Let
$p\ind{t}=\argmin_{p\in\Delta(\Act)}\sup_{M\in\cM}\gameval{\Mhat\ind{t}}(p,M)$.~~~\algcommentlight{Minimizer
  for $\comp(\cM,\Mhat\ind{t})$; cf. Eq. \pref{eq:game_value}}.%
\State\textcolor{blue!70!black}{\optiontwo:}
\State\hspace{\algorithmicindent}Let
$\cM\ind{t}=\crl[\big]{M\in\cM\ind{t-1}\mid{}\sum_{i=1}^{t-1}\En_{\act\sim{}p\ind{i}}\brk[\big]{\DhelsX{\big}{M(\act)}{\Mhat\ind{i}(\act)}}\leq{}R^{2}}$,
with $\cM\ind{0}\ldef{}\cM$.\label{line:confidence_set}%
\State\hspace{\algorithmicindent}Let
$p\ind{t}=\argmin_{p\in\Delta(\Act)}\sup_{M\in\cM\ind{t}}\gameval{\Mhat\ind{t}}(p,M)$.~~~\algcommentlight{Minimizer for $\comp(\cM\ind{t},\Mhat\ind{t})$.}%
\State{}Sample decision $\act\ind{t}\sim{}p\ind{t}$ and update estimation
oracle with $(\act\ind{t},r\ind{t}, \obs\ind{t})$.
\EndFor
\end{algorithmic}
\caption{Estimation to Decision-Making Meta-Algorithm (\mainalg)}
\label{alg:main}
\end{algorithm}

 This theorem is a special case of a more general
result, \pref{thm:upper_general} in \pref{sec:algorithm}.
The upper bound is seen to have a similar form to the lower bound in
\pref{thm:lower_main}.
In particular, suppressing precise dependence
on the failure parameter
$\delta$ and logarithmic factors, we have
\begin{align*}
  &\text{\pref{thm:lower_main}:} \quad\RegDM \approxgeq{} \max_{\gamma{}}\min\crl[\bigg]{\sup_{\Mbar\in\cM}\comp(\cMloc[\vepslowg](\Mbar),\Mbar)\cdot{}T,\gamma},\\
  &\text{\pref{thm:upper_main}:}
\quad\RegDM\approxleq{}\min_{\gamma}\max\crl[\bigg]{\sup_{\Mbar\in\conv(\cM)}\comp(\cMloc[\vepsupg](\Mbar),\Mbar)\cdot{}T, \gamma\cdot\log\abs{\cM}}.
\end{align*}
The min-max and max-min above can be exchanged under mild conditions on
$\comploc{\veps}(\cM)$. The most important difference between the
rates is the presence of the term
$\log\abs{\cM}$ in \pref{thm:upper_main}, which
serves as an upper bound on the complexity of statistical estimation
for the model class $\cM$. A secondary difference is that the upper
bound in \pref{thm:upper_main} uses the class $\conv(\cM)$, while the lower bound in \pref{thm:lower_main}
only considers $\cM$.

In terms of failure probability, the lower bound in \pref{thm:lower_main} provides a
  meaningful converse to \pref{thm:upper_main} in the moderate
  probability regime where $\delta\approxleq{}\compthmone$; setting $\delta\propto{}1/\sqrt{T}$
  suffices for non-trivial classes.\footnote{We use the term
    ``non-trivial'' to refer to any class that embeds
    the two-armed bandit problem.} With this choice, the
  dependence on $\delta$ in the lower bound vanishes and the
  $\log(1/\delta)$ factor in the upper bound becomes $\log(T)$. We
  conclude that (up to the other differences discussed above), the upper
  bound cannot be improved beyond this logarithmic factor in this
  probability regime. In addition, whenever the localized classes
  $\cMloc(\Mbar)$ and $\cMinf(\Mbar)$ have similar complexity (such is
  the case for all of the examples in \pref{table:lower}),
  \pref{thm:lower_main_expectation} provides an in-expectation lower
  bound of the same order.

Gaps between the upper and lower
bounds are discussed in further detail in \pref{sec:main_discussion}.

\subsubsection{\preft{thm:upper_main}: Proof Sketch}
  
  We find it most illuminating to sketch how to prove a non-localized version of
  \pref{thm:upper_main}, which is a nearly immediate consequence of
  the definition of the \CompText. Consider \pref{alg:main} with \optionone and
  parameter $\gamma>0$. Let
  $\EstHel\ldef{}\sum_{t=1}^{T}\En_{\act\ind{t}\sim{}p\ind{t}}\brk[\big]{\DhelsX{\big}{\Mstar(\act\ind{t})}{\Mhat\ind{t}(\act\ind{t})}}$
  denote the cumulative Hellinger error of the estimation oracle
  $\AlgEst$. We write
\begin{align*}
  \RegDM &=
           \sum_{t=1}^{T}\En_{\act\ind{t}\sim{}p\ind{t}}\brk*{\fmstar(\pimstar)-\fmstar(\act\ind{t})}\\
         &=
           \sum_{t=1}^{T}\En_{\act\ind{t}\sim{}p\ind{t}}\brk*{\fmstar(\pimstar)-\fmstar(\act\ind{t})}
           - \gamma{}\cdot
           \En_{\act\ind{t}\sim{}p\ind{t}}\brk*{\Dhels{\Mstar(\act\ind{t})}{\Mhat\ind{t}(\act\ind{t})}}
           + \gamma\cdot{}\EstHel.
\end{align*}
For each $t$, since $\Mstar\in\cM$, we have
\begin{align}
  &\En_{\act\ind{t}\sim{}p\ind{t}}\brk*{\fmstar(\pimstar)-\fmstar(\act\ind{t})}
           - \gamma{}\cdot
  \En_{\act\ind{t}\sim{}p\ind{t}}\brk*{\Dhels{\Mstar(\act\ind{t})}{\Mhat\ind{t}(\act\ind{t})}}\notag\\
&  \leq  \sup_{M\in\cM}\En_{\act\ind{t}\sim{}p\ind{t}}\brk*{\fm(\pim)-\fm(\act\ind{t})}
           - \gamma{}\cdot
                                                                                                   \En_{\act\ind{t}\sim{}p\ind{t}}\brk*{\Dhels{M(\act\ind{t})}{\Mhat\ind{t}(\act\ind{t})}}\notag\\
  &  = \inf_{p\in\Delta(\Act)}\sup_{M\in\cM}\En_{\act\sim{}p}\brk*{\fm(\pim)-\fm(\act)
           - \gamma{}\cdot
    \Dhels{M(\act)}{\Mhat\ind{t}(\act)}}\notag\\
  &  = \comp(\cM,\Mhat\ind{t}).\label{eq:minimax_regret_basic}
\end{align}
We conclude that
\[
  \RegDM \leq{} \max_{t}\comp(\cM,\Mhat\ind{t})\cdot{}T + \gamma\cdot\EstHel.
\]
From here, all that remains is to tune $\gamma$ and show that we can
choose the oracle
$\AlgEst$ such that $\EstHel\approxleq{}\log\abs{\cM}$ with high
probability and $\Mhat\ind{t}\in\conv(\cM)$.

Achieving the localized result in \pref{thm:upper_main} requires more
effort, and relies on \optiontwo. The key idea is to use certain properties of the minimax
program \pref{eq:algorithm_argmin} to relate the value of the
\CompText for the confidence sets $\cM\ind{t}$ to the value of the
localized \CompShort appearing in \pref{eq:upper_main}.
\qed

\subsubsection{Infinite Model Classes}

In order to incorporate rich, potentially nonparametric classes of
models, we provide a generalization of \pref{thm:upper_main} based on
covering numbers. This extension is important in our applications to
bandits and reinforcement learning (\pref{sec:bandit,sec:rl}).

\begin{definition}[Model class covering number]
  A model class $\cM'\subseteq\cM$ is said to be an $\veps$-cover for
  $\cM$ if
  \begin{equation}
  \label{eq:m_cover}
    \forall{}M\in\cM\quad\exists{}M'\in\cM'\quad\text{s.t.}\quad
    \sup_{\act\in\Act}\Dhels{M'(\act)}{M(\act)}\leq{}\veps^{2}.
  \end{equation}
  We let $\cN(\cM,\veps)$ denote the size of the smallest such
  cover, and define
  \begin{equation}
    \label{eq:m_comp}
  \MComp \ldef{} \inf_{\veps\geq{}0}\crl*{
    \log\cN(\cM,\veps)+\veps^{2}T
  }
\end{equation}
as a fundamental complexity parameter associated with $\cM$.
\end{definition}
Basic examples include parametric models in $d$ dimensions (i.e.,
$\log\cN(\cM,\veps)\propto{}d\log(1/\veps)$) where
$\MComp=\bigoht(d)$, and nonparametric models with
$\log\cN(\cM,\veps)\propto\veps^{-\rho}$ for $\rho>0$, where
$\MComp=\bigoh(T^{\frac{\rho}{2+\rho}})$.

    Beyond requiring bounded covering numbers, we place a (fairly weak) regularity condition on the class of densities associated with the
model class $\cM$, which is standard within the literature on density
estimation \citep{opper99logloss,bilodeau2021minimax}. Recall
(cf. \pref{sec:prelims}) that $\densm(\cdot,\cdot\mid{}\cdot)$ denotes
the conditional density for $M$ under the common conditional measure $\nu(\cdot,\cdot\mid{}\cdot)$.
\begin{assumption}[Bounded densities]
  \label{ass:finite_measure}
  There exists a constant $B\geq{}e$ such that:
  \begin{enumerate}
    \item $\dom(\Rspace\times\Ospace\mid{}\act)\leq{}B$ for all $\act\in\Act$.
    \item
      $\sup_{\act\in\Act}\sup_{(r,o)\in\Rspace\times\Ospace}\densm(r,o\mid{}\act)\leq{}B$
      for all $M\in\cM$.
    \end{enumerate}
  \end{assumption}

  Our result scales with $\log(B)$, which is constant for bandit
  problems with bounded rewards, logarithmic for reinforcement
  learning problems with finite state/action spaces, and polynomial in dimension for continuous reinforcement learning problems. See
  \pref{sec:rl} for further discussion.

  \begin{restatable}[Main upper bound---general version]{theorem}{uppermaininfinite}
  \label{thm:upper_main_infinite}
  Let $\delta\in(0,1)$ be given, and let $\Cloc$ be as in
  \pref{ass:localization_constant}. Assume that
  $\Rspace\subseteq\brk*{0,1}$ and that \pref{ass:finite_measure}
  holds. Define $C_1=\bigoh(\Cloc^2\log_{\Cloc}(T)\log^2(BT))$, $C_2=\bigoh(\log^2(BT))$, and $\vepsupg=C_2\prn[\big]{\frac{\gamma}{T}(\MComp +\log(\delta^{-1})) + \sup_{\Mbar\in\conv(\cM)}\comp(\cM,\Mbar)
    + \gamma^{-1}}$. Then \pref{alg:main},
with an appropriate choice of parameters and estimation oracle, guarantees
that with probability at least $1-\delta$,
\begin{equation}
  \label{eq:upper_main_infinite}
  \RegDM
  \leq{}
  C_1\cdot{}\min_{\gamma>0}\max\crl[\bigg]{\sup_{\Mbar\in\conv(\cM)}\comp(\cMloc[\vepsupg](\Mbar),\Mbar)\cdot{}T,\;
    \gamma\cdot{}(\MComp +\log(\delta^{-1}))}.
\end{equation}
\end{restatable}

\subsection{Learnability}
\label{sec:learnability}

In statistical learning and related settings, a long line of
research on \emph{learnability} provides necessary and sufficient
conditions under which a hypothesis class under consideration is
\emph{learnable} in the sense that there exist algorithms with
non-trivial sample complexity
\citep{vapnik1995nature,alon1997scale,shalev2010learnability,rakhlin2010online,daniely2011multiclass}. Equipped
with our main upper and lower bounds, we use the \CompText to provide
an analogous characterization of learnability (i.e., existence of
algorithms with sublinear regret) in the \FrameworkShort framework.

Our characterization of learnability applies to model classes $\cM$ that are i)
convex, and ii) admit non-trivial estimation complexity. In addition, we make the following mild regularity assumption.
  \begin{assumption}
    \label{ass:mild}
    There exists $M_0\in\cM$ such that $f^{\sss{M_0}}$ is constant.
  \end{assumption}
  This assumption is satisfied for all of the examples considered in this paper. The characterization is as follows.
\begin{restatable}[Learnability]{theorem}{learnabilitymain}
  \label{thm:learnability_main}
  Assume that $\Rspace\subseteq\brk*{0,1}$, and that
  \pref{ass:finite_measure} and \pref{ass:mild} hold. Suppose that $\cM$ is convex
  and has $\MComp=\bigoht(T^{q})$ for some $q<1$. Then:
  \begin{enumerate}
  \item If there exists $\rho>0$ such that $\lim_{\gamma\to\infty}
    \comp(\cM)\cdot\gamma^{\rho}=0$, then there exists an algorithm
    for which
        \[
      \lim_{T\to\infty}\frac{\MinimaxReg}{T^{p}}=0
    \]
    for some $p<1$.
  \item If $\lim_{\gamma\to\infty}\comp(\cM)\cdot\gamma^{\rho}>0$
    for all $\rho>0$, then any algorithm must have
        \[
      \lim_{T\to\infty}\frac{\MinimaxReg}{T^{p}}=\infty
    \]
      for all $p < 1$.
  \end{enumerate}
\end{restatable}
\pref{thm:learnability_main} shows that for any convex model class
where the model estimation complexity $\MComp$ is sublinear,
polynomial decay of the \CompText is sufficient to achieve sublinear
regret. Conversely, if the \CompText does not decay polynomially, no
algorithm can achieve sublinear regret. We emphasize that the latter
result (necessity of polynomial decay) applies to any model class,
regardless of whether it admits low estimation
complexity. Using results from follow-up work of
  \citet{foster2023tight}, the assumption of convexity
  can be removed; see discussion in \cref{sec:main_discussion}.

It should be noted that learnability is---by
definition---a coarse property, and is best thought of as a basic
sanity check.  As we show in \pref{sec:examples,sec:bandit,sec:rl},
our machinery is considerably more precise, and yields quantitative upper and lower bounds that
accurately reflect problem-dependent parameters such as dimension.

\subsection{Tighter Guarantees Based on Decision Space Complexity}
\label{sec:main_bayes}
Up to this point, all of the upper bounds we have presented depend
on the model estimation complexity $\MComp$. In general, low model estimation complexity is not required
to achieve low regret for decision making. This is because our end goal is to make good
\emph{decisions}, so we can give up on accurately
estimating the model in regions of the decision space that do not help
to distinguish the relative quality of decisions.
Our final result for this section provides a tighter bound that
replaces the model estimation complexity $\MComp$ with an analogous estimation complexity
parameter for the decision space $\Pim$, albeit at the cost of removing the
localization found in \pref{thm:upper_main}. This result is non-constructive in nature, and is proven
by using the minimax theorem to move to the Bayesian setting where
the true model is drawn from a worst-case prior that is known to
the learner (cf. \pref{sec:minimax_swap}), then running \mainalg over a special model class constructed with knowledge
of the prior. We refer ahead to \pref{sec:dual} for more background on
this approach.

Our main upper bound---which applies to both finite and infinite classes---is
stated in terms of the following, weaker notion of covering number, which is tailored to decision making.
\begin{definition}[Decision space covering number]
A decision set $\Act'\subseteq\Pim$ is an $\veps$-cover for $\Pim$ if
    \begin{equation}
  \label{eq:action_cover}
  \forall{}M\in\cM\quad\exists{}\act'\in\Act'\quad\text{s.t.}\quad
  \fm(\pim) - \fm(\act')\leq\veps.
  \end{equation}
  We let $\cN(\Pim,\veps)$ denote the size of the smallest cover, and
  define
  \begin{equation}
  \ActComp
    = \inf_{\veps\geq{}0}\crl*{
      \log\cN(\Pim,\veps) + \veps{}T
    }\label{eq:action_complexity}
  \end{equation}
as a fundamental complexity measure associated with $\Pim$.
\end{definition}
The result has the same structure as \pref{thm:upper_main}, but
replaces $\MComp$ with $\ActComp$.
\begin{restatable}{theorem}{uppermainbayes}
  \label{thm:upper_main_bayes}
  Whenever the conclusion of \pref{prop:minimax_swap} holds, there exists an algorithm that ensures that 
\begin{align}
  \label{eq:upper_main_bayes}
  \En\brk*{\RegDM}
  &\leq{}
  2\cdot\min_{\gamma>0}\max\crl*{\comp(\conv(\cM))\cdot{}T,\;\;
    \inf_{\veps\geq{}0}\crl*{\gamma\cdot\log\ActCov +\veps\cdot{}T}}\\
  &\leq{}
  2\cdot\min_{\gamma>0}\max\crl*{\comp(\conv(\cM))\cdot{}T,\;
    (1+\gamma)\cdot{}\ActComp}.
\end{align}
In particular, when $\abs{\Pim}<\infty$, we have
\begin{align}
  \En\brk*{\RegDM}    &\leq{}
  2\cdot\min_{\gamma>0}\max\crl*{\comp(\conv(\cM))\cdot{}T,
      \gamma\cdot{}\log\abs{\Pim}}.
\end{align}
\end{restatable}
As a concrete example, for the multi-armed bandit, where
$\Pim=\brk{A}$ and $\comp(\conv(\cM))\approxleq{}\frac{A}{\gamma}$, we
have $\MComp=\bigoht(A)$, while $\ActComp=\log{}A$; the latter bound
leads to near-optimal regret $\sqrt{AT\log{}A}$. There are
some applications for which $\MComp$ correctly characterizes the
precise minimax rate, but in general the estimation complexity for $\cM$ can be arbitrarily large relative to $\Pim$ (see
\pref{sec:rl} for discussion in the context of reinforcement learning). On the other hand, for all of
the applications we are aware of, $\ActComp$ has a smaller
contribution to regret than the \CompText itself after balancing
$\gamma$ (e.g., $\log{}A$ versus $A$ for the multi-armed bandit).

As a result of \pref{thm:upper_main_bayes}, we obtain the following
strengthening of \pref{thm:learnability_main}.

\begin{restatable}[Learnability---refined version]{theorem}{learnabilitymainbayes}
\label{thm:learnability_main_bayes}
Suppose that $\ActComp=\bigoht(T^{q})$ for some $q<1$, and that the
conclusion of \pref{prop:minimax_swap} holds, but place no
assumption on $\MComp$. Then the conclusion of \pref{thm:learnability_main} continues to hold.
\end{restatable}

\subsection{Discussion and Subsequent Improvements}
\label{sec:main_discussion}

We close by highlighting some gaps between our quantitative upper and lower bounds, and opportunities for
improvement. We also discuss some gaps in our understanding of regret in the frequentist
setting and Bayesian setting. Many of these issues are immaterial from the perspective of understanding learnability,
but are important for understanding precise minimax rates.
\subsubsection{When are the Upper and Lower Bounds Tight?}

The lower and
upper bounds from \pref{thm:lower_main,thm:lower_main_expectation,thm:upper_main} have
a similar functional form, with the main differences
involving localization, model estimation complexity, and convexity. In
what follows, we discuss each of these gaps, highlight improvements
obtained in follow-up work \citep{foster2023tight}, and discuss further opportunities for
improvement. See \cref{sec:related} for additional discussion of
subsequent work.

\paragraph{Localization and in-expectation lower bounds}
Both \pref{thm:lower_main} and \pref{thm:upper_main} depend on the
\CompText for the localized model class $\cMloc[\veps](\Mbar)$ rather
than the full model class. The localization radius in the upper bound
is larger than that of the lower bound, which can lead to looseness for
  some classes (see \citet{foster2023tight} for an
  example). The follow-up work of \cite{foster2023tight}
  obtains matching upper and lower bounds that remove this issue by
  working with a ``constrained''
  variant of the \CompShort.\footnote{Informally, the constrained DEC
    replaces the soft penalty term of $-\gamma \En_{\pi\sim
      p}\brk*{\Dhels{M(\pi)}{\Mbar(\pi)}}$ with a hard constraint of
    the form $\En_{\pi\sim p}\brk*{\Dhels{M(\pi)}{\Mbar(\pi)}}\leq
    \veps$.} As shown in
  \citet{foster2023tight}, the constrained DEC is equivalent to the
  form of the localized \CompShort appearing in \cref{thm:lower_main}
  under mild regularity conditions, indicating that this form of the
  \CompShort is fundamental, and \cref{thm:upper_main} can be tightened.

Next, let us discuss the connection between localization and
in-expectation (versus in-probability) lower bounds. \pref{thm:lower_main_expectation} provides lower bounds based on the
\CompShort for expected regret, but relies the $L_{\infty}$ notion of
localization in \pref{eq:localized_infinity}. While this notion
suffices to derive the lower bounds in \pref{table:lower}, it is not
difficult to construct examples where $L_{\infty}$-localization is
too loose, where the lower bound cannot be achieved (in contrast to
the weaker notion of localization in \pref{eq:localized}, which is
achieved by \pref{alg:main}). The improvements obtained in
  \cite{foster2023tight} lead to lower bounds that simultaneously
  obtain the favorable notion of localization in
  \cref{thm:lower_main}, yet hold in expectation.

More broadly, we do not yet fully understand what features of the
class $\cM$ determine whether localization is
needed to attain precise minimax rates. On one hand, it is not hard to
construct examples for which localization can significantly improve
regret, and hence is necessary for any lower bound. For example, if $\cM$ is large but there is a critical
radius $\veps$ for which $\cMloc[\veps](\Mbar)=\crl{\Mbar}$, the regret bound in \pref{thm:upper_general} (the general version of \pref{thm:upper_main}) is constant. However, for all of
the applications in bandits and reinforcement learning that we are
aware of, localization seems to only lead to improvements in logarithmic
factors. We refer interested readers to \pref{app:structural}, which
shows that for certain ``reasonable'' model classes, the global and
local \CompShort coincide up to constant factors. See 
  \citet{foster2023tight} for further (stylized) examples of classes
  in which localization is needed to attain tight rates.

\paragraph{Convexity}
Ignoring localization, the lower bounds in \pref{thm:lower_main,thm:lower_main_expectation} scale
with $\sup_{\Mbar\in\cM}\comp(\cM,\Mbar)$, while the upper bound in
\pref{thm:upper_main} scales
with
$\sup_{\Mbar\in\conv(\cM)}\comp(\cM,\Mbar)$. That is, the
upper bounds require evaluating the \CompShort with reference models in the
convex hull of $\cM$, while the lower bounds do not. The follow-up work of
  \cite{foster2023tight} strengthens our proof techniques to give lower
  bounds that scale with
  $\sup_{\Mbar\in\conv(\cM)}\comp(\cM,\Mbar)$, indicating that
  convexity for the reference model $\Mbar$ is fundamental.

\paragraph{Estimation error}
The most notable distinction between our upper and lower bounds is the
dependence on the complexity of estimation for $\cM$ ($\MComp$ for
\pref{thm:upper_main} and $\ActComp$ for \pref{thm:upper_main_bayes}).
The correct dependence on the estimation
complexity is a subtle issue which cannot be fully addressed without
introducing additional complexity parameters.
\begin{itemize}
\item For the finite-armed bandit problem where $\Pim=\brk{A}$,
we have $\ActComp=\log{}A$. As a result, \pref{thm:upper_main_bayes} gives an upper bound scaling with
  $\sqrt{AT\log{}A}$, while the lower bounds from
  \pref{thm:lower_main,thm:lower_main_expectation} scale as $\sqrt{AT}$. In this case,
  the lower bound coincides with the minimax rate
  \citep{audibert2009minimax}, and the upper bound can be improved by $\sqrt{\log{}A}$.
\item For linear bandits on the unit ball in $\bbR^{d}$, $\ActComp$ scales as $\bigoht(d)$. Here,
  \pref{thm:upper_main_bayes} gives regret $\bigoht(d\sqrt{T})$, while
  the lower bounds from
  \pref{thm:lower_main,thm:lower_main_expectation} scale with $\sqrt{dT}$. In this
  case, the upper bound is tight, and the lower bound can be improved
  by $\sqrt{d}$ \citep{lattimore2020bandit}.
\end{itemize}
These observations highlight that fully resolving the correct dependence on
estimation complexity requires additional problem complexity parameters. We leave this issue for future work, but we emphasize that
for all the applications we are aware of, the decision space
estimation complexity $\ActComp$ has a smaller contribution to regret than
the \CompText itself (after balancing $\gamma$), and does not appear
to be the deciding factor in whether a problem is learnable.

Finally, we mention in passing that the bound on the estimation
complexity in \pref{thm:upper_main_infinite} depends on a rather coarse
point-wise notion of model covering number, and it would be useful to
improve this result to take advantage of refined (e.g.,
empirical or sequential) covering numbers \citep{StatNotes2012}; this is
out of scope for the present work.

  \paragraph{Other technical differences}
\pref{thm:lower_main} restricts to $\gamma>8e\sqrt{T}$, whereas
  \pref{thm:upper_main} allows for any $\gamma>0$. This is essentially
  without loss of generality: Whenever
  $\compthmone\geq{}\gamma^{-1}$, which holds for any non-trivial
  class, the optimal choice for $\gamma$ satisfies $\gamma\geq{}\sqrt{T}$.\looseness=-1

\subsubsection{Regret Bounds: Frequentist vs. Bayesian and
  Constructive vs. Non-Constructive}
\label{sec:bayesian}
Our localized upper bounds, \pref{thm:upper_main} and
\pref{thm:upper_main_infinite}, are proven directly in the
\FrameworkShort framework described in \pref{sec:intro} (the ``frequentist''
setting) by running
\pref{alg:main} with a particular choice of estimation
oracle. \pref{thm:upper_main_bayes}, which replaces the model estimation
complexity $\MComp$ with the decision space estimation complexity
$\ActComp$, is proven non-constructively by exhibiting an upper bound
on regret in the Bayesian setting (cf. \pref{sec:minimax_swap}).
Informally, working in the Bayesian
setting is a powerful tool, because it allows one to ``average out'' models that agree on the optimal
action, thereby reducing the estimation complexity. Follow-up
  work of \cite{foster2022complexity} provides a frequentist algorithm
  which achieves the same guarantee by exploiting a constructive
  analogue of this principle. See \pref{sec:related} for further
  discussion.
An interesting question that remains is whether one can achieve the
dependence on $\cN(\Pim,\veps)$ in \pref{thm:upper_main_bayes} and the
localization in \pref{thm:upper_main_infinite} simultaneously. 

  \section{The \mainalgsection Meta-Algorithm: General Toolkit}
  \label{sec:algorithm}

  In this section we provide a general toolkit for deriving regret bounds and
  efficient algorithms using \mainalg
  meta-algorithm (\pref{alg:main}). We begin with a regret bound for
  \mainalg with generic
  online estimation oracles (\pref{sec:oracle}), then provide a dual
  Bayesian view of the \CompText and \mainalg algorithm
  (\pref{sec:dual}), and finally
 close with some simple extensions
  (\pref{sec:general_distance}).

  \subsection{Guarantees for General Online Estimation Oracles}
  \label{sec:oracle}
  The \mainalg meta-algorithm can be applied with any user-specified estimation oracle
  $\AlgEst$. For general oracles, the decision making performance of
  \mainalg depends on the estimation performance of the oracle, which
  we measure via cumulative Hellinger error:
  \begin{equation}
    \label{eq:hellinger_error}
        \EstHel \ldef{} \sum_{t=1}^{T}\En_{\act\ind{t}\sim{}p\ind{t}}\brk*{\DhelsX{\Big}{\Mstar(\act\ind{t})}{\Mhat\ind{t}(\act\ind{t})}}.
      \end{equation}
Our most general theorem shows that for any choice of oracle, \mainalg inherits the Hellinger
estimation error as a bound on decision making regret, thereby bridging estimation and decision making.

To state the result, we define $\cMhat\ind{t}$
  as any set such that $\Mhat\ind{t}\in\cMhat\ind{t}$ almost surely for all $t$, and
  define $\cMhat=\cup_{t\geq{}1}\cMhat\ind{t}$.
  To provide high probability guarantees, we make the following
  assumption.
  \begin{assumption}
    \label{ass:hellinger_oracle}
    The estimation oracle $\AlgEst$ guarantees for any $T\in\bbN$ and
    $\delta\in(0,1)$, with probability
    at least $1-\delta$, $\EstHel\leq{}\EstProbHel$, where
    $\EstProbHel$ is a known upper bound.
  \end{assumption}
For example, as we show in the sequel, whenever $\cM$ is
finite, Vovk's aggregating algorithm \citep{vovk1995game} ensures that
$\En\brk*{\EstHel}\leq{}\log\abs{\cM}$, and satisfies
\pref{ass:hellinger_oracle} with $\EstProbHel=2\log(\abs{\cM}/\delta)$.

\begin{theorem}
\label{thm:upper_general}
\pref{alg:main} with \optionone and exploration parameter $\gamma>0$ guarantees
that
\begin{equation}
  \label{eq:upper_general1}
\RegDM \leq{}
  \sup_{\Mbar\in\cMhat}\comp(\cM,\Mbar)\cdot{}T + \gamma\cdot\EstHel
\end{equation}
almost surely. Furthermore, when \pref{ass:hellinger_oracle} holds, \pref{alg:main} with \optionone guarantees that for any $\delta\in(0,1)$, with probability at least $1-\delta$,
\begin{equation}
  \label{eq:upper_general2}
  \RegDM \leq{}
  \sup_{\Mbar\in\cMhat}\comp(\cM,\Mbar)\cdot{}T + \gamma\cdot\EstProbHel.
\end{equation}
Finally, fix $\delta\in(0,1)$ and consider \pref{alg:main} with \optiontwo and $R^{2}=\EstProbHel$. Suppose that $\AlgEst$ has
$\cMhat\ind{t}\subseteq\conv(\cM\ind{t})$ for all $t$, and that $\Rspace\subseteq\brk*{0,1}$. Then for any fixed $T\in\bbN$ and
$\gamma>0$, with probability at least $1-\delta$, 
  \begin{align}
  \RegDM
  &\leq{}
    \sum_{t=1}^{T}\sup_{\Mbar\in\conv(\cM)}\comp(\cMloc[\veps_t](\Mbar),\Mbar)
    + \gamma\cdot{}\EstProbHel,
    \label{eq:upper_general3}
\end{align}
where $\veps_t\ldef{}6\frac{\gamma}{t}\EstProbHel + \sup_{\Mbar\in\conv(\cM)}\comp(\cM,\Mbar)
+ (2\gamma)^{-1}$.
\end{theorem}
The guarantees in \pref{eq:upper_general1} and
\pref{eq:upper_general2} concern the simpler variant of \mainalg
(\optionone), which is more practical to implement but does not achieve localization. When invoked with Vovk's aggregating
algorithm as described above, the regret for \optionone scales as roughly $\min_{\gamma>0}\max\crl*{
  \sup_{\Mbar\in\cMhat}\comp(\cM,\Mbar)\cdot{}T,\;
  \gamma\cdot\log\abs{\cM}}$, which matches
\pref{thm:upper_main} modulo localization. The guarantee in \pref{eq:upper_general3} concerns \optiontwo, and constitutes our most general regret guarantee bound on
the localized \CompText. When invoked with a variant
of the aggregating algorithm designed to satisfy the requisite
condition that $\cMhat\ind{t}\subseteq\conv(\cM\ind{t})$ (\pref{app:online}),
this result recovers \pref{thm:upper_main} as a special
case.

\begin{remark}[Inexact minimizers]
  All of the results regarding \mainalg (\pref{alg:main}) are stated
  for the case where $p\ind{t}$ is chosen to exactly solve the minimax problem
  \[
    \argmin_{p\in\Delta(\Act)}\sup_{M\in\cM\ind{t}}  \gameval{\Mhat\ind{t}}(p, M),
  \]
  for a model class $\cM\ind{t}$ and estimator $\Mhat\ind{t}$. This
  leads to regret bounds that scale with $\comp(\cM\ind{t},\Mhat\ind{t})$.
  If we instead use a distribution $p\ind{t}$ that \emph{certifies an
    upper bound on the \CompShort} in the sense that
    \[
    \sup_{M\in\cM\ind{t}}  \gameval{\Mhat\ind{t}}(p\ind{t}, M)
    \leq{} g_{\gamma}(\cM\ind{t},\Mhat\ind{t})
  \]
  for some function $g_{\gamma}(\cdot,\cdot)$, then all of the main
  theorems  (\pref{thm:upper_main,thm:upper_main_infinite,thm:upper_main_bayes,thm:upper_general,prop:bayes_basic,thm:upper_general_distance}) continue to hold, with occurrences of
  $\comp(\cM,\Mbar)$ replaced by $g_{\gamma}(\cM,\Mbar)$. We tacitly
  use this fact in later sections.
\end{remark}

\subsubsection{Online Estimation Oracles: Density Estimation and Examples}
A topic we have not yet addressed is how to go about minimizing the Hellinger estimation error in
\pref{eq:hellinger_error}. Building on classical literature in statistical
estimation, we show that this problem can generically be solved via
\emph{online density estimation}. For each example $(\act\ind{t},
r\ind{t}, \obs\ind{t})$, define the
logarithmic loss for a model $M$ as\looseness=-1
\begin{equation}
  \label{eq:log_loss}
  \logloss\ind{t}(M) = \log\prn*{
    \frac{1}{\densm(r\ind{t}, \obs\ind{t}\mid{}\act\ind{t})}
  },
\end{equation}
where we recall that $\densm(\cdot,\cdot\mid{}\act)$ is the conditional
density for $(r,\obs)$ under $M$ (cf. \pref{sec:prelims}). As an
intermediate quantity, we consider regret under the logarithmic loss:
\begin{equation}
  \label{eq:log_loss_regret}
  \RegLog
  = \sum_{t=1}^{T}\logloss\ind{t}(\Mhat\ind{t}) - \inf_{M\in\cM}\sum_{t=1}^{T}\logloss\ind{t}(M).
\end{equation}
Regret minimization with the logarithmic loss (also known as \emph{sequential
probability assignment}) is a fundamental problem in
online learning. Efficient algorithms are known for model classes of interest
\citep{cover1991universal,vovk1995game,kalai2002efficient,hazan2015online,orseau2017soft,rakhlin2015sequential,foster2018logistic,luo2018efficient},
and this is complemented by theory which provides minimax rates for generic
model classes \citep{shtarkov1987universal,opper99logloss,cesabianchi99logloss,bilodeau2020tight}. Canonical
examples include finite classes, where Vovk's aggregating algorithm
guarantees $\RegLog\leq\log\abs{\cM}$ for every sequence,\footnote{See
  \pref{app:online} for detailed guarantees.} and linear
models (i.e., $M(r,o\mid{}\act)=\tri*{\phi(r,o,\act),\theta}$ for a
fixed feature map in $\phi\in\bbR^{d}$), where algorithms with
$\RegLog=\bigoh(d\log(T))$ are known
\citep{rissanen1986complexity,shtarkov1987universal}. All of these
algorithms satisfy $\cMhat=\conv(\cM)$.\footnote{In fact, even for
  improper estimators that do not satisfy $\cMhat=\conv(\cM)$, it is
  always possible to project $\Mhat\ind{t}$ onto
  $\conv(\cM)$ while maintaining the estimator's bound on $\En\brk{\EstHel}$.} We refer the
reader to Chapter 9 of \cite{PLG} for further examples and
discussion.

The following result shows that a bound on the log-loss regret immediately yields a bound on the Hellinger estimation error. %
\begin{proposition}
  \label{prop:log_loss_hellinger}
  For any estimation algorithm $\AlgEst$,
  whenever \pref{ass:realizability} holds,
  \begin{equation}
    \En\brk*{\EstHel} \leq{} \En\brk*{\RegLog}.
  \end{equation}
Furthermore, for any $\delta\in(0,1)$, with
probability at least $1-\delta$,
\begin{equation}
    \EstHel \leq{} \RegLog + 2\log(\delta^{-1}).
\end{equation}
\end{proposition}
For a proof, refer to \pref{lem:logloss_hellinger_ol} in
\pref{app:technical}, which gives a more general version of this
result. Further examples of estimation oracles are given throughout \pref{sec:examples,sec:bandit,sec:rl}.

\subsection{Dual Perspective and Connection to Posterior Sampling}
\label{sec:dual}

The \CompText \pref{eq:comp} is a min-max optimization problem, and
can be interpreted as a game in which the learner (the ``min'' player) aims to find a
decision distribution $p$ that optimally trades off regret and information
acquisition in the face of an adversary (the ``max'' player) that
selects a worst-case model in $\cM$. We can define a natural
\emph{dual} (or, max-min) analogue of the \CompShort via
\begin{equation}
  \label{eq:comp_dual}
  \compb(\cM,\Mbar) =
  \sup_{\mu\in\Delta(\cM)}\inf_{p\in\Delta(\Act)}\En_{M\sim\mu}\En_{\act\sim{}p}\biggl[\fm(\pim)-\fm(\pi)
    -\gamma\cdot\Dhels{M(\act)}{\Mbar(\act)}
    \biggr].
  \end{equation}
The dual \CompText has the following Bayesian interpretation. The
adversary selects a \emph{prior} distribution $\mu$ over models in
$\cM$, and the learner (with knowledge of the prior) finds a
decision distribution $p$ that balances the tradeoff between regret
and information acquisition when the underlying model is drawn from
$\mu$.

The connection between the primal and dual \CompShort is
analogous to the connection between primal and dual regret
(cf. \pref{sec:minimax_swap}). As with regret, the primal and dual \CompShort can be
shown to coincide under mild regularity conditions.\footnote{As with
  \pref{prop:minimax_swap}, we expect that this result can be proven
  to hold under weaker conditions.}

  \begin{restatable}{proposition}{minimaxswapdec}
  \label{prop:minimax_swap_dec}
  Suppose that $\Act$ is finite and $\Rspace$ is bounded. Then for all
  models $\Mbar$,
  \begin{equation}
    \label{eq:minimax_swap_dec}
    \comp(\cM,\Mbar) = \compb(\cM,\Mbar).
  \end{equation}
  \end{restatable}
As a consequence, any bound on the dual \CompShort immediately yields
a bound on the primal \CompShort. This perspective is useful because
it allows us to bring existing tools for Bayesian bandits and
reinforcement learning to bear on the primal \CompText. For example, \emph{probability matching} is a well-known Bayesian
strategy that---when applied to our setting---uses the distribution
$p$ induced by sampling $M\sim\mu$ and selecting
$\act_{\sss{M}}$. Using analysis from \cite{russo2014learning}, one
can show that this strategy certifies
\[
\compb(\cM,\Mbar) \leq{} \frac{\abs{\Act}}{4\gamma}
\]
for the finite-armed bandit; see \pref{sec:examples} for a
proof. Using more sophisticated analysis techniques, we use this
approach to bound the \CompText for a general class of structured
bandit problems with bounded \emph{star number} in \pref{sec:bandit},
and provide bounds for reinforcement learning in \pref{sec:rl}. In fact, many prior results for the Bayesian setting can be viewed as
implicitly providing bounds on the dual \CompText
\citep{russo2014learning,bubeck2015bandit,bubeck2016multi,russo2018learning,lattimore2019information,lattimore2020improved};
cf. \pref{sec:related}.

\begin{algorithm}[t]
    \setstretch{1.3}
     \begin{algorithmic}[1]
       \State \textbf{parameters}:
       \Statex[1] Prior $\mu\in\Delta(\cM)$.
       \Statex[1] Exploration parameter $\gamma>0$.
  \For{$t=1, 2, \cdots, T$}
  \State Define
  $\Mbar\ind{t}(\act)=\En\brk*{\Mstar(\act)\mid\hist\ind{t-1}}$ and
  $\Mbar_{\act'}\ind{t}(\act)=\En\brk*{\Mstar(\act)\mid\pistar=\act',
    \hist\ind{t-1}}$ for all $\act'\in\Act$.
  \State Compute coarsened posterior $\mu\ind{t}\in\Delta(\cM)$ via $\mu\ind{t}(\crl{\Mbar\ind{t}_{\act}})=\bbP(\pistar=\act\mid{}\hist\ind{t-1})$.
  \State Let
$p\ind{t}=\argmin_{p\in\Delta(\Act)}\En_{M\sim\mu\ind{t}}\En_{\act\sim{}p}\brk*{\fm(\pim)-\fm(\pi)
    -\gamma\cdot\Dhels{M(\act)}{\Mbar\ind{t}(\act)}}$.
\State{}Sample decision $\act\ind{t}\sim{}p\ind{t}$ and update $\hist\ind{t}=\hist\ind{t-1}\cup\crl{(\act\ind{t},r\ind{t},\obs\ind{t})}$.
\EndFor
\end{algorithmic}
\caption{\mainalgB}
\label{alg:bayes_basic}
\end{algorithm}

\paragraph{Bayesian \mainalg as a generalization of posterior
  sampling}
Beyond the primal and dual
\CompText, there are deeper connections between our approach and
Bayesian approaches. Consider the Bayesian version of the \FrameworkShort
framework from \pref{sec:minimax_swap}, in which the underlying
model $\Mstar$ is drawn from a known prior $\mu\in\Delta(\cM)$, and
our objective is to minimize the Bayesian regret
\[
  \En_{\Mstar\sim\mu}\En\sups{\Mstar}\brk*{\RegDM}.
\]
For this setting, the celebrated \emph{posterior sampling} (or, Thompson
sampling) algorithm
\citep{thompson1933likelihood,agrawal2012analysis,russo2014learning}
applies the probability matching idea as follows. At each round $t$,
we first compute the conditional law for $\Mstar$ given the observed
history $\hist\ind{t-1}$, which we denote by $\mu\ind{t} =
\mu\ind{t}(\cdot\mid{}\hist\ind{t-1})$. We then sample
$M\sim{}\mu\ind{t}$ and play $\act\ind{t}=\pim$. This algorithm leads
to Bayesian regret bounds for basic problem settings such as finite-armed bandits, linear bandits, and tabular
reinforcement learning \citep{osband2013more,russo2014learning,osband2017posterior}. 

We provide a Bayesian analogue of the \mainalg algorithm (\pref{alg:bayes_basic}), which
may be viewed as a generalization of posterior sampling. At each round, the algorithm computes an induced model
$\Mbar\ind{t}(\act)=\En\brk*{\Mstar(\act)\mid\hist\ind{t-1}}$, and a
distribution $\mu\ind{t}\in\Delta(\cM)$ which may be viewed as a
coarsened version of the posterior distribution, then
solves the optimization problem
\begin{equation}
\argmin_{p\in\Delta(\Act)}\En_{M\sim\mu\ind{t}}\En_{\act\sim{}p}\brk*{\fm(\pim)-\fm(\pi)
  -\gamma\cdot\Dhels{M(\act)}{\Mbar\ind{t}(\act)}}.
\label{eq:bayes_argmin}
\end{equation}
The algorithm then samples the decision $\act\ind{t}$ from the resulting
distribution. It is straightforward to show that this algorithm has
the following guarantee, which specializes
\pref{thm:upper_main_bayes}.\footnote{The proof of
  \pref{thm:upper_main_bayes} (see
  Eq. \pref{eq:upper_main_bayes_quant}) extends this result to
  infinite classes, and is proven in \pref{app:upper_main_bayes}.}
\begin{restatable}{theorem}{bayesbasic}
  \label{prop:bayes_basic}
  For parameter $\gamma>0$, \pref{alg:bayes_basic} guarantees that
  \begin{equation}
    \label{eq:bayes_basic}
    \En_{\Mstar\sim\mu}\En\sups{\Mstar}\brk*{\RegDM}
    \leq{} \sup_{\Mbar\in\conv(\cM)}\compb(\conv(\cM),\Mbar)\cdot{}T +\gamma\cdot\log\abs{\Pim}.
  \end{equation}
\end{restatable}
This algorithm generalizes posterior sampling by replacing the naive
probability matching strategy with the optimization problem in
\pref{eq:bayes_argmin}. It is also closely related to the
\emph{information-directed sampling} algorithm of
\cite{russo2018learning}. However, \pref{alg:bayes_basic} can be applied in settings
where information directed-sampling fails. See \cref{sec:related} for
further discussion of connections to posterior sampling.

\subsection{General Divergences and Randomized Estimators}
\label{sec:general_distance}

\begin{algorithm}[t]
    \setstretch{1.3}
     \begin{algorithmic}[1]
       \State \textbf{parameters}:
       \Statex[1] Online estimation oracle $\AlgEst$.
       \Statex[1] Exploration parameter $\gamma>0$.
       \Statex[1] Divergence $\Dgen{\cdot}{\cdot}$.
  \For{$t=1, 2, \cdots, T$}
  \State Compute randomized estimate $\nu\ind{t} = \AlgEst\ind{t}\prn[\Big]{ \crl*{(\act\ind{i},
r\ind{i},\obs\ind{i})}_{i=1}^{t-1} }$.
\State 
$p\ind{t}\gets\argmin_{p\in\Delta(\Act)}\sup_{M\in\cM}\En_{\act\sim{}p}\brk*{\fm(\pim)-\fm(\pi)
    -\gamma\cdot\En_{\Mhat\sim\nu\ind{t}}\Dgen{M(\act)}{\Mhat(\act)}}$.
\algcommentlight{Eq. \pref{eq:comp_general_randomized}}.%
\State{}Sample decision $\act\ind{t}\sim{}p\ind{t}$ and update estimation
oracle with $(\act\ind{t},r\ind{t}, \obs\ind{t})$.
\EndFor
\end{algorithmic}
\caption{\mainalg for General Divergences and Randomized Estimators}
\label{alg:main_generalized}
\end{algorithm}

In this section we give a generalization of the \mainalg algorithm
that incorporates two extra features: \emph{general divergences} and
\emph{randomized estimators}.

\paragraph{General divergences}
The \CompText measures estimation error via the
Hellinger distance $\Dhels{M(\act)}{\Mbar(\act)}$. By providing a
characterization for learnability (\pref{sec:learnability}), we show
that the choice of Hellinger distance here is
fundamental. Nonetheless, for specific applications and model classes,
it can be useful to work with alternative distance measures and divergences. For a
general divergence $D:\conv(\cM)\times\conv(\cM)\to\bbR_{+}$,
we define
\begin{equation}
  \label{eq:comp_general}
  \compgen(\cM,\Mbar) =
    \inf_{p\in\Delta(\Act)}\sup_{M\in\cM}\En_{\act\sim{}p}\biggl[\fm(\pim)-\fm(\pi)
    -\gamma\cdot\Dgen{M(\act)}{\Mbar(\act)}
    \biggr].
  \end{equation}
  This variant of the \CompShort naturally leads to regret bounds in

\paragraph{Randomized estimators}
The basic version of \mainalg (\pref{alg:main}) assumes that at each
round, the estimation algorithm $\AlgEst$ provides a point estimate
$\Mhat\ind{t}$. In some settings, it useful to consider
\emph{randomized estimators} that, at each round, produce a
distribution $\nu\ind{t}$ over models. For this setting, we further
generalize the \CompShort by defining
\begin{equation}
  \label{eq:comp_general_randomized}
  \compgen(\cM,\nu) =
    \inf_{p\in\Delta(\Act)}\sup_{M\in\cM}\En_{\act\sim{}p}\biggl[\fm(\pim)-\fm(\pi)
    -\gamma\cdot\En_{\Mbar\sim\nu}\brk*{\Dgen{M(\act)}{\Mbar(\act)}}
    \biggr]
  \end{equation}
  for distributions $\nu\in\Delta(\cM)$. 
\paragraph{Algorithm}  
  A generalization of \mainalg that incorporates general divergences and
  randomized estimators is given in \pref{alg:main_generalized}. The
  algorithm is identical to \mainalg with \optionone, with the only
  differences being that i) we play the distribution that solves the minimax
  problem \pref{eq:comp_general_randomized} with the user-specified divergence
  $\Dgen{\cdot}{\cdot}$ rather than squared Hellinger distance, and ii)
  we use the randomized estimate $\nu\ind{t}$ rather than a point estimate. Our
  performance guarantee for this algorithm depends on the estimation
  performance of the oracle's randomized estimates $\nu\ind{1},\ldots,\nu\ind{T}$ with respect to the given
  divergence $D$, which we define as
    \begin{equation}
    \label{eq:general_error}
        \EstD \ldef{} \sum_{t=1}^{T}\En_{\act\ind{t}\sim{}p\ind{t}}\En_{\Mhat\ind{t}\sim\nu\ind{t}}\brk*{\Dgen{\Mstar(\act\ind{t})}{\Mhat\ind{t}(\act\ind{t})}}.
  \end{equation}
Let $\cMhat$ be any set
for which $\Mhat\ind{t}\in\cMhat$ for all $t$ almost surely.
We have the following guarantee.
\begin{restatable}{theorem}{uppergeneraldistance}
  \label{thm:upper_general_distance}
\pref{alg:main_generalized} with exploration parameter $\gamma>0$ guarantees
that
\begin{equation}
  \label{eq:upper_general_distance}
\RegDM \leq{}
\sup_{\nu\in\Delta(\cMhat)}\compgen(\cM,\nu)\cdot{}T + \gamma\cdot\EstD
\end{equation}
almost surely.
\end{restatable}

\paragraph{On the use of general divergences}
In bandit problems, it is often convenient to work with the
divergence
$\Dsq{M(\act)}{\Mbar(\act)}\ldef{}(\fm(\act)-\fmbar(\act))^2$, which
uses the mean reward function as a sufficient statistic. With this
choice of distance, the minimax problem in \pref{eq:comp_general}
recovers the action selection strategy used in the \squarecb
contextual bandit algorithm \citep{abe1999associative,foster2020beyond}. This distance
also recovers a generalization of \squarecb for infinite, linearly-structured action spaces
given in \citet{foster2020adapting}. Note that whenever $\cFm$ has range in
$\brk*{0,1}$, one has
$\compH(\cM,\Mbar)\leq{}\compSq(\cM,\Mbar)$,\footnote{Throughout the
  paper we abbreviate $\compH=\compgen[\Dhelshort]$,
  $\compKL=\compgen[\Dklshort]$, $\compSq=\compgen[\Dsqshort]$, and so on.}
and in general this inequality is strict, so working with this
distance is not always sufficient. The advantage, however, is that for
$\Dsqshort$ the estimation error in \pref{eq:general_error} can be
minimized directly using square loss estimation rather than density
estimation, which can lead to simpler algorithms and tighter bounds. For example, one can generically
obtain $\EstSq\leq{}\log\abs{\cFm}$ for finite classes, which can be
much tighter than the analogous bound $\EstHel\leq{}\log\abs{\cM}$ for
Hellinger error. See \cite{foster2020beyond} for
more background on square loss estimation oracles.

Another advantage of working with general distances is to simplify
analysis of the minimax value. For example, some of our lower
bounds on the \CompShort proceed by moving to KL divergence and lower
bounding
$\compH(\cM,\Mbar)\geq\compKL(\cM,\Mbar)$.

We mention without proof that it is also possible to extend
\optiontwo of \mainalg to accommodate general divergences. In this case, we
require the additional assumption that the divergence $D$ is bounded, symmetric, and
satisfies the triangle inequality up to a multiplicative constant.

\paragraph{On the use of randomization}
Subsequent work of \citet{foster2022note} shows that for Hellinger
distance and other divergences satisfying mild technical assumptions,
the randomized variant of the \CompShort in
\cref{eq:comp_general_randomized} is equal to
$\sup_{\Mbar\in\conv(\cM)}\compgen(\cM,\Mbar)$ up to
constants. We restate the result here.
\begin{lemma}[\citet{foster2022note}]
  \label{lem:randomization_doesnt_help}
  For all reference models $\Mbar$ (not necessarily in $\cM$), we have
  that for all $\gamma>0$,
  \begin{align}
    \comp(\cM,\Mbar)
    \leq{} \sup_{\nu\in\Delta(\cM)}\comp[\gamma/4]^{\mathrm{H}}(\cM,\nu).
  \end{align}
\end{lemma}
While working with randomization offers no improvement in
statistically, in some cases a distribution $p\in\Delta(\Pi)$
that minimizes $\compgen(\cM,\nu)$ can be simpler to compute than a
distribution that minimizes $\compgen(\cM,\Mbar)$ for
$\Mbar\in\conv(\cM)$. We refer to the follow-up work of
\citet{chen2022unified} for concrete examples of randomized estimators.

\section{Illustrative Examples}
\label{sec:examples}

In this section, we show how to bound the \CompText and efficiently instantiate
the \mainalg meta-algorithm for two canonical problem settings:
multi-armed bandits and tabular (finite
state/action) reinforcement learning. Through the duality introduced
in \pref{sec:dual}, we showcase the use of both frequentist and
Bayesian approaches to bound the DEC. The proofs in this section illustrate key concepts and technical tools that prove useful when
we consider richer, more structured settings in \pref{sec:bandit} and \pref{sec:rl}.

\subsection{Multi-Armed Bandits}
For the first example, we provide upper and lower bounds on the \CompText for the classical multi-armed bandit setting. Here, we
have $\Act=\brk*{A}$, $\Rspace=\brk*{0,1}$, and $\Obs=\NullObs$, and
the model class $\cM=\crl*{M:M(\act)\in\Delta(\Rspace)}$ consists of
all possible distributions over $\Rspace$.

To derive upper bounds, it is simpler to work with the square loss variant of the \CompShort
from \pref{sec:general_distance},
\[
  \compSq(\cM,\Mbar) =
  \inf_{p\in\Delta(\Pi)}\sup_{M\in\cM}\En_{\act\sim{}p}\brk*{\fm(\pim)-\fm(\pi)
  -\gamma\cdot(\fm(\act)-\fmbar(\act))^2},
\]
which has $\comp(\cM,\Mbar)\leq\compSq(\cM,\Mbar)$ whenever rewards
lie in $\brk*{0,1}$. Intuitively, the reason why working with this
coarse notion of distance suffices is that---beyond the mean reward
for each decision---the
noise distribution provides little information about the underlying
problem instance for this unstructured setting.

\subsubsection{Bayesian Upper Bound via Posterior Sampling}
We first show how to bound the \CompText using a simple, yet
indirect approach that leverages minimax duality. Recall from
\pref{sec:dual} that by the minimax theorem
(\pref{prop:minimax_swap_dec}), we have
\begin{equation}
  \label{eq:minimax_swap_tabular}
  \compSq(\cM,\Mbar)
  = \compSqdual(\cM,\Mbar)
  =   \sup_{\mu\in\Delta(\cM)}\inf_{p\in\Delta(\Act)}\En_{M\sim\mu}\En_{\act\sim{}p}\brk*{\fm(\pim)-\fm(\pi)
    -\gamma\cdot(\fm(\act)-\fmbar(\act))^2}.
\end{equation}
Hence, to bound the \CompShort, it suffices to show that for any
prior $\mu\in\Delta(\cM)$, we can choose the distribution
$p\in\Delta(\Act)$ such that the quantity in
\pref{eq:minimax_swap_tabular} is bounded. We have the following
result.
\begin{proposition}
  \label{prop:posterior_mab}
  Consider the multi-armed bandit setting with $\cR=\bbR$.  For any $\mu\in\Delta(\cM)$, the distribution $p(\act)=\mu(\crl{M :
    \pim=\act})$ satisfies
  \[
\En_{M\sim\mu}\En_{\act\sim{}p}\brk*{\fm(\pim)-\fm(\pi)
  -\gamma\cdot(\fm(\act)-\fmbar(\act))^2}
\leq{} \frac{A}{\gamma},
\]
for all $\Mbar$ and $\gamma>0$. Consequently, $\compSqdual(\cM,\Mbar)\leq\frac{A}{\gamma}$.
\end{proposition}
\begin{proof}[\pfref{prop:posterior_mab}]
  Let $\mu$ be given and $\Mbar$ and $\gamma>0$ be fixed. Observe that by the AM-GM inequality
  \[
    \En_{M\sim\mu}\En_{\act\sim{}p}\brk*{\fm(\pim)-\fm(\pi)}
      \leq{}
      \En_{M\sim\mu}\En_{\act\sim{}p}\brk*{\fm(\pim)-\fmbar(\pi)}
        +
        \frac{\gamma}{2}\En_{M\sim\mu}\En_{\act\sim{}p}\brk*{(\fm(\pi)-\fmbar(\pi))^2}
        + (2\gamma)^{-1}.
      \]
      Hence, it suffices to bound the first term above. Since
      $\act\sim{}p$ and $\pim$ under $M\sim\mu$ are identically
      distributed, we have
      \begin{align*}
        \En_{M\sim\mu}\En_{\act\sim{}p}\brk*{\fm(\pim)-\fmbar(\pi)}
        = \En_{M\sim\mu}\brk*{\fm(\pim)-\fmbar(\pim)}.
      \end{align*}
      Next, we apply Cauchy-Schwarz to bound by
      \begin{align*}
        \En_{M\sim\mu}\brk*{\fm(\pim)-\fmbar(\pim)}
        &=
        \En_{M\sim\mu}\brk*{\frac{p^{1/2}(\pim)}{p^{1/2}(\pim)}(\fm(\pim)-\fmbar(\pim))}\\
        &\leq{} \prn*{\En_{M\sim\mu}\brk*{\frac{1}{p(\pim)}}}^{1/2}\cdot\prn*{\En_{M\sim\mu}\brk*{p(\pim)(\fm(\pim)-\fmbar(\pim))^{2}}}^{1/2}.
      \end{align*}
      For the first term, we have
      \[
        \En_{M\sim\mu}\brk*{\frac{1}{p(\pim)}}
        = \sum_{\act\in\brk*{A} : p(\act)>0}\frac{\mu(\crl{M :
              \pim=\act})}{p(\act)}
        \leq A,
      \]
      while the second term has
      \begin{align*}
        \En_{M\sim\mu}\brk*{p(\pim)(\fm(\pim)-\fmbar(\pim))^{2}}
        \leq{}
          \sum_{\act\in\brk*{A}}p(\act)\En_{M\sim\mu}\brk*{(\fm(\act)-\fmbar(\act))^{2}}
        = \En_{M\sim\mu}\En_{\act\sim{}p}\brk*{(\fm(\act)-\fmbar(\act))^{2}}.
      \end{align*}
      Hence, by the AM-GM inequality,
      \begin{align*}
        \En_{M\sim\mu}\brk*{\fm(\pim)-\fmbar(\pim)}
        &\leq{}
        \prn*{A\cdot\En_{M\sim\mu}\En_{\act\sim{}p}\brk*{(\fm(\act)-\fmbar(\act))^{2}}}^{1/2}\\
          &\leq{}
          \frac{\gamma}{2}\En_{M\sim\mu}\En_{\act\sim{}p}\brk*{(\fm(\act)-\fmbar(\act))^{2}}
         + \frac{A}{2\gamma}.
      \end{align*}
\end{proof}~\\
The key idea above, which originates from
\cite{russo2014learning}, is the use of Cauchy-Schwarz to ``decouple''
the random variables $M\sim\mu$ and $\pim$; this idea is generalized in \pref{sec:bandits_star}.

\paragraph{Regret upper bound}
By plugging the bound on the \CompText from \pref{prop:posterior_mab}
into \pref{thm:upper_main_bayes} (and noting that $\conv(\cM)=\cM$), we conclude existence of an
algorithm with
\begin{equation}
  \En\brk*{\RegDM}
  \leq{}   2\cdot\min_{\gamma>0}\max\crl*{\frac{AT}{\gamma},
      \gamma\log{}A} = 4\sqrt{AT\log{}A}.\label{eq:posterior_mab}
  \end{equation}
  This matches the optimal rate for the multi-armed bandit
  problem up to the $\log{}A$ factor
  \citep{audibert2009minimax}. Furthermore, via
  \pref{alg:bayes_basic}, this approach provides an explicit algorithm
  for the Bayesian setting.
  \subsubsection{Frequentist Upper Bound via Inverse Gap Weighting}
  The Bayesian approach in the prequel does not lead to an explicit strategy
  that achieves the value of the frequentist \CompText. We now give
  an explicit approach based on the \emph{inverse gap weighting}
  technique \citep{abe1999associative,foster2020beyond} recently
  popularized by \cite{foster2020beyond}.\looseness=-1
  \begin{proposition}
    \label{prop:igw_mab}
      Consider the multi-armed bandit setting with $\cR=\bbR$. For any $\Mbar$ and $\gamma>0$, define
    \begin{equation}
      \label{eq:igw}
      p(\act) = \frac{1}{\lambda + 2\gamma(\fmbar(\pimbar)-\fmbar(\act))},
    \end{equation}
where $\lambda\in\brk{1,A}$ is chosen such that
$\sum_{\act}p(\act)=1$.\footnote{The normalizing constant
  $\lambda\in\brk*{1,A}$ is always guaranteed to exist because we have
  $\frac{1}{\lambda}\leq{}\sum_{\act}p(\act)\leq\frac{A}{\lambda}$,
  and because $\lambda\mapsto\sum_{\act}p(\act)$ is continuous over $\brk{1,A}$.} This strategy guarantees
  \[
\sup_{M\in\cM}\En_{\act\sim{}p}\brk*{\fm(\pim)-\fm(\pi)
  -\gamma\cdot(\fm(\act)-\fmbar(\act))^2}
\leq{} \frac{A}{\gamma},
\]
and consequently certifies that $\compSq(\cM,\Mbar)\leq\frac{A}{\gamma}$.
\end{proposition}
Conceptually, this result shows that \igwtext can be thought of as a
frequentist counterpart to posterior sampling.
  \begin{proof}[\pfref{prop:igw_mab}]
    Let $M\in\cM$ be fixed. As with the proof of
    \pref{prop:posterior_mab}, we first apply the AM-GM inequality to bound
  \[
    \En_{\act\sim{}p}\brk*{\fm(\pim)-\fm(\pi)}
      \leq{}
      \En_{\act\sim{}p}\brk*{\fm(\pim)-\fmbar(\pi)}
        +
        \frac{\gamma}{2}\En_{\act\sim{}p}\brk*{(\fm(\pi)-\fmbar(\pi))^2}
        + (2\gamma)^{-1}.
      \]
      Next, we write
      \[
        \En_{\act\sim{}p}\brk*{\fm(\pim)-\fmbar(\pi)}
        = \fm(\pim)-\fmbar(\pimbar)
        + \En_{\act\sim{}p}\brk*{\fmbar(\pimbar)-\fmbar(\act)},
      \]
      and observe that
      \[
        \En_{\act\sim{}p}\brk*{\fmbar(\pimbar)-\fmbar(\act)}
        =
        \sum_{\act\neq\pimbar}\frac{\fmbar(\pimbar)-\fmbar(\act)}{\lambda
          + 2\gamma(\fmbar(\pimbar)-\fmbar(\act))}
        \leq \frac{A-1}{2\gamma}.
      \]
      To bound the final term, we apply the AM-GM inequality:
      \begin{align*}
        \fm(\pim)-\fmbar(\pimbar)
        &= \fm(\pim)-\fmbar(\pim) - (\fmbar(\pimbar)-\fmbar(\pim))\\
        &\leq{} \frac{\gamma}{2}p(\pim)(\fm(\pim)-\fmbar(\pim))^2  +
          \frac{1}{2\gamma{}p(\pim)}- (\fmbar(\pimbar)-\fmbar(\pim))\\
        &\leq{} \frac{\gamma}{2}\En_{\act\sim{}p}\brk*{(\fm(\act)-\fmbar(\act))^2}  + \frac{1}{2\gamma{}p(\pim)}- (\fmbar(\pimbar)-\fmbar(\pim)).
      \end{align*}
      We then use the definition of $p$ to bound
      \begin{align*}
        \frac{1}{2\gamma{}p(\pim)}- (\fmbar(\pimbar)-\fmbar(\pim))
        = \frac{\lambda +
        2\gamma(\fmbar(\pimbar)-\fmbar(\pim))}{2\gamma} -
        (\fmbar(\pimbar)-\fmbar(\pim))
        = \frac{\lambda}{2\gamma}
        \leq \frac{A}{2\gamma}.
      \end{align*}
      Combining all of these inequalities, we conclude that
      \begin{align*}
        \En_{\act\sim{}p}\brk*{\fm(\pim)-\fm(\pi)}
        \leq{} \gamma
        \En_{\act\sim{}p}\brk*{(\fm(\act)-\fmbar(\act))^2} + \frac{A}{\gamma}.
      \end{align*}
  \end{proof}

\paragraph{Regret upper bound}  
Applying \pref{thm:upper_general_distance} with the bound
on the \CompShort above, we conclude that \mainalg with \igwtext ensures that for any
estimation oracle,
\[
\RegDM \leq{} \frac{AT}{\gamma} + \gamma\cdot\EstSq,
\]
where
\[
\EstSq = \sum_{t=1}^{T}\En_{\act\ind{t}\sim{}p\ind{t}}\brk*{(\fmhatt(\act\ind{t})-\fmstar(\act\ind{t}))^2}.
\]
As a concrete example, following \cite{foster2020beyond}, we can
estimate the rewards using the
Vovk-Azoury-Warmuth algorithm
\citep{azoury2001relative,vovk1998competitive}, which has
$\En\brk{\EstSq}\leq{}A\log(T)$. The resulting algorithm has
$\En\brk*{\RegDM}\leq{}A\sqrt{T\log(T)}$, and runs
in time $\bigoh(A)$ per round. We mention in passing that with a more
careful analysis, this approach can recover
the $\sqrt{AT\log{}A}$ rate \pref{eq:posterior_mab} derived through
the Bayesian approach, but we do not pursue this here.
\subsubsection{Lower Bound}

We conclude this example by proving a lower bound on the \CompText which matches the upper bounds derived above. We use
this example to illustrate a general strategy to lower bound the
\CompShort, which is used extensively throughout
\pref{sec:bandit,sec:rl}. The lower bound strategy is based on the following
notion of a \emph{hard family of models}.
\begin{definition}[\hardfamily]
  A reference model $\Mbar\in\cM$ and collection
  $\crl*{M_1,\ldots,M_N}$ with $N\geq{}2$ are said to be an
  \hardfamily if the following properties hold.
  \begin{enumerate}
  \item \emph{Regret property.} There exist functions
    $\hardu_i:\Act\to\brk{0,1}$, with
    $\sum_{i}\hardu_i(\act)\leq{}\frac{N}{2}\;\forall\act$, such that
    \[
      \fmi(\pimi) - \fmi(\act) \geq{}\alpha\cdot(1-\hardu_i(\act)).
    \]
  \item \emph{Information property.} There exist functions $\hardv_i:\Act\to\brk{0,1}$, with
    $\sum_{i}\hardv_i(\act)\leq{}1\;\forall\act$, such that
        \[
          \Dhels{M_i(\act)}{\Mbar(\act)}\leq{}\beta\cdot\hardv_i(\act)+\delta.
    \]
  \end{enumerate}
\end{definition}
Informally, any \hardfamily leads to a difficult decision making
problem when $N$ is large because a given action can have low regret or
large information gain on at most one model in the family. The
following lemma makes this idea precise.
\begin{lemma}
  \label{lem:hard_family}
  Let $\cM'=\crl*{M_1,\ldots,M_N}\subseteq\cM$ be an \hardfamily with
  respect to $\Mbar$. Then for all $\gamma>0$
  \begin{equation}
    \label{eq:hard_family}
    \comp(\cM',\Mbar) \geq \compb(\cM',\Mbar)
    \geq \frac{\alpha}{2} - \gamma\prn*{\frac{\beta}{N}+\delta}.
  \end{equation}
\end{lemma}
\begin{proof}
Choose $\mu=\unif(\crl{M_1,\ldots,M_N})$ as the uniform distribution over models in
the family. Then, for any distribution $p\in\Delta(\Act)$, we
have
\begin{align*}
  \compb(\cM',\Mbar)
  &\geq{} \En_{i\sim\mu}\En_{\act\sim{}p}\brk*{
  \fmi(\pimi) - \fmi(\act) -\gamma\cdot\Dhels{M_i(\act)}{\Mbar(\act)}
    }\\
  &\geq{} \En_{i\sim\mu}\En_{\act\sim{}p}\brk*{
    \alpha(1-\hardu_i(\act))-\gamma(\beta\hardv_i(\act)+\delta)
    }\\
  &= \En_{\act\sim{}p}\En_{i\sim\mu}\brk*{
    \alpha(1-\hardu_i(\act))-\gamma(\beta\hardv_i(\act)+\delta)
    }.    
\end{align*}
Observe that for any fixed decision $\act\in\Act$,
$\En_{i\sim\mu}\brk{u_i(\act)}=\frac{1}{N}\sum_{i=1}^{N}u_i(\act)\leq1/2$,
and likewise $\En_{i\sim\mu}\brk{v_i(\act)}\leq{}1/N$. Consequently,
\begin{align*}
  \compb(\cM,\Mbar)
  \geq{}
  \frac{\alpha}{2}-\gamma\prn*{\frac{\beta}{N}+\delta}.
\end{align*}

\end{proof}

Equipped with this lemma, we prove the following lower bound.
\begin{proposition}
  \label{prop:mab_lower}
  For the multi-armed bandit problem with $A\geq{}2$ actions, $\cR=\brk{0,1}$, and
  $\Mbar(\act)=\Ber(\nicefrac{1}{2})$,
  \[
    \comp(\cMinf[\vepsg](\Mbar),\Mbar)\geq{} 2^{-6}\cdot{}\frac{A}{\gamma}
  \]
  for all $\gamma\geq{}\frac{A}{3}$, where $\vepsg=\frac{A}{12\gamma}$.
\end{proposition}

\begin{proof}[\pfref{prop:mab_lower}]
  We construct a family of models $\cM'=\crl*{M_1,\ldots,M_A}$ by defining
  $M_i(\act)=\Ber(\nicefrac{1}{2}+\Delta\indic\crl{\act=i})$ for
  $\Delta\in(0,\nicefrac{1}{2})$, and setting $\hardu_i(\act)=\hardv_i(\act)=\indic\crl{i=\act}$. We have
  \[
    \fmi(\pimi)-\fmi(\act)\geq\Delta(1-\indic\crl{\act=i}),\mathand
    \Dhels{M_i(\act)}{\Mbar(\act)}\leq{}3\Delta^2\indic\crl*{\act=i},
  \]
  where the second inequality is
  \pref{lem:divergence_bernoulli}. Since
  $\sum_{\act}u_i(\act)=\sum_{\act}v_i(\act)=1$, this establishes
  that $\cM'$ is a
  $(\Delta,3\Delta^2,0)$-family, so by \pref{lem:hard_family} we have that for any $\gamma>0$,
  \[
    \comp(\cM',\Mbar)
    \geq{} \frac{\Delta}{2} - \gamma\frac{3\Delta^{2}}{A}.
  \]
  By choosing $\Delta=\frac{A}{12\gamma}$, this gives
    \[
    \comp(\cM',\Mbar)
    \geq{} (48)^{-1}\frac{A}{\gamma}.
  \]
  Furthermore,
  it is clear that $\cM'\subseteq\cMinf[\Delta](\Mbar)$, and that
  we may take
  $\abscontp=\bigoh(1)$ in \pref{thm:lower_main,thm:lower_main_expectation}
  whenever $\Delta\leq1/4$; it suffices to restrict to $\gamma\geq{}A/3$.
\end{proof}
\paragraph{Regret lower bound}
Applying \pref{thm:lower_main_expectation} with the lower bound on the
\CompShort from \pref{prop:mab_lower}, we are guaranteed that\footnote{Recall that the sub-family of models constructed in
  \pref{prop:mab_lower}
satisfies
  $\abscontp=\bigoh(1)$ for \pref{thm:lower_main_expectation}.}
\[
  \En\brk*{\RegDM}
  \geq{} c'\cdot{}\frac{AT}{\gamma},
\]
so long as $\frac{A}{12\gamma}\leq\vepslowg=c''\frac{\gamma}{T}$,
where $c,c',c''>0$ are numerical constants. If we
choose $\gamma=C\cdot\sqrt{AT}$ for $C>0$ sufficiently large, we have
that
\[
\En\brk*{\RegDM} \geq{} \Omega(\sqrt{AT}),
\]
which matches the minimax rate for the problem
\citep{audibert2009minimax}.

Note that compared to specialized approaches tailored to the multi-armed bandit
setting, the constants in this lower bound are rather loose.
This is a consequence of the high generality of our framework.

\subsection{Tabular Reinforcement Learning}
\label{sec:tabular}

As our second example, we show how to bound the \CompText and develop efficient
algorithms for reinforcement learning in the
tabular (or, finite state/action) setting. Here, $\cM$ is the collection
of all non-stationary MDPs with state space $\cS=\brk{S}$, action space
$\cA=\brk{A}$, and horizon $H$, and $\Act=\PiRNS$ is the
collection of all randomized, non-stationary Markov policies (cf. \pref{ex:rl}). We assume that
rewards are normalized such that $\sum_{h=1}^{H}r_h\in\brk*{0,1}$ almost
surely, and take $\cR=\brk*{0,1}$ \citep{jiang2018open,zhang2021reinforcement}. Recall that for each
$M\in\cM$, $\crl{\Pm_h}_{h=1}^{H}$ and $\crl{\Rm_h}_{h=1}^{H}$
denote the associated transition kernels and reward distributions.

Our approach to bounding the \CompShort for this setting parallels the
approach to the multi-armed bandit. First, we provide an upper bound on
the dual \CompShort based on a new analysis of posterior sampling,
then we provide a new, efficient strategy to bound the
frequentist \CompShort which combines the idea of \igwtext with the
notion of a \emph{policy cover}. The latter result shows for the first
time how to adapt \igwtext to reinforcement learning, and answers a
question raised by \cite{foster2020beyond}.

\subsubsection{Bayesian Upper Bound via Posterior Sampling}
The following result shows that posterior sampling leads to a $\poly(S,A,H)$
bound on the dual Bayesian \CompText for tabular reinforcement learning.
\begin{proposition}
  \label{prop:posterior_tabular}
  Consider the tabular reinforcement learning setting with
  $\sum_{h=1}^{H}r_h\in\cR\ldef\brk{0,1}$. For any $\mu\in\Delta(\cM)$
  and reference model $\Mbar$ (not necessarily in $\cM$), the distribution $p(\act)=\mu(\crl{M :
    \pim=\act})$ satisfies
  \[
\En_{M\sim\mu}\En_{\act\sim{}p}\brk*{\fm(\pim)-\fm(\pi)
  -\gamma\cdot\Dhels{M(\act)}{\Mbar(\act)}}
\leq{} 26\frac{H^2SA}{\gamma},
\]
for all $\Mbar$ and $\gamma>0$. Consequently, $\compdual(\cM,\Mbar)\leq{}26\frac{H^2SA}{\gamma}$.
\end{proposition}
For simplicity, we prove the result only for $\Mbar\in\cM$ here. The more general result
for arbitrary $\Mbar$ is a special case of
\cref{thm:posterior_bilinear}. Before proving the result, we state a useful \emph{change of measure}
lemma which shows that to bound the \CompShort for any episodic
reinforcement learning setting, it suffices to consider state-action
distributions induced by the reference model $\Mbar$. This lemma is
used for many subsequent reinforcement learning results.
\begin{restatable}[Change of measure for reinforcement learning]{lemma}{changeofmeasure}
  \label{lem:change_of_measure}
  Consider any family of MDPs $\cM$ and reference MDP $\Mbar\in\cM$, all of
  which satisfy $\sum_{h=1}^{H}r_h\in\brk*{0,1}$. Suppose that $\mu\in\Delta(\cM)$
  and $p\in\Delta(\PiGen)$ are such that
    \begin{align}
      &\En_{M\sim\mu}\En_{\pi\sim{}p}\brk*{\fm(\pim) - \fmbar(\act)}\label{eq:com0}\\
        &\leq{} C_1 + C_2\cdot \En_{M\sim\mu}\En_{\pi\sim{}p}\Enm{\Mbar}{\pi}\brk*{\sum_{h=1}^{H}          \Dhels{\Pm(s_h,a_h)}{\Pmbar(s_h,a_h)}
    + \Dhels{\Rm(s_h,a_h)}{\Rmbar(s_h,a_h)}
          },\label{eq:com1}
    \end{align}
    for parameters $C_1,C_2>0$. Then for all $\eta>0$, it holds that
  \begin{equation}
    \label{eq:com2}
    \En_{M\sim\mu}\En_{\pi\sim{}p}\brk*{\fm(\pim) - \fm(\act)}
    \leq{} C_1 + (4\eta)^{-1} + (40HC_2+\eta)\cdot\En_{M\sim\mu}\En_{\pi\sim{}p}\brk*{
    \Dhels{M(\act)}{\Mbar(\act)}
    }.
  \end{equation}
\end{restatable}
See \pref{app:rl_technical} for the proof. Equipped with this result,
we proceed to prove \pref{prop:posterior_tabular}. As in the multi-armed
bandit setting (\pref{prop:posterior_mab}), the thrust of the argument is to
use Cauchy-Schwarz to decouple the random variables $\pim$ and $M$ under
$M\sim\mu$. Compared to the multi-armed bandit setting though, we
perform this decoupling by working in the space of \emph{occupancy
  measures} for the reference MDP $\Mbar$; this is facilitated by the
change of measure lemma. A second point of comparison is that
unlike the multi-armed bandit setting, where it was sufficient to work
with the square loss \CompText $\compSq$, here it is
critical to work with an information-theoretic divergence.
\begin{proof}[\pfref{prop:posterior_tabular}]
In light of \pref{lem:change_of_measure}, it suffices to bound the quantity
  \begin{align*}
    \En_{M\sim\mu}\En_{\pi\sim{}p}\brk*{\fm(\pim) - \fmbar(\act)}
  \end{align*}
  in terms of the quantity on the right-hand side of \pref{eq:com1}.
We start by writing
  \[
    \En_{M\sim\mu}\En_{\pi\sim{}p}\brk*{\fm(\pim) - \fmbar(\pi)}
    = \En_{M\sim\mu}\brk*{\fm(\pim)} -\En_{\pi\sim{}p}\brk*{\fmbar(\pi)}.
  \]
  By definition, $\pi\sim{}p$ is identical in law to $\pim$
  under $M\sim\mu$, so this is equal to
  \[
\En_{M\sim\mu}\brk*{\fm(\pim)}
-\En_{M\sim\mu}\brk*{ \fmbar(\pim)}
= \En_{M\sim\mu}\brk*{\fm(\pim) - \fmbar(\pim)}.
\]
Next, using a classical simulation lemma (\pref{lem:simulation_basic}), we have
\begin{align*}
  \En_{M\sim\mu}\brk*{\fm(\pim)- \fmbar(\pim)}
  &\leq{}
    \sum_{h=1}^{H}\En_{M\sim\mu}\Enm{\Mbar}{\pim}\brk*{\Dtv{\Pm_h(s_h,a_h)}{\Pmbar_h(s_h,a_h)}+
    \Dtv{\Rm_h(s_h,a_h)}{\Rmbar_h(s_h,a_h)}}\\
    &=
    \sum_{h=1}^{H}\En_{M\sim\mu}\Enm{\Mbar}{\pim}\brk*{\errm_h(s_h,a_h)},
\end{align*}
where we define $\errm_h(s,a) \ldef{} \Dtv{\Pm(s,a)}{\Pbar(s,a)}+\Dtv{\Rm(s,a)}{\Rmbar(s,a)}$.
We proceed to bound each term in the sum. %
Recalling that $\dmpi_h(s,a)$ denotes the marginal distribution
over $(s_h,a_h)$ for policy $\pi$ under model $M$, we have
\[
  \sum_{h=1}^{H}\En_{M\sim\mu}\Enm{\Mbar}{\pim}\brk*{\errm_h(s_h,a_h)}
  = 
  \sum_{h=1}^{H}\En_{M\sim\mu}\brk*{\sum_{s,a}\dm{\Mbar}{\pim}_h(s,a)\errm_h(s,a)
    }.
  \]
Define $\dbar_h(s,a) =
  \En_{M\sim\mu}\brk*{\dm{\Mbar}{\pim}(s,a)}$. Then by
  Cauchy-Schwarz, for each $h$, we have
  \begin{align*}
    &\En_{M\sim\mu}\brk*{\sum_{s,a}\dm{\Mbar}{\pim}_h(s,a)\errm_h(s,a)}\\
&=\En_{M\sim\mu}\brk*{\sum_{s,a}\dm{\Mbar}{\pim}_h(s,a)
\prn*{\frac{\dbar_h(s,a)}{\dbar_h(s,a)}}^{1/2}\errm_h(s,a)}\\
    &\leq{}\prn*{\En_{M\sim\mu}\brk*{\sum_{s,a}\frac{\dm{\Mbar}{\pim}_h(s,a)}{\dbar_h(s,a)}}}^{1/2}
      \cdot\prn*{\En_{M\sim\mu}\brk*{\sum_{s,a}\dbar_h(s,a)(\errm_h(s,a))^2}}^{1/2}.
  \end{align*}
  For the first term above, we observe that
  \begin{align*}
    \En_{M\sim\mu}\brk*{\sum_{s,a}\frac{\dm{\Mbar}{\pim}_h(s,a)}{\dbar_h(s,a)}}
    = \sum_{s,a}\frac{\En_{M\sim\mu}\brk[\big]{
    \dm{\Mbar}{\pim}_h(s,a)}}{\dbar_h(s,a)}
        = \sum_{s,a}\frac{\dbar_h(s,a)}{\dbar_h(s,a)} = SA.
  \end{align*}
  For the second term, we again use that $\pi\sim{}p$ and $\pim$ under
$M\sim\mu$ are identical in law to write
  \begin{align*}
    \En_{M\sim\mu}\brk*{\sum_{s,a}\dbar_h(s,a)(\errm_h(s,a))^2}
    &= \sum_{s,a}\dbar_h(s,a)
      \En_{M\sim\mu}\brk*{(\errm_h(s,a))^2}\\
    &=
      \sum_{s,a}\En_{M\sim\mu}\brk*{\dm{\Mbar}{\pim}_h(s,a)}\En_{M\sim\mu}\brk*{(\errm_h(s,a))^2}\\
        &=
      \sum_{s,a}\En_{\pi\sim{}p}\brk*{\dm{\Mbar}{\pi}_h(s,a)}\En_{M\sim\mu}\brk*{(\errm_h(s,a))^2}\\
        &=
          \En_{M\sim\mu}\En_{\pi\sim{}p}\brk*{\sum_{s,a}\dm{\Mbar}{\pi}_h(s,a)(\errm_h(s,a))^2}\\
    &=
      \En_{M\sim\mu}\En_{\pi\sim{}p}\Enm{\Mbar}{\pi}\brk*{(\errm_h(s_h,a_h))^2}.
  \end{align*}
  Combining these observations and applying the AM-GM inequality, we
  have that for any $\eta>0$,
  \begin{align*}
    \sum_{h=1}^{H}\En_{M\sim\mu}\Enm{\Mbar}{\pim}\brk*{\errm_h(s_h,a_h)}
    \leq{}  \frac{HSA}{2\eta}
    + \frac{\eta}{2}\En_{M\sim\mu}\En_{\pi\sim{}p}\Enm{\Mbar}{\pi}\brk*{\sum_{h=1}^{H}(\errm_h(s_h,a_h))^2}.
  \end{align*}
Using that $(x+y)^2\leq{}2(x^2+y^2)$, we conclude that
  \begin{align*}
    &\En_{M\sim\mu}\brk*{\fm(\pim) -
    \fmbar(\pim)}\\
    &\leq{} \frac{HSA}{2\eta} + \eta\cdot
    \En_{M\sim\mu}\En_{\pi\sim{}p}\Enm{\Mbar}{\pi}\brk*{\sum_{h=1}^{H}\Dtvs{\Pm_h(s_h,a_h)}{\Pmbar_h(s_h,a_h)}
    + \Dtvs{\Rm_h(s_h,a_h)}{\Rmbar_h(s_h,a_h)}}.
  \end{align*}
With this result in hand, we apply \pref{lem:change_of_measure}, which
implies that for any $\eta'>0$,
\begin{align*}
      \En_{M\sim\mu}\En_{\pi\sim{}p}\brk*{\fm(\pim) - \fm(\act)}
    \leq{} \frac{HSA}{2\eta} + (4\eta')^{-1} + (40H\eta+\eta')\cdot\En_{M\sim\mu}\En_{\pi\sim{}p}\brk*{
    \Dhels{M(\act)}{\Mbar(\act)}
    }.
\end{align*}
We conclude by setting $\eta=\eta'=\frac{\gamma}{41H}$.

\end{proof}

The bound on the \CompText from \pref{prop:posterior_tabular}, which
scales with $\frac{H^{2}SA}{\gamma}$, is optimal in terms of
dependence on $S$ and $A$, as we show below. The dependence on horizon
can be improved from $H^{2}$ to $H$ for MDPs with time-homogeneous
dynamics.

In \pref{sec:rl}, we generalize the decoupling idea used in this
proof to MDPs with low Bellman rank and, more generally,
any \emph{bilinear class}. Indeed, the only property of the tabular
RL setting that is essential to the proof above is that the Bellman residuals have
rank at most $SA$.

\subsubsection{Frequentist Upper Bound via Policy Cover Inverse Gap Weighting}
\newcommand{\PiCov}{\Psi}
\newcommand{\pihsa}{\pi_{h,s,a}}%

\begin{algorithm}[htp]
    \setstretch{1.3}
     \begin{algorithmic}[1]
       \State \textbf{parameters}:
       \Statex[1] Estimated model $\Mbar\in\cM$.
       \Statex[1] Exploration parameter $\eta>0$.
       \State Define \emph{inverse gap weighted policy cover} $\PiCov=\crl{\pi_{h,s,a}}_{h\in\brk{H},s\in\brk{S},a\in\brk{A}}$ via
       \begin{equation}
         \pi_{h,s,a} =
         \argmax_{\pi\in\PiRNS}\frac{\dm{\Mbar}{\pi}_h(s,a)}{2HSA +
           \eta(\fmbar(\pimbar)-\fmbar(\pi))}.
         \label{eq:pc_igw1}
       \end{equation}
       \State For each policy
       $\pi\in\PiCov\cup\crl{\pimbar}$, define
       \begin{equation}
         p(\pi) = \frac{1}{\lambda + \eta(\fmbar(\pimbar)-\fmbar(\pi))},\label{eq:pc_igw2}
       \end{equation}
       where $\lambda\in\brk{1,2HSA}$ is chosen such that $\sum_{\pi}p(\pi)=1$.
       \State \textbf{return} $p$.
\end{algorithmic}
\caption{Policy Cover Inverse Gap Weighting (\pcigw)}
\label{alg:policy_cover_igw}
\end{algorithm}

We now provide an explicit, efficiently computable strategy which
bounds the frequentist \CompText for tabular reinforcement learning. The strategy,
which we call \emph{Policy Cover Inverse Gap Weighting}, is displayed
in \pref{alg:policy_cover_igw}. As the name suggests, our approach
combines the inverse gap weighting technique introduced in the
multi-armed bandit setting with the notion of a \emph{policy cover}---that is, a collection of
policies that ensures good coverage on every state \citep{du2019latent,misra2020kinematic,jin2020reward}.

\pref{alg:policy_cover_igw} consists of two steps. First, in \pref{eq:pc_igw1}, we
compute the collection of policies
$\PiCov=\crl{\pi_{h,s,a}}_{h\in\brk{H},s\in\brk{S},a\in\brk{A}}$ that
constitutes a policy cover for the estimated model $\Mbar\in\cM$. In prior work, this is accomplished by
computing the policy $\pihsa=\argmax_{\pi\in\PiRNS}\dm{\Mbar}{\pi}_h(s,a)$
that maximizes the occupancy measure for $\Mbar$ for each $(h,s,a)$ tuple. The main twist here is that we instead
consider the policy
\[
  \pi_{h,s,a} =
  \argmax_{\pi\in\PiRNS}\frac{\dm{\Mbar}{\pi}_h(s,a)}{2HSA +
    \eta(\fmbar(\pimbar)-\fmbar(\pi))}
\]
that maximizes the ratio of the occupancy measure and the
regret gap under $\Mbar$. This \emph{inverse gap weighted policy
  cover} balances exploration and exploration by trading off coverage
with suboptimality, and is critical to deriving a tight bound on the
\CompShort that leads to $\sqrt{T}$-regret.

With the policy cover in hand, the second step of
\pref{alg:policy_cover_igw} computes the exploratory distribution $p$
by simply applying inverse gap weighting to the elements of the cover and the greedy policy $\pimbar$.

The following result---proven in \pref{app:examples}---shows that the
\pcigw strategy can be implemented efficiently. Briefly, the idea is
to solve \pref{eq:pc_igw1} by taking a dual approach and optimizing
over occupancy measures rather
than policies. With this parameterization, \pref{eq:pc_igw1} becomes a
linear-fractional program, which can then be transformed into a standard
linear program using classical techniques.
    \begin{restatable}{proposition}{efficienttabular}
  \label{prop:efficient_tabular}
  The \pcigw algorithm (\pref{alg:policy_cover_igw}) can be implemented in $\poly(H,S,A,\log(\eta))$
  time via linear programming.
\end{restatable}

Our bound on the \CompText for the \pcigw algorithm is as follows.

\begin{proposition}
  \label{prop:igw_tabular}
  Consider the tabular reinforcement learning setting with
  $\sum_{h=1}^{H}r_h\in\cR\ldef{}\brk{0,1}$. For any $\gamma>0$ and
  $\Mbar\in\cM$, the \pcigw strategy in
  \pref{alg:policy_cover_igw}, with $\eta=\frac{\gamma}{21H^2}$, ensures that
  \begin{align*}
\sup_{M\in\cM}\En_{\act\sim{}p}\brk*{\fm(\pim)-\fm(\pi)
  -\gamma\cdot\Dhels{M(\act)}{\Mbar(\act)}}
\leq{} 95\frac{H^3SA}{\gamma},
  \end{align*}
  and consequently certifies that $\comp(\cM,\Mbar)\leq95\frac{H^3SA}{\gamma}$.
\end{proposition}
This result shows for the first time how to leverage the inverse gap
weighting technique for provable exploration in reinforcement
learning. In \pref{sec:rl}, we show to extend this approach to any
family of MDPs with \emph{bilinear class}
structure.%

The proof of \pref{prop:igw_tabular} follows the structure of the
proof of \pref{prop:posterior_tabular} closely; the main
difference is that the decoupling step is replaced by an argument
based on the inverse gap weighted policy cover in \pref{eq:pc_igw1}.

\begin{proof}[\pfref{prop:igw_tabular}]%
  \newcommand{\allpolicies}{\PiCov\cup\crl{\pimbar}}%
  \newcommand{\etaalt}{\eta'}%
  We first verify that the strategy in \pref{eq:pc_igw2} is indeed
  well-defined, in the sense that a normalizing constant
  $\lambda\in[1, 2HSA]$ always exists.
  \begin{proposition}
    There is a unique choice for $\lambda>0$ such that $\sum_{\pi}p(\pi)=1$, and its value lies in $[1,2HSA]$.
  \end{proposition}
  \begin{proof}
    Let $f(\lambda)=\sum_{\pi}\frac{1}{\lambda +
      \eta(\fmbar(\pimbar)-\fmbar(\pi))}$. We
    first observe that if $\lambda>2HSA$, then
    $f(\lambda)\leq{}\sum_{\pi}\frac{1}{\lambda}=\frac{HSA+1}{\lambda}<1.$
    On the other hand for $\lambda\in(0,1)$,
    $f(\lambda)\geq{}\frac{1}{\lambda +
      \eta(\fmbar(\pimbar)-\fmbar(\pimbar))}
    = \frac{1}{\lambda}>1$.
    Hence, since $f(\lambda)$ is continuous and strictly decreasing over
  $(0,\infty)$, there exists a unique $\lambda^{\star}\in[1,
  2HSA]$ such that $f(\lambda^{\star})=1$.
  \end{proof}

With this out of the way, we show that the \pcigw strategy achieves the desired bound on the
\CompShort. Let $M\in\cM$ be fixed. As in
\pref{prop:posterior_tabular}, we bound the quantity
  \begin{align*}
    \En_{\pi\sim{}p}\brk*{\fm(\pim) - \fmbar(\act)}
  \end{align*}
  in terms of the quantity on the right-hand side of \pref{eq:com1},
  then apply change of measure (\pref{lem:change_of_measure}). We
  begin with the decomposition
  \begin{align}
    \En_{\pi\sim{}p}\brk*{\fm(\pim) - \fmbar(\pi)}
    =     \underbrace{\En_{\pi\sim{}p}\brk*{\fmbar(\pimbar) -
    \fmbar(\pi)}}_{(\mathrm{I})}
    + \underbrace{\fm(\pim) - \fmbar(\pimbar)}_{(\mathrm{II})}.
    \label{eq:igw_tabular0}
  \end{align}
  For the first term $(\mathrm{I})$, which may be thought of as exploration bias, we have
  \begin{align}
    \En_{\pi\sim{}p}\brk*{\fmbar(\pimbar) - \fmbar(\pi)}
    =\sum_{\act\in\PiCov\cup\crl{\pimbar}}\frac{\fmbar(\pimbar) - \fmbar(\act)}{\lambda
    + \eta(\fmbar(\pimbar) - \fmbar(\act))}
    \leq{} \frac{2HSA}{\eta},
    \label{eq:igw_tabular0.5}
  \end{align}
  where we have used that $\lambda\geq{}0$. We next bound the second
  term $(\mathrm{II})$, which entails showing that the \pcigw
  distribution \emph{explores enough}. We have
  \begin{equation}
    \label{eq:igw_tabular1}
    \fm(\pim) - \fmbar(\pimbar)
    = \fm(\pim) - \fmbar(\pim) - (\fmbar(\pimbar) - \fmbar(\pim)).
  \end{equation}
  Following \pref{prop:posterior_tabular}, we use the simulation lemma
  to bound
    \begin{align*}
      \fm(\pim)- \fmbar(\pim)
  &\leq{}
    \sum_{h=1}^{H}\Enm{\Mbar}{\pim}\brk*{\Dtv{\Pm_h(s_h,a_h)}{\Pmbar_h(s_h,a_h)}+
    \Dtv{\Rm_h(s_h,a_h)}{\Rmbar_h(s_h,a_h)}}\\
      &=\sum_{h=1}^{H}\sum_{s,a}\dm{\Mbar}{\pim}_h(s,a)\errm_h(s,a),
    \end{align*}
where $\errm_h(s,a) \ldef \Dtv{\Pm(s,a)}{\Pmbar(s,a)} +
\Dtv{\Rm(s,a)}{\Rmbar(s,a)}$. Define $\dbar_h(s,a) =
    \En_{\act\sim{}p}\brk[\big]{\dm{\Mbar}{\pi}_h(s,a)}$. Then, using the
    AM-GM inequality, we have that for any $\etaalt>0$,
    \begin{align*}
      \sum_{h=1}^{H}\sum_{s,a}\dm{\Mbar}{\pim}_h(s,a)\brk*{\errm_h(s,a)}
      &=
        \sum_{h=1}^{H}\sum_{s,a}\dm{\Mbar}{\pim}_h(s,a)\prn*{\frac{\dbar_h(s,a)}{\dbar_h(s,a)}}^{1/2}(\errm_h(s,a))^2\\
      &\leq{} \frac{1}{2\etaalt}\sum_{h=1}^{H}\sum_{s,a}\frac{(\dm{\Mbar}{\pim}_h(s,a))^{2}}{\dbar_h(s,a)}
+
        \frac{\eta'}{2}\sum_{h=1}^{H}\sum_{s,a}\dbar_h(s,a) (\errm_h(s,a))^2\\
      &= \frac{1}{2\etaalt}\sum_{h=1}^{H}\sum_{s,a}\frac{(\dm{\Mbar}{\pim}_h(s,a))^2}{\dbar_h(s,a)}
              +        \frac{\eta'}{2}\sum_{h=1}^{H}\En_{\pi\sim{}p}\Enm{\Mbar}{\pi}\brk*{(\errm_h(s_h,a_h))^2}.
    \end{align*}
    The second term is exactly the upper bound we want, so it remains to bound
    the ratio of occupancy measures in the first term. Observe that
    for each $(h,s,a)$, we have
    \begin{align*}
      \frac{\dm{\Mbar}{\pim}_h(s,a)}{\dbar_h(s,a)}
      \leq{}
      \frac{\dm{\Mbar}{\pim}_h(s,a)}{\dm{\Mbar}{\pihsa}_h(s,a)}\cdot\frac{1}{p(\pihsa)}
      \leq{}
      \frac{\dm{\Mbar}{\pim}_h(s,a)}{\dm{\Mbar}{\pihsa}_h(s,a)}\prn*{2HSA +
      \eta(\fmbar(\pimbar) -\fmbar(\pihsa)},
    \end{align*}
    where the second inequality follows from the definition of $p$ and
    the fact that $\lambda\leq{}2HSA$. Furthermore, since
    \[
      \pihsa = \argmax_{\pi\in\PiRNS}\frac{\dm{\Mbar}{\pi}_h(s,a)}{2HSA +
           \eta(\fmbar(\pimbar)-\fmbar(\pi))},
       \]
       and since $\pim\in\PiRNS$, we can upper bound by
       \begin{equation}
       \label{eq:pcigw_multiplicative}
               \frac{\dm{\Mbar}{\pim}_h(s,a)}{\dm{\Mbar}{\pim}_h(s,a)}\prn*{2HSA +
        \eta(\fmbar(\pimbar) -\fmbar(\pim)}
      = 2HSA + \eta(\fmbar(\pimbar) -\fmbar(\pim).
    \end{equation}
    As a result, we have
    \[
      \sum_{h=1}^{H}\sum_{s,a}\frac{(\dm{\Mbar}{\pim}_h(s,a))^2}{\dbar_h(s,a)}
      \leq{}
      \sum_{h=1}^{H}\sum_{s,a}\dm{\Mbar}{\pim}_h(s,a)(2HSA +
      \eta(\fmbar(\pimbar) -\fmbar(\pim))
      = 2H^{2}SA +\eta{}H(\fmbar(\pimbar) -\fmbar(\pim)).
    \]
    Putting everything together and returning to \pref{eq:igw_tabular1}, this establishes that
    \begin{align*}
      \fm(\pim) - \fmbar(\pimbar)
      \leq{} \frac{H^{2}SA}{\eta'}
      +\frac{\eta'}{2}\sum_{h=1}^{H}\En_{\pi\sim{}p}\Enm{\Mbar}{\pi}\brk*{(\errm_h(s_h,a_h))^2}
      +\frac{\eta{}H}{2\eta'}(\fmbar(\pimbar) -\fmbar(\pim))
                 -(\fmbar(\pimbar) - \fmbar(\pim)).
    \end{align*}
    We set $\eta' = \frac{\eta{}H}{2}$ so that the latter terms cancel
    and we are left with
    \[
      \fm(\pim) - \fmbar(\pimbar)
      \leq{} \frac{2HSA}{\eta} + \frac{\eta{}H}{4}\sum_{h=1}^{H}\En_{\pi\sim{}p}\Enm{\Mbar}{\pi}\brk*{(\errm_h(s_h,a_h))^2}.
    \]
    Combining this with \pref{eq:igw_tabular0} and
    \pref{eq:igw_tabular0.5} gives
    \begin{align*}
      &\En_{\pi\sim{}p}\brk*{\fm(\pim) - \fmbar(\pi)}\\
      &\leq{} \frac{4HSA}{\eta} +
        \frac{\eta{}H}{4}\sum_{h=1}^{H}\En_{\pi\sim{}p}\Enm{\Mbar}{\pi}\brk*{(\errm_h(s_h,a_h))^2}\\
      &\leq{} \frac{4HSA}{\eta} + \frac{\eta{}H}{2}\sum_{h=1}^{H}\En_{\pi\sim{}p}\Enm{\Mbar}{\pi}\brk*{\Dtvs{\Pm(s_h,a_h)}{\Pmbar(s_h,a_h)}+\Dtvs{\Rm(s_h,a_h)}{\Rmbar(s_h,a_h)}}.
    \end{align*}
We conclude by applying the change-of-measure lemma (\pref{lem:change_of_measure}), which
implies that for any $\eta'>0$,
\begin{align*}
\En_{\pi\sim{}p}\brk*{\fm(\pim) - \fm(\act)}
    \leq{} \frac{4HSA}{\eta} + (4\eta')^{-1} + (20H^2\eta+\eta')\cdot\En_{\pi\sim{}p}\brk*{
    \Dhels{M(\act)}{\Mbar(\act)}
    }.
\end{align*}
The result follows by choosing $\eta=\eta'=\frac{\gamma}{21H^2}$.
\end{proof}

\subsubsection{Regret Upper Bound}
Using an approach in \pref{sec:rl_bilinear_estimation}, we can obtain an efficient online estimation algorithm
for tabular MDPs which guarantees that \dfedit{$\Mhat\ind{t}\in\cM$} and
\[
\EstHel =
\sum_{t=1}^{T}\En_{\act\ind{t}\sim{}p\ind{t}}\brk*{\Dhels{\Mstar(\act\ind{t})}{\Mhat\ind{t}(\act\ind{t})}}
\leq{} \bigoht(HS^{2}A).
\]
Combining this with the \pcigw strategy (\pref{prop:igw_tabular}) and
\mainalg (\pref{thm:upper_general}),
we obtain an efficient algorithm with
\[
\RegDM \leq \bigoht(\sqrt{H^{4}S^{3}A^2T}).
\]
Regret bounds for tabular reinforcement learning have received extensive investigation, but this result is exciting
because it represents a completely new algorithmic approach. In
particular, this is the first frequentist reinforcement learning
algorithm we are aware of that does not make use of confidence sets or
optimism.%

Note that while this bound scales as $\sqrt{\poly(S,A,H)T}$ as desired,
it does fall short of the minimax rate, which is $\sqrt{HSAT}$ for our
setting (since we consider $\sum_{h=1}^{H}r_h\in\brk{0,1}$ and
time-inhomogeneous dynamics). We emphasize that obtaining the tightest
possible dependence on problem parameters is not the focus of this work,
but it would be interesting to improve this approach to match the
minimax rate.

\subsubsection{Lower Bound}
We close the tabular reinforcement learning example by complementing
our upper bounds with a lower bound on the \CompText.
\begin{proposition}
  \label{prop:lower_tabular}
  Let $\cM$ be the class of tabular MDPs with $S\geq{}2$ states, $A\geq{}2$ actions,
  and $\sum_{h=1}^{H}r_h\in\cR\ldef\brk*{0,1}$. If $H\geq{}2\log_2(S/2)$, then
  there exists $\Mbar\in\cM$ such that for
all $\gamma\geq{}HSA/24$,
\[
\comp(\cMinf[\vepsg](\Mbar),\Mbar) \geq{} 2^{-9}\cdot\frac{HSA}{\gamma},
\]
  where $\vepsg=\frac{HSA}{96\gamma}$.
\end{proposition}
As with the multi-armed bandit example, the proof of this result
proceeds by constructing a hard family of models and appealing to
\pref{lem:hard_family}. We follow a familiar tree MDP construction
(e.g., \cite{osband2016lower,domingues2021episodic}).
\begin{proof}[\pfref{prop:lower_tabular}]%
  \newcommand{\hstar}{h^{\star}}%
  \newcommand{\sstar}{s^{\star}}%
    \newcommand{\Msa}{M_{s,a}}%
    \newcommand{\Mhsa}{M_{h,s,a}}%
    \newcommand{\Mhsastar}{M_{\hstar,\sstar,\astar}}%
    \newcommand{\pisa}{\pi\subs{M_{s,a}}}%
    \newcommand{\Pmsa}{P\sups{M_{s,a}}}%
    \newcommand{\Rmsa}{R\sups{M_{s,a}}}%
    \newcommand{\fsa}{f\sups{\Msa}}%
    \newcommand{\fhsa}{f\sups{\Mhsa}}%
    \newcommand{\term}{\mathsf{term}}%
    \newcommand{\wait}{\mathsf{wait}}%
Assume without loss of generality that $S$ is a multiple of $2$. Let
$\Delta\in(0,1/2)$ be a parameter. We consider the following class of
MDPs.
\begin{itemize}
  \item Define $H_1\ldef\log_2(S/2)$ and $H_2=H-H_1$.
  \item Let $S'\ldef{}S/2$. The state space $\cS$ is chosen to consist of a depth-$H_1$ binary tree
    (which has $S'$ leaves and $\sum_{i=0}^{\log_2(S/2)}2^{i}=S-1$
    total states), along with a single terminal state $\term$. We let
    $\cS'$ denote the collection of leaf states.
  \item All MDPs $M$ in the family have the same (deterministic) dynamics
    $\Pm=P$. The agent begins at the root state in the tree, and for each
      $h < H_1$, there are two available actions, $\mathsf{left}$ and
      $\mathsf{right}$, which determine whether the next state is the
      left or right successor node in the tree. For $h\geq{}H_1$,
      there are two cases:
      \begin{itemize}
      \item For $s\in\cS'$, the agent can choose to ``wait'' using action
        $\wait$, or choose an action from $\cA'\ldef\brk{A'}$, where
        $A'\ldef{}A-1$. The $\wait$ action causes the agent to stay in
        $s$ (i.e. $P(s\mid{}s,\wait)=1)\;\forall{}s\in\cS'$), while
        actions in $\brk{A'}$ cause the agent to immediately transit
        to $\term$ (i.e.,
        $P(\term\mid{}s,a)=1\;\forall{}s\in\cS',a\in\brk{A'}$).
      \item The terminal state $\term$ is self-looping and (i.e. $P(\term\mid{}\term,\cdot)=1$).
      \end{itemize}
    \item Let $\cH'=\crl{H_1,\ldots,H}$.  For each $\hstar\in\cH', \sstar\in\cS', \astar\in\cA'$, we
      define an MDP $\Mhsastar$ which has the dynamics described above and
      the following reward functions:
      \begin{itemize}
      \item $\Rm_h(s,a)=0$ a.s. for all $h<H_1$.
      \item $\Rm_h(\term,\cdot)=0$ a.s. for all $h>H_1$
      \item $\Rm_h(s,\wait)=0$ a.s. for all $s\in\cS'$ for all $h\geq{}H_1$
      \item $\Rm_h(s,a)=\Ber\prn*{\tfrac{1}{2}+\Delta\cdot\indic\crl{h=\hstar,s=\sstar,a=\astar}}$
        for $s\in\cS'$, $a\in\brk{A'}$, $h\geq{}H_1$.
      \end{itemize}

  \end{itemize}
  We choose $\cM'=\crl*{M_{s,a}}_{h\in\cH',s\in\cS',a\in\cA'}$. Finally,
  we take the reference MDP $\Mbar=(P,\Rmbar)$ to have the
same dynamics and rewards as above, except that
$\Rmbar_h(s,a)=\Ber(1/2)$ for all $h\in\cH', s\in\cS', a\in\cA'$.

We claim that $\cM'$ is a hard family of MDPs in the sense of
\pref{lem:hard_family}. To do so, we define
\[
  \hardu_{h,s,a}(\act)=\hardv_{h,s,a}(\act)=\Prm{\Mbar}{\act}(s_h=s,a_h=a).
\]
We note that for any $\act$,
\[
  \sum_{s\in\cS',a\in\cA',h\in\cH'}\hardu_{h,s,a}(\act)
  =
  \sum_{s\in\cS',a\in\cA'}\sum_{h=H_1}^{H}\Prm{\Mbar}{\act}(s_h=s,a_h=a)
  \leq{}\sum_{s\in\cS',a\in\cA'}\Prm{\Mbar}{\act}(s_{H_1}=s,a_{H_1}=a)\leq{}1,
\]
where the inequality uses that (i) all $a\in\cA'$ transit to the terminal
state $\term$ for $h\geq{}H_1$, and (ii) each state $s\in\cS'$ is only
reachable for $h\geq{}H_1$ if $s_{H_1}=s$. This verifies that
$u_{h,s,a}$ and $v_{h,s,a}$ satisfy the preconditions of Lemma 1.

To proceed, we first observe that the optimal policy for $\Mhsa$ takes the single sequence
of actions in the tree that leads to state $s$, uses action $\wait$ until step
$h$, then selects action $a$ at step
$h$ for expected reward $\frac{1}{2}+\Delta$. As a result,
\[
f\sups{\Mhsa}(\pihsa) - \fhsa(\pi) = \Delta\cdot\Prm{\Mhsa}{\pi}\prn*{(s_h,a_h)\neq{}(s,a)}=\Delta(1-\hardu_{h,s,a}(\act))
\]
where we recall that all MDPs share the same dynamics. Furthermore, since the rewards for $\Mbar$ and $\Mhsa$ are
identical unless $(s_h,a_h)=(s,a)$, we have\footnote{Since squared
  Hellinger distance is an $f$-divergence, it satisfies $\Dhels{\bbP_{Y\mid{}X}\tens{}\bbP_X}{\bbQ_{Y\mid{}X}\tens{}\bbP_X}=\En_{X\sim{}\bbP_X}\brk*{\Dhels{\bbP_{Y\mid{}X}}{\bbQ_{Y\mid{}X}}}$.}
\begin{align*}
\Dhels{\Mhsa(\act)}{\Mbar(\act)}
  &=\Prm{\Mbar}{\pi}\prn*{(s_h,a_h)=(s,a)}\cdot\Dhels{\Ber(\nicefrac{1}{2})}{\Ber(\nicefrac{1}{2}+\Delta)}\\
  &\leq{}\Prm{\Mbar}{\pi}\prn*{(s_h,a_h)=(s,a)}\cdot{}3\Delta^{2}\\
  &= \hardv_{h,s,a}(\act)\cdot 3\Delta^{2},
\end{align*}
by \pref{lem:divergence_bernoulli}. It follows that $\cM'$ is a
$(\Delta,3\Delta^2,0)$-family, so \pref{lem:hard_family} implies that
\[
  \compb(\cM',\Mbar)
  \geq{} \frac{\Delta}{2}- \gamma\frac{3\Delta^{2}}{\abs{\cH'}\abs{\cS'}\abs{\cA'}} \geq \frac{\Delta}{2}- \gamma\frac{24\Delta^{2}}{HSA}.
\]
We set $\Delta=\frac{HSA}{96\gamma}$,
which leads to value
\[
  (384)^{-1}\frac{HSA}{\gamma}.
\]
We conclude
by noting that $\cM'\subseteq\cMinf[\Delta](\Mbar)$, and %
that we may take $\abscontp=\bigoh(1)$ in \pref{thm:lower_main,thm:lower_main_expectation}
whenever $\Delta\leq1/4$; it suffices to take $\gamma\geq{}HSA/24$. 
\end{proof}

\paragraph{Regret lower bound}
We apply \pref{thm:lower_main_expectation} with the lower bound on the \CompShort
from \pref{prop:lower_tabular}, which gives that for any $\gamma\geq{}c\sqrt{T}$,
$\En\brk*{\RegDM}
  \geq{} c'\cdot\frac{HSAT}{\gamma}$,
so long as $\frac{HSA}{96\gamma}\leq\vepslowg=c''\frac{\gamma}{T}$,
where $c,c',c''>0$ are numerical constants.\footnote{As in the
  multi-armed bandit example, we have $\abscontp=\bigoh(1)$ in
  \pref{thm:lower_main_expectation} for the MDP
  family $\cM'$.} If we
choose $\gamma=C\cdot\sqrt{AT}$ for sufficiently large $C>0$, we have
that
\[
\En\brk*{\RegDM} \geq{} \Omega(\sqrt{HSAT}).
\]

\subsection{Discussion}

We showed how to bound the \CompText for multi-armed bandits and
tabular reinforcement learning via Bayesian approaches (posterior
sampling) and frequentist approaches (inverse gap weighting). The
analyses for both techniques parallel each other, and leverage
decoupling and change of measure arguments. We build on these ideas in \pref{sec:bandit} and
\pref{sec:rl} to derive algorithms and bounds for more complex settings.

While we have considered posterior
sampling (the most ubiquitous Bayesian approach) and inverse gap
weighting (a somewhat more recent frequentist approach) and highlighted parallels, upper
confidence bound-based approaches have been conspicuously absent up to
this point, and do not appear to be sufficient
to bound the \CompText. Informally, this is because the \CompShort
considers estimation error under the \emph{learner's own distribution}
(i.e., future estimation error after the learner commits to the
exploration strategy), while UCB and other confidence-based approaches
explore based on estimation error on historical data. Further research
is required to better understand whether this distinction is fundamental.

\section{Application to Bandits
}
\label{sec:bandit}
As a special case of our main results, we obtain algorithms and lower
bounds for structured bandits with large action spaces
(\pref{ex:bandit}).
In this
section, we highlight some well-known instances of the structured bandit
problem recovered by our results (\pref{sec:bandits_familiar}), then
provide a new guarantee based on
a combinatorial parameter called the \emph{star number}
\citep{hanneke2015minimax,foster2020instance} (\pref{sec:bandits_star}). The latter result improves upon
previous regret bounds based on the eluder dimension.

\subsection{Familiar Examples}
\label{sec:bandits_familiar}
We consider three canonical structured bandit settings, linear
bandits, convex bandits, and non-parametric bandits, and provide tight
upper
and lower bounds on the \CompText.
We then provide additional lower bounds for bandits with ReLU rewards and for various bandit problems with gaps.

Throughout the section we assume $\cR\subseteq\brk*{0,1}$ unless
otherwise specified. For a
given function class $\cF\subseteq\prn{\Act\to\cR}$, we define $\cMf=\crl*{M : \fm\in\cF}$ as the induced class of
models. We tacitly make use of the fact that all of the lower bound constructions $\cM'\subseteq\cM$ in this section satisfy $\abscontp=\bigoh(1)$ for \pref{thm:lower_main}.

\subsubsection{Linear Bandits}
\label{sec:linear}

In the linear bandit setting \citep{abe1999associative,auer2002finite,dani2008stochastic,chu2011contextual,abbasi2011improved}, we set $\Act\subseteq\bbR^{d}$,
define $\cF=\crl*{\act\mapsto\tri{\theta,\act}\mid{}\theta\in\Theta}$ for a
parameter set $\Theta\subseteq\bbR^{d}$, then take $\cM=\cMf$ as the
induced model class. The following recent result from
\cite{foster2020adapting} gives an efficient algorithm that leads to
upper bounds on the \CompText for this setting.
  \begin{proposition}[Upper bound for linear bandits \citep{foster2020adapting}]
    \label{prop:bandit_upper_linear}
    Consider the linear bandit setting with $\cR=\bbR$. Let $\Mbar\in\conv(\cM)$ and
    $\gamma>0$ be given, and define
    \[
      p=\argmax_{p\in\Delta(\Act)}\crl*{\En_{\act\sim{}p}\brk[\big]{
        \fmbar(\act)} + \frac{1}{4\gamma}\log\det(\En_{\act\sim{}p}\brk*{\act\act^{\trn}})
    }.
  \]
  This strategy certifies that
  \[
    \compSq(\cM,\Mbar) \leq{} \frac{d}{4\gamma}.
  \]
  \end{proposition}
  Combining this strategy with \pref{thm:upper_main_bayes}, we obtain
  $\En\brk{\RegDM}\leq{}\bigoh(\sqrt{dT\log\abs{\Act}})$ when
  $\abs{\Act}<\infty$, and for infinite action spaces we
  obtain $\En\brk{\RegDM}\leq\bigoht(d\sqrt{T})$ whenever $\Theta$ and
  $\Act$ have bounded diameter; both results are optimal. More generally, using
  this strategy within the \mainalg algorithm (\pref{thm:upper_general})
  yields $\En\brk{\RegDM}\leq\bigoh(\sqrt{dT\cdot{}\EstHel})$.

We now turn our attention to lower bounds. Note that if we take
$\Theta$ to be the $\ls_{\infty}$-ball and $\Act$ to be the $\ls_{1}$-ball, a trivial
lower bound on the \CompShort is $\frac{d}{\gamma}$, since this embeds
the finite-armed bandit setting. The following result shows that the same lower bound
holds under euclidean geometry.
  \begin{proposition}[Lower bound for linear bandits]
    \label{prop:bandit_lower_linear}
    Consider the linear bandit setting with $\cR=\brk{-1,+1}$, and let
    $\Act=\Theta=\crl*{v\in\bbR^{d}\mid{}\nrm*{v}_{2}\leq{}1}$. Then
    for all $d\geq{}4$ and $\gamma\geq\frac{2d}{3}$, there exists $\Mbar\in\cM$ such that
  \[
    \comp(\cMinf[\vepsg](\Mbar),\Mbar) \geq{} \frac{d}{12\gamma},
  \]
where $\vepsg = \frac{d}{3\gamma}$.
\end{proposition}
Combining this result with \pref{thm:lower_main_expectation} leads to a lower
bound of the form $\En\brk*{\RegDM}\geq\bigom(\sqrt{dT})$.

\subsubsection{Convex Bandits}
The convex bandit problem
\citep{kleinberg2004nearly,flaxman2005online,agarwal2013stochastic,bubeck2016multi,bubeck2017kernel,lattimore2020improved}
is a generalization of the linear bandit. We take
$\Act\subseteq\bbR^{d}$, define
\[
  \cF=\crl*{f:\Act\to\brk{0,1}\mid{}\text{$f$ is concave and
      $1$-Lipschitz w.r.t $\ls_2$}},
\] and take $\cM=\cMf$ as the
induced model class.\footnote{We consider concave rather than concave
  functions because we work with rewards instead of losses.}
The following recent result of \cite{lattimore2020improved} provides a
bound on the \CompText for this setting.\footnote{The statement
  presented here requires very slight modifications to the construction in
  \cite{lattimore2020improved}. Namely, certain parameters that scale
  with $T$ in the original construction must be replaced by $\gamma$.}
\begin{proposition}[Upper bound for convex bandits
  (\citet{lattimore2020improved}, Theorem 3)]
  \label{prop:bandit_upper_convex}
  For the convex bandit setting with $\cR=\brk{0,1}$, we have
  \[
    \compSq(\cM,\Mbar) \leq{} \bigoh\prn*{
      \frac{d^{4}}{\gamma}\cdot\polylog(d,\mathrm{diam}(\Act),\gamma)
    }
  \]
  for all $\Mbar\in\conv(\cM)$ and $\gamma>0$.
\end{proposition}
Since this setting has $\ActComp\leq\bigoht(d)$ whenever
$\diam(\Act)=\bigoh(1)$, combining the bound above with
\pref{thm:upper_main_bayes} leads to regret $\En\brk{\RegDM}\leq{}\bigoht(d^{2.5}\sqrt{T})$.

We mention in passing that previous results in the line of work on
Bayesian regret for convex
bandits \citep{bubeck2015bandit,bubeck2016multi} can also be
interpreted as bounds on the \CompText. While the optimal
dependence on $d$ for this setting is not yet understood, a lower bound of $\sqrt{dT}$
follows from the result for the linear setting.

\subsubsection{Nonparametric Bandits}
\newcommand{\met}{\rho}
\newcommand{\Mcov}{\cN_{\met}}
For the next example, we consider a standard nonparametric bandit problem: Lipschitz
bandits in metric spaces \citep{auer2007improved,kleinberg2019bandits}. We
take $\Act$ to be a metric space equipped with metric $\met$, then
take $\cM=\cMf$, where we define
\[
\cF = \crl*{f:\Act\to\brk{0,1} \mid{} \text{$f$ is $1$-Lipschitz w.r.t $\met$}}.
\]
Our results are stated in terms of covering numbers with respect to
the metric $\rho$. Let us say that $\Act'\subseteq\Act$ is an $\veps$-cover with
respect to $\met$ if
\[
    \forall{}\act\in\Act\quad\exists{}\act'\in\Act'\quad\text{s.t.}\quad
    \met(\act,\act')\leq\veps,
  \]
  and let $\Mcov(\Act,\veps)$ denote the size of the smallest such cover.
\begin{proposition}[Upper bound for Lipschitz bandits]
    \label{prop:bandit_upper_lipschitz}
    Consider the Lipschitz bandit setting with $\cR=\bbR$, and suppose that
    $\Mcov(\Act,\veps)\leq\veps^{-d}$ for $d>0$.
    Let $\Mbar\in\conv(\cM)$ and $\gamma\geq{}1$ be given and consider the following algorithm:
    \begin{enumerate}
    \item Let $\Act'\subseteq\Act$ witness the covering number
      $\Mcov(\Act,\veps)$ for a parameter $\veps>0$.
    \item Perform the inverse gap weighting strategy in \pref{eq:igw}
      over $\Act'$.
    \end{enumerate}
    By setting $\veps=\gamma^{-\frac{1}{d+1}}$, this strategy certifies that
    \[
            \compSq(\cM,\Mbar) \leq{} 2\gamma^{-\frac{1}{d+1}}.
    \]
\end{proposition}
Since $\ActComp\leq\bigoht(d)$ for this setting, if we apply
\pref{thm:upper_main_bayes} with this result we obtain
\[
\En\brk{\RegDM} \leq \bigoht(T^{\frac{d+1}{d+2}})
\]
which matches the minimax rate derived in \cite{kleinberg2019bandits}. We
complement this with a lower bound.
\begin{proposition}[Lower bound for Lipschitz bandits]
    \label{prop:bandit_lower_lipschitz}
Consider the Lipschitz bandit setting with $\cR=\brk{0,1}$.  Suppose
that $\Mcov(\Act,\veps)\geq{}\veps^{-d}$ for $d\geq{}1$. Then for all
$\gamma\geq{}1$, there exists $\Mbar\in\cM$ such that
\[
  \comp(\cMinf[\vepsg](\Mbar),\Mbar) \geq{} 2^{-7}\gamma^{-\frac{1}{d+1}},
\]
where $\vepsg=6^{-2}\gamma^{-\frac{1}{d+1}}$.
\end{proposition}
Plugging this guarantee into \pref{thm:lower_main_expectation}, we obtain a
lower bound of the form $\En\brk*{\RegDM}\geq\bigomt(T^{\frac{d+1}{d+2}})$, which again recovers
the minimax rate. Extending these upper and lower
bounds on the \CompShort to accommodate other (e.g., \Holder)
nonparametric bandit problems is straightforward.
\subsubsection{ReLU Bandits}  

We now consider a bandit setting based on the well-known ReLU activation function
$\relu(x)=\max\crl{x,0}$. Here, we take
$\Act=\crl*{\act\in\bbR^{d}\mid\nrm{\act}_2\leq{}1}$ and $\cM=\cMf$,
where $\cF$ is the class of value functions of the form
\begin{equation}
  \label{eq:relu}
  f(\act) = \relu(\tri{\theta,\act}-b),
\end{equation}
where $\theta\in\Theta\ldef\crl*{\theta\in\bbR^{d}\mid{}\nrm{\theta}_2\leq{}1}$ is an unknown parameter vector and
$b\geq{}0$ is a known bias parameter.

Note that since we work with rewards, this setting would be a special
case of bandit convex optimization if we were to replace the $+\relu(\cdot)$
function in \pref{eq:relu} with $-\relu(\cdot)$,
and in this case it would be possible to appeal to \pref{prop:bandit_upper_convex} to
derive a $\sqrt{\poly(d)T}$ bound on regret. However, the $+\relu(\cdot)$
formulation in \pref{eq:relu} cannot be viewed as an instance of
bandit convex optimization, and the following proposition shows
that this setting is intractable.\looseness=-1
\begin{proposition}[Lower bound for ReLU bandits]
  \label{prop:bandit_lower_relu}
Consider the ReLU bandit setting with $\cR=\brk*{-1,+1}$. For all $d\geq{}16$, there exists $\Mbar\in\cM$ such that for
  all $\gamma>0$,
  \begin{equation}
    \label{eq:16}
    \comp(\cMinf[\vepsg](\Mbar),\Mbar) \geq{} \frac{e^{d/8}}{24\gamma}\wedge\frac{1}{8},
  \end{equation}
  where $\vepsg=\frac{e^{d/8}}{6\gamma}\wedge\frac{1}{2}$.
\end{proposition}

By plugging this result into \pref{thm:lower_main} and setting
$\gamma=T$, we conclude that any algorithm must have
\[
\RegDM \geq \bigom(\min\crl{e^{\Omega(d)},T})
\]
with constant probability, which further implies that $\En\brk*{\RegDM} \geq \bigom(\min\crl{e^{\Omega(d)},T})$.
This recovers recent impossibility results
\citep{dong2021provable,li2021eluder}.

\subsubsection{Gap-Dependent Lower Bounds}
In multi-armed bandits, \emph{gap-dependent} regret bounds that adapt
to the gap
\[
\Delm\ldef{}\min_{\act\neq\pim}\crl*{\fm(\pim) - \fm(\act)}
\]
between the best and second-best action have been the subject of extensive
investigation
\citep{lai85asymptotically,burnetas1996optimal,garivier2016explore,garivier2019explore,kaufmann2016complexity,lattimore2018refining,garivier2016optimal}. Gap-dependent
guarantees in reinforcement learning with function approximation have
also received recent interest as a means to bypass certain intractability
results \citep{du2019provably,wang2021exponential}. Here we prove
lower bounds on the \CompText for finite-armed bandits and linear
bandits when the class $\cM$ is constrained to have gap
$\Delta$. These examples highlight that the \CompShort leads to
meaningful lower bounds even for ``easy'' problems with low
statistical complexity
\begin{proposition}[Multi-armed bandits with gaps]
  \label{prop:bandit_gap_tabular}
  Let $\cM$ be the class of all multi-armed bandit problems over
  $\Act=\brk{A}$ with $\cR=\brk{0,1}$ and gap $\Delta>0$. For all
  $\Delta\in(0,1/8)$, there exists $\Mbar\in\cM$, such that
  \[
    \comp(\cMinf[\Delta](\Mbar),\Mbar)\geq{}
    \frac{\Delta}{4}\indic\crl*{
      \gamma\leq\frac{A}{48\Delta}
    }
  \]
  for all $\gamma>0$.
\end{proposition}
Applying this result within \pref{thm:lower_main_expectation}, we are guaranteed
that
$\En\brk*{\RegDM} \geq{} c\cdot
      \Delta{}T\indic\crl*{\gamma\leq{}(48)^{-1}\frac{A}{\Delta}}$,  as long as $\Delta \leq \vepslowg = c'\frac{\gamma}{T}$, where
  $c,c'>0$ are numerical constants. In particular, if we set
  $\gamma=c_1\frac{A}{\Delta}$ for $c_1>0$ sufficiently small, we have $\En\brk*{\RegDM} \geq
  \bigom\prn*{\Delta{}T}$ for all $T\leq{}c_2\frac{A}{\Delta^2}$, where $c_2>0$ is a numerical constant. Since minimax regret is non-decreasing with $T$, this
  implies a lower bound of the form
  \begin{equation}
    \label{eq:mab_gap}
    \En\brk*{\RegDM} \geq \bigom\prn*{
      \min\crl*{\Delta{}T, \frac{A}{\Delta}}
    }.
  \end{equation}
  Up to a $\log(T)$ factor, this matches the usual $\frac{A}{\Delta}$
  scaling found in standard gap-dependent lower bounds
\citep{lai85asymptotically,burnetas1996optimal,garivier2016explore,garivier2019explore,kaufmann2016complexity,lattimore2018refining,garivier2016optimal}
when $T$ is sufficiently large. We
caution however that these results are
\emph{instance-dependent} in nature, and provide gap-dependent lower
bounds on the regret for any particular problem instance, whereas
\pref{eq:mab_gap} is a minimax lower bound over the class of all
possible models with gap $\Delta$. This is a consequence of the
fact that the \CompText and our associated lower bounds capture
minimax complexity for decision making, which is fundamentally different from instance-dependent complexity; see \pref{sec:related} for more discussion.

We proceed with an analogous lower bound for the linear setting in \pref{sec:linear}.
\begin{proposition}[Linear bandits with gaps]
  \label{prop:bandit_gap_linear}
  For every $\Delta\in(0,1/4)$, there exists a collection $\cM$ of linear bandit models with $\cR=\brk{-1,+1}$, 
  $\Act\subseteq\Theta=\crl*{v\in\bbR^{d}\mid{}\nrm*{v}_{2}\leq{}1}$, and gap
  $\Delta>0$, such that for some $\Mbar\in\cM$,
  \[
    \comp(\cMinf[\Delta](\Mbar),\Mbar) \geq{}     \frac{\Delta}{4}\indic\crl*{
      \gamma\leq\frac{d}{12\Delta}
    }
  \]
for all $\gamma>0$.
\end{proposition}
Using similar calculations to the multi-armed bandit example, we
deduce that any algorithm must have

  \begin{equation}
    \label{eq:linear_gap}
    \En\brk*{\RegDM} \geq \bigom\prn*{
      \min\crl*{\Delta{}T, \frac{d}{\Delta}}
    }.
  \end{equation}
\subsection{Disagreement Coefficient, Star Number, and Eluder Dimension}
\label{sec:bandits_star}

In this section we show that the \CompText can recover regret bounds for
bandits based on a well-known combinatorial parameter called the
\emph{eluder dimension}
\citep{russo2013eluder,wang2020provably,ayoub2020model,jin2021bellman},
and then give a new bound based on a closely related but tighter parameter
called the \emph{star number}
\citep{hanneke2015minimax,foster2020instance}.

We begin by defining the eluder dimension for a value function class $\cF$.
\begin{definition}
\label{def:eluder}
  Let $\cF:\Act\to\bbR$ be given, and define $\ElCheck(\cF,\Delta)$ as the length of the longest sequence
  of decisions $\act_1,\ldots,\act_s\in\Act$ such that for all $i$, there
  exists $f_i\in\cF$ such that
  \begin{equation}
    \label{eq:eluder}
\abs*{f_i(\act_i)}>\Delta,\quad\text{and}\quad\sum_{j<i}f^2_i(\act_j)\leq\Delta^{2}.
  \end{equation}
  The eluder dimension is defined as $\El(\cF,\Delta)=\sup_{\Delta'\geq{}\Delta}\ElCheck(\cF,\Delta')\vee{}1$.
\end{definition}
The star number was originally introduced in the context of active
learning with binary classifiers by \citet{hanneke2015minimax}. We
work with a scale-sensitive variant introduced by
\cite{foster2020instance}, which can be thought of as a
``non-sequential'' analogue of the eluder dimension.
\begin{definition}
\label{def:star}
  Let $\cF:\Act\to\bbR$ be given, and define $\StarCheck(\cF,\Delta)$ as the length of the longest sequence
  of decisions $\act_1,\ldots,\act_s\in\Act$ such that for all $i$, there
  exists $f_i\in\cF$ such that
  \begin{equation}
    \label{eq:star}
\abs*{f_i(\act_i)}>\Delta,\quad\text{and}\quad\sum_{j\neq{}i}f^{2}_i(\act_j)\leq\Delta^{2}.
  \end{equation}
  The star number is defined as $\Star(\cF,\Delta)=\sup_{\Delta'\geq{}\Delta}\StarCheck(\cF,\Delta')\vee{}1$.
\end{definition}
It is clear from this definition that $\StarDim(\cF,\Delta)\leq\ElDim(\cF,\Delta)$. In general, the star
number can be arbitrarily small compared to the eluder dimension \citep{foster2020instance}.

The following result
shows that boundedness of star number and eluder dimension always
implies boundedness of the \CompText.
\begin{theorem}
  \label{thm:eluder_star}
  Consider any class $\cM$ with $\Rspace=\brk{0,1}$ and reference
  model $\Mbar$ (not necessarily in $\cM$) with $\fmbar\in\brk{0,1}$. Suppose the
  conclusion of \pref{prop:minimax_swap_dec} holds. Then for all 
  $\gamma\geq{}e$, we have
  \begin{equation}
    \label{eq:comp_star}
    \compSq(\cM,\Mbar) \leq{} \bigoh(1)\cdot\inf_{\Delta>0}\crl*{\Delta
      + \frac{\sup_{M\in\cM}\StarDim^{2}(\cFm-\fm,\Delta)\log^{2}(\gamma)}{\gamma}}.
  \end{equation}
  and
  \begin{equation}
    \label{eq:comp_eluder}
    \compSq(\cM,\Mbar) \leq{} \bigoh(1)\cdot\inf_{\Delta>0}\crl*{\Delta
      + \frac{\sup_{M\in\cM}\El(\cFm-\fm,\Delta)\log^{2}(\gamma)}{\gamma}}.
  \end{equation}
\end{theorem}
An upper bound in terms of the star number immediately
implies an upper bound in terms of the eluder dimension, but we present
separate bounds for each parameter because Eq. \pref{eq:comp_star} has
quadratic dependence on the star number, while Eq. \pref{eq:comp_eluder}
has linear dependence on the eluder dimension. Abbreviating
$\ElDim\equiv\sup_{f\in\cFm}\ElDim(\cFm-f,1/T)$ and $\StarDim\equiv\sup_{f\in\cFm}\Star(\cFm-f,1/T)$, this result recovers the
$\bigoht(\sqrt{\ElDim{}T})$ regret bound derived in
\cite{russo2013eluder} as a special case by appealing to
\pref{thm:upper_main_bayes}, but also implies a regret bound of the
form $\bigoht(\sqrt{\StarDim^2{}T})$, which can be arbitrarily tighter.

Compared to the earlier examples in this section,
one should not expect to derive a matching lower bound on the
\CompText. Unlike the
\CompShort itself, the star number and eluder dimension only provide
sufficient conditions for sample-efficient learning, and neither parameter plays a fundamental role in
determining the minimax regret. For
example, both parameters are exponential in the
dimension for bandit
convex optimization \citep{li2021eluder}, while
\pref{prop:bandit_upper_convex} shows that the \CompShort is polynomial.

\pref{thm:eluder_star} is a corollary of a more general result that 
provides a \emph{prior-dependent} upper bound on the dual Bayesian
\CompText when posterior sampling is applied. For a given prior
$\mu\in\Delta(\cM)$ and distance $\Dgen{\cdot}{\cdot}$, define
\begin{equation}
  \label{eq:comp_dual_prior}
  \compgendual(\mu,\Mbar) =
  \inf_{p\in\Delta(\Act)}\En_{M\sim\mu}\En_{\act\sim{}p}\biggl[\fm(\pim)-\fm(\pi)
    -\gamma\cdot\Dgen{M(\act)}{\Mbar(\act)}
    \biggr],
  \end{equation}
  so that
  $\compgendual(\cM,\Mbar)=\sup_{\mu\in\Delta(\cM)}\compgendual(\mu,\Mbar)$. Our result is stated in
terms of a parameter called the (scale-sensitive) disagreement
coefficient, introduced in \cite{foster2020instance}.
\begin{definition}
\label{def:disagreement}
  For a distribution $\rho\in\Delta(\Act)$, the disagreement coefficient is
  defined as
  \begin{equation}
  \label{eq:disagreement}
  \sdis(\cF,\Delta_0,\veps_0;\rho) =
  \sup_{\Delta\geq\Delta_0,\veps\geq{}\veps_0}\crl*{\frac{\Delta^{2}}{\veps^{2}}\cdot
    \bbP_{\act\sim\rho}\prn*{
      \exists f\in\cF: \abs*{f(\act)}>\Delta, \En_{\act\sim\rho}\brk*{f^{2}(\act)}\leq\veps^{2}
    }
    }\vee{}1.
  \end{equation}
\end{definition}
Our main result is as follows.
\begin{theorem}
  \label{thm:disagreement}
  For any class $\cF\subseteq(\Act\to\brk{0,1})$, prior
  $\mu\in\Delta(\cM)$, and reference model $\Mbar$ (not necessarily in
  $\cM)$ with $\fmbar\in\brk{0,1}$, the posterior sampling
  strategy, which plays $\rho_{\mu}(\act)\ldef{}\mu(\crl{\pim=\act})$, certifies that
    \begin{equation}
    \label{eq:comp_disagreement}
    \compSqdual(\mu,\Mbar)
    \leq{} \inf_{\Delta>0}\crl*{
      2\Delta +
      24\frac{\sup_{M\in\cM}\sdis(\cFm-\fm,\Delta,\gamma^{-1};\rho_{\mu})\log^2(\gamma{})}{\gamma}}
  \end{equation}
  for all $\gamma\geq{}e$.
\end{theorem}
\pref{thm:disagreement} generalizes the decoupling argument used to prove the
upper bound on the \CompText for multi-armed bandits in
\pref{sec:examples}, with the disagreement coefficient providing a
bound on the price of decoupling. \pref{thm:eluder_star} follows immediately from
this theorem, along with the following technical result
from \cite{foster2020instance}.
  \begin{lemma}[\cite{foster2020instance}]
    \label{lem:disagreement_to_ratio}
For all $\rho\in\Delta(\Act)$ and $\Delta,\veps>0$, we have
$\sdis(\cF,\Delta,\veps;\rho)\leq{}4(\Star(\cF,\Delta))^{2}$ and
$\sdis(\cF,\Delta,\veps;\rho)\leq{}4\El(\cF,\Delta)$.
  \end{lemma}

\section{Application to Reinforcement Learning}
\label{sec:rl}

We now turn our focus to episodic reinforcement
learning with function approximation (\pref{ex:rl}), providing upper
and lower bounds on regret. Our main results are as follows.
\begin{enumerate}
\item First, in \pref{sec:rl_bilinear_basic}, we show how to extend
  the techniques in \pref{sec:examples} (posterior sampling and the \pcigw method)
  to bound the \CompShort and obtain regret bounds for \emph{bilinear classes}
  \citep{du2021bilinear}, a large class of reinforcement learning
  problems which captures many settings where sample-efficient
  reinforcement learning is possible. Our results here are applicable
  when the class of models has moderate model estimation complexity.
  \item Next, in \pref{sec:rl_bilinear_refined}, we provide tighter guarantees for
  bilinear classes that scale only with the estimation complexity for the
  underlying class of value functions. This result recovers a broader
  set of sample-efficient learning guarantees, but is somewhat more
  specialized to the bilinear class framework.
\item Finally, in \pref{sec:rl_lower}, we derive \emph{lower bounds} for reinforcement learning. As a highlight, we show that the \CompText recovers
  exponential lower bounds for reinforcement learning with linearly
  realizable function approximation \citep{wang2021exponential}.
\end{enumerate}
See \pref{table:frameworks} for a summary of the relationship
  between the \CompShort and other complexity measures in RL; precise
  results concerning Bellman-Eluder dimension are deferred to \cref{sec:bedim}.

All of the results in this section take $\Act = \PiGen$ and
$\cR\subseteq\brk{0,1}$ (that is, $\sum_{h=1}^{H}r_h\in\brk{0,1}$)
unless otherwise specified \citep{jiang2018open,zhang2021reinforcement}.

\paragraph{Reinforcement learning: Model-based, model-free, and beyond}
Recall that for reinforcement learning, each model
$M\in\cM$ consists of a collection of probability transition functions
$\Pm_{1},\ldots,\Pm_{H}$ and reward distributions
$\Rm_1,\ldots,\Rm_{H}$. This formulation can be viewed as an instance
of \emph{model-based} reinforcement learning, where one uses function
approximation to directly model the dynamics of the
environment. What is perhaps less obvious is that this formulation
also suffices to capture model-free methods and direct policy search
methods.%
\begin{itemize}
\item For model-free (or, value function approximation) methods, one typically assumes that we are given a class of
  $Q$-value functions $\cQ=\cQ_1\times\cdots\times\cQ_H$ that
  is \emph{realizable} in the sense that it contains the optimal
  $Q$-function for every problem instance under consideration. This is
  captured in the \FrameworkShort framework by taking
  \begin{equation}
    \label{eq:model_free}
    \cM_{\cQ} = \crl*{ M \mid \Qmstar_h \in\cQ_h\;\;\forall{}h\in\brk{H}}
  \end{equation}
  as the induced class of models, and hence we can derive upper
  and lower bounds for this setting. A well-known special case is that
  of \emph{linearly realizable} function approximation, where each
class $\cQ_h$ is linear; this setting is addressed in
  \pref{sec:rl_lower}. Naturally, one can modify the definition in
  \pref{eq:model_free} to incorporate commonly used additional assumptions such as
  low rank structure \citep{jin2020provably} or completeness under
  Bellman backups.
\item Direct policy search methods do not model the dynamics or value
  functions, and instead work directly with a given class of
  policies $\Pi$ (specified via function approximation). Here, a
  natural notion of realizability (e.g., \cite{mou2020sample}) is to
  assume the policy class contains the optimal policy for all problem
  instances under consideration. This is captured by taking
  \[
    \cM_{\Pi} = \crl*{M \mid{} \pim\in\Pi}
  \]
as the induced class of models. We do not focus on this setting here, as few positive results are known.
\end{itemize}
From the perspective of lower bounds, this viewpoint is without loss of
generality, though more care is required to derive tight upper bounds
(cf. \cref{sec:rl_bilinear_refined}),
since our generic results in \cref{sec:framework} scale with the model estimation complexity $\log\abs{\cM}$.

\subsection{Bilinear Classes: Basic Results}
\label{sec:rl_bilinear_basic}

  \begin{figure}[tb]
\centering
\includegraphics[width=.65\textwidth]{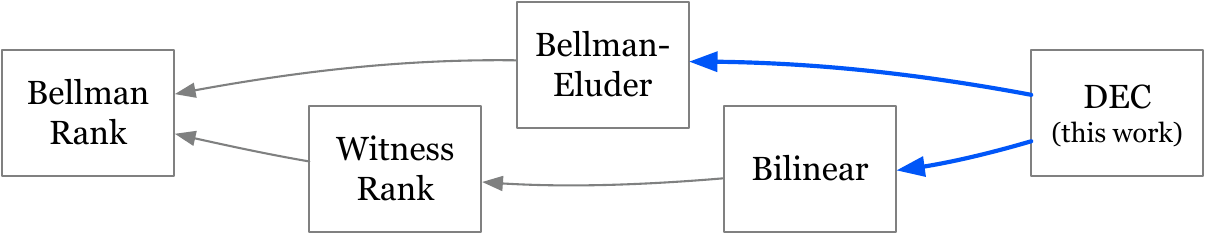}
\caption{
  Relationship between \CompText and existing frameworks for
  generalization in reinforcement learning. An arrow indicates that the
  head framework is subsumed by the tail framework.
}
\label{table:frameworks}
\end{figure}

In this section, we use the \mainalg algorithm to provide regret bounds for reinforcement learning with bilinear classes
\citep{du2021bilinear}. The bilinear class framework generalizes a number of
previous structural conditions, most notably Bellman rank \citep{jiang2017contextual}, and
captures most known settings where sample-efficient reinforcement
learning is possible. The following is an adaptation of the definition from \cite{du2021bilinear}.
\begin{definition}[Bilinear class]
  \label{def:bilinear}
  A model class $\cM$ is said to be bilinear relative to
    reference model $\Mbar$ if:
  \begin{enumerate}
  \item There exist functions $W_h(\cdot\midsem\Mbar):\cM\to\bbR^{d}$,
    $X_h(\cdot\midsem\Mbar):\cM\to\bbR^{d}$ such that for all
    $M\in\cM$ and $h\in\brk{H}$,
    \begin{equation}
      \label{eq:bilinear_residual}
\abs*{\Enm{\Mbar}{\pim}\brk*{
          \Qmstar_h(s_h, a_h) - r_h - \Vmstar_h(s_{h+1})
        }}
\leq\abs{\tri{W_h(M;\Mbar),X_h(M;\Mbar)}}.
    \end{equation}
    We assume that $W_h(\Mbar;\Mbar)=0$.
  \item Let $z_h = (s_h, a_h, r_h, s_{h+1})$. There exists a collection of estimation policies
    $\crl*{\piestm}_{M\in\cM}$ and estimation functions $\crl{\lestm(\cdot;\cdot)}_{M\in\cM}$
    such that for all $M, M'\in\cM$ and $h\in\brk{H}$,
    \begin{equation}
      \label{eq:bilinear_tv}
\tri{X_h(M;\Mbar), W_h(M';\Mbar)}
      = \Enm{\Mbar}{\pi^{\vphantom{\mathrm{est}}}_{M}\circ_{h}\piest_{M}}\brk*{
        \lestm(M';z_h)
        }.
      \end{equation}
      If $\piestm=\pim$, we say that estimation is on-policy.
    \end{enumerate}
    If $\cM$ is bilinear relative to all $\Mbar\in\cM$, we say that
$\cM$ is a bilinear class. We let $\dimbi(\cM,\Mbar)$ denote the minimal dimension $d$ for which the
    bilinear class property holds relative to $\Mbar$, and define
    $\dimbifull=\sup_{\Mbar\in\cM}\dimbi(\cM,\Mbar)$. We let
    $\Lbi(\cM;\Mbar)\geq{}1$ denote any almost sure upper bound on
    $\abs{\lestm(M';z_h)}$ under $\Mbar$, and let $\Lbifull=\sup_{\Mbar\in\cM}\Lbi(\cM;\Mbar)$.
\end{definition}
For the remainder of this section, we define concatenated factorizations
$X(M;\Mbar),W(M;\Mbar)\in\bbR^{dH}$ via
\begin{equation}
  \label{eq:bilinear_full}
X(M;\Mbar) = (X_1(M;\Mbar), \ldots,X_H(M;\Mbar)),\mathand W(M;\Mbar) = (W_1(M;\Mbar), \ldots,W_H(M;\Mbar)).
\end{equation}
  Basic examples of bilinear classes (cf. \cite{du2021bilinear}) include:
\begin{multicols}{2}
  \begin{itemize}
  \item Linear MDPs \citep{yang2019sample,jin2020provably}.
  \item Block MDPs and reactive POMDPs
  \citep{krishnamurthy2016pac,du2019latent}.
  \item FLAMBE/feature selection in low rank MDPs \citep{agarwal2020flambe}.
  \item MDPs with Linear $Q^{\star}$ and $V^{\star}$ \citep{du2021bilinear}.
  \item MDPS with Low Occupancy Complexity \citep{du2021bilinear}.
  \item Linear mixture MDPs
  \citep{modi2020sample,ayoub2020model}.
  \item Linear dynamical systems (LQR) \citep{dean2020sample}.
  \end{itemize}
\end{multicols}
Further examples include $Q^{\star}$-irrelevant state aggregation
  \citep{li2009unifying,dong2019provably} and classes with low
  Bellman rank \citep{jiang2017contextual} or Witness rank
  \citep{sun2019model}.

The results in this subsection are applicable to any bilinear class
$\cM$ for which
the model
estimation complexity $\MComp$ is
non-trivial. Guarantees under more general conditions are given in
\pref{sec:rl_bilinear_refined}.
\subsubsection{Bounding the \CompText: Posterior Sampling and \pcigw}
\label{sec:rl_bilinear_dec}

We now show how to bound the \CompText for reinforcement learning with
bilinear classes. Our development here parallels that of the tabular setting
in \pref{sec:tabular}. We first provide a bound on the Bayesian
\CompShort via posterior sampling, then provide an efficient algorithm
that leads to a bound on the frequentist \CompShort by adapting the inverse
gap weighting technique.

\paragraph{Bounding the \CompShort with Posterior Sampling}
For $\alpha\in[0,1]$, let $\pialpham$ be the randomized policy that---for each $h$---plays $\pi_{\sss{M,h}}$ with probability
$1-\alpha/H$ and $\piest_{\sss{M,h}}$ with probability
$\alpha/H$. Our guarantee for (modified) posterior sampling is as follows.\looseness=-1
\begin{restatable}{theorem}{posteriorbilinear}
  \label{thm:posterior_bilinear}
  Let $\cM$ be a bilinear class and let $\Mbar$ be an arbitrarily reference model. Let
  $\mu\in\Delta(\cM)$ be given, and consider the modified posterior
  sampling strategy that samples $M\sim\mu$ and plays $\pialpham$,
  where $\alpha\in[0,1]$ is a parameter.
  \begin{itemize}
    \item If $\piestm = \pim$ (i.e., estimation is on-policy), this strategy with $\alpha=0$ certifies that
    \[
      \compdual(\cM,\Mbar) \leq{}
      \frac{\jqedit{4}H^{2}\Lbifulls{}\jqedit{\dimbi(\cM)}}{\gamma}
    \]
    for all $\gamma>0$.
  \item For general estimation policies, this strategy
    with $\alpha=\prn[\big]{\frac{\jqedit{4}H^{3}\Lbifulls{}\jqedit{\dimbi(\cM)}}{\gamma}}^{1/2}$ certifies that
    \[
      \compdual(\cM,\Mbar) \leq{}
      \prn*{\frac{\jqedit{16}H^{3}\Lbifulls{}\jqedit{\dimbi(\cM)}}{\gamma}}^{1/2}
    \]
     whenever $\gamma\geq{} \jqedit{16}H^{3}\Lbifulls{}\jqedit{\dimbi(\cM)}$.
  \end{itemize}
\end{restatable}
Focusing on dimension, the \CompShort bound in the on-policy
case where $\piestm=\pim$ scales as $\frac{\dimbifull}{\gamma}$, which leads to regret
\[
\RegDM \approxleq \sqrt{\dimbifull{}\cdot{}T\cdot{}\EstHel}
\]
for any online estimation algorithm; recall that for finite
 classes, one can take $\EstHel\approxleq\log\abs{\cM}$.
In the general case, the \CompShort bound scales as
$\sqrt{\frac{\dimbifull}{\gamma}}$ due to the use of forced exploration, which leads to
\[
\RegDM \approxleq{} (\dimbifull\cdot\EstH)^{1/3}\cdot{}T^{2/3}.
\]
In terms of $T$ dependence, this matches the regret bound implied by the results in
\cite{du2021bilinear}, which also rely on forced
exploration; algorithms with $\sqrt{T}$-regret for general bilinear
classes are not currently known. In both of these examples, the final regret bound depends on the
estimation complexity for the class of models, which is not required
information-theoretically
\citep{jin2020provably,agarwal2020flambe}. This is an instance of one
of the main gaps between our generic upper and lower bounds
discussed in \pref{sec:main_discussion}. We give a specialized
approach to remove this
issue in the sequel.%

See \pref{app:rl_extensions} for an
extension of \cref{thm:posterior_bilinear} which covers the closely related setting of
Bellman-eluder dimension \citep{jin2021bellman}.

\paragraph{Bounding the \CompShort with \pcigw}

Our frequentist algorithm, \pcigwb (\pref{alg:igw_bilinear}), is an adaptation
of the \pcigw algorithm used in the tabular setting. The algorithm is
based on the primitive of \emph{G-optimal design}, which we use to
generalize the notion of policy cover (i.e., a collection of policies that maximizes
the visitation probability for any given state-action pair) used in \pref{alg:policy_cover_igw}.
\begin{definition}[G-optimal design]
  Let a set $\cX\subseteq\bbR^{d}$ be given. A distribution
  $p\in\Delta(\cX)$ is said to be a G-optimal design with
  approximation factor $\Copt\geq{}1$ if 
  \begin{equation}
    \label{eq:optimal_design}
    \sup_{x\in\cX}\tri*{\Sigma_p^{\pinv}x,x}\leq{}\Copt\cdot{}d,
  \end{equation}
  where $\Sigma_p\ldef\En_{x\sim{}p}\brk*{xx^{\trn}}$.
\end{definition}
The following result guarantees existence of an exact optimal design
with $\Copt=1$; we consider designs with $\Copt>1$ for computational reasons.
  \begin{fact}[\cite{kiefer1960equivalence}]
    For any compact $\cX\subseteq\bbR^{d}$, there exists an optimal design
    with $\Copt=1$.
  \end{fact}

The \pcigwb algorithm (\pref{alg:igw_bilinear}) combines inverse gap
weighting with optimal design. Let $\Mbar\in\cM$ be the estimated
model. For each layer $h$, the algorithm
computes an (approximate) G-optimal design for the collection of
vectors $\crl{Y_h(M;\Mbar)}_{M\in\cM}$, where 
\[
  Y_h(M;\Mbar) \ldef{} \frac{X_h(M;\Mbar)}{\sqrt{1+\eta(\fmbar(\pimbar)-\fmbar(\pim))}.
  }
\]
We denote resulting optimal design by
$\qopt_h\in\Delta(\cM)$. Analogous to the tabular setting, the optimal design for the unweighted
factors $\crl*{X_h(M;\Mbar)}_{M\in\cM}$ would suffice to ensure good
exploration, but the optimal design for the weighted factors
$\crl*{Y_h(M;\Mbar)}_{M\in\cM}$ balances exploration and regret, and
is critical for deriving $\sqrt{T}$-type regret bounds.

For the next
step, we mix the optimal designs for
each layer via
$q\ldef{}\frac{1}{2H}\sum_{h=1}^{H}\qopt_h +
\frac{1}{2}\delta_{\Mbar}$; we also mix in the estimated model
$\Mbar$. Finally, we compute a distribution over
policies via inverse gap weighting:
\[
  p(\pialpham) = \frac{q(M)}{\lambda + \eta(\fmbar(\pimbar)-\fmbar(\pim))}.
\]
We use the mixed policies $\pialpham$ to allow for a small amount of
forced exploration (controlled by the parameter $\alpha$) in the off-policy case where $\piestm\neq\pim$.

\pref{alg:igw_bilinear} is efficient whenever i) we can compute an approximate
optimal design for the factors $\crl*{Y_h(M;\Mbar)}_{M\in\cM}$
efficiently, and ii) the design has small support. The final
guarantee for the algorithm scales linearly with the approximation
factor $\Copt$. In \cref{sec:rl_bilinear_computational}, we show that both desiderata can be achieved
(with $\Copt=\bigoh(d)$) whenever the learner has access to a certain
planning oracle, leading to an efficient algorithm.

The main guarantee for the \pcigwb algorithm is as follows.

  \begin{restatable}{theorem}{igwbilinear}
    \label{thm:igw_bilinear}
      Let $\cM$ be a bilinear class. Let $\gamma>0$ and $\Mbar\in\cM$
      be given, and consider the \pcigwb strategy in
      \pref{alg:igw_bilinear}. Suppose the optimal design solver in
      \pref{line:approximate_optimal_design} has approximation factor
      $\Copt\geq{}1$.
  \begin{itemize}
    \item If $\piestm = \pim$ (i.e., estimation is on-policy), this
      strategy with $\eta=\frac{\gamma}{3H^3\Copt\Lbifulls{}\dimbi(\cM,\Mbar)}$ and $\alpha=0$ certifies that
    \[
      \comp(\cM,\Mbar) \leq{}
      \frac{9H^{3}\Copt\Lbifulls{}\dimbi(\cM,\Mbar)}{\gamma}.
    \]
  \item For general estimation policies, choosing $\eta=\frac{\gamma}{6H^4\Copt\Lbifulls{}\dimbi(\cM,\Mbar)}$ and $\alpha=\prn*{\frac{18H^{4}\Copt\Lbifulls{}\dimbi(\cM,\Mbar)}{\gamma}}^{1/2}$ certifies that
    \[
      \comp(\cM,\Mbar) \leq{}
      \prn*{\frac{72H^{4}\Copt\Lbifulls{}\dimbi(\cM,\Mbar)}{\gamma}}^{1/2},
    \]
     whenever $\gamma\geq{} 72H^{4}\Copt\Lbifulls{}\dimbi(\cM,\Mbar)$.
  \end{itemize}
\end{restatable}
This guarantee matches the bound on the \CompShort for posterior
sampling (\pref{thm:posterior_bilinear}) up to a factor of $H$, though
it requires that $\Mbar\in\cM$. We refer to
\cref{sec:igw_bilinear} for an efficient implementation,
as well as examples of online estimators that satisfy the conditions
of the theorem.
  \begin{algorithm}[t]
    \setstretch{1.3}
    \begin{algorithmic}[1]
      \Statex[0] \algcomment{Exploration for bilinear classes via
        inverse gap weighting.}
       \State \textbf{parameters}:
       \Statex[1] Model class $\cM$ and reference model $\Mbar\in\cM$ with bilinear dimension $d$.
       \Statex[1] Learning rate $\eta>0$.
       \Statex[1] Forced exploration parameter $\alpha>0$.
       \State \multiline{For each $M\in\cM$, let $\pialpham$ be the randomized
       policy that for each $h$ plays $\pi_{\sss{M,h}}$ with probability
       $1-\alpha/H$ and $\piest_{\sss{M,h}}$ with probability
       $\alpha/H$.}
       \For{$t=1, 2, \cdots, T$}
       \State For each $M\in\cM$ and $h\in\brk*{H}$, define
       \begin{equation}
         \label{eq:bilinear_reweighted}
         Y_h(M;\Mbar) =     \frac{X_h(M;\Mbar)}{\sqrt{1+\eta(\fmbar(\pimbar)-\fmbar(\pim))}
         }.
       \end{equation}
       \Statex[1] \algcommentlight{The distribution $\qopt_h$ is assumed to
         solve \pref{eq:optimal_design} with approximation factor $\Copt$.}
       \State For each $h$, obtain $\qopt_h\in\Delta(\cM)$ from optimal
       design solver (e.g., \pref{alg:igw_spanner}) for $\crl*{Y_h(M;\Mbar)}_{M\in\cM}$.\label{line:approximate_optimal_design}
       \State Define $q = \frac{1}{2}\sum_{h=1}^{H}\qopt_h + \frac{1}{2}\delta_{\sss{\Mbar}}$.
       \State \multiline{For each $M\in\supp(q)$, let
         \[
           p(\pialpham) = \frac{q(M)}{\lambda + \eta(\fmbar(\pimbar)-\fmbar(\pim))},
         \]
         where $\lambda\in[1/2,1]$ is chosen such that
         $\sum_{M\in\cM}p(M)=1$.}
     \State \textbf{return} $p$.
\EndFor
\end{algorithmic}
\caption{\pcigwb}
\label{alg:igw_bilinear}
\end{algorithm}

  \subsection{Bilinear Classes: Refined Guarantees}
  \label{sec:rl_bilinear_refined}

  We now use the \mainalg framework to give refined
  regret bounds that sharpen our results for bilinear classes in the
  prequel. Instead of scaling with the estimation error for the
model class $\cM$, our results here scale only with the estimation complexity for the class of
induced $Q$-functions
\[
  \cQm = \crl*{\Qmstar\mid{}M\in\cM},
\]
which makes them better suited to model-free reinforcement learning settings. 

Our refined results require the following
modifications to the basic bilinear class framework:
\begin{itemize}
\item  We assume that for all $M\in\cM$, we have
  $X_h(M;\Mbar)=X_h(\pi_M;\Mbar)$.
\item We assume that $\piestm=\pim$ for all $M\in\cM$ (i.e., estimation is on-policy),
  and replace Property 2 of \pref{def:bilinear} with the
following condition: For all $M, M'\in\cM$, $h\in\brk{H}$
    \begin{equation}
      \label{eq:bilinear_complete}
\tri{X_h(M;\Mbar), W_h(M';\Mbar)}^2
\leq \Lbifulls\cdot{}\Enm{\Mbar}{\pim}\brk*{
        \prn[\Big]{\brk{\cTm[M']_h\Vmstar[M']_{h+1}}(s_h,a_h)
        - \brk{\cTm[\Mbar]_h\Vmstar[M']_{h+1}}(s_h,a_h)}^2
        }.
  \end{equation}
\end{itemize}
Finally, we make a Bellman completeness assumption. For a given MDP $M$, let $\cTm_h$ denote the Bellman operator for layer
$h$, defined via $\cTm_h\brk*{f}(s,a) = \En\sups{M}\brk*{r_h +
  \max_{a'}f(s_{h+1},a')\mid{}s_h=s,a_h=a}$. In addition, let $\cQmh\ldef\crl*{\Qmstar_h\mid{}M\in\cM}$.
\begin{assumption}[Completeness]
  \label{ass:completeness}
For all $h$ and $M\in\cM$, $\brk{\cTmstar_h\Vmstar_{h+1}}\in\cQmh$.
\end{assumption}
The first two properties above are satisfied in most standard
model-free settings, including Linear and Linear Bellman Complete
MDPs, Low Occupancy Complexity MDPs, and Linear
$Q^{\star}$/$V^{\star}$ MDPs. \pref{ass:completeness} is satisfied for Linear and Linear Bellman
Complete MDPs, but in general may or may not be satisfied for the
problem under consideration. One can extend the results we provide
here to the off-policy estimation case and the case where \pref{ass:completeness}
does not hold using similar arguments, at the cost of a more
complicated analysis and worse dependence on $T$. We do not pursue
this here, since our goal is only to give a taste for how the \mainalg
framework can recover existing results.

We derive tighter guarantees for bilinear classes by replacing the Hellinger divergence found in the
definition of $\comp(\cM,\Mbar)$ with the following divergence
tailored to the bilinear class setting:
\begin{equation}
  \label{eq:divergence_bilinear}
  \Dbi{M(\act)}{\Mbar(\act)}
  \ldef \sum_{h=1}^{H}\tri*{X_h(\pi;\Mbar),W_h(M;\Mbar)}^{2}.
\end{equation}
This quantity is always upper bounded by Hellinger distance,
but leads to tighter rates because it only depends on the models
under consideration through their Bellman residuals.

Following the development for general divergences described
in \pref{sec:general_distance}, we consider a variant of \mainalgB
(\pref{alg:bayes_basic}) tailored to the bilinear divergence $\Dbi{\cdot}{\cdot}$. For each
time $t$:
\begin{itemize}
\item Compute posterior $\mu\ind{t}\in\Delta(\cM)$ given $\hist\ind{t-1}$.
\item Solve
  \begin{equation}
    \label{eq:comp_argmin_bilinear}
    \argmin_{p\in\Delta(\Pi)}\sup_{M\in\cM}\En_{\pi\sim{}p}\brk*{
      \fm(\pim) -\fm(\pi)
      -\gamma\cdot\En_{\Mbar\sim\mu\ind{t}}\brk*{\Dbi{M(\act)}{\Mbar(\act)}}
      },
    \end{equation}
    which corresponds to the optimization problem defining the generalized \CompText
    $\compgen[\Dbishort](\cM,\mu\ind{t})$ in \pref{eq:comp_general_randomized}.
  \end{itemize}
  We abbreviate $\compbi(\cM,\nu)\equiv\compgen[\Dbishort](\cM,\nu)$ going
  forward. The following result---which is proven using the same
  approach as \pref{thm:posterior_bilinear}---shows that by moving to
  $\Dbi{\cdot}{\cdot}$, the \CompShort is bounded by the bilinear dimension.
\begin{restatable}{proposition}{decbilinear}
  \label{prop:dec_bilinear}
  Let $\cM$ be a bilinear class. Then for all $\gamma>0$ and
  $\nu\in\Delta(\cM)$,
  \[
    \compbidual(\cM,\nu) \leq{} \frac{H\cdot{}\dimbi(\cM)}{\gamma}.
  \]
\end{restatable}

Equipped with a bound on the \CompText, we proceed to give a regret
bound for the \mainalg variant above. Our regret bound depends on the following notion of
estimation complexity for the value function class $\cQm$.
  \begin{definition}
    Models $M_1,\ldots,M_N$ are said to be an $\veps$-cover for $\cQm$ if for
    all $M\in\cM$, there exists $i\in\brk{N}$ such that
    \[
      \max_{h\in\brk*{H}}\sup_{s\in\cS,a\in\cA}\abs*{\Qmstar[M_i]_h(s,a) - \Qmstar_h(s,a)}\leq{}\veps.
    \]
    Let $\cN(\cQm,\veps)$ denote the size of the smallest such
    cover, and define
    $\QComp=\inf_{\veps>0}\crl*{\log\cN(\cQm,\veps) + \veps^{2}T}$.
  \end{definition}
Note that for any setting in which the optimal $Q$-functions are
linear functions in dimension $d$ (e.g.,
Linear and Linear Bellman Complete MDPs, Linear $Q^{\star}/V^{\star}$),
we have $\Qcomp=\bigoht(dH)$.
  
Our refined regret bound for bilinear classes is as follows.  
  \begin{restatable}{theorem}{bilinearrefined}
  \label{thm:bilinear_refined}
  Suppose that \pref{ass:completeness} holds, and that estimation is
  on-policy (i.e., $\piestm=\pim$). Consider the \mainalgB algorithm with
  the optimization problem given in \pref{eq:comp_argmin_bilinear}.
  This algorithm guarantees that for any prior $\mu\in\Delta(\cM)$ and $\gamma>0$, 
  \[
    \En_{\Mstar\sim\mu}\En\sups{\Mstar}\brk*{\RegDM}
    \leq{} \bigoht\prn*{
      \sqrt{H^2\Lbifulls\cdot{}\dimbifulls{}\cdot{}T\cdot{}\Qcomp}
      }.
    \]
      Consequently, we have $\MinimaxReg\leq \bigoht\prn[\big]{
      \sqrt{H^2\Lbifulls\cdot{}\dimbifulls{}\cdot{}T\cdot{}\Qcomp}
      }$.\looseness=-1
\end{restatable}
As a concrete example, this leads to $\sqrt{\poly(d,H)\cdot{}T}$
regret for Linear and Linear Bellman Complete MDPs. The main idea
behind the proof is as follows. Using the definition of the
\CompShort, we show that
\[
  \En\brk*{\RegDM}
  \approxleq{} \sup_{\nu\in\Delta(\cM)}\compbidual(\cM,\nu)\cdot{}T
  + \gamma\cdot\sum_{t=1}^{T}\En\brk*{\Dbi{\Mstar(\pi\ind{t})}{\Mhat\ind{t}(\pi\ind{t})}}.
\]
The first term above is bounded by \pref{prop:dec_bilinear}, and we
bound the estimation error term involving $\Dbi{\cdot}{\cdot}$ using
an argument based on the elliptic potential \citep{du2021bilinear,jin2021bellman}.

\paragraph{Discussion}
We emphasize that the results here are specialized to bilinear classes, and require
modifications to the basic framework in \pref{sec:framework}. Deriving
tighter regret bounds as a direct corollary of our general results is
an important topic for further research. To this end, we refer
  the reader to the follow-up work of \citet{foster2022note}, which
  provides a general approach to deriving model-free reinforcement
  learning guarantees through the \CompShort framework; see also
  discussion in \pref{sec:related}.

\subsection{Lower Bounds}
\label{sec:rl_lower}

We close this section by complementing our upper bounds on the
\CompText with lower bounds. We first give a basic lower bound for
bilinear classes, then give lower bounds for learning with linearly
realizable $Q$-functions.

\subsubsection{Lower Bound for Bilinear Classes}
The upper bounds for bilinear classes in
\pref{thm:posterior_bilinear,thm:igw_bilinear} scale with
$\frac{d}{\gamma}$, where $d$ is the bilinear dimension. The following
result shows that this dependence is unavoidable.
  \begin{proposition}[Lower bound for linear MDPs]
    \label{prop:mdp_lower_linear}
    For all $d\geq{}4$ there exists a bilinear class $\cM$ with
    dimension $d$, horizon $H=1$, $\cR=\brk*{-1,+1}$, and
    $\nrm{X(M;\Mbar)}_2,\nrm{W(M;\Mbar)}_2\leq{}1$ for all
    $M,\Mbar\in\cM$, such that for all $\gamma\geq\frac{d}{3}$,
  \[
    \comp(\cMinf[\vepsg](\Mbar),\Mbar) \geq{} \frac{d}{12\gamma}
  \]
for some $\Mbar\in\cM$, where $\vepsg = \frac{d}{3\gamma}$.
\end{proposition}
\begin{proof}[\pfref{prop:mdp_lower_linear}]
  Immediate consequence of \pref{prop:bandit_lower_linear}.
\end{proof}

Invoking this result within \pref{thm:lower_main_expectation} leads to a lower
bound on regret of the form $\En\brk*{\RegDM}\geq\bigom(\sqrt{dT})$, which
matches the upper bound in \pref{sec:rl_bilinear_basic} in terms of
dependence on $d$ and $T$.

\subsubsection{Lower Bound for Linearly Realizable MDPs}
Learning with \emph{linearly realizable} function approximation is a
well-studied problem in reinforcement learning
\citep{weisz2021exponential,wang2021exponential}. Here, we are given a
feature map $\phi(s,a)\in\bbR^{d}$ satisfying $\nrm*{\phi(s,a)}_2\leq{}1$, and consider a class of value functions $\cQ=\cQ_1\times\cdots\times\cQ_H$
given by
\begin{equation}
  \label{eq:linearly realizable}
  \cQ_h = \crl*{(s,a)\mapsto\tri*{\theta,\phi(s,a)}\mid{}\theta\in\bbR^{d},\nrm*{\theta}_2\leq{}1}.
\end{equation}
We assume that $\cQ$ is realizable in the sense that it
contains the optimal $Q$-function for all problem instances under
consideration, which corresponds to choosing $\cM$ to be
\[
\cM_{\cQ} = \crl*{ M \mid{} \Qmstar_h\in\cQ_h\;\forall{}h}.
\]
\cite{weisz2021exponential} provide an exponential lower bound which
establishes that sample-efficient reinforcement learning is not possible in
this setting, and \cite{wang2021exponential} show that this lower
bound continues to hold even when the instances under consideration
have constant suboptimality gap. The following result shows that the
\CompText is exponential for this setting, thereby recovering these results.
\begin{restatable}[Lower bound for linearly realizable MDPs]{proposition}{linearlyrealizable}
  \label{prop:linear_qstar}
  For all $d\geq{}2^{9}$ and $H\geq{}2$, there exists a family of
  linearly realizable MDPs $\cM$ with $\cR=\brk{-1,+1}$ such that for
  all $\gamma>0$, there exists $\Mbar\in\cM$ for which
  \begin{equation}
    \label{eq:linear_qstar}
    \comp(\cM_{1/2}(\Mbar),\Mbar) \geq{} \frac{1}{48}\indic\crl*{\gamma\leq{}2^{-9}\min\crl{2^{H},\exp(2^{-10}d)}}.
  \end{equation}
\end{restatable}
Using \pref{thm:lower_main} with $\gamma\propto{}T\log(T)$, we
conclude that any algorithm for the linearly realizable setting must
have
\[
  \RegDM \geq{} \bigomt\prn*{
    \min\crl*{2^{H}, 2^{\bigom(d)}, T}
    }
\]
with constant probability, which in turn implies that $\En\brk*{\RegDM} \geq{} \bigomt\prn*{
    \min\crl*{2^{H}, 2^{\bigom(d)}, T}
    }$.

\subsubsection{Lower Bound for Linearly Realizable MDPs with
  Deterministic Dynamics and Gap}
While \pref{prop:linear_qstar} shows that linear realizability is not
sufficient for sample-efficient reinforcement learning, it is known
that linear realizability \emph{does} suffice if one restricts to
deterministic dynamics and rewards. In particular,
\cite{wen2017efficient} provide an algorithm for this setting that has
\[
\RegDM \leq \bigoh(dH)
\]
whenever the feature dimension is $d$ and $\sum_{h=1}^{H}r_h\in\brk{0,1}$.\footnote{See \cite{du2019provably,wang2021exponential} for
  generalizations to various types of nearly-deterministic systems.} The following result provides a complementary lower bound.
\begin{restatable}[Lower bound for deterministic linearly realizable MDPs]{proposition}{deterministiclinear}
  \label{prop:mdp_gap_linear}
There exists a collection $\cM$ of
  linearly realizable MDPs $\cM$ with $\cR=\brk{0,1}$,
  $H=1$, constant suboptimality gap, and deterministic rewards and
  dynamics, such that for all $\gamma>0$, there exists $\Mbar\in\cM$
  such that
  \[
    \comp(\cMinf[1/3](\Mbar),\Mbar) \geq{}     \frac{1}{12}\indic\crl*{
      \gamma\leq\frac{d}{48}
    }.
  \]
\end{restatable}
By \pref{thm:lower_main}, this result implies that any algorithm
for the linearly realizable setting with deterministic dynamics and
rewards must have
\[
  \RegDM \geq{} \bigomt\prn*{
    \min\crl*{d, T}
    }
\]
with constant probability, which in turn implies that $\En\brk*{\RegDM} \geq \bigomt\prn*{
    \min\crl*{d, T}
    }$.

\section{Incorporating Contextual Information}
\label{sec:contextual}

In this section we consider a \emph{contextual} variant of the
\Framework framework in which the learner is given additional
side information (in the form of a covariate or context) before each
decision is made. This setting encompasses the well-studied
contextual bandit problem \citep{auer2002non,langford2008epoch,chu2011contextual,beygelzimer2011contextual,agarwal2014taming,foster2020beyond} as well, as well as various
contextual reinforcement learning problems
\citep{abbasi2014online,modi2018markov,dann2019policy,modi2020no}. We show that the \mainalg paradigm seamlessly extends
to incorporate contextual information, leading to new, efficient algorithms.

\paragraph{Setting} In the contextual \FrameworkShort framework, we
adopt the following protocol for $T$
rounds, where for each round $t=1,\ldots,T$:
\begin{enumerate}
\item Nature provides the learner with a \emph{context}
  $\con\ind{t}\in\Cspace$, where $\Cspace$ is the \emph{context space}.
\item The \learner selects a decision $\act\ind{t}\in\Act$.
  \item Nature selects a reward $r\ind{t}\in\RewardSpace$ and
    observation $\obs\ind{t}\in\ObsSpace$ based on the decision, which
    are then observed by the learner.
  \end{enumerate}
  We allow each context $x\ind{t}$ to be chosen in an arbitrary,
  potentially adaptive fashion but---following the development so
  far---assume that rewards and observations are stochastic. We work with models of the form $M(\cdot,\cdot)$, where
  $M(x,\act)$ denotes the distribution over $(r,\obs)$ when $x$ is the
  context and decision $\act$ is selected. We assume that at each
  timestep, given $x\ind{t}$, the pair $(r\ind{t}, \obs\ind{t})$ is drawn independently
  from an unknown model $\Mstar(x\ind{t}, \pi\ind{t})$. As before, we
  assume access to a class of models $\cM$ that contains the true
  model $\Mstar$.
\begin{assumption}
  \label{ass:realizability_contextual}
  The model class $\cM$ contains the true model $\Mstar$.
\end{assumption}
For each model $M\in\cM$, let
$\fm(\con,\act)\ldef{}\En_{r\sim{}M(\con,\act)}\brk*{r(\pi)}$ denote
the mean reward function and let 
$\cpolm(x)\ldef{}\argmax_{\act\in\Act}\fm(\con,\act)$ denote the
decision-making policy with
the greatest expected reward under $M$. Finally, we define
$\cFm=\crl*{\fm\mid{}M\in\cM}$ as the induced class of mean reward
functions.\footnote{We use the notation $\cpolm$ (compared to $\pim$
  in the non-contextual setting) to distinguish \emph{decisions} from
  \emph{policies that map contexts to decisions}.}

We evaluate the \learner's performance in terms of regret to the optimal
decision-making policy for $\Mstar$:
\begin{equation}
  \label{eq:regret}
  \RegCDM
  \ldef \sum_{t=1}^{T}\En_{\act\ind{t}\sim{}p\ind{t}}\brk*{\fstar(\con\ind{t},\cpolstar(\con\ind{t})) - \fstar(\con\ind{t},\act\ind{t})},
\end{equation}
where we abbreviate $\fstar=\fmstar$ and
$\cpolstar=\cpolm[\Mstar]$.

The advantage of this formulation---which generalizes similar
formulations for the contextual bandit problem
\citep{chu2011contextual,beygelzimer2011contextual,foster2020beyond}---is
that we accommodate arbitrarily generated sequences of contexts, which may correspond to, e.g., users arriving
at a website as they please. Note that the fully stochastic contextual
bandit problem, in which $x\ind{1},\ldots,x\ind{T}$ are \iid, is a
special case of the basic \FrameworkShort framework in
\pref{sec:intro} (with policies as decisions), and does not require dedicated treatment.

\subsection{The Contextual \mainalgsection Algorithm}

\begin{algorithm}[htp]
    \setstretch{1.3}
     \begin{algorithmic}[1]
       \State \textbf{parameters}:
       \Statex[1] Online estimation oracle $\AlgEst$.
       \Statex[1] Exploration parameter $\gamma>0$.
       \Statex[1] Divergence $\Dgen{\cdot}{\cdot}$.
       \For{$t=1, 2, \cdots, T$}
       \State Receive context $x\ind{t}$.
  \State Compute estimate $\Mhat\ind{t} = \AlgEst\ind{t}\prn[\Big]{
    \crl*{(x\ind{i}, \act\ind{i}, r\ind{i},\obs\ind{i})}_{i=1}^{t-1} }$.
  \State Define\vspace{-10pt}
  \[p\ind{t}=\argmin_{p\in\Delta(\Act)}\sup_{M\in\cM}\En_{\act\sim{}p}\brk*{\fm(\con\ind{t},\cpolm(\con\ind{t}))-\fm(\con\ind{t},\act)
    -\gamma\cdot\Dgen{M(\con\ind{t},\act)}{\Mhat\ind{t}(\con\ind{t},\act)}}.\]\vspace{-18pt}
\State{}Sample decision $\act\ind{t}\sim{}p\ind{t}$ and update estimation
oracle with $(\con\ind{t},\act\ind{t},r\ind{t}, \obs\ind{t})$.
\EndFor
\end{algorithmic}
\caption{Contextual \mainalg}
\label{alg:contextual}
\end{algorithm}

\pref{alg:contextual} is a 
contextual generalization of the \mainalg meta-algorithm. The algorithm has the
same structure as the basic (non-contextual) \mainalg algorithm
(\optionone), with the main difference being that we use the context
$x\ind{t}$ to form the minimax problem solved at each step. Following
\pref{sec:general_distance}, the algorithm is stated in terms of an arbitrary user-specified divergence
$\Dgen{\cdot}{\cdot}$ for added flexibility.

In more detail, at each time $t$, the algorithm first receives the
context $x\ind{t}$, then obtains an estimated model
$\Mhat\ind{t}(\cdot,\cdot)$ from the estimation oracle $\AlgEst$; here, unlike
the non-contextual setting, the oracle can make use of previous contexts to
form the estimate. Given the estimator and context, the algorithm
solves the optimization problem
\begin{equation}
  \label{eq:opt_contextual}
  p\ind{t}=\argmin_{p\in\Delta(\Act)}\sup_{M\in\cM}\En_{\act\sim{}p}\brk*{\fm(\con\ind{t},\cpolm(\con\ind{t}))-\fm(\con\ind{t},\act)
    -\gamma\cdot\Dgen{M(\con\ind{t},\act)}{\Mhat\ind{t}(\con\ind{t},\act)}},
\end{equation}
which corresponds to the non-contextual optimization problem
\pref{eq:comp_general} applied to the projected class
$\cMx[\con\ind{t}]$, where 
\[
  \cMx[\con] \ldef{} \crl*{M(\con,\cdot)\mid{}M\in\cM}.
\]
Finally, the algorithm
samples $\act\ind{t}\sim{}p\ind{t}$ and updates the estimation oracle
with the example $(x\ind{t},\act\ind{t},r\ind{t},\obs\ind{t})$. Notably, the per-round computational complexity is
exactly the same as in the non-contextual setting.

Conceptually, \pref{alg:contextual} can be interpreted as a
universal generalization of the \squarecb algorithm of
\citet{foster2020beyond} from finite-action contextual bandits to
arbitrary contextual decision making problems. Rather than using the inverse gap weighting
strategy in \squarecb, which is tailored to contextual bandits with finite actions, we simply
solve the optimization problem that defines the \CompText for the
current context, thereby accommodating any learnable contextual decision making problem.

The performance guarantee for the contextual \mainalg algorithm depends on the estimation
  performance of the oracle $\AlgEst$ with respect to the divergence
  $D$, on the \emph{observed sequence of contexts}:
    \begin{equation}
      \label{eq:general_error}
        \EstCD \ldef{} \sum_{t=1}^{T}\En_{\act\ind{t}\sim{}p\ind{t}}\brk*{\Dgen{\Mstar(\con\ind{t},\act\ind{t})}{\Mhat\ind{t}(\con\ind{t},\act\ind{t})}}.
      \end{equation}
      Let $\cMhat$ be any set such that $\Mhat\ind{t}\in\cMhat$ for
      all $t$ almost surely, and recall that in the non-contextual
      setting, $\compgen(\cM,\cMhat)\ldef\sup_{\Mbar\in\cMhat}\compgen(\cM,\Mbar)$. We have the following regret
      bound, generalizing \pref{thm:upper_general_distance}.
\begin{restatable}{theorem}{uppercontextual}
  \label{thm:upper_contextual}
\pref{alg:contextual} with exploration parameter $\gamma>0$ guarantees
that
\begin{equation}
  \label{eq:upper_contextual}
\RegCDM \leq{}
\sup_{x\in\cX}\compgen(\cMx,\cMhatx)\cdot{}T + \gamma\cdot\EstCD
\end{equation}
almost surely.
\end{restatable}
This result shows that any interactive decision making problem that
is learnable in the non-contextual setting is also learnable in the
presence of arbitrarily selected contexts, as long as estimation is
feasible.

\subsection{Application to Contextual Bandits}
The contextual bandit problem
\citep{chu2011contextual,beygelzimer2011contextual,foster2020beyond}
is the most basic special case of the contextual \FrameworkShort
setting, and corresponds to the case in which there are no auxiliary
observations (i.e., $\Ospace=\NullObs$). The contextual bandit problem with
finite actions is quite well-studied but contextual bandits with
continuous, structured action spaces are comparatively under-explored. As a consequence  of the
bounds on the \CompText from \pref{sec:bandit}, we show that \pref{alg:contextual} leads to
efficient contextual algorithms for structured action
spaces. For each of these examples, we take $\cMhat=\conv(\cM)$.
\begin{itemize}
\item \emph{Finite actions.} In the finite-action setting where $\Act=\brk{A}$, we
  have
  $\sup_{x\in\cX}\compSq(\cMx,\cMhatx)\leq\frac{A}{\gamma}$,
  and this is achieved efficiently through the inverse gap weighting
  strategy (cf. \pref{prop:igw_mab}). In this case,
  \pref{alg:contextual} reduces to the \squarecb algorithm of
  \cite{foster2020beyond}, which achieves
  \[
    \RegCDM \leq{} \bigoh\prn[\big]{\sqrt{AT\cdot\EstSq}}
  \]
  after tuning $\gamma$.  For finite classes where $\abs{\cFm}<\infty$ and $\Rspace=\brk{0,1}$, this
  leads to $\RegCDM \leq \bigoh\prn[\big]{\sqrt{AT\log\abs{\cFm}}}$
  for an appropriate choice of estimation oracle \citep{foster2020beyond}.

\item \emph{Linear action spaces.} Suppose actions are linearly structured in the sense that
  $\Act\subseteq\bbR^{d}$, and each $\fm\in\cFm$ factorizes as
  \[
    \fm(\con, \act)=\tri{\gm(\con),\act}
  \]
  for some $\gm:\Cspace\to\bbR^{d}$. Here, we can efficiently achieve
  $\sup_{x\in\cX}\compSq(\cMx,\cMhatx)\leq\frac{d}{\gamma}$
  using the strategy from \pref{prop:bandit_upper_linear}. In this
  case, \pref{alg:contextual} provides an efficient algorithm with 
  \[
    \RegCDM \leq{} \bigoht\prn[\big]{\sqrt{dT\cdot\EstSq}}
  \]
  after tuning $\gamma$. This recovers the result from \cite{foster2020adapting}.
\item \emph{Continuous actions with concave rewards.} If $\Act\subseteq\bbR^{d}$ and the function
  $\act\mapsto\fm(\con,\act)$ is concave for all $\con\in\Cspace$ and
  $M\in\cM$, \pref{prop:bandit_upper_convex} implies that
  $\sup_{x\in\cX}\compSq(\cMx,\cMhatx)\leq\bigoht(\frac{\poly(d)}{\gamma})$,
  so that \pref{alg:contextual} enjoys
  \[
    \RegCDM \leq{} \bigoht\prn[\big]{\sqrt{\smash[t]{\poly(d)}T\cdot\EstSq}}.
  \]
  This yields the first oracle-efficient algorithm for contextual bandits
  with continuous actions and concave (resp. convex) rewards, though we emphasize
  that more research is required to understand when the optimization
  problem \pref{eq:opt_contextual} can be solved efficiently for this setting.
\end{itemize}

Beyond these examples, the appeal of \pref{alg:contextual} is that it
allows one to immediately translate any future bounds on the
\CompText for structured bandit problems into oracle-efficient
algorithms for contextual bandits with structured action spaces.

\subsection{Application to Contextual Reinforcement Learning}
Contextual reinforcement learning (sometimes referred to as the contextual MDP
problem) is setting that generalizes both contextual bandits and
reinforcement learning \citep{abbasi2014online,modi2018markov,dann2019policy,modi2020no}. At each time, the learner receives a context
$x\ind{t}$, selects a policy $\act\ind{t}$, then executes the policy
in an (unknown) finite-horizon MDP and observes a trajectory. Formally, this setting is
simply a special case of the contextual \FrameworkShort framework in
which $M(x,\cdot)$ is a finite-horizon MDP for all $M\in\cM$ and $x\in\cX$.
Similar to the contextual bandit problem, the underlying MDP changes
from round based on the context, so it is essential to---via modeling and function approximation---generalize across similar contexts.

\paragraph{Example: Contextual reinforcement learning with finite states and actions}
The most basic and well-studied contextual reinforcement learning
problem is the setting where, given the context $x$, each model
$M(x,\cdot)$ is a finite state/action MDP with $\cS=\brk{S}$ and
$\cA=\brk{A}$
\citep{abbasi2014online,modi2018markov,dann2019policy,modi2020no}. Here,
we can apply the \pcigw strategy (\pref{alg:policy_cover_igw}) which,
via \pref{prop:igw_tabular}, certifies that 
\[
  \sup_{x\in\cX}\comp(\cMx,\cMhatx)\leq\bigoh\prn*{\frac{H^{3}SA}{\gamma}},
\]
as long as $\cMhatx$ is a tabular MDP for all $x\in\cX$. With this
choice, the contextual \mainalg algorithm guarantees that for any
estimation oracle $\AlgEst$, 
  \[
    \RegCDM \leq{} \bigoh\prn*{\sqrt{H^{3}SAT\cdot\EstSq}},
  \]
  after tuning $\gamma$. This result is computationally efficient---beyond updating the estimation oracle, the only computational
  overhead at each round is to run the \pcigw algorithm with the
  estimated model---and constitutes the first universal reduction from contextual
  reinforcement learning to supervised online learning.

Note that while we require that each MDP has finite states and actions, we place no
  assumption on the context space $\Cspace$, and no assumption on how
  each model $x\mapsto{}M(x,\cdot)$ varies as a function of the
  context. In particular, one can leverage rich, flexible function
  approximation (e.g., neural networks or kernels) to learn the
  mapping from contexts to MDPs by simply choosing an appropriate
  estimation oracle $\AlgEst$. In constrast, previous approaches are limited to
  linear or generalized linear function approximation.
\section{Additional Related Work and Follow-Up Work}
\label{sec:related}

We now highlight some relevant lines of research not already covered
by our discussion.

\subsection{Statistical Estimation}
Our results build on a long line of work on minimax lower bounds for
statistical estimation
\citep{lecam1973convergence,ibragimov1981statistical,assouad1983deux,birge1995estimation}. Here,
the seminal work of
\cite{donoho1987geometrizing,donoho1991geometrizingii,donoho1991geometrizingiii}
shows that for a large class of nonparametric estimation problems, the
local (and in some cases, global) minimax rates are characterized by a local
\emph{modulus of continuity} with respect to Hellinger
distance, defined via:
\begin{equation}
  \label{eq:hellinger_modulus}
  \omega_{\veps}(\cM,\Mbar)\ldef{}\sup_{M\in\cM}\crl*{
    \nrm*{f^{\sss{M}}-f^{\sss{\Mbar}}} : \Dhels{M}{\Mbar}\leq\veps^{2}
    },
  \end{equation}
  for an appropriate norm $\nrm{\cdot}$. An important special case
  concerns parametric models, where---under mild regularity
  conditions---the modulus of continuity \pref{eq:hellinger_modulus}
  asymptotically coincides with the
   Fisher information
  \citep{lecam2000asymptotics,van2000asymptotic}.
  
  One can view the \CompText \pref{eq:comp} as an interactive decision
  making analogue of the modulus of continuity. To make the connection
  more apparent, consider a regularized (as opposed to constrained)
  variant of \pref{eq:hellinger_modulus}:
  \begin{equation}
  \label{eq:hellinger_modulus2}
  \omega_{\gamma}(\cM,\Mbar)\ldef{}\sup_{M\in\cM}\crl*{
    \nrm*{f^{\sss{M}}-f^{\sss{\Mbar}}} - \gamma\cdot\Dhels{M}{\Mbar}
    }.
  \end{equation}
  Like the \CompText, the modulus of continuity
  \pref{eq:hellinger_modulus2} captures a worst-case tradeoff between
  risk and information gain relative to a reference
  model $\Mbar$, with the scale
  parameter $\gamma>0$ controlling the tradeoff. The key difference is
that because the classical estimation setting is purely passive, the
modulus of continuity \pref{eq:hellinger_modulus2} does not involve decisions made by the learner, and hence is expressed as a ``max''
rather than a ``min-max''.

  \subsection{Instance-Dependent Complexity for Structured Bandits}
An important special case of our general decision making framework is
the problem of structured bandits with large action spaces
(\pref{sec:bandit}). The pioneering work of
\citet{graves1997asymptotically} gives a characterization for the optimal asymptotic instance-dependent rates for this problem, as well
as a broader class of problems called \emph{controlled Markov
  chains}. In detail, for a model class $\cM$ and reference model
$\Mbar\in\cM$, consider the complexity measure
\begin{equation}
  \label{eq:graves_lai}
  \glcomp(\cM,\Mbar) \ldef \inf_{w\in\bbR^{\Act}_{+}}\crl*{
    \sum_{\act\in\Act}w_{\act}(\fmbar(\pimbar)-\fmbar(\act)) \mid{}
    \forall{}M\in\cC(\cM,\Mbar) : \sum_{\act\in\Act}w_{\act}\kl{\Mbar(\pi)}{M(\pi)}
    \geq{} 1
    },\tag{G\&L}
  \end{equation}
  where $\cC(\cM,\Mbar) \ldef{} \crl*{M\in\cM: \fm(\pimbar)=\fmbar(\pimbar),
    \pim\neq{}\pimbar}$ is a set of ``confusing'' alternatives.
  \cite{graves1997asymptotically} show that for every problem instance $\Mstar$, any ``uniformly consistent''
  algorithm must incur regret
  $(1-o(1))\cdot{}\glcomp(\cM,\Mstar)\log(T)$ asymptotically, and
  that regret $(1+o(1))\cdot{}\glcomp(\cM,\Mstar)\log(T)$ is asymptotically
  achievable. A
  more recent line of work attempts to achieve this fundamental limit
  non-asymptotically
  \citep{combes2017minimal,degenne2020structure,jun2020crush}, and \cite{ok2018exploration}
  extends these results to the reinforcement learning setting under
  strong ergodicity assumptions.

The Graves-Lai complexity measure \pref{eq:graves_lai} has evident
structural similarities to the \CompText, which can be made especially clear by
considering the following regularized variant:
\begin{equation}
  \label{eq:graves_lai2}
\glcomp_{\gamma}(\cM,\Mbar) \ldef{} \inf_{w\in\bbR^{\Act}_{+}}\sup_{M\in\cC(\cM,\Mbar)}\crl*{
    \sum_{\act\in\Act}w_{\act}(\fmbar(\pimbar)-\fmbar(\act)) -
    \gamma\cdot{}\prn*{\sum_{\act\in\Act}w_{\act}\kl{\Mbar(\pi)}{M(\pi)}-
    1}
    }.
  \end{equation}
  Note that $w\in\bbR^{\Act}_{+}$ may be interpreted as an
  unnormalized distribution over decisions. With this perspective, the most important difference between this complexity measure and
  the \CompText is that \pref{eq:graves_lai2} only considers regret
  under the nominal model $\Mbar$, while the \CompShort considers
  regret under a worst-case model selected by nature. We believe
  this difference is a fundamental consequence of considering minimax
  regret under finite samples rather than asymptotic
  instance-dependent regret. In particular, all existing results that
  provide finite-sample guarantees based on the Graves-Lai complexity
  measure \citep{combes2017minimal,degenne2020structure,jun2020crush}
  require strong assumptions on the problem structure (e.g.,
  finite actions) in order to
  control the error incurred by evaluating \pref{eq:graves_lai} with
  a plug-in estimator for the true model $\Mstar$. This highlights that
  \pref{eq:graves_lai} alone is not be sufficient to capture
  optimal instance-dependent guarantees with finite samples. In contrast, we
  avoid similar assumptions because the \CompShort incorporates
  uncertainty in a stronger fashion. We refer to the follow-up work of
  \citet{dong2022asymptotic,wagenmaker2023instance} for further
  discussion and background.

\subsection{Posterior Sampling and the Information Ratio}
For structured bandits in the Bayesian framework, \cite{russo2014learning,russo2018learning} introduce a parameter
measure known as the \emph{information ratio} which bounds the
Bayesian regret for posterior sampling and a related strategy called
information-directed sampling. For a given distribution
$\mu\in\Delta(\cM)$ (typically the posterior distribution at a given
round), estimator $\Mbar$, and action distribution $p$, the
information ratio is given by 
\[
\frac{\prn*{\En_{\act\sim{}p}\En_{M\sim\mu}\brk*{\fm(\act)-\fm(\pim)}}^2}{
\En_{\act\sim{}p}\En_{M\sim\mu}\brk*{\Dkl{M(\act)}{\Mbar(\act)}}}.
\]
Note that we consider the original
  definition from \cite{russo2018learning}, which uses KL
  divergence, but the concept of information ratio readily extends to other
  divergences such as Hellinger distance (a-la \pref{sec:general_distance}).

\cite{russo2014learning} show that when the distribution $p$ is chosen
via posterior
sampling (i.e., sample $M'\sim\mu$ and follow
$\pi\subs{M'}$), the information ratio is bounded by $\ActSize$ for
finite-armed bandits; similar bounds hold for linear bandits. A more
general strategy, \emph{information-directed sampling}
\citep{russo2018learning} directly optimizes the information ratio for
the given prior and estimator:
\[
p_{\mathrm{ids}} = \argmin_{p\in\Delta(\Act)}\frac{\prn*{\En_{\act\sim{}p}\En_{M\sim\mu}\brk*{\fm(\act)-\fm(\pim)}}^2}{
\En_{\act\sim{}p}\En_{M\sim\mu}\brk*{\Dkl{M(\act)}{\Mbar(\act)}}}.
\]
To relate these concepts and techniques to our own results, we
consider a natural ``worst-case'' complexity measure for the Bayesian
setting based on the information ratio:
\begin{equation}
  \label{eq:information_ratio_bayes}
\InfB(\cM,\Mbar) = \max_{\mu\in\Delta(\cM)}\min_{p\in\Delta(\Act)}\frac{\prn*{\En_{\act\sim{}p}\En_{M\sim\mu}\brk*{\fm(\act)-\fm(\pim)}}^2}{
  \En_{\act\sim{}p}\En_{M\sim\mu}\brk*{\Dkl{M(\act)}{\Mbar(\act)}}}.
\end{equation}
We show that boundedness of this parameter implies boundedness
of the KL divergence variant of the \CompText.
\begin{restatable}{proposition}{informationratiodec}
  \label{prop:information_ratio_bound}
  For all $\Mbar\in\cM$ and $\gamma>0$,
  \[
    \compKLdual(\cM,\Mbar) \leq{} \frac{\InfB(\cM,\Mbar)}{4\gamma}.
  \]
\end{restatable}
Defining
$\InfB(\cM)=\sup_{\Mbar\in\cM}\InfB(\cM,\Mbar)$, the
results of \cite{russo2018learning} imply that information-directed sampling attains
Bayesian regret
$\bigoh\prn[\big]{\sqrt{\InfB(\conv(\cM))\cdot{}T\log\ActSize}}$. By combining
\pref{prop:information_ratio_bound} with \pref{thm:upper_main_bayes}, we recover
this result.

\begin{proof}[\pfref{prop:information_ratio_bound}]
  Recall that
  \begin{align*}
    \compKLdual(\cM,\Mbar)
    =
    \sup_{\mu\in\Delta(\cM)}\inf_{p\in\Delta(\Act)}\En_{\act\sim{}p}\En_{M\sim\mu}\brk*{
      \fm(\pim) - \fm(\pi)
      - \gamma\cdot{}\Dkl{M(\act)}{\Mbar(\act)}
                      }.
  \end{align*}
  For all $\mu\in\Delta(\cM)$ and $p\in\Delta(\Act)$, we can apply the AM-GM inequality to bound
  \begin{align*}
    &\En_{\act\sim{}p}\En_{M\sim\mu}\brk*{
    \fm(\pim) - \fm(\pi)}\\
    & =     \En_{\act\sim{}p}\En_{M\sim\mu}\brk*{
    \fm(\pim) -
      \fm(\pi)}\cdot{}\frac{\prn*{\En_{\act\sim{}p}\En_{M\sim\mu}\brk*{\Dkl{M(\act)}{\Mbar(\act)}}}^{1/2}}{\prn*{\En_{\act\sim{}p}\En_{M\sim\mu}\brk*{\Dkl{M(\act)}{\Mbar(\act)}}}^{1/2}}\\
        & \leq{}   \frac{1}{4\gamma}\cdot{}\frac{\prn*{\En_{\act\sim{}p}\En_{M\sim\mu}\brk*{\fm(\act)-\fm(\pim)}}^2}{
  \En_{\act\sim{}p}\En_{M\sim\mu}\brk*{\Dkl{M(\act)}{\Mbar(\act)}}} +   \gamma\cdot{}\En_{\act\sim{}p}\En_{M\sim\mu}\brk*{\Dkl{M(\act)}{\Mbar(\act)}}.
  \end{align*}
Since this holds uniformly for all distributions $p$, we can choose $p$ to
minimize the information ratio.
\end{proof}

The proof of \pref{prop:information_ratio_bound} shows that the
information ratio and the \CompShort are closely related, and suggests
that the information ratio might be thought of as a parameter-free
analogue of the dual Bayesian \CompShort. However, the following proposition shows that in
general, the information ratio can be arbitrarily large compared to
\CompShort.
    \begin{restatable}{proposition}{informationratiobayes}
  \label{prop:information_ratio_bayes_lower}
Consider the Lipschitz bandit problem in which $\Act=\brk*{0,1}^{d}$,
$\cFm=\crl[\big]{f:\brk*{0,1}^{d}\to\brk*{0,1} \mid{} \abs{f(x)-f(y)}\leq{}\nrm{x-y}_{\infty}}$, $\Obs=\NullObs$, and
$\cM=\crl[\big]{\act\mapsto\cN(f(\act),1)\mid{}f\in\cFm}$. For this setting,
the information ratio is infinite for all $d\geq{}1$:
\[
\InfB(\cM)=+\infty.
\]
On the other hand, we have
$\comp(\cM)\leq\bigoht(\gamma^{-\frac{1}{d+1}})$, and consequently
\pref{thm:upper_main_bayes} recovers the optimal regret bound $\En\brk*{\RegDM}\leq{}\bigoht(T^{\frac{d+1}{d+2}})$.
\end{restatable}
This result continues to hold if the KL divergence in
$\InfB(\cM)$ is replaced by squared error, TV distance, or Hellinger
distance. The proof of the result indicates that ratios---which
inherently suffer from boundedness issues and numerical instability---may not
lead to fundamental complexity measures. We note that in spite
of this limitation, the information ratio was an important influence
on the present work.%

Beyond the issues highlighted above, the information ratio is
tied to the Bayesian bandit setting, and does not immediately yield
algorithms for the frequentist setting that is the focus of this
paper. To address this issue, one might consider the following
frequentist analogue:
\[
\InfF(\cM,\Mbar)\ldef{}\min_{p\in\Delta(\Act)}\max_{\mu\in\Delta(\cM)}\frac{\prn*{\En_{\act\sim{}p}\En_{M\sim\mu}\brk*{\fm(\act)-\fm(\pim)}}^2}{
\En_{\act\sim{}p}\En_{M\sim\mu}\brk*{\Dkl{M(\act)}{\Mbar(\act)}}}.
\]
While it always holds that
\[
\InfB(\cM,\Mbar)  \leq \InfF(\cM,\Mbar),
\]
we show that this inequality is strict in general: even for the
multi-armed bandit, the frequentist
information ratio $\InfF(\cM,\Mbar)$ can be arbitrarily large compared to the
Bayesian version.
    \begin{restatable}{proposition}{informationratiofreq}
  \label{prop:information_ratio_separation}
  Consider the multi-armed bandit setting with $\Act=\brk{A}$ and
  $\cM=\crl[\big]{M(\act)\ldef{}\cN(f(\act),1/2)  :
    f\in\brk*{0,1}^{A}}$. For any $\Mbar\in\cM$ for which $\fmbar\in\mathrm{int}(\brk{0,1}^{A})$
  and $\min_{\act\neq\pimbar}\crl{\fmbar(\act)-\fmbar(\pimbar)}>0$, we have
  \[
    \InfF(\cM,\Mbar)=+\infty.
  \]
On the other hand, this example has $\comp(\cM)\leq{}\InfB(\cM)\leq\bigoh\prn*{A/\gamma}$ for all $\gamma>0$.  
\end{restatable}
This result should be contrasted with the situation for the
\CompShort, where the ``min-max'' and ``max-min'' variants coincide
under mild regularity conditions (cf. \pref{sec:dual}).

\paragraph{Further results}
Building on \cite{russo2014learning,russo2018learning}, subsequent
works have generalized their information-theoretic approach to provide
minimax regret bounds for various settings, including bandit
convex optimization \citep{bubeck2015bandit,bubeck2016multi,lattimore2020improved,lattimore2021minimax} and
partial monitoring
\citep{lattimore2019information,kirschner2020information}. In
parallel, a line of work
\citep{lattimore2020exploration,lattimore2021mirror,lattimore2022minimax}
develops \emph{frequentist} regret bounds for structured bandits and
partial monitoring based on an elegant algorithm design principle
known as \emph{exploration-by-optimization}. Notably, 
\citet{lattimore2021mirror,lattimore2022minimax} make use of a
``parameterized'' form of the information ratio which can overcome the
obvious pathologies highlighted in the prequel, and
\citet{lattimore2022minimax} shows that this quantity acts as a lower
bound on the minimax rate for general partial monitoring problems in
an adversarial setting.\footnote{\dfedit{The lower bounds in
  \citet{lattimore2022minimax} are loose by $\poly(\abs{\Pi})$
  factors, and hence cannot meaningfully reflect dependence on
  problem-dependent parameters in the same fashion as our results.}} Follow-up
work of \citet{foster2022complexity} shows that the parameterized
information ratio coincides with the ``convexified'' \CompText given by
$\comp(\conv(\cM))$, and hence can only give tight guarantees for
convex classes, not for general models.

We mention in passing the works of
  \citet{kirschner2018information,kirschner2021asymptotically}, which
  give frequentist algorithms inspired by information-directed sampling and the
  information ratio. These works rely on notions of information
  gain tailored to linear function approximation, and it is not clear
  whether they extend to more general settings.

\subsection{Additional Follow-Up Work}
We close this section by highlighting relevant related work
  published in the time since the initial arXiv preprint of this paper became available.

\paragraph{Tighter guarantees based on the \CompShort}
The follow-up work of \citet{foster2023tight}
provides guarantees that improve upon our main upper and lower bounds by working
with a constrained variant of the \CompShort. For a fixed
parameter $\veps>0$, the \emph{Constrained \CompShort} is given by
\begin{align*}
\comp[\veps]^c(\cM,\Mbar) \ldef{} \inf_{p\in \Delta(\Pi)}\sup_{M\in \cM}\crl*{ \En_{\pi\sim p}[\fm(\pim) - \fm(\pi)] \mid{} \En_{\pi\sim p}[M(\pi),\Mbar(\pi)] \leq \veps^2},
\end{align*}
with $\comp[\veps]^c(\cM)\ldef
\sup_{\Mbar\in\conv(\cM)}\comp[\veps]^c(\cM\cup\crl{\Mbar},\Mbar)$. Working
with this variant of the \CompShort allows one to remove all but one of the
gaps discussed \cref{sec:main_discussion}: the lower bounds 1) hold in
expectation, and 2) allow the reference to belong to $\conv(\cM)$,
while the upper bounds enjoy a localization radius that matches the
lower bound. To our knowledge, this work offers the only subsequent improvement
to our core results. For the special case of reinforcement learning,
\citet{foster2022note} give model-free guarantees that improve upon
those in \cref{sec:rl} by combining the \CompText with a technique
known as \emph{optimistic estimation} \citep{zhang2021feel}; see
discussion below.

\paragraph{Generalizations of the \CompShort framework}
A number of follow-up works generalize the \CompText, as well as our
upper and lower bound techniques, to more general settings, including
1) PAC decision making \citep{foster2023tight}, 2) decision making
with adversarial outcomes \citep{foster2022complexity}, and 3) multi-agent decision making \citep{foster2023complexity}.

\paragraph{Optimistic estimation and posterior sampling}
Concurrent work of \citet{zhang2021feel} provides frequentist
regret bounds for contextual bandits and linearly-structured reinforcement learning
settings based on a posterior sampling-like algorithm referred to
as \emph{Feel-Good Thompson Sampling}. This approach makes use of
randomized online estimation algorithms in the vein of
\cref{eq:general_error}, but augments the estimation objective with a
``feel-good'' bonus that biases the estimator toward over-estimating
the optimal value (we refer to this technique as \emph{optimistic
  estimation}). When combined with posterior sampling, optimistic
estimation leads to frequentist regret bounds that match other
well-known approaches. A number of subsequent works extend this approach to more general
reinforcement learning settings
\citep{dann2021provably,agarwal2022model,zhong2022posterior,liu2023one},
also using posterior sampling as the exploration mechanism. Subsequent
work of \citet{foster2022note} highlights that for general decision
making settings, posterior sampling may lead to arbitrary sub-optimal
guarantees, and hence the complexity measures studied in these works
do not lead to lower bounds on optimal regret for the general settings
considered in this paper. However, \citet{foster2022note} show that the
optimistic estimation principle can be combined with the \CompText
(acting as a more general mechanism for exploration), leading to
guarantees for model-free reinforcement learning that improve upon
those in \cref{sec:rl}; we refer to this work for additional
discussion and background.

\paragraph{Further general-purpose complexity measures}
Since the initial preprint of this paper became available, a number of
subsequent works \citep{chen2022general,zhong2022posterior,xie2023role} have proposed
other complexity measures that lead to upper bounds on sample
complexity for general reinforcement learning settings. Compared to
the \CompText, these complexity measures are sufficient but not
\emph{necessary} for low sample complexity, and do not lead to lower
bounds in general. As such, they are best thought of as further generalizations of
other sufficient conditions such as Bellman Rank and Bellman-Eluder dimension.

\section{Discussion}
\label{sec:discussion}

We have developed a theory of learnability and unified algorithm design
principle for reinforcement learning and interactive decision
making. Our results provide a solid foundation on which to build the
theory of data-driven decision making going forward, and we are excited about many directions for future research.
    \begin{itemize}
    \item \emph{Computation.} While \mainalg{} provides a unified
      algorithm design principle for decision making, we have mainly
      focused on statistical rather than computational aspects of the
      algorithm in this paper, outside of special cases. Going
      forward, we intend to fully explore when and how the algorithm can
      be implemented efficiently, which we believe to have strong practical implications.

    \item \emph{Beyond reinforcement learning.} The examples in this
      paper focus on bandits and reinforcement learning, but
      the \FrameworkShort{} framework is substantially more general, and encompasses rich
      settings such as POMDPs. It remains to fully understand the implications of our results for these settings.
    \end{itemize}
Beyond these questions, we anticipate many natural extensions to our
framework and results, including adaptive or instance-dependent
guarantees, and incorporating constraints such as safety.

We close by mentioning some technical questions. First, while our
upper bounds generally achieve polynomial
sample complexity whenever this is achievable, more work is required
to understand to what extent the estimation complexity in our bounds
can be tightened or removed, and to achieve the sharpest possible
guarantees. Second, a natural question is whether our algorithmic results can
be extended to support offline oracles for estimation, in the same vein
as \cite{simchi2020bypassing}.

\subsection*{Acknowledgements}

We thank Ayush Sekhari and Karthik Sridharan for helpful discussions,
and thank Zak Mhammedi for useful comments and feedback. AR acknowledges support from the ONR through awards N00014-20-1-2336 and N00014-20-1-2394, from the NSF through award DMS-2031883, and from the ARO through award W911NF-21-1-0328.

\bibliography{refs}

\clearpage

\appendix

\section{Technical Tools}
\label{app:technical}

In this section of the appendix we provide a collection of basic
technical results used throughout the paper: Tail bounds for sequences
of random variables (\pref{app:tail}), inequalities for
information-theoretic divergences (\pref{app:information}), and regret
bounds for online learning (\pref{app:online}).

\subsection{Tail Bounds}
\label{app:tail}

\begin{lemma}[Azuma-Hoeffding]
  \label{lem:hoeffding}
    Let $(X_t)_{t\leq{T}}$ be a sequence of real-valued random
    variables adapted to a filtration $\prn{\filt_t}_{t\leq{}T}$. If $\abs*{X_t}\leq{}R$ almost surely, then with probability at least $1-\delta$,
    \[
      \abs*{\sum_{t=1}^{T}X_t - \En_{t-1}\brk*{X_t}} \leq{} R\cdot\sqrt{8T\log(2\delta^{-1})}.
    \]
\end{lemma}

\begin{lemma}[Freedman's inequality (e.g., \citet{agarwal2014taming})]
  \label{lem:freedman}
  Let $(X_t)_{t\leq{T}}$ be a real-valued martingale difference
  sequence adapted to a filtration $\prn{\filt_t}_{t\leq{}T}$. If
  $\abs*{X_t}\leq{}R$ almost surely, then for any $\eta\in(0,1/R)$, with probability at least $1-\delta$,
    \[
      \sum_{t=1}^{T}X_t \leq{} \eta\sum_{t=1}^{T}\En_{t-1}\brk*{X_t^{2}} + \frac{\log(\delta^{-1})}{\eta}.
    \]
  \end{lemma}
  The following result is an immediate consequence of \pref{lem:freedman}.
      \begin{lemma}
      \label{lem:multiplicative_freedman}
            Let $(X_t)_{t\leq{T}}$ be a sequence of random
      variables adapted to a filtration $\prn{\filt_{t}}_{t\leq{}T}$. If
  $0\leq{}X_t\leq{}R$ almost surely, then with probability at least
  $1-\delta$,
  \begin{align*}
    &\sum_{t=1}^{T}X_t \leq{}
                        \frac{3}{2}\sum_{t=1}^{T}\En_{t-1}\brk*{X_t} +
                        4R\log(2\delta^{-1}),
    \intertext{and}
      &\sum_{t=1}^{T}\En_{t-1}\brk*{X_t} \leq{} 2\sum_{t=1}^{T}X_t + 8R\log(2\delta^{-1}).
  \end{align*}
    \end{lemma}

      \begin{lemma}
    \label{lem:martingale_chernoff}
    For any sequence of real-valued random variables $\prn{X_t}_{t\leq{}T}$ adapted to a
    filtration $\prn{\filt_t}_{t\leq{}T}$, it holds that with probability at least
    $1-\delta$, for all $T'\leq{}T$,
    \begin{equation}
      \label{eq:martingale_chernoff}
      \sum_{t=1}^{T'}X_t \leq
      \sum_{t=1}^{T'}\log\prn*{\En_{t-1}\brk*{e^{X_t}}} + \log(\delta^{-1}).
    \end{equation}
  \end{lemma}
\begin{proof}[\pfref{lem:martingale_chernoff}]
  We claim that the sequence
  \[
    Z_{\tau} \ldef \exp\prn*{\sum_{t=1}^{\tau}X_t- \log\prn*{\En_{t-1}\brk*{e^{X_t}}}}
  \]
 is a nonnegative supermartingale with
  respect to the filtration $(\filt_{\tau})_{\tau\leq{}T}$. Indeed, for any choice of $\tau$, we have
  \begin{align*}
    \En_{\tau-1}\brk*{Z_{\tau}} &=    \En_{\tau-1}\brk*{\exp\prn*{\sum_{t=1}^{\tau}X_t-
    \log\prn*{\En_{t-1}\brk*{e^{X_t}}}}}\\
    &= \exp\prn*{\sum_{t=1}^{\tau-1}X_t-
      \log\prn*{\En_{t-1}\brk*{e^{X_t}}}}\cdot\En_{\tau-1}\brk*{\exp\prn*{X_\tau-
      \log\prn*{\En_{\tau-1}\brk*{e^{X_\tau}}}}}\\
&= \exp\prn*{\sum_{t=1}^{\tau-1}X_t-
                                                     \log\prn*{\En_{t-1}\brk*{e^{X_t}}}}\\
    &= Z_{\tau}.
  \end{align*}
Since $Z_0=1$, Ville's inequality (e.g., \citet{howard2020time}) implies that for
all $\lambda>0$,
\[
\bbP_0(\exists{}\tau : Z_{\tau}>\lambda)\leq{}\frac{1}{\lambda}.
\]
The result now follows by the Chernoff method.
\end{proof}

\subsection{Information Theory}
\label{app:information}
We begin with some basic inequalities between divergences.
  \begin{lemma}[e.g., \cite{polyanskiy2014lecture}]
    \label{lem:pinsker}
    The following inequalities hold:
    \begin{itemize}
    \item $\Dtv{\bbP}{\bbQ}\leq\sqrt{\frac{1}{2}\Dkl{\bbP}{\bbQ}}$ and
      $\Dtvs{\bbP}{\bbQ}\leq\Dhels{\bbP}{\bbQ}\leq{}2\Dtv{\bbP}{\bbQ}$.%
    \item $\Dhels{\bbP}{\bbQ}\leq{}\Dkl{\bbP}{\bbQ}$.
    \end{itemize}
  \end{lemma}
  
  \begin{lemma}[Divergence for Gaussians distributions]
  \label{lem:divergence_gaussian}
  We have
  $\Dkl{\cN(\mu_1,\sigma^2)}{\cN(\mu_2,\sigma^2)}=\frac{1}{2\sigma^2}(\mu_1-\mu_2)^2$
  and
  $\Dhels{\cN(\mu_1,\sigma^2)}{\cN(\mu_2,\sigma^2)}=1-\exp\prn*{-\frac{(\mu_1-\mu_2)^2}{8\sigma^2}}$
  for all $\mu_1,\mu_2\in\bbR$.
  \end{lemma}
  
  \begin{lemma}[Divergence for Bernoulli distributions]
    \label{lem:divergence_bernoulli}
    For all $p,q\in\brk*{0,1}$,
    \[
    \Dhels{\Ber(p)}{\Ber(q)}
    \leq (p-q)^2\cdot\prn*{\frac{1}{p+q}+\frac{1}{1-p+1-q}}.
    \]
    In particular, for all $\Delta\in(0,1/2)$, we have
    \[
          \Dhels{\Ber(\nicefrac{1}{2}+\Delta)}{\Ber(\nicefrac{1}{2})}\leq3\Delta^2,
  \]
  and for all $\Delta\in(0,1/4)$, we have
  \[
\Dhels{\Ber(\nicefrac{1}{2}+\Delta)}{\Ber(\nicefrac{1}{2}+2\Delta)}\leq{}5\Delta^2.
  \]
\end{lemma}

  \begin{lemma}[Divergence for Rademacher distributions]
  \label{lem:divergence_rademacher}
  For all $\mu_1, \mu_2\in\brk{-1,+1}$, we have
  \[
  \Dhels{\Rad(\mu_1)}{\Rad(\mu_2)}
  \leq{} \frac{(\mu_1-\mu_2)^2}{2}\prn*{
  \frac{1}{2+\mu_1+\mu_2}+   \frac{1}{2-\mu_1-\mu_2}
  }.
  \]
  In particular, for all $\mu_1,\mu_2\in\brk{-1/2, +1/2}$, 
  \[
    \Dhels{\Rad(\mu_1)}{\Rad(\mu_2)} \leq (\mu_1-\mu_2)^2,
  \]
and for all $\mu\in\brk{-1,+1}$,  \[
    \Dhels{\Rad(\mu)}{\Rad(0)}\leq{}\frac{3}{4}\mu^2.
  \]
\end{lemma}
\begin{proof}[\pfref{lem:divergence_rademacher}]
  This is an immediate consequence of \pref{lem:divergence_bernoulli},
  since $\Dhels{\Rad(\Delta)}{\Rad(0)} = \Dhels{\Ber(\nicefrac{(1+\Delta)}{2})}{\Ber(\nicefrac{1}{2})}$.
\end{proof}

\begin{lemma}
  \label{lem:hellinger_pair}
    For any pair of random variables $(X,Y)$,
    \[
      \En_{X\sim{}\bbP_X}\brk*{\Dhels{\bbP_{Y\mid{}X}}{\bbQ_{Y\mid{}X}}} \leq{} 
      4\Dhels{\bbP_{X,Y}}{\bbQ_{X,Y}}.      
    \]
  \end{lemma}
  \begin{proof}[\pfref{lem:hellinger_pair}]
    Since squared Hellinger distance is an $f$-divergence, we have
    \begin{align*}
      \En_{X\sim{}\bbP_X}\brk*{\Dhels{\bbP_{Y\mid{}X}}{\bbQ_{Y\mid{}X}}}
      & = \Dhels{\bbP_{Y\mid{}X}\tens{}\bbP_X}{\bbQ_{Y\mid{}X}\tens{}\bbP_X}.
        \end{align*}
        Next, using that Hellinger distance satisfies the
        triangle inequality, along with the elementary inequality
        $(a+b)^2\leq{}2(a^2+b^2)$, we have,
        \begin{align*}
\En_{X\sim{}\bbP_X}\brk*{\Dhels{\bbP_{Y\mid{}X}}{\bbQ_{Y\mid{}X}}}      & \leq{} 2\Dhels{\bbP_{Y\mid{}X}\tens{}\bbP_X}{\bbQ_{Y\mid{}X}\tens{}\bbQ_X}
        + 2\Dhels{\bbQ_{Y\mid{}X}\tens{}\bbP_X}{\bbQ_{Y\mid{}X}\tens{}\bbQ_X}\\
      & = 2\Dhels{\bbP_{X,Y}}{\bbQ_{X,Y}}
        + 2\Dhels{\bbP_X}{\bbQ_X}\\
      & \leq{} 4\Dhels{\bbP_{X,Y}}{\bbQ_{X,Y}},
    \end{align*}
    where the final line follows from the data processing inequality.
  \end{proof}
  
  \begin{lemma}[Lemma 4 of \citet{yang1998asymptotic}\protect\footnote{This
      result is stated in \citet{yang1998asymptotic} in terms of
      densities, but the variant here is an immediate consequence.}]
    \label{lem:kl_hellinger}
    Let $\bbP$ and $\bbQ$ be probability distributions over a measurable space
    $(\cX,\filt)$. If
    $\sup_{F\in\filt}\frac{\bbP(F)}{\bbQ(F)}\leq{}V$, then
    \begin{equation}
      \label{eq:kl_hellinger}
      \Dkl{\bbP}{\bbQ}\leq{}(2+\log(V))\Dhels{\bbP}{\bbQ}.
    \end{equation}
  \end{lemma}

Next, we state a ``multiplicative'' variant of Pinsker's inequality, which provides faster rates at the cost of multiplicative rather than additive error.

  \begin{restatable}[Multiplicative Pinsker-type inequality for Hellinger distance]{lemma}{mpmin}
  \label{lem:mp_min}
  Let $\bbP$ and $\bbQ$ be probability measures on $(\cX,\filt)$. For all
  $h:\cX\to\bbR$ with $0\leq{}h(X)\leq{}R$ almost surely under $\bbP$
  and $\bbQ$, we have
  \begin{align}
    \label{eq:mp_min_sqrt}
    \abs*{\En_{\bbP}\brk*{h(X)} - \En_{\bbQ}\brk*{h(X)}} \leq \sqrt{2R(\En_{\bbP}\brk*{h(X)} + \En_{\bbQ}\brk*{h(X)})\cdot\Dhels{\bbP}{\bbQ}}.
  \end{align}
In particular, 
  \begin{align}
    \label{eq:mp_min}
\En_{\bbP}\brk*{h(X)}
    &\leq{} 3\En_{\bbQ}\brk*{h(X)} + 4R\Dhels{\bbP}{\bbQ}.
  \end{align}
\end{restatable}
\begin{proof}[\pfref{lem:mp_min}]
Let a measurable event $A$ be fixed. Let $p = \bbP(A)$ and
$q=\bbQ\prn{A}$. Then we have
\[
\frac{(p-q)^{2}}{2(p+q)} \leq{}(\sqrt{p}-\sqrt{q})^{2}\leq{}\Dhels{(p,1-p)}{(q,1-q)}\leq{}\Dhels{\bbP}{\bbQ},
\]
where the second inequality is the data-processing inequality. It follows that
\[
\abs*{p-q}\leq{}\sqrt{2(p+q) \Dhels{\bbP}{\bbQ}},
\]
To deduce the final result for $R=1$, we observe that
$\En_{\bbP}\brk*{h(X)}=\int_{0}^{1}\bbP\prn{h(X)>t}dt$ and likewise
for $\En_{\bbQ}\brk*{h(X)}$, then apply Jensen's inequality. The result for
general $R$ follows by rescaling.

The inequality in \pref{eq:mp_min} follows by applying the AM-GM
inequality to \pref{eq:mp_min_sqrt} and rearranging.

\end{proof}

\begin{lemma}
  \label{lem:mult_pinsker_as}
Let $\bbP$ and $\bbQ$ be a pair of probability measures for a random
variable $(\cZ,\filt)$. Suppose we can write $Z=X-Y$, where $X,Y\geq{}0$ and
  $\abs{X-Y}\leq\veps$ almost surely under $\bbP$ and $\bbQ$. Then
    \begin{align}
    \abs*{\En_{\bbP}\brk*{Z}-\En_{\bbQ}\brk*{Z}}
    &\leq{} \sqrt{8\veps\cdot{}\prn*{\En_{\bbP}\brk*{X+Y}+\En_{\bbQ}\brk*{X+Y}}\cdot\Dhels{\bbP}{\bbQ}}.
\end{align}
\end{lemma}
\begin{proof}[\pfref{lem:mult_pinsker_as}]%
  \newcommand{\Zp}{Z_{+}}%
  \newcommand{\Zm}{Z_{-}}%
  \newcommand{\Ep}{\En_{\bbP}}%
    \newcommand{\Eq}{\En_{\bbQ}}%
  Let $Z_{+}=\brk{Z}_{+}$ and $Z_{-}=\brk*{-Z}_{+}$, so that $Z=\Zp -
  \Zm$. We have
  \begin{align*}
    \abs*{\Ep\brk{Z}-\Eq\brk{Z}}
    \leq{}     \abs*{\Ep\brk{\Zp}-\Eq\brk{\Zp}} +     \abs*{\Ep\brk{\Zm}-\Eq\brk{\Zm}}.
  \end{align*}
  Note that $0\leq{}\Zp,\Zm\leq\veps$ almost surely, so by
  \pref{lem:mp_min}, we have
  \[
    \abs*{\Ep\brk{\Zp}-\Eq\brk{\Zp}}\leq{} \sqrt{2\veps\prn*{\En_{\bbP}\brk*{\Zp}+\En_{\bbQ}\brk*{\Zp}}\cdot\Dhels{\bbP}{\bbQ}}
  \]
  and
  \[
  \abs*{\Ep\brk{\Zm}-\Eq\brk{\Zm}}\leq{}
  \sqrt{2\veps\prn*{\En_{\bbP}\brk*{\Zm}+\En_{\bbQ}\brk*{\Zm}}\cdot\Dhels{\bbP}{\bbQ}}
\]
To conclude, we use that $\Zp\leq{}X+Y$ and $\Zm\leq{}X+Y$ to simplify.
  
\end{proof}

Finally, we provide an approximate chain rule-type inequality for the squared Hellinger distance, which allows the distance between two joint distributions to be decomposed into a sum of distances between conditional distributions.

\begin{lemma}[Subadditivity for squared Hellinger distance]
  \label{lem:hellinger_chain_rule}
  Let $(\cX_1,\filt_1),\ldots,(\cX_n,\filt_n)$ be a sequence of
  measurable spaces, and let $\cX\ind{i}=\prod_{i=t}^{i}\cX_t$ and
  $\filt\ind{i}=\bigotimes_{t=1}^{i}\filt_t$. For each $i$, let
  $\bbP\ind{i}(\cdot\mid{}\cdot)$ and $\bbQ\ind{i}(\cdot\mid{}\cdot)$ be probability kernels from
  $(\cX\ind{i-1},\filt\ind{i-1})$ to $(\cX_i,\filt_i)$. Let $\bbP$ and
  $\bbQ$ be
  the laws of $X_1,\ldots,X_n$ under
  $X_i\sim{}\bbP\ind{i}(\cdot\mid{}X_{1:i-1})$ and
  $X_i\sim{}\bbQ\ind{i}(\cdot\mid{}X_{1:i-1})$ respectively. Then it
  holds that
\begin{align}
  \label{eq:hellinger_chain_rule}
  \Dhels{\bbP}{\bbQ}
  &\leq{}
10^2\log(n)\cdot\En_{\bbP}\brk*{\sum_{i=1}^{n}\Dhels{\bbP\ind{i}(\cdot\mid{}X_{1:i-1})}{\bbQ\ind{i}(\cdot\mid{}X_{1:i-1})}}.
\end{align}
Furthermore, if there exists a constant $V\geq{}e$ such that
$\sup_{x_{1:i-1}\in\cX\ind{i-1}}\sup_{A\in\filt_{i}}\frac{\bbP\ind{i}(A\mid{}x_{1:i-1})}{\bbQ\ind{i}(A\mid{}x_{1:i-1})}\leq{}V$
for all $i$, then
\begin{align}
  \label{eq:hellinger_chain_rule2}
  \Dhels{\bbP}{\bbQ}
  &\leq{}
    3\log(V)\cdot\En_{\bbP}\brk*{\sum_{i=1}^{n}\Dhels{\bbP\ind{i}(\cdot\mid{}X_{1:i-1})}{\bbQ\ind{i}(\cdot\mid{}X_{1:i-1})}}.
\end{align}

\end{lemma}

\begin{proof}[\pfref{lem:hellinger_chain_rule}]
We first prove the result in \pref{eq:hellinger_chain_rule}. For
a parameter $\lambda\in\prn*{0,e^{-1}}$, define a probability kernel
\begin{align*}
&\bbP\ind{i}_{\lambda}(\cdot\mid{}x_{1:i-1}) = (1-\lambda)
  \bbP\ind{i}(\cdot\mid{}x_{1:i-1})
                 + \lambda\bbQ\ind{i}(\cdot\mid{}x_{1:i-1}).
\end{align*}
Let $\bbPl$ be
  the law of $X_1,\ldots,X_n$ under
  $X_i\sim{}\bbPl\ind{i}(\cdot\mid{}X_{1:i-1})$.
Since Hellinger distance satisfies
the triangle inequality, we have
\begin{align*}
  \Dhel{\bbP}{\bbQ}
  &\leq{}   \underbrace{\Dhel{\bbPl}{\bbQ}}_{\textrm{I}} +   \underbrace{\Dhel{\bbP}{\bbPl}}_{\textrm{II}}.
\end{align*}
We proceed to bound each term individually.
\paragraph{Term I}
We have
\[
  \Dhels{\bbPl}{\bbQ}
  \leq{} \Dkl{\bbQ}{\bbPl} = \En_{\bbQ}\brk*{\sum_{i=1}^{n}\Dkl{\bbQ\ind{i}(\cdot\mid{}X_{1:i-1})}{\bbPl\ind{i}(\cdot\mid{}X_{1:i-1})}},
\]
where the right-hand expression follows from the chain rule for KL divergence.
We proceed by appealing to \pref{lem:kl_hellinger}.
Since 
$\bbP\ind{i}_{\lambda}=(1-\lambda)\bbP\ind{i}+\lambda\bbQ\ind{i}$, we
have that for any $x_{1:i-1}\in\cX\ind{i-1}$ and measurable event $A\in\filt_{i}$,
\[
  \frac{\bbQ\ind{i}(A\mid{}x_{1:i-1})}{\bbPl\ind{i}(A\mid{}x_{1:i-1})}
  \leq{} \frac{1}{\lambda}.
\]
As a result, \pref{lem:kl_hellinger} implies that for all $x_{1:i-1}\in\cX\ind{i-1}$,
\begin{align*}
  \Dkl{\bbQ\ind{i}(\cdot\mid{}x_{1:i-1})}{\bbPl\ind{i}(\cdot\mid{}x_{1:i-1})}
  &\leq{}(2+\log(1/\lambda))
    \Dhels{\bbQ\ind{i}(\cdot\mid{}x_{1:i-1})}{\bbPl\ind{i}(\cdot\mid{}x_{1:i-1})}\\
  &\leq{}3\log(1/\lambda)\Dhels{\bbQ\ind{i}(\cdot\mid{}x_{1:i-1})}{\bbPl\ind{i}(\cdot\mid{}x_{1:i-1})},
\end{align*}
since $\lambda\leq{}2/e$. Moreover, by convexity of squared Hellinger
distance,
\begin{align*}
  \Dhels{\bbQ\ind{i}(\cdot\mid{}x_{1:i-1})}{\bbPl\ind{i}(\cdot\mid{}x_{1:i-1})}
  &\leq{} (1-\lambda)
  \Dhels{\bbQ\ind{i}(\cdot\mid{}x_{1:i-1})}{\bbP\ind{i}(\cdot\mid{}x_{1:i-1})}
  + \lambda
  \Dhels{\bbQ\ind{i}(\cdot\mid{}x_{1:i-1})}{\bbQ\ind{i}(\cdot\mid{}x_{1:i-1})}\\
  &\leq{} \Dhels{\bbQ\ind{i}(\cdot\mid{}x_{1:i-1})}{\bbP\ind{i}(\cdot\mid{}x_{1:i-1})}.
\end{align*}
We conclude that
\begin{align*}
  \Dhels{\bbPl}{\bbQ}
  \leq{} 3\log(1/\lambda)\cdot\En_{\bbQ}\brk*{\sum_{i=1}^{n}\Dhels{\bbQ\ind{i}(\cdot\mid{}X_{1:i-1})}{\bbP\ind{i}(\cdot\mid{}X_{1:i-1})}}.
\end{align*}
\paragraph{Term II}
We begin in a similar fashion as the first term and bound
\[
  \Dhels{\bbP}{\bbPl}
  \leq{} \Dkl{\bbP}{\bbPl} = \En_{\bbP}\brk*{\sum_{i=1}^{n}\Dkl{\bbP\ind{i}(\cdot\mid{}X_{1:i-1})}{\bbPl\ind{i}(\cdot\mid{}X_{1:i-1})}}.
\]
For any measurable event $A\in\filt_i$ and $x_{1:i-1}\in\cX\ind{i-1}$, we have
\[
  \frac{\bbP\ind{i}(A\mid{}x_{1:i-1})}{\bbPl\ind{i}(A\mid{}x_{1:i-1})}
  \leq{} \frac{1}{1-\lambda},
\]
so that \pref{lem:kl_hellinger} yields
\begin{align*}
  \Dkl{\bbP\ind{i}(\cdot\mid{}x_{1:i-1})}{\bbPl\ind{i}(\cdot\mid{}x_{1:i-1})}
  &\leq{}(2+\log(1/(1-\lambda))
    \Dhels{\bbP\ind{i}(\cdot\mid{}x_{1:i-1})}{\bbPl\ind{i}(\cdot\mid{}x_{1:i-1})}\\
  &\leq{}3\Dhels{\bbP\ind{i}(\cdot\mid{}x_{1:i-1})}{\bbPl\ind{i}(\cdot\mid{}x_{1:i-1})}\\
  &\leq{}3\lambda{}\Dhels{\bbQ\ind{i}(\cdot\mid{}x_{1:i-1})}{\bbP\ind{i}(\cdot\mid{}x_{1:i-1})},
\end{align*}
where the last inequality uses convexity of squared Hellinger distance.

Altogether we have
\[
  \Dhels{\bbPl}{\bbP}
  \leq{} 3\lambda{}\cdot{}\En_{\bbP}\brk*{\sum_{i=1}^{n}\Dhels{\bbP\ind{i}(\cdot\mid{}X_{1:i-1})}{\bbQ\ind{i}(\cdot\mid{}X_{1:i-1})}}.
\]
Since the sum inside the expectation on the right-hand side is bounded
by $2n$ with probability $1$, \pref{lem:mp_min} gives the following
upper bound:
\begin{align*}
  9\lambda{}\cdot{}\En_{\bbQ}\brk*{\sum_{i=1}^{n}\Dhels{\bbP\ind{i}(\cdot\mid{}X_{1:i-1})}{\bbQ\ind{i}(\cdot\mid{}X_{1:i-1})}}
  + 24\lambda{}n\cdot{}\Dhels{\bbP}{\bbQ}.
\end{align*}
\paragraph{Completing the proof}
Combining our bounds on terms I and II and using the elementary fact
that $(x+y)^2\leq{}2(x^2+y^2)$, we have
\begin{align*}
  \Dhels{\bbP}{\bbQ}
  \leq{}
\prn*{6\log(1/\lambda)+18\lambda}\En_{\bbQ}\brk*{\sum_{i=1}^{n}\Dhels{\bbP\ind{i}(\cdot\mid{}X_{1:i-1})}{\bbQ\ind{i}(\cdot\mid{}X_{1:i-1})}}
  + 48\lambda{}n\Dhels{\bbP}{\bbQ}.
\end{align*}
We set $\lambda=(96n)^{-1}$ (so that $\lambda\leq{}1/e$ as required)
which, after rearranging, gives
\[
  \frac{1}{2}\Dhels{\bbP}{\bbQ}
  \leq{}
\prn*{6\log(96n)+\frac{18}{96n}}\En_{\bbQ}\brk*{\sum_{i=1}^{n}\Dhels{\bbP\ind{i}(\cdot\mid{}X_{1:i-1})}{\bbQ\ind{i}(\cdot\mid{}X_{1:i-1})}}.
\]
To conclude, we note that $2\prn*{6\log(96n)+\frac{18}{96n}}
  \leq{} 10^2\log(n)$. This establishes
  \pref{eq:hellinger_chain_rule}.
  
To prove
\pref{eq:hellinger_chain_rule2}, we simply write
\begin{align*}
  \Dhels{\bbP}{\bbQ}
  \leq{} \Dkl{\bbP}{\bbQ}
  &\leq{}
    \En_{\bbP}\brk*{\sum_{i=1}^{n}\Dkl{\bbP\ind{i}(\cdot\mid{}X_{1:i-1})}{\bbQ\ind{i}(\cdot\mid{}X_{1:i-1})}}\\
  &\leq{}
    3\log(V)\En_{\bbP}\brk*{\sum_{i=1}^{n}\Dhels{\bbP\ind{i}(\cdot\mid{}X_{1:i-1})}{\bbQ\ind{i}(\cdot\mid{}X_{1:i-1})}}.
\end{align*}
where the first inequality is the chain rule and the second inequality
uses \pref{lem:kl_hellinger} along with the assumption that $V\geq{}e$.

\end{proof}

\subsection{Online Learning}
\label{app:online}

\newcommand{\Xspace}{\cX}
\newcommand{\Yspace}{\cY}
\newcommand{\Xsig}{\mathscr{X}}
\newcommand{\Ysig}{\mathscr{Y}}
\newcommand{\Xpair}{(\Xspace,\Xsig)}
\newcommand{\Ypair}{(\Yspace,\Ysig)}
\newcommand{\Kclass}{\cK(\Xspace,\Yspace)}
\newcommand{\idx}{i}
\newcommand{\jdx}{j}
\newcommand{\Iset}{\cI}
\newcommand{\Jset}{\cJ}
\newcommand{\qdist}{q}
\newcommand{\gstar}{g_{\star}}
\newcommand{\istar}{i^{\star}}

\newcommand{\Bconst}{b_T}

  \newcommand{\jstar}{j^{\star}}
  \newcommand{\hhat}{\wh{h}}
  \newcommand{\Entil}{\wt{\En}}
  \newcommand{\hstar}{h_{\jstar}}
\newcommand{\indt}{\indic_t}

In this section, we provide results for online learning in a conditional
density estimation setting. The setup is as follows. Let $\Xpair$ be the
\emph{covariate/context space} and $\Ypair$ be the \emph{outcome space}. Let
$\dom$ be an unnormalized kernel from $\Xpair$ to $\Ypair$ (i.e., for every
$A\in\Ysig$, $x\mapsto{}\nu(A\mid{}x)$ is $\Xsig$-measurable, and
for all $x\in\Xspace$, $\nu(\cdot\mid{}x)$ is a $\sigma$-finite
measure; cf. \pref{sec:prelims}), and let
$\Kclass$ denote the collection of all regular conditional
densities with respect to $\dom$. That is, each $g\in\Kclass$ is a
jointly measurable map $g:\Xspace\times\Yspace\to\bbR_{+}$ such that
for all $x\in\cX$
\[
y\mapsto{}g(y\mid{}x)
\]
is a probability density with respect to $\dom(\cdot\mid{}x)$.

\subsubsection{Online Learning, Regret, and Estimation Error}

We consider the following online learning process: For $t=1,\ldots,T$:
\begin{itemize}
\item Learner predicts $\ghat\ind{t}\in\Kclass$.
\item Nature reveals $x\ind{t}\in\Xspace$ and $y\ind{t}\in\Yspace$ and
  learner suffers loss
  \[
    \logloss\ind{t}(\ghat\ind{t})\ldef{}\log\prn*{
      \frac{1}{\ghat\ind{t}(y\ind{t}\mid{}x\ind{t})}
    }.
  \]
\end{itemize}

Define $\hist\ind{t}=(x\ind{1},y\ind{1}),\ldots,(x\ind{t},y\ind{t})$
and let $\filt\ind{t}=\sigma(\hist\ind{t})$. Let $\Iset$ be an
abstract ``index'' set. We consider a so-called \emph{expert setting} in which we are given a class of history-dependent functions
\[
  \cG=\crl*{g_{\idx}}_{\idx\in\Iset},
\]
where each ``expert'' $g_{\idx} = (g_{\idx}\ind{1},\ldots,g\ind{T}_{\idx})$ is a
sequence of functions of the form
\[
g_{\idx}\ind{t}(\cdot\mid\cdot\midsem \hist\ind{t-1})\in\Kclass.
\]
Each such function can be thought of as a conditional density that becomes
known to the learner after observing the history $\hist\ind{t-1}$
(i.e., at the beginning of round $t$). We abbreviate $g_{\idx}\ind{t}(\cdot\mid\cdot) =
g_{\idx}\ind{t}(\cdot\mid\cdot\midsem \hist\ind{t-1})$ when the
history is clear from context and define
$\cG\ind{t}=\crl*{g\ind{t}_{\idx}}_{\idx\in\Iset}$. We occasionally
overload the density $g_i\ind{t}(x)\equiv{}g_i\ind{t}(\cdot\mid{}x)$ with its induced
probability measure.

We define regret to the expert class $\cG$ via
\begin{equation}
  \label{eq:regret_ol}
    \RegLog = \sum_{t=1}^{T}\logloss\ind{t}(\ghat\ind{t}) - \inf_{\idx\in\Iset}\sum_{t=1}^{T}\logloss\ind{t}(g\ind{t}_{\idx}).
\end{equation}
For the main results in this section, we make the following realizability assumption.
\begin{assumption}
  \label{ass:realizability_ol}
  There exists $\gstar\ldef{}g_{\istar}\in\cG$ such that for all
  $t\in\brk*{T}$,
  \[
    y\ind{t}\sim{}\gstar\ind{t}(\cdot\mid{}x\ind{t}) \mid{} x\ind{t},
    \hist\ind{t-1}.
  \]
\end{assumption}
Under \pref{ass:realizability_ol}, our main object of interest will be the Hellinger
estimation error
\begin{equation}
  \label{eq:hellinger_ol}
  \EstHel\ldef\sum_{t=1}^{T}\En_{t-1}\brk*{\Dhels{\ghat\ind{t}(x\ind{t})}{\gstar\ind{t}(x\ind{t})}},
\end{equation}
where
$\En_{t}\brk*{\cdot}\ldef{}\En\brk*{\cdot\mid{}\filt\ind{t}}$. The
following lemma shows that minimizing the log-loss regret
\pref{eq:regret_ol} suffices to
minimize the Hellinger estimation error.
\begin{lemma}
  \label{lem:logloss_hellinger_ol}
For any estimation algorithm, whenever \pref{ass:realizability_ol} holds,
\begin{equation}
  \label{eq:logloss_hellinger_ol1}
  \En\brk*{\EstHel} \leq{} \En\brk*{\RegLog}.
  \end{equation}
  Furthermore, for any $\delta\in(0,1)$, with
probability at least $1-\delta$,
\begin{equation}
  \label{eq:logloss_hellinger_ol2}
  \EstHel \leq{} \RegLog + 2\log(\delta^{-1}).
  \end{equation}
  Finally, consider a sequence of $\crl*{0,1}$-valued random
  variables $(\indt)_{t\leq{}T}$, where $\indt$ is
  $\filt\ind{t-1}$-measurable. For any $\delta\in(0,1)$, we have that
  with probability at least $1-\delta$,
  \begin{equation}
      \label{eq:logloss_hellinger_ol3}
\sum_{t=1}^{T}\En_{t-1}\brk*{\Dhels{\ghat\ind{t}(x\ind{t})}{\gstar\ind{t}(x\ind{t})}}\indt
      \leq{} \sum_{t=1}^{T}\prn*{\logloss\ind{t}(\ghat\ind{t}) -
      \logloss\ind{t}(\gstar\ind{t})}\indt
      + 2\log(\delta^{-1}).
  \end{equation}
\end{lemma}

\subsubsection{Finite Classes}

We begin with a basic guarantee for finite expert classes. We assume
$\abs{\Iset}=\abs{\cG}$ without loss of generality.
\begin{lemma}[\citet{PLG}]
  \label{lem:vovk_fast}
  Consider Vovk's aggregating algorithm, which predicts via
  \[
    \ghat\ind{t}=\En_{\idx\sim{}q\ind{t}}\brk*{g\ind{t}_i},\mathwhere
q\ind{t}(i)\propto\exp\prn*{-\sum_{s=1}^{t-1}\logloss\ind{s}(g\ind{s}_{\idx})}.
  \]
  This algorithm guarantees that
  \[
    \RegLog \leq{} \log\abs{\cG}.
  \]
  Consequently, we have $\En\brk*{\EstHel}\leq{}\log\abs{\cG}$ and
  $\EstHel \leq{} \log\abs{\cG} + 2\log(\delta^{-1})$ with probability at
  least $1-\delta$.
\end{lemma}

\subsubsection{Infinite Classes}
We now provide results for infinite expert classes based on covering
numbers.
\begin{definition}
  \label{def:g_cover}
  We say that $\Jset\subseteq\Iset$ is an $\veps$-cover for $\cG$ if
  \[
    \forall{}\idx\in\Iset\quad\exists{}\jdx\in\Jset\quad\text{s.t.}\quad
    \sup_{t\leq{}T}\sup_{\hist\ind{t-1}}\sup_{x\in\Xspace}\Dhels{g\ind{t}_{\jdx}(x\midsem\hist\ind{t-1})}{g\ind{t}_{\idx}(x\midsem\hist\ind{t-1})}\leq{}\veps^{2}.
  \]
  We let $\cN(\cG,\veps)$ denote the size of the smallest such
  cover and define
  \begin{equation}
    \label{eq:g_comp}
    \GComp \ldef \inf_{\veps\geq{}0}\crl*{
      \log\cN(\cG,\veps)+\veps^{2}T
    }
  \end{equation}
  as an associated complexity parameter.
\end{definition}

We require the following mild regularity assumption.
\begin{assumption}
  \label{ass:finite_measure_ol}
  There exists a constant $B\geq{}e$ such that:
  \begin{enumerate}
    \item $\dom(\cY\mid{}x)\leq{}B$ for all $x\in\cX$.
    \item $\sup_{x\in\Xspace}\sup_{y\in\Yspace}g\ind{t}_i(y\mid{}x)\leq{}B$.
  \end{enumerate}

\end{assumption}

\paragraph{Algorithm}
We consider the following algorithm.   Define $B_x = \dom(\cY\mid{}x)$ for each $x\in\cX$. Let
  $\alpha\in(0,1)$ be a parameter, and for each $i$ define a
  smoothed conditional density
  \[
    h\ind{t}_\idx(y\mid{}x) = (1-\alpha)g\ind{t}_\idx(y\mid{}x) + \alpha{}B_x^{-1}.
  \]
  One can verify that this is indeed a probability density with
  respect to $\nu$, since for all $x$, $\int{}B_x^{-1}\nu(dy\mid{}x)=1$.
  Let $\Jset\subseteq\Iset$ witness the covering number
  $\cN(\cG,\veps)$, and let $\jstar\in\Jset$ be the corresponding
  cover element for $\gstar$; note that
  $\abs{\Jset}\leq\cN(\cG,\veps)$. We run the aggregating algorithm over
  $\crl*{h_{\jdx}}_{\jdx\in\Jset}$: we set
  \[
    q\ind{t}(j)\propto\exp\prn*{-\sum_{s=1}^{t-1}\logloss\ind{s}(h\ind{s}_{\jdx})}
  \]
  for each $j\in\Jset$, then predict with $\ghat\ind{t}=\En_{\jdx\sim{}q\ind{t}}\brk*{g\ind{t}_{\jdx}}$.

\begin{lemma}
  \label{lem:vovk_infinite_fast}
  Suppose \pref{ass:finite_measure_ol} holds. Then the algorithm above, with an appropriate setting for $\alpha$ and
$\veps$, guarantees that 
  \begin{equation}
    \label{eq:vovk_infinite_fast_expectation}
    \En\brk*{\EstH}
    \leq{}34\cdot{}\inf_{\veps>0}\crl*{\Bconst{}\cdot\veps^{2}T
      + \log\cN(\cG,\veps)} = \bigoh(\Bconst\cdot{}\GComp),
  \end{equation}
  where $\Bconst\ldef{}\log(2B^2T)$. Furthermore, the algorithm satisfies $\ghat\ind{t}\in\conv(\cG\ind{t})$.
  
  The same algorithm guarantees that for all $\delta\in(0,e^{-1})$, with probability at least $1-\delta$,
  \begin{align}
    \label{eq:vovk_infinite_fast}
    \EstH
    &\leq{}40\cdot{}\inf_{\veps>0}\crl*{\Bconst{}\cdot\veps^{2}T
      + \log\cN(\cG,\veps)}
      + 424\Bconst^{2}\log(\delta^{-1})\\
    &\leq{}\bigoh(\Bconst\cdot{}\GComp+\Bconst^{2}\log(\delta^{-1})).\notag
  \end{align}
\end{lemma}

\subsubsection{Expert Classes with Time-Varying Availability}
\newcommand{\cov}{\iota}
\newcommand{\covi}{\iota(\idx)}
\newcommand{\gc}{\check{g}}
\newcommand{\gcheck}{\gc}
\newcommand{\gheck}{\gc}
\newcommand{\vepscheck}{\check{\veps}}
\newcommand{\cGcheck}{\check{\cG}}
\renewcommand{\hbar}{\bar{h}}
\newcommand{\gbar}{\bar{g}}

We now add an additional twist to the setting in considered in the
prequel At each timestep $t$, we receive a subset
$\Iset\ind{t}\subseteq\Iset$, which is a measurable function of
$\hist\ind{t-1}$, and we are required to respect the constraint that
$\ghat\ind{t}\in\conv\prn*{\crl{g\ind{t}_{\idx}}_{\idx\in\Iset}}$. $\Iset\ind{t}$
may be thought of as a set of ``available'' experts. In particular, we use the
results in this section to prove \pref{thm:upper_main}, where
$\cI\ind{t}$ corresponds to a subset of models in a confidence set
constructed based on the data observed so far. Here, satisfying the
constraint that
$\ghat\ind{t}\in\conv\prn*{\crl{g\ind{t}_{\idx}}_{\idx\in\Iset}}$ allows us to appeal to the regret bound in Eq. \pref{eq:upper_general3} of \pref{thm:upper_general}.

As
in the previous section, we give a result for arbitrary, infinite
expert classes based on covering numbers.
\paragraph{Algorithm}
We give algorithm that builds on the algorithm for
infinite classes in \pref{lem:vovk_infinite_fast}, but incorporates tricks
from the \emph{sleeping experts} literature
\citep{cesa1997how,kleinberg2010regret}. With parameters $\veps>0$,
$\alpha\in(0,1)$, we proceed as follows:
\begin{itemize}
\item Let $\Jset\subseteq\Iset$ witness the covering number
  $\cN(\cG,\veps)$. Let $\cov:\Iset\to\Jset$ be any function that
  maps an index $i\in\Iset$ to a covering element
  $\cov(i)$ such that
  $\Dhels{g\ind{t}_{\idx}(x;\hist\ind{t-1})}{g\ind{t}_{\cov(\idx)}(x;\hist\ind{t-1})}\leq{}\veps^{2}$
  for all $x$, $\hist\ind{t-1}$.
\item For each $1\leq{}t\leq{}T$:
  \begin{itemize}
  \item Let $\Jset\ind{t}=\crl*{\covi}_{\idx\in\Iset\ind{t}}$.
  \item Define a modified class of experts
    $\cGcheck\ldef\crl{\gc\ind{t}_j}_{j\in\cJ}$ as follows: For each $j\in\cJ$:
    \begin{itemize}
    \item If $j\notin\Jset\ind{t}$ or
      $j\in\Jset\ind{t}\cap\Iset\ind{t}$, set $\gc\ind{t}_j=g\ind{t}_j$.
    \item If $j\in\Jset\ind{t}$ but $j\notin\Iset\ind{t}$, take
      $\gcheck\ind{t}_j=g\ind{t}_{k}$, where $k\in\Iset\ind{t}$ is any
      element such that
      \begin{equation}
        \label{eq:altcover}
        \Dhels{g\ind{t}_k(x;\hist\ind{t-1})}{g\ind{t}_{j}(x;\hist\ind{t-1})}\leq\veps^{2}\;\;\forall{}x.
      \end{equation}
      Such an element is guaranteed to exist, or else we would not
      have $j\in\Jset\ind{t}$ to begin with. Furthermore, since
      $\Iset\ind{t}$ is a measurable function of $\hist\ind{t-1}$, the
      resulting class
      $\cGcheck=\crl*{\gcheck\ind{t}_j}_{j\in\cJ}$ is itself a
      valid time-varying expert class.
    \end{itemize}
  \item For each $j\in\Jset$, define
    $h\ind{t}_j(y\mid{}x)=(1-\alpha)\gcheck\ind{t}_j(y\mid{}x) +
    \alpha{}B_x^{-1}$.
  \item Let $q\ind{t}$ denote the distribution produced by running the
    aggregating algorithm on the expert class
    $\crl*{\hbar\ind{t}_j}_{j\in\Jset}$, which we define inductively
    as follows:
    \begin{itemize}
    \item $q\ind{t}(j) \propto \exp\prn*{
      -\sum_{i=1}^{t-1}\logloss\ind{t}(\hbar\ind{i})
      }\;\;\forall{}j\in\Jset$.
    \item Define $\qbar\ind{t}(j) \ldef
      \frac{q\ind{t}(j)\indic\crl{j\in\Jset\ind{t}}}{\sum_{j\in\Jset\ind{t}}q\ind{t}(j)}$
      and
      $v\ind{t}(y\mid{}x)\ldef\sum_{j\in\Jset}\qbar\ind{t}(j)h\ind{t}_j(y\mid{}x)$.
    \item Set $\hbar\ind{t}_j=h\ind{t}_j$ if $j\in\Jset\ind{t}$ and
      $\hbar\ind{t}_j=v\ind{t}$ otherwise. Note that $q\ind{t}$
      depends only on $\crl{\hbar_j\ind{1},\ldots,\hbar\ind{t-1}_j}_{j\in\Jset}$, and
      hence all of these quantities are well-defined and
      $\crl*{\hbar_j}_{j\in\cJ}$ is a valid class of time-varying experts.
    \end{itemize}

  \item Predict via $\ghat\ind{t}=\En_{j\sim{}\qbar\ind{t}}\brk*{\gcheck\ind{t}_j}$.
  \end{itemize}
\end{itemize}

\begin{lemma}
  \label{lem:vovk_sleeping_fast}
  Suppose \pref{ass:finite_measure_ol} holds. Then the algorithm above, with an appropriate setting for $\alpha$ and
$\veps$, guarantees that for all $\delta\in(0,e^{-1})$, with probability at least $1-\delta$,
  \begin{equation}
    \label{eq:vovk_infinite_sleeping_fast}
    \sum_{t=1}^{T}\En_{t-1}\brk*{\Dhels{\ghat\ind{t}(x\ind{t})}{\gstar\ind{t}(x\ind{t})}}\indic\crl*{\istar\in\Iset\ind{t}}
    \leq{}160\cdot{}\inf_{\veps>0}\crl*{\Bconst{}\cdot\veps^{2}T
      + \log\cN(\cG,\veps)}
    + 424\Bconst^{2}\log(\delta^{-1}),
  \end{equation}
  where $\Bconst\ldef{}\log(2B^2T)$. Furthermore, the algorithm satisfies $\ghat\ind{t}\in\conv\prn*{\crl{g\ind{t}_{\idx}}_{\idx\in\Iset\ind{t}}}$.
\end{lemma}

For finite classes, we prove a slightly tighter version of this guarantee
that does not depend on the parameter $B$. We consider the following
simplified version of the algorithm above.
\begin{itemize}
\item For each $1\leq{}t\leq{}T$:
\begin{itemize}
  \item Let $q\ind{t}$ denote the distribution produced by running the
    aggregating algorithm on the expert class
    $\crl*{\gbar\ind{t}_i}_{i\in\Iset}$, which we define inductively
    as follows:
    \begin{itemize}
    \item $q\ind{t}(i) \propto \exp\prn*{
      -\sum_{i=1}^{t-1}\logloss\ind{t}(\gbar\ind{i})
      }\;\;\forall{}i\in\Iset$.
    \item Define $\qbar\ind{t}(i) \ldef
      \frac{q\ind{t}(i)\indic\crl{i\in\Iset\ind{t}}}{\sum_{i\in\Iset\ind{t}}q\ind{t}(i)}$
      and
      $v\ind{t}(y\mid{}x)\ldef\sum_{i\in\Iset}\qbar\ind{t}(i)h\ind{t}_i(y\mid{}x)$.
    \item Set $\gbar\ind{t}_i=g\ind{t}_i$ if $i\in\Iset\ind{t}$ and
      $\gbar\ind{t}_i=v\ind{t}$ otherwise. As before, $q\ind{t}$
      depends only on $\crl{\gbar_i\ind{1},\ldots,\gbar\ind{t-1}_i}_{i\in\Iset}$, and
      hence all of these quantities are well-defined and
      $\crl*{\gbar_i}_{i\in\cI}$ is a valid class of time-varying experts.
    \end{itemize}
  \item Predict via $\ghat\ind{t}=\En_{i\sim{}\qbar\ind{t}}\brk*{g\ind{t}_i}$.
  \end{itemize}
\end{itemize}
\begin{lemma}
  \label{lem:vovk_sleeping_fast_finite}
  The algorithm above guarantees that for all $\delta\in(0,e^{-1})$, with probability at least $1-\delta$,
  \begin{equation}
    \label{eq:vovk_sleeping_fast_finite}
    \sum_{t=1}^{T}\En_{t-1}\brk*{\Dhels{\ghat\ind{t}(x\ind{t})}{\gstar\ind{t}(x\ind{t})}}\indic\crl*{\istar\in\Iset\ind{t}}
    \leq{}\log\abs{\cG}
    + 2\log(\delta^{-1}).
  \end{equation}
  Furthermore, the algorithm satisfies $\ghat\ind{t}\in\conv\prn*{\crl{g\ind{t}_{\idx}}_{\idx\in\Iset\ind{t}}}$.
\end{lemma}

\subsubsection{Proofs}

\begin{proof}[\pfref{lem:logloss_hellinger_ol}]
  We first-prove the in-expectation result. From \pref{eq:regret_ol}, we have that
  \[
    \sum_{t=1}^{T}\logloss\ind{t}(\ghat\ind{t})
    -
    \sum_{t=1}^{T}\logloss\ind{t}(g\ind{t}_{\istar})\leq{}\RegLog.
  \]
  Taking expectations, \pref{ass:realizability_ol} implies that
  \[
    \sum_{t=1}^{T}\En\brk*{\Dkl{g\ind{t}_{\istar}(x\ind{t})}{\ghat\ind{t}(x\ind{t})}}
    \leq{}\En\brk*{\RegLog}.
  \]
  The result now follows from \pref{lem:pinsker}, which states that $\Dhels{g\ind{t}_{\istar}(x\ind{t})}{\ghat\ind{t}(x\ind{t})}\leq{} \Dkl{g\ind{t}_{\istar}(x\ind{t})}{\ghat\ind{t}(x\ind{t})}$.

  We now prove the high-probability result in
  \pref{eq:logloss_hellinger_ol3} by adapting an argument from
  \citet{zhang2006from} (see also \citet{geer2000empirical}); the result in
  \pref{eq:logloss_hellinger_ol2} is a special case. Define
$X_t=\frac{1}{2}(\logloss\ind{t}(\ghat\ind{t}) -
\logloss\ind{t}(g_{\istar}\ind{t}))$ and $Z_t=X_t\cdot\indt$. Applying
\pref{lem:martingale_chernoff} with the sequence $(-Z_t)_{t\leq{}T}$,
we are guaranteed that with probability at least $1-\delta$,
\[
  \sum_{t=1}^{T}-\log\prn*{\En_{t-1}\brk*{e^{-Z_t}}}
  \leq{} \sum_{t=1}^{T}Z_t + \log(\delta^{-1})
  = \frac{1}{2}\sum_{t=1}^{T}\prn*{\logloss\ind{t}(\ghat\ind{t}) -
      \logloss\ind{t}(\gstar\ind{t})}\indt + \log(\delta^{-1})
\]
Let $t$ be fixed, and note that since $\indt$ is $\filt\ind{t-1}$-measurable,
\[
  \En_{t-1}\brk*{e^{-Z_t}}
  = (1-\indt) + \indt \En_{t-1}\brk*{e^{-X_t}}.
\]
Next, we observe that
\begin{align*}
  \En_{t-1}\brk*{e^{-X_t}\mid{}x\ind{t}}
  &= \En_{t-1}\brk*{
  \sqrt{\frac{\ghat\ind{t}(y\ind{t}\mid{}x\ind{t})}{g\ind{t}_{\istar}(y\ind{t}\mid{}x\ind{t})}}\mid{}x\ind{t}
  }\\
    &=
\int{}g\ind{t}_{\istar}(y\mid{}x\ind{t})\sqrt{\frac{\ghat\ind{t}(y\mid{}x\ind{t})}{g\ind{t}_{\istar}(y\mid{}x\ind{t})}}\dom(dy\mid{}x\ind{t})\\
  &=
    \int{}\sqrt{g\ind{t}_{\istar}(y\mid{}x\ind{t})\ghat\ind{t}(y\mid{}x\ind{t})}\dom(dy\mid{}x\ind{t})
    = 1 - \frac{1}{2}\Dhels{g\ind{t}_{\istar}(x\ind{t})}{\ghat\ind{t}(x\ind{t})}.
\end{align*}
Hence,
\[
  \En_{t-1}\brk*{e^{-Z_t}} = 1 - \frac{1}{2}\Dhels{g\ind{t}_{\istar}(x\ind{t})}{\ghat\ind{t}(x\ind{t})}\indt
\]
and, since $-\log(1-x)\geq{}x$ for $x\in\brk*{0,1}$, we conclude
that
\[
  \frac{1}{2}\sum_{t=1}^{T}\En_{t-1}\brk*{\Dhels{g\ind{t}_{\istar}(x\ind{t})}{\ghat\ind{t}(x\ind{t})}}\indt
  \leq{} \frac{1}{2}\sum_{t=1}^{T}\prn*{\logloss\ind{t}(\ghat\ind{t}) -
      \logloss\ind{t}(\gstar\ind{t})}\indt + \log(\delta^{-1}).
\]
\end{proof}

\begin{proof}[\pfref{lem:vovk_infinite_fast}]
  Recall that we run the aggregating algorithm over
  $\crl*{h_{\jdx}}_{\jdx\in\Jset}$ which, per \pref{lem:vovk_fast}
  guarantees that
  \[
    \sum_{t=1}^{T}\logloss\ind{t}(\hhat\ind{t})
    -
    \min_{\jdx\in\Jset}\sum_{t=1}^{T}\logloss\ind{t}(h\ind{t}_{\jdx})
    \leq \log\abs{\Jset},
  \]
  where $\hhat\ind{t}\ldef\En_{j\sim{}q\ind{t}}\brk*{h\ind{t}_{\jdx}}$. In
  particular, we have
    \[
    \sum_{t=1}^{T}\logloss\ind{t}(\hhat\ind{t})
    -
    \sum_{t=1}^{T}\logloss\ind{t}(h\ind{t}_{\jstar})
    \leq \log\abs{\Jset}.
  \]
  \paragraph{In-expectation bound}
  Let $X_t = \logloss\ind{t}(\hhat\ind{t}) -
  \logloss\ind{t}(h\ind{t}_{\jstar})$, which has
  \[
    \abs{X_t}
    \leq{} \log(2B^{2}/\alpha)\rdef{}R
  \]
  by \pref{ass:finite_measure_ol}. Our starting point is to write
  \[
    \En\brk*{\sum_{t=1}^{T}\En_{t-1}\brk*{X_t}}
    \leq{}\log\abs{\Jset}.
  \]
    Observe that
  \begin{align*}
    \En_{t-1}\brk*{X_t}
    &= \En_{t-1}\brk*{
      \Dkl{\gstar\ind{t}(x\ind{t})}{\hhat\ind{t}(x\ind{t})}
    } - \En_{t-1}\brk*{
      \Dkl{\gstar\ind{t}(x\ind{t})}{\hstar\ind{t}(x\ind{t})}
      }\\
    &\geq{} \En_{t-1}\brk*{
      \Dhels{\gstar\ind{t}(x\ind{t})}{\hhat\ind{t}(x\ind{t})}
    } - \En_{t-1}\brk*{
      \Dkl{\gstar\ind{t}(x\ind{t})}{\hstar\ind{t}(x\ind{t})}
      },
      \end{align*}
  where the second inequality uses \pref{lem:pinsker}.   Since
  $\gstar\ind{t}(y\mid{}x)/\hstar\ind{t}(y\mid{}x)\leq{}B^{2}/\alpha$,
  \pref{lem:kl_hellinger} gives
  \begin{align*}
    \En_{t-1}\brk*{
    \Dkl{\gstar\ind{t}(x\ind{t})}{\hstar\ind{t}(x\ind{t})}
    }
    &\leq{}
    3R \En_{t-1}\brk*{
    \Dhels{\gstar\ind{t}(x\ind{t})}{\hstar\ind{t}(x\ind{t})}
    }\\
      &\leq{}
    3R \En_{t-1}\brk*{
    \Dhels{\gstar\ind{t}(x\ind{t})}{g_{\jstar}\ind{t}(x\ind{t})}
        } + 6R\alpha\\
          &\leq{}
            6R(\alpha+\veps^{2}).
  \end{align*}
      Furthermore, recalling that $\ghat\ind{t}\ldef{}\En_{j\sim{}q\ind{t}}\brk*{g\ind{t}_{\jdx}}$ we have
    \[
      \Dhels{\gstar\ind{t}(x\ind{t})}{\ghat\ind{t}(x\ind{t})}
    \leq{} 2 \Dhels{\gstar\ind{t}(x\ind{t})}{\hhat\ind{t}(x\ind{t})}
  + 2 \Dhels{\ghat\ind{t}(x\ind{t})}{\hhat\ind{t}(x\ind{t})}
\leq{} 2 \Dhels{\gstar\ind{t}(x\ind{t})}{\hhat\ind{t}(x\ind{t})}
  + 4\alpha.
\]
Combining all of these results, we have that
\begin{equation}
  \label{eq:ol_infinite_expectation}
  \sum_{t=1}^{T}\En_{t-1}\brk*{\Dhels{\gstar\ind{t}(x\ind{t})}{\ghat\ind{t}(x\ind{t})}}
  \leq{} 2\sum_{t=1}^{T}\En_{t-1}\brk*{X_t}
  + 16R(\alpha+\veps^{2})T.
\end{equation}
and
\[
  \En\brk*{\EstHel}
  \leq{} 2\log\abs{\cJ}
  + 16R(\alpha+\veps^{2})T.
\]
To conclude, we set $\alpha=(T\log(2B^{2}T))$, so that
\[
R\alpha{}T=\frac{\log(2B^{2}T\log(2B^{2}T))}{\log(2B^2T)}\leq{}2,
\]
which allows us to upper bound
\[
  \En\brk*{\EstHel}
  \leq{} 34\log\abs{\cJ}
  + 32\log(2B^2T)\veps^{2}T.
\]

\paragraph{High-probability bound}
Applying \pref{lem:freedman} to
  the sequence $X'_t\ldef{}\En_{t-1}\brk*{X_t}-X_t$, we are guaranteed
  that for any $\eta\in(0,R^{-1})$, with probability at least $1-\delta$,
  \begin{align}
    \sum_{t=1}^{T}\En_{t-1}\brk*{X_t}
    &\leq{} \sum_{t=1}^{T}X_t +
    \eta\sum_{t=1}^{T}\En_{t-1}\brk*{(X_t-\En_{t-1}\brk*{X_t})^{2}}
      + \eta^{-1}\log(\delta^{-1})\notag\\
    &\leq{} \log\abs{\Jset} +
    \eta\sum_{t=1}^{T}\En_{t-1}\brk*{X_t^{2}}
          + \eta^{-1}\log(\delta^{-1}).\label{eq:ol_infinite_basic}
  \end{align}
  Let $t$ be fixed. Let $\Entil_{t}$ denote the conditional
  expectation for a modified process in which we draw
  $y\ind{t}\sim{}h\ind{t}_{\jstar}(\cdot\mid{}x\ind{t};\hist\ind{t-1})\mid{}x\ind{t},\hist\ind{t-1}$,
  rather than $y\ind{t}\sim{}\gstar(\cdot\mid{}\cdots)$. Using
  \pref{lem:mp_min}, we have
  \[
    \En_{t-1}\brk*{X_t^{2}}
    \leq{} 3 \Entil_{t-1}\brk*{X_t^{2}} + 4R^{2}\En_{t-1}\brk*{\Dhels{\gstar\ind{t}(x\ind{t})}{h\ind{t}_{\jstar}(x\ind{t})}}.
  \]
  Furthermore, by convexity and boundedness of squared Hellinger
  distance along with the covering property for $\cJ$, we have
  \[
    \En_{t-1}\brk*{\Dhels{\gstar\ind{t}(x\ind{t})}{h\ind{t}_{\jstar}(x\ind{t})}}
    \leq{} 2\alpha +
    \En_{t-1}\brk*{\Dhels{\gstar\ind{t}(x\ind{t})}{g\ind{t}_{\jstar}(x\ind{t})}}\leq{}2\alpha
    + \veps^{2}.
    \]
  We now appeal to the following lemma.
  \begin{lemma}[Central-to-Bernstein \citep{mehta2017fast}]
    \label{lem:central_to_bernstein}
    Let $X$ be a random variable taking values in $\brk*{-R,R}$. If
    $\En\brk*{e^{-X}}\leq{}1$, then $\En\brk*{X^{2}}\leq{}4(1+R)\En\brk*{X}$.
  \end{lemma}
  Observe that
  \[
    \Entil_{t-1}\brk*{e^{-X_t}\mid{}x\ind{t}}
    =
    \Entil_{t-1}\brk*{\frac{\hhat\ind{t}(y\ind{t}\mid{}x\ind{t})}{\hstar\ind{t}(y\ind{t}\mid{}x\ind{t})}\mid{}x\ind{t}}
    = \int\hstar(y\mid{}x\ind{t})\frac{\hhat\ind{t}(y\mid{}x\ind{t})}{\hstar\ind{t}(y\mid{}x\ind{t})}\dom(dy\mid{}x\ind{t})=1,
  \]
  so \pref{lem:central_to_bernstein} implies that
  \[
    \Entil_{t-1}\brk*{X_t^{2}}\leq{} 8R\Entil_{t-1}\brk*{X_t} = 8R\En_{t-1}\brk*{\Dkl{\hstar\ind{t}(x\ind{t})}{\hhat\ind{t}(x\ind{t})}}.
  \]
  Furthermore, since 
  $\hstar\ind{t}(y\mid{}x)/\hhat\ind{t}(y\mid{}x)\leq{}2B^{2}/\alpha$,
  \pref{lem:kl_hellinger} implies that
  \[
    \En_{t-1}\brk*{\Dkl{\hstar\ind{t}(x\ind{t})}{\hhat\ind{t}(x\ind{t})}}
    \leq{}(2+\log(2B^{2}/\alpha))\En_{t-1}\brk*{\Dhels{\hstar\ind{t}(x\ind{t})}{\hhat\ind{t}(x\ind{t})}}
    \leq{}3R\En_{t-1}\brk*{\Dhels{\hstar\ind{t}(x\ind{t})}{\hhat\ind{t}(x\ind{t})}}.
  \]
  Putting everything together, we have that
  \begin{align*}
    \sum_{t=1}^{T}\En_{t-1}\brk{X_{t}^{2}}
    &\leq{}
    24R^{2}\sum_{t=1}^{T}\En_{t-1}\brk*{\Dhels{\hstar\ind{t}(x\ind{t})}{\hhat\ind{t}(x\ind{t})}}
      + 4R^{2}(2\alpha+\veps^{2})T\\
    &\leq{}
      24R^{2}\sum_{t=1}^{T}\En_{t-1}\brk*{\Dhels{\gstar\ind{t}(x\ind{t})}{\hhat\ind{t}(x\ind{t})}}
      + 56R^{2}(\alpha+\veps^{2})T,
  \end{align*}
  where the second inequality uses convexity of squared Hellinger
  distance. We set $\eta=1/96R^{2}\leq{}R^{-1}$ and plug this inequality into
  \pref{eq:ol_infinite_basic} to get
  \begin{align}
    \sum_{t=1}^{T}\En_{t-1}\brk*{X_t}
    &\leq{} 
\log\abs{\Jset} +
\frac{1}{4}\sum_{t=1}^{T}\En_{t-1}\brk*{\Dhels{\gstar\ind{t}(x\ind{t})}{\hhat\ind{t}(x\ind{t})}}
+ (\alpha+\veps^{2})T
      + 96R^{2}\log(\delta^{-1})\notag\\
    &\leq{} 
\log\abs{\Jset} +
\frac{1}{4}\sum_{t=1}^{T}\En_{t-1}\brk*{\Dhels{\gstar\ind{t}(x\ind{t})}{\ghat\ind{t}(x\ind{t})}}
+ 2(\alpha+\veps^{2})T
      + 96R^{2}\log(\delta^{-1}).
        \label{eq:ol_infinite_basic2}
  \end{align}
  Combining this with \pref{eq:ol_infinite_expectation}, we have
    \begin{align*}
      &\sum_{t=1}^{T}\En_{t-1}\brk*{\Dhels{\gstar\ind{t}(x\ind{t})}{\ghat\ind{t}(x\ind{t})}}\\
      &\leq{} 2\sum_{t=1}^{T}\sum_{t=1}^{T}\En_{t-1}\brk*{X_t}
        + 16R(\alpha+\veps^{2})T\\
            &\leq{} 2\log\abs{\Jset}
              +\frac{1}{2}\sum_{t=1}^{T}\En_{t-1}\brk*{\Dhels{\gstar\ind{t}(x\ind{t})}{\ghat\ind{t}(x\ind{t})}}
              + 20R(\alpha+\veps^{2})T
              + 192R^{2}\log(\delta^{-1}).
    \end{align*}
    Rearranging gives
    \begin{align*}
      \sum_{t=1}^{T}\En_{t-1}\brk*{\Dhels{\gstar\ind{t}(x\ind{t})}{\ghat\ind{t}(x\ind{t})}}
      \leq{}
      4\log\abs{\Jset}
              + 40R(\alpha+\veps^{2})T
              + 384R^{2}\log(\delta^{-1}).
    \end{align*}
    We set $\alpha=1/T$ and use that $\delta\leq{}e^{-1}$ to upper
    bound 
    \[
      \EstHel \leq{} 4\log\abs{\Jset}
      + 40R\veps^{2}T
+ 424R^{2}\log(\delta^{-1}).
    \]

  \end{proof}

  \begin{proof}[\pfref{lem:vovk_sleeping_fast}]
  Let $\jstar=\cov(\istar)$, and note that $\istar\in\Iset\ind{t}$
  implies that $\jstar\in\Jset\ind{t}$. We first observe that since
  $\crl*{\hbar_j}_{j\in\cJ}$ is a valid class of experts and
  $q\ind{t}$ is selected by running the aggregating algorithm on this
  class, \pref{lem:vovk_fast} guarantees that
  \[
    \sum_{t=1}^{T}\logloss\ind{t}(\En_{j\sim{}q\ind{t}}\brk{\hbar\ind{t}_j})
    -\min_{j\in\Jset}\sum_{t=1}^{T}\logloss\ind{t}(\hbar\ind{t}_j)
    \leq{} \log\abs{\cJ}
  \]
  and in particular,
  \begin{equation}
    \sum_{t=1}^{T}\logloss\ind{t}(\En_{j\sim{}q\ind{t}}\brk{\hbar\ind{t}_j})
    -\sum_{t=1}^{T}\logloss\ind{t}(\hbar\ind{t}_{\jstar})
    \leq{} \log\abs{\cJ}.\label{eq:sleeping_basic}
  \end{equation}
  Next, we use the definition of $\hbar_{j}\ind{t}$, which implies
  that
  \begin{align*}
    \En_{j\sim{}q\ind{t}}\brk{\hbar\ind{t}_j}
    &= \sum_{j\in\cJ\ind{t}}q\ind{t}(j)h\ind{t}_j
    + \sum_{j\notin\cJ\ind{t}}q\ind{t}(j)v\ind{t}\\
    &=
\prn*{\sum_{j\in\cJ\ind{t}}q\ind{t}(j)}\sum_{j\in\cJ\ind{t}}\qbar\ind{t}(j)h\ind{t}_j+\prn*{1-\sum_{j\in\cJ\ind{t}}q\ind{t}(j)}\sum_{j\in\cJ\ind{t}}\qbar\ind{t}(j)h\ind{t}_j\\
    &= \En_{j\sim{}\qbar\ind{t}}\brk*{h\ind{t}_j}.
  \end{align*}
  Likewise, we have
  \[
    \hbar\ind{t}_{\jstar} = \left\{
      \begin{array}{ll}
        h\ind{t}_{\jstar},&\quad{}\jstar\in\cJ\ind{t},\\
        \En_{j\sim{}\qbar\ind{t}}\brk*{h\ind{t}_j},&\quad{}\jstar\notin\cJ\ind{t}.        
      \end{array}
      \right.
  \]
  As a result, \pref{eq:sleeping_basic} implies that
  \begin{align}
    \label{eq:sleeping_basic2}
    \sum_{t=1}^{T}\prn*{\logloss\ind{t}(\En_{j\sim{}\qbar\ind{t}}\brk{h\ind{t}_j})
    -\logloss\ind{t}(h\ind{t}_{\jstar})}\indic\crl*{\jstar\in\cJ\ind{t}} \leq\log\abs{\cJ}.
  \end{align}
  From here, the proof closely follows that of
  \pref{lem:vovk_infinite_fast}, so we omit details for certain steps. We
  will use the following fact.
    \begin{proposition}
      \label{prop:sleeping_cover}The set $\crl*{\gcheck_{j}}_{j\in\cJ}$
      is a $\vepscheck\ldef{}2\veps$-cover for $\cG$, and has
      $\crl*{\gcheck\ind{t}_{j}}_{j\in\cJ\ind{t}}\subseteq\crl*{g\ind{t}_{i}}_{i\in\cI\ind{t}}$
      for all $t$.
    \end{proposition}
    \begin{proof}
      Consider a fixed step $t$ and
      let $i\in\Iset$ and $j=\cov(i)$. For the first case, if $j\notin\Jset\ind{t}$ or
      $j\in\Jset\ind{t}\cap\Iset\ind{t}$, we have
      $\gcheck\ind{t}_j=g\ind{t}_j$, so
      $\sup_{x}\Dhels{\gcheck\ind{t}_j(x)}{g_i\ind{t}(x)}\leq\veps^{2}$
      by the cover property of $\cJ$. Furthermore, if
      $j\in\cJ\ind{t}$, we have
      $\gcheck\ind{t}_j\in\crl*{g\ind{t}_{i}}_{i\in\cI\ind{t}}$ by
      definition. For the second case, if
      $j\in\Jset\ind{t}\setminus\Iset\ind{t}$, we have
      $\gcheck_j\ind{t}=g_k\ind{t}$ for some $k\in\cI\ind{t}$ such
      that
      $\sup_{x}\Dhels{g\ind{t}_j(x)}{g_k\ind{t}(x)}\leq\veps^{2}$. In
      this case it is immediate that
      $\gcheck\ind{t}_j\in\crl*{g\ind{t}_{i}}_{i\in\cI\ind{t}}$, and
      by the triangle inequality we have that for all $x$,
      \[
        \Dhel{\gcheck\ind{t}_j(x)}{g_i\ind{t}(x)}
        \leq{} \Dhel{g\ind{t}_j(x)}{g_i\ind{t}(x)} +
        \Dhel{g\ind{t}_j(x)}{g_k\ind{t}(x)} \leq{} 2\veps.
      \]
    \end{proof}
  Let
  $\indt\ldef{}\indic\crl*{\jstar\in\cJ\ind{t}}$, and let $X_t=\prn*{\logloss\ind{t}(\En_{j\sim{}\qbar\ind{t}}\brk{h\ind{t}_j})
    -\logloss\ind{t}(h\ind{t}_{\jstar})}$ and $Z_t=X_t\cdot\indt$,
  which has $\abs{Z_t}\leq{}\abs{X_t}\leq{}\log(2B^2/\alpha)\ldef{}R$. Applying \pref{lem:freedman} to
  the sequence $Z'_t\ldef{}\En_{t-1}\brk*{Z_t}-Z_t$, we are guaranteed
  that for any $\eta\in(0,R^{-1})$, with probability at least $1-\delta$,
  \begin{align}
    \sum_{t=1}^{T}\En_{t-1}\brk*{X_t}\indt
    &\leq{} \log\abs{\Jset} +
    \eta\sum_{t=1}^{T}\En_{t-1}\brk*{X_t^{2}}\indt
          + \eta^{-1}\log(\delta^{-1}),\label{eq:ol_sleeping_basic}
  \end{align}
  where we have used that $\indt$ is $\filt\ind{t-1}$-measurable. Defining $\hhat\ind{t}=\En_{j\sim\qbar\ind{t}}\brk*{h\ind{t}_j}$, we
have
  \begin{align*}
    \En_{t-1}\brk*{X_t}
    &= \En_{t-1}\brk*{
      \Dkl{\gstar\ind{t}(x\ind{t})}{\hhat\ind{t}(x\ind{t})}
    } - \En_{t-1}\brk*{
      \Dkl{\gstar\ind{t}(x\ind{t})}{\hstar\ind{t}(x\ind{t})}
      }\\
    &\geq{} \En_{t-1}\brk*{
      \Dhels{\gstar\ind{t}(x\ind{t})}{\hhat\ind{t}(x\ind{t})}
    } - \En_{t-1}\brk*{
      \Dkl{\gstar\ind{t}(x\ind{t})}{\hstar\ind{t}(x\ind{t})}
      }.
  \end{align*}
  Following \pref{lem:vovk_infinite_fast}, we compute
  \begin{align*}
    \En_{t-1}\brk*{
    \Dkl{\gstar\ind{t}(x\ind{t})}{\hstar\ind{t}(x\ind{t})}
    }
      &\leq{}
    3R \En_{t-1}\brk*{
    \Dhels{\gstar\ind{t}(x\ind{t})}{\gcheck_{\jstar}\ind{t}(x\ind{t})}
        } + 6R\alpha\\
          &\leq{}
            6R(\alpha+\vepscheck^{2}).
  \end{align*}
  Recalling that
  $\ghat\ind{t}\ldef{}\En_{j\sim{}\qbar\ind{t}}\brk*{\gcheck\ind{t}_{\jdx}}$,
  we have
    \[
      \Dhels{\gstar\ind{t}(x\ind{t})}{\ghat\ind{t}(x\ind{t})}
    \leq{} 2 \Dhels{\gstar\ind{t}(x\ind{t})}{\hhat\ind{t}(x\ind{t})}
  + 2 \Dhels{\ghat\ind{t}(x\ind{t})}{\hhat\ind{t}(x\ind{t})}
\leq{} 2 \Dhels{\gstar\ind{t}(x\ind{t})}{\hhat\ind{t}(x\ind{t})}
  + 4\alpha.
\]
We conclude that 
\begin{equation}
  \label{eq:ol_sleeping_expectation}
  \sum_{t=1}^{T}\En_{t-1}\brk*{\Dhels{\gstar\ind{t}(x\ind{t})}{\ghat\ind{t}(x\ind{t})}}\indt
  \leq{} 2\sum_{t=1}^{T}\sum_{t=1}^{T}\En_{t-1}\brk*{X_t}\indt
  + 16R(\alpha+\vepscheck^{2})T.
\end{equation}
Next, using the same calculation as in \pref{lem:vovk_infinite_fast}
(with $g\ind{t}$ replaced by $\gcheck\ind{t}$), we have that
  \begin{align*}
    \sum_{t=1}^{T}\En_{t-1}\brk{X_{t}^{2}}\indt
    &\leq{}
      24R^{2}\sum_{t=1}^{T}\En_{t-1}\brk*{\Dhels{\gstar\ind{t}(x\ind{t})}{\hhat\ind{t}(x\ind{t})}}\indt
      + 56R^{2}(\alpha+\vepscheck^{2})T\\
        &\leq{}
      24R^{2}\sum_{t=1}^{T}\En_{t-1}\brk*{\Dhels{\gstar\ind{t}(x\ind{t})}{\ghat\ind{t}(x\ind{t})}}\indt
      + 104R^{2}(\alpha+\vepscheck^{2})T.
  \end{align*}
  We set $\eta=1/96R^{2}\leq{}R^{-1}$ and plug this inequality into
  \pref{eq:ol_sleeping_basic} to get
  \begin{align}
    \sum_{t=1}^{T}\En_{t-1}\brk*{X_t}\indt
    &\leq{} 
\log\abs{\Jset} +
\frac{1}{4}\sum_{t=1}^{T}\En_{t-1}\brk*{\Dhels{\gstar\ind{t}(x\ind{t})}{\ghat\ind{t}(x\ind{t})}}\indt
+ 2(\alpha+\vepscheck^{2})T
      + 96R^{2}\log(\delta^{-1}).
        \label{eq:ol_sleeping_basic2}
  \end{align}
  Combining this with \pref{eq:ol_sleeping_expectation}, we have
    \begin{align*}
      &\sum_{t=1}^{T}\En_{t-1}\brk*{\Dhels{\gstar\ind{t}(x\ind{t})}{\ghat\ind{t}(x\ind{t})}}\indt\\
      &\leq{} 2\sum_{t=1}^{T}\sum_{t=1}^{T}\En_{t-1}\brk*{X_t}\indt
        + 16R(\alpha+\vepscheck^{2})T\\
            &\leq{} 2\log\abs{\Jset}
              +\frac{1}{2}\sum_{t=1}^{T}\En_{t-1}\brk*{\Dhels{\gstar\ind{t}(x\ind{t})}{\ghat\ind{t}(x\ind{t})}}\indt
              + 20R(\alpha+\vepscheck^{2})T
              + 192R^{2}\log(\delta^{-1}),
    \end{align*}
which, after rearranging, gives
    \begin{align*}
      \sum_{t=1}^{T}\En_{t-1}\brk*{\Dhels{\gstar\ind{t}(x\ind{t})}{\ghat\ind{t}(x\ind{t})}}\indt
      \leq{}
      4\log\abs{\Jset}
              + 40R(\alpha+\vepscheck^{2})T
              + 384R^{2}\log(\delta^{-1}).
    \end{align*}
    Since $\istar\in\cI\ind{t}$ implies $\jstar\in\cJ\ind{t}$ by
    construction, this implies that
        \begin{align*}
      \sum_{t=1}^{T}\En_{t-1}\brk*{\Dhels{\gstar\ind{t}(x\ind{t})}{\ghat\ind{t}(x\ind{t})}}\indic\crl{\istar\in\cI\ind{t}}
      \leq{}
      4\log\abs{\Jset}
              + 40R(\alpha+\vepscheck^{2})T
              + 384R^{2}\log(\delta^{-1}).
    \end{align*}
    We set $\alpha=1/T$ and use that $\delta\leq{}e^{-1}$ to upper
    bound 
    \[
      \EstHel \leq{} 4\log\abs{\Jset}
      + 40R\vepscheck^{2}T
      + 424R^{2}\log(\delta^{-1})
      \leq{} 4\log\abs{\Jset}
      + 160R\veps^{2}T
+ 424R^{2}\log(\delta^{-1}).
    \]
Since
$\ghat\ind{t}=\En_{i\sim\qbar\ind{t}}\brk*{\gheck\ind{t}_j}$ and
$\qbar\ind{t}$ is supported on $\cJ\ind{t}$,
\pref{prop:sleeping_cover} implies that $\ghat\ind{t}\in\conv(\crl{g\ind{t}_{i}}_{i\in\cI\ind{t}})$.
\end{proof}

\begin{proof}[\pfref{lem:vovk_sleeping_fast_finite}]
  We first observe that since
  $\crl*{\gbar_j}_{i\in\cI}$ is a valid class of experts and
  $q\ind{t}$ is selected by running the aggregating algorithm on this
  class, \pref{lem:vovk_fast} guarantees that
  \begin{equation}
    \sum_{t=1}^{T}\logloss\ind{t}(\En_{i\sim{}q\ind{t}}\brk{\gbar\ind{t}_i})
    -\sum_{t=1}^{T}\logloss\ind{t}(\gbar\ind{t}_{\istar})
    \leq{} \log\abs{\cI}.\label{eq:sleeping_basic_finite}
  \end{equation}
  Next, we use the definition of $\gbar_{i}\ind{t}$, which implies
  that
  \begin{align*}
    \En_{j\sim{}q\ind{t}}\brk{\gbar\ind{t}_i}
    &= \sum_{i\in\cI\ind{t}}q\ind{t}(i)g\ind{t}_i
    + \sum_{i\notin\cI\ind{t}}q\ind{t}(i)v\ind{t}\\
    &=
\prn*{\sum_{i\in\cI\ind{t}}q\ind{t}(i)}\sum_{i\in\cI\ind{t}}\qbar\ind{t}(i)g\ind{t}_i+\prn*{1-\sum_{i\in\cI\ind{t}}q\ind{t}(i)}\sum_{i\in\cI\ind{t}}\qbar\ind{t}(i)g\ind{t}_i\\
    &= \En_{i\sim{}\qbar\ind{t}}\brk*{g\ind{t}_i}.
  \end{align*}
  Likewise, we have
  \[
    \gbar\ind{t}_{\istar} = \left\{
      \begin{array}{ll}
        g\ind{t}_{\istar},&\quad{}\istar\in\cI\ind{t},\\
        \En_{i\sim{}\qbar\ind{t}}\brk*{g\ind{t}_i},&\quad{}\istar\notin\cI\ind{t}.        
      \end{array}
      \right.
  \]
  As a result, \pref{eq:sleeping_basic_finite} implies that
  \begin{align}
    \label{eq:sleeping_basic_finite2}
    \sum_{t=1}^{T}\prn*{\logloss\ind{t}(\En_{i\sim{}\qbar\ind{t}}\brk{g\ind{t}_i})
    -\logloss\ind{t}(g\ind{t}_{\istar})}\indic\crl*{\istar\in\cI\ind{t}} \leq\log\abs{\cI}.
  \end{align}
  The result now follows from \pref{lem:logloss_hellinger_ol}.
\end{proof}

\section{\CompText: Structural Results}
\label{app:structural}
  \newcommand{\fmp}{f\sups{M'}}
  \newcommand{\fmbarp}{f\sups{\Mbar'}}
  \newcommand{\Mnot}{M_0}
  \newcommand{\fmnot}{f\sups{M_0}}

In this section we provide conditions under which one can lower bound the
localized \CompText by the global version.

\begin{lemma}[Local-to-global lemma]
  \label{lem:local_to_global}
  Let $\cM$ have $\cFm\subseteq(\Act\to\brk*{0,1})$. Suppose there exists $\Mnot\in\cM$ such that $\starhull(\cM,\Mnot)\subseteq\cM$ and $\fmnot$ is constant. Then 
  \begin{equation}
    \label{eq:local_to_global}
    \sup_{\Mbar\in\cM}\comp(\cM_{\veps}(\Mbar),\Mbar)
    \geq{} \veps\cdot{}\sup_{\Mbar\in\cM}\comp(\cM,\Mbar)
  \end{equation}
  for all $\veps\in[0,1]$ and $\gamma\geq{}0$.
\end{lemma}
\begin{proof}[\pfref{lem:local_to_global}]
  For any pair $M,\Mbar\in\cM$, define $M'=(1-\veps)\Mnot + \veps{}M$ and $\Mbar'=(1-\veps)\Mnot+\veps{}\Mbar$, and note that
  $M',\Mbar'\in\starhull(\cM,\Mnot)\subseteq\cM$. Furthermore, observe that
  \begin{align*}
    \fmbarp(\pi_{\sss{\Mbar'}})-\fmp(\pi_{\sss{M'}})
    &= \veps(\fmbar(\pimbar)-\fm(\pim)) + (1-\veps)(\fmnot(\pimbar)-\fmnot(\pim))\\
    &= \veps(\fmbar(\pimbar)-\fm(\pim))
    \geq{}-\veps,
  \end{align*}
since $\fmnot$ is constant and $\cFm$ takes values in
$\brk*{0,1}$. Hence, $M'\in\cM_{\veps}(\Mbar')$. It follows that for
any $\Mbar\in\cM$,
\begin{align*}
  \sup_{\Mbar\in\cM}\comp(\cM_{\veps}(\Mbar),\Mbar)
  & \geq{} \comp(\cM_{\veps}(\Mbar'), \Mbar')\\
  &\geq{} \inf_{p\in\Delta(\Act)}\sup_{M\in\cM}\En_{\act\sim{}p}\brk*{
    \fmp(\pi_{\sss{M'}}) - \fmp(\act)
    - \gamma\cdot\Dhels{\Mbar'(\act)}{M'(\act)}
    }.\\
    &= \inf_{p\in\Delta(\Act)}\sup_{M\in\cM}\En_{\act\sim{}p}\brk*{
      \veps\cdot(\fm(\pim) - \fm(\act))
    - \gamma\cdot\Dhels{\Mbar'(\act)}{M'(\act)}
    },
\end{align*}
where we again use that $\fmnot$ is constant.
Next, using the joint convexity of squared Hellinger distance, we have
\begin{align*}
  \Dhels{\Mbar'(\act)}{M'(\act)}
  \leq{} \veps\Dhels{\Mbar(\act)}{M(\act)}
  + (1-\veps)\Dhels{\Mnot(\act)}{\Mnot(\act)}
  = \veps\Dhels{\Mbar(\act)}{M(\act)}.
\end{align*}
It follows that
\begin{align*}
  \comp(\cM_{\veps}(\Mbar),\Mbar)&\geq\inf_{p\in\Delta(\Act)}\sup_{M\in\cM}\En_{\act\sim{}p}\brk*{
      \veps\cdot(\fm(\pim) - \fm(\act))
    - \gamma\veps\cdot\Dhels{\Mbar(\act)}{M(\act)}
                 } \\
      &= \veps\cdot\comp(\cM,\Mbar).
\end{align*}
Since this holds uniformly for all $\Mbar\in\cM$, the result is established.
\end{proof}

Next, we prove a stronger local-to-global lemma from a \emph{structured Gaussian bandit} setting in which there are no
observations (i.e., $\Ospace=\NullObs$), and rewards are Gaussian in
the sense that $M(\act)= \cN(\fm(\act), \sigma^2)$ for all $M\in\cM$. We place
no assumption on structure of the mean reward function class itself.

\begin{lemma}[Local-to-global lemma for Gaussian bandits]
  \label{lem:local_to_global_gaussian}
  For any structured Gaussian bandit class $\cM$ with
  $\starhull(\cFm,0)\subseteq\cFm\subseteq(\Act\to\brk{0,1})$ and $\sigma^2\geq{}1/8$, we have
  \begin{equation}
    \label{eq:local_to_global_gaussian}
    \sup_{\Mbar\in\cM}\comp(\cM_{\veps}(\Mbar),\Mbar)
    \geq{} \veps\cdot{}\sup_{\Mbar\in\cM}\comp[4e(\veps\gamma)](\cM,\Mbar)
  \end{equation}
  for all $\veps\in[0,1]$ and $\gamma\geq{}0$.
\end{lemma}
For most of the applications considered in this paper,
$\comp(\cM)\propto\frac{C}{\gamma}$, where $C$ is a problem-dependent
parameter such as dimension. In this case, the dependence on $\veps$ on the right-hand
side of \pref{eq:local_to_global_gaussian} vanishes, and we are left with
\[
    \sup_{\Mbar\in\cM}\comp(\cM_{\veps}(\Mbar),\Mbar)
    \geq{} \frac{1}{4e}\cdot\sup_{\Mbar\in\cM}\comp[\gamma](\cM,\Mbar).
\]

\begin{proof}[\pfref{lem:local_to_global_gaussian}]
  For any $M,\Mbar\in\cM$, define $M',\Mbar'$ via
$f\sups{M'}=\veps\fm$ and $f\sups{\Mbar'}=\veps\fmbar$, and note that
$f\sups{M'},f\sups{\Mbar'}\in\starhull(\cFm,0)\subseteq\cFm$. Observe
that $\fmbarp(\pi_{\sss{\Mbar'}})-\fmp(\pi_{\sss{M'}})\geq{}-\veps$ (since $\cFm$ takes values in
$\brk*{0,1}$) and hence $M'\in\cM_{\veps}(\Mbar')$. It follows that
for any fixed $\Mbar\in\cM$,
\begin{align*}
  \sup_{\Mbar\in\cM}\comp(\cM_{\veps}(\Mbar),\Mbar)
  & \geq{} \comp(\cM_{\veps}(\Mbar'), \Mbar')\\
  &\geq{} \inf_{p\in\Delta(\Act)}\sup_{M\in\cM}\En_{\act\sim{}p}\brk*{
    \fmp(\pi_{\sss{M'}}) - \fmp(\act)
    - \gamma\cdot\Dhels{\Mbar'(\act)}{M'(\act)}
    }.\\
    &= \inf_{p\in\Delta(\Act)}\sup_{M\in\cM}\En_{\act\sim{}p}\brk*{
      \veps\cdot(\fm(\pim) - \fm(\act))
    - \gamma\cdot\Dhels{\Mbar'(\act)}{M'(\act)}
    }.
\end{align*}
Next, we compute 
\begin{align*}
  \Dhels{\Mbar'(\act)}{M'(\act)}
  \leq{} \Dkl{\Mbar'(\act)}{M'(\act)}
  = \frac{1}{2\sigma^2}(\fmbarp(\act)-\fmp(\act))^2
  = \frac{\veps^2}{2\sigma^{2}}(\fmbar(\act)-\fm(\act))^2
\end{align*}
and %
\begin{align*}
  \Dhels{\Mbar(\act)}{M(\act)}
  = 1 - \exp\prn*{-
  \frac{(\fmbar(\act)-\fm(\act))^2}{8\sigma^2}
  }
  \geq{} \frac{1}{8e\sigma^{2}}(\fmbar(\act)-\fm(\act))^2,
\end{align*}
where the inequality uses that $e^{-x}\leq{}1-\frac{x}{e}$ for
$x\leq{}1$ along with the assumption that $\fm,\fmbar\in\brk*{0,1}$
and $\sigma^{2}\geq1/8$.
It follows that
\begin{align*}
&\inf_{p\in\Delta(\Act)}\sup_{M\in\cM}\En_{\act\sim{}p}\brk*{
      \veps\cdot(\fm(\pim) - \fm(\act))
    - \gamma\cdot\Dhels{\Mbar'(\act)}{M'(\act)}
                 } \\
  &\geq\inf_{p\in\Delta(\Act)}\sup_{M\in\cM}\En_{\act\sim{}p}\brk*{
      \veps\cdot(\fm(\pim) - \fm(\act))
    - 4e\veps^2\gamma\cdot\Dhels{\Mbar(\act)}{M(\act)}
    } \\
    &= \veps\cdot{}\inf_{p\in\Delta(\Act)}\sup_{M\in\cM}\En_{\act\sim{}p}\brk*{
      \fm(\pim) - \fm(\act)
    - 4e\veps\gamma\cdot\Dhels{\Mbar(\act)}{M(\act)}
      } \\
      &\geq{} \veps\cdot\comp[4e\veps\gamma](\cM,\Mbar).
\end{align*}
Since this holds uniformly for all $\Mbar$, the result is established.

\end{proof}

\section{Proofs from \preft{sec:framework}}
\label{app:framework}

\subsection{Proofs for Lower Bounds}

\subsubsection{Proof of \preft{thm:lower_main}}

\lowermain*

{

\newcommand{\Enot}{\neg{}\cE}

\newcommand{\eg}{\cE^{\sM}}
\newcommand{\ef}{\cE^{\sMbar}}
\newcommand{\cEm}{\cE^{\sM}}
\newcommand{\cEmbar}{\cE^{\sMbar}}
\newcommand{\egnot}{\neg\eg}
\newcommand{\efnot}{\neg\ef}

\newcommand{\indg}{\indic\crl*{\cEm}}
\newcommand{\indf}{\indic\crl*{\cEmbar}}
\newcommand{\indgnot}{\indic\crl*{\neg\cEm}}
\newcommand{\indfnot}{\indic\crl*{\neg\cEmbar}}

\newcommand{\klfg}{\Dhels{\PM}{\PMbar}}
\newcommand{\Pfdot}{\PMbar(\cdot)}
\newcommand{\Pgdot}{\PM(\cdot)}

\newcommand{\Cn}{C_T}

\begin{proof}[\pfref{thm:lower_main}]We will prove the following
  slightly more refined result: For any $\gamma>\sqrt{\Ct{}T}$ and nominal model
  $\Mbar\in\cM$ and any algorithm, there exists a model $M\in\cMloc[\vepslowg{}](\Mbar)$ such that
  \begin{equation}
    \label{eq:lower_main_local}
    \RegDM\geq{}
(6\Ct)^{-1}\cdot\min\crl[\bigg]{\prn*{\comp(\cMloc[\vepslowg{}](\Mbar),\Mbar)-15\delta}\cdot{}T, \gamma}
  \end{equation}  
 with probability at least $\delta/2$. The result in \pref{eq:lower_main} is an
immediate corollary.

Let $T\in\bbN$ be fixed, and let $\veps=c_1\frac{\gamma}{T}$, where
$c_1$ is a free parameter to be specified later. Let
$\gamma\geq{}\sqrt{T}/c_2$ be given, where $c_2$ is another free parameter.
  
Consider a fixed algorithm
$p = \crl*{p\ind{t}(\cdot\mid\cdot)}_{t=1}^{T}$. Recall that $\Prm{M}{p}$ is the law of $\hist\ind{T}$
when $M$ is the underlying model and $p$ is the algorithm. We
abbreviate this to $\bbP\sups{M}$, and let
$\Em\brk*{\cdot}$ denote the corresponding expectation.

Let $\Mbar\in\cM$ be fixed. Let $\phat\ldef\frac{1}{T}\sum_{t=1}^{T}p\ind{t}(\cdot\mid{}\hist\ind{t-1})\in\Delta(\Act)$,
which is a random variable, and let $\pmbar=\Embar\brk*{\phat}$. Since $\pmbar\in\Delta(\Act)$, the definition
of $\comp(\cMloc[\veps](\Mbar),\Mbar)$ guarantees that
\begin{align*}
  \sup_{M\in\cM_{\veps}(\Mbar)}
\En_{\act\sim\pmbar}\brk*{\fm(\pim) - \fm(\act)
  -\gamma\cdot\Dhels{M(\act)}{\Mbar(\act)}}
  \geq{} \comp(\cMloc[\veps](\Mbar),\Mbar).
\end{align*}
Let $M\in\cMloc[\veps](\Mbar)$ attain the supremum above.\footnote{If the supremum is not attained, one
can consider a limit sequence. We omit the details.} Then, rearranging, we have
\begin{align}
  \En_{\act\sim\pmbar}\brk*{\fm(\pim) - \fm(\act)}
  \geq{} \gamma\cdot \En_{\act\sim\pmbar}\brk*{\Dhels{M(\act)}{\Mbar(\act)}}
  + \comp(\cMloc[\veps](\Mbar),\Mbar).
  \label{eq:lower_basic}
\end{align}
For the remainder of the proof, we abbreviate
$\comp{}\equiv\comp{}(\cM_{\veps}(\Mbar),\Mbar)$. We define
$\delm=\En_{\act\sim{}\phat}\brk{\fm(\pim)-\fm(\act)}$ and
$\delmbar=\En_{\act\sim\phat}\brk*{\fmbar(\pimbar)-\fmbar(\act)}$,
which are random variables. Note that we have $\delm=\frac{1}{T}\RegDM$ when
$\hist\ind{T}\sim{}\PM$, and likewise for $\delmbar$. With this notation,
we can rewrite \pref{eq:lower_basic} as 
\begin{align}
  \Embar\brk*{\delm} \geq{}
  \gamma\cdot \Embar\brk*{\En_{\act\sim\phat}\brk*{\Dhels{M(\act)}{\Mbar(\act)}}}
  + \comp.
   \label{eq:minimax_g_aux}
\end{align}

Let $c_3>0$ be a final free parameter.
We consider two cases. First, if either
$\PM(\delm>c_3\frac{\gamma}{T})>\delta$ or
$\PMbar(\delmbar>c_3\frac{\gamma}{T})>\delta$, the main result is
implied whenever $c_3$ is sufficiently large, since
$\delm$ is equal in law to $\frac{1}{T}\RegDM$ when $M$ is the
underlying problem instance. Hence, for the second case, we define $\cEm=\crl*{\delm\leq{}c_3\frac{\gamma}{T}}$
and assume that
$\PMbar\brk{\cEmbar},\PM\brk{\cEm}\geq{}1-\delta$.

The strategy for the remainder of the proof is as follows. Given the
assumption that $\PMbar\brk{\cEmbar},\PM\brk{\cEm}\geq{}1-\delta$, we
will prove a lower bound on the regret for model $M$ in terms of
$\comp$, using the
lower bound in \pref{eq:minimax_g_aux} as a starting point. The main
challenge is to relate the quantity $\Embar\brk{\delm}$ on the left-hand side of
\pref{eq:minimax_g_aux} to the quantity $\Em\brk{\delm}$, which
corresponds to regret when actions are selected under the law induced
by $M$. We address this issue using change of measure arguments, which
take advantage of the localization property and the assumption that $\PMbar\brk{\cEmbar},\PM\brk{\cEm}\geq{}1-\delta$.

To begin, we recall the following result from \pref{app:technical}.
\mpmin*
We apply \pref{eq:mp_min} with $X=\cH\ind{T}$,
$h(X)=\delm\indic\crl*{\cEm}$, $\bbP=\PMbar$, and
$\bbQ=\PM$, which gives
\begin{align*}
  \Embar\brk*{\delm\indic\crl{\cEm}}
  &\leq{}
    3\cdot{}\Em\brk*{\delm\indg} +
  4c_3\gammat\cdot{}\Dhels{\PM}{\PMbar},
\end{align*}
where we have used that $0\leq{}\delm\indg\leq{}c_3\frac{\gamma}{T}$
almost surely. We can further bound by
    \begin{align}
      \Embar\brk*{\delm\indic\crl{\cEm}}    &\leq{}
                                              3\cdot{}\Em\brk{\delm} + 
                                              3\delta{} + 4c_3\gammat\cdot{}\klfg\label{eq:lb1}
\end{align}
using that $\fm\in\brk*{0,1}$ and
$\PM\brk{\neg\cEm}\leq\delta$. Next, we observe that
\begin{equation}
  \Embar\brk*{\delm\indg} =
\Embar\brk*{\delm}
  -  \Embar\brk*{\delm\indgnot}.\label{eq:lb2}
\end{equation}
Combining \pref{eq:minimax_g_aux}, \pref{eq:lb1}, and \pref{eq:lb2},
we have
\begin{align*}
  \Em\brk{\delm} &\geq{} \frac{1}{3}\Embar\brk{\delm}
  - \frac{1}{3}\Embar\brk*{\delm\indgnot} - \delta -
                                              \frac{4c_3}{3}\gammat\cdot{}\klfg\\
  &\geq{} \frac{1}{3}\comp
    - \frac{1}{3}\Embar\brk*{\delm\indgnot} - \delta  + \frac{\gamma}{3}\cdot \Embar\brk*{\En_{\act\sim\phat}\brk*{\Dhels{M(\act)}{\Mbar(\act)}}}-\frac{4c_3}{3}\gammat\cdot{}\klfg.
\end{align*}
We bound the error term involving $\indgnot$ using the following
technical lemma.
\begin{lemma}
  \label{lem:f_eg_error_aux}
  When $\PMbar\brk{\cEmbar},\PM\brk{\cEm}\geq{}1-\delta$, we have
  \[
  \Embar\brk*{\delm\indgnot}
  \leq{} (7c_1 + 14c_3)\cdot\gammat\klfg + \sqrt{14\klfg\Embar\brk*{
        \En_{\act\sim\phat}\brk*{\Dhels{M(\act)}{\Mbar(\act)}}
        }}+ 7\delta{}.
  \]
\end{lemma}
Applying \pref{lem:f_eg_error_aux} and rearranging, we have
\begin{align*}
  \Em\brk{\delm} 
  &\geq{} \frac{1}{3}\comp
    - 4\delta  + \frac{\gamma}{3}\cdot
    \Embar\brk*{\En_{\act\sim\phat}\brk*{\Dhels{M(\act)}{\Mbar(\act)}}}-3^{-1}(7c_1
    + 18c_3)\gammat\cdot{}\klfg\\
  &~~~~~~-\frac{1}{3}\sqrt{14\klfg\Embar\brk*{
        \En_{\act\sim\phat}\brk*{\Dhels{M(\act)}{\Mbar(\act)}}
        }}.
\end{align*}
Next, we invoke \pref{lem:hellinger_chain_rule}, which implies that
\[
  \Dhels{\PM}{\PMbar}
  \leq{}
  \Cn\sum_{t=1}^{T}\Embar\brk*{\Dhels{M(\act\ind{t})}{\Mbar(\act\ind{t})}}
  = \Cn{}T\Embar\brk*{\En_{\act\sim\phat}\brk*{\Dhels{M(\act)}{\Mbar(\act)}}
    },
  \]
  where $\Cn\ldef{}2^{8}(\log(2T)\wedge{}\log(\Vmbar(\cM)))$ and
$\Vmbar(\cM)\ldef{}\sup_{M\in\cM}\sup_{\act\in\Act}\sup_{A\in\Rsig\otimes\Osig}\frac{\Mbar(A\mid\act)}{M(A\mid\act)}\leq{}\abscont$ (cf. \pref{thm:lower_main}).\footnote{We apply
    \pref{lem:hellinger_chain_rule} with the sequence
    $X_1,\ldots,X_{2T}$, where $X_t=\act\ind{t}$,
    $X_{t+1}=(r\ind{t},\obs\ind{t})$, and use that the conditional
    distribution of $\act\ind{t}$ is identical under $\PM$ and $\PMbar$.}
  This gives
    \begin{align*}
  \Em\brk{\delm} 
  &\geq{} \frac{1}{3}\comp
    - 4\delta  + 3^{-1}\gamma(1 - 7\Cn{}c_1-18\Cn{}c_3)\cdot
    \Embar\brk*{\En_{\act\sim\phat}\brk*{\Dhels{M(\act)}{\Mbar(\act)}}}\\
  &~~~~~~-3^{-1}\sqrt{14\Cn{}T}\Embar\brk*{
        \En_{\act\sim\phat}\brk*{\Dhels{M(\act)}{\Mbar(\act)}}
    }\\
      &\geq{} \frac{1}{3}\comp
        - 4\delta  + 3^{-1}\gamma(1 - 7\Cn{}c_1-\sqrt{14\Cn}c_2 + 18\Cn{}c_3)\cdot
    \Embar\brk*{\En_{\act\sim\phat}\brk*{\Dhels{M(\act)}{\Mbar(\act)}}},
    \end{align*}
    where the final line uses the assumption that $\gamma\geq{}\sqrt{T}/c_2$.
It follows that if we (conservatively) choose $c_1=c_3=(126\Cn)^{-1}$
and $c_2=(126\Cn)^{-1/2}$, we have
\[
\frac{1}{T}\Em\brk*{\RegDM} = \Em\brk{\delm} \geq{}3^{-1}(\comp-12\delta).
\]
This establishes a lower bound on the expected regret. To
conclude, we show that this implies a lower bound in
probability.
\begin{restatable}{lemma}{markovlb}
  \label{prop:markov_lb}
  For any real-valued  random variable $Z$ with $\En\brk{Z}\geq\mu$ and $Z\leq{}R\in\bbR_{+}$
  almost surely, $\bbP\brk{Z>\mu/2}\geq{} \frac{\mu}{2R}$.
\end{restatable}
\pref{prop:markov_lb} implies that
$\PM\brk*{\RegDM>6^{-1}(\comp{} -
  12\delta{})T}\geq{}6^{-1}\frac{(\comp{} -
  12\delta{})T}{\Em\brk{\RegDM}}\geq{}6^{-1}(\comp{}-12\delta)$,
since $\RegDM\leq{}T$. We consider two cases. First, if
$\comp{}<15\delta$, $\PM\brk*{\RegDM>6^{-1}(\comp{} -
  15\delta{})T}=1$ trivially. Otherwise, we have
\[
\PM\brk*{\RegDM>6^{-1}(\comp{} -
  15\delta{})T}
\geq{} \PM\brk*{\RegDM>6^{-1}(\comp{} -
  12\delta{})T}\geq{} 6^{-1}(15\delta-12\delta)= \delta/2.
\]
\end{proof}

\subsubsection{Proofs for Auxiliary Lemmas (\preft{thm:lower_main})}

\begin{proof}[\pfref{lem:f_eg_error_aux}]
We first note that by \pref{lem:mp_min}, we have
  \begin{equation}
    \label{eq:pf_bound}
    \PMbar\brk{\egnot} \leq{} 3\PM\brk{\egnot} +
    4\klfg
    \leq{} 3\delta +     4\klfg.
  \end{equation}
  We consider two cases. First, if $\klfg\leq\delta$, then
  $\PMbar\brk{\egnot}\leq{}7\delta$ by \pref{eq:pf_bound}. Since
  $\fm\in\brk{0,1}$, this implies that
  $\Embar\brk*{\delm\indgnot}\leq{}\PMbar\brk{\egnot}\leq{}7\delta{}$ as
desired. 

For the second case where $\klfg\geq{}\delta$, \pref{eq:pf_bound} implies that
\begin{equation}
  \PMbar\brk{\egnot} \leq{} 7\cdot\klfg,\label{eq:lb_casetwo}
  \end{equation} and we proceed by breaking the error into three terms:
  \begin{align*}
    \Embar\brk*{\delm\indgnot} =
    \underbrace{\Embar\brk*{\delmbar\indgnot}}_{\I}
    + \underbrace{(\fm(\pim)-\fmbar(\pimbar))\PMbar\brk{\egnot}}_{\II}
    +     \underbrace{\Embar\brk*{\En_{\act\sim\phat}\brk*{\fmbar(\act)-\fm(\act)}\indgnot}}_{\III}.
  \end{align*}
  \paragraph{Term I}
  We express this quantity as 
  \begin{align*}
    \Embar\brk*{\delmbar\indgnot}
    = \Embar\brk*{\delmbar\indic\crl{\ef\wedge\egnot}}
    + \Embar\brk*{\delmbar\indic\crl{\efnot\wedge\egnot}}.
  \end{align*}
  For the first term above, we have
  \[
  \Embar\brk*{\delmbar\indic\crl{\ef\wedge\egnot}}\leq{}c_3\gammat\PMbar\brk{\egnot}
  \leq{} 7c_3\cdot\gammat{}\klfg,
\]
where we have used that $\delmbar\leq{}c_3\gammat$ under
$\ef$, along with \pref{eq:lb_casetwo}. For the second term above, since
$\indic\crl{A\wedge{}B}\leq{}\frac{1}{2}(\indic\crl{A}+\indic\crl{B})$, we
have \begin{align*}
  \Embar\brk*{\delmbar\indic\crl{\efnot\wedge\egnot}}
    &\leq{}\frac{1}{2}\Embar\brk*{\delmbar\indic\crl{\egnot}}
    +\frac{1}{2}\Embar\brk*{\delmbar\indic\crl{\efnot}}
  \\
      &\leq{}\frac{1}{2}\Embar\brk*{\delmbar\indic\crl{\egnot}}
    +\frac{1}{2}\delta{},
\end{align*}
since we have $\PMbar\brk{\efnot}\leq{}\delta$ by assumption, and
$\fmbar\in\brk*{0,1}$. After putting both bounds together and rearranging,
we have
\[
\I \leq{} 14c_3\cdot\gammat\klfg + \delta{}.
\]
\paragraph{Term II}
From the definition of $\cMloc[\veps](\Mbar)$, we have
$\fm(\pim)-\fmbar(\pimbar)\leq{}\veps = c_1\frac{\gamma}{T}$, so that
\[
\II \leq{} 7c_1\cdot\gammat\klfg.
\]
\paragraph{Term III}
By Cauchy-Schwarz, we have
\begin{align*}
\III =
  \Embar\brk*{\En_{\act\sim\phat}\brk*{\fmbar(\act)-\fm(\act)}\indgnot}
  \leq{} \sqrt{
    \Embar\brk*{\prn*{\En_{\act\sim\phat}\brk*{\fmbar(\act)-\fm(\act)}}^{2}}\PMbar\brk*{\egnot}
    }.
\end{align*}
Furthermore, we have
\begin{align*}
  \Embar\brk*{\prn*{\En_{\act\sim\phat}\brk*{\fmbar(\act)-\fm(\act)}}^{2}}
  &=
    \Embar\brk*{\prn*{\frac{1}{T}\sum_{t=1}^{T}\En_{\act\sim{}p\ind{t}}\brk*{\fmbar(\act)-\fm(\act)}}^{2}}\\
  &\leq{} \frac{1}{T}\sum_{t=1}^{T}\Embar\brk*{
\prn*{\En_{\act\sim{}p\ind{t}}\brk*{\fmbar(\act)-\fm(\act)}}^2
    }.
\end{align*}
We can further bound the right-hand side by
\begin{align*}
  \frac{1}{T}\sum_{t=1}^{T}\Embar\brk*{
        \prn*{\fmbar(\act\ind{t})-\fm(\act\ind{t})}^2
       }
      &\leq{} \frac{2}{T}\sum_{t=1}^{T}\Embar\brk*{
          \Dtvs{M(\act\ind{t})}{\Mbar(\act\ind{t})}
          }\\
  &\leq{} \frac{2}{T}\sum_{t=1}^{T}\Embar\brk*{
          \Dhels{M(\act\ind{t})}{\Mbar(\act\ind{t})}
  }\\
      &=2\Embar\brk*{
        \En_{\act\sim\phat}\brk*{\Dhels{M(\act)}{\Mbar(\act)}}
        }.
\end{align*}
where the last inequality is from \pref{lem:pinsker}.
Combining this with \pref{eq:lb_casetwo}, we have
\begin{align*}
  \III \leq{} \sqrt{14\klfg\Embar\brk*{
        \En_{\act\sim\phat}\brk*{\Dhels{M(\act)}{\Mbar(\act)}}
        }}.
\end{align*}

\paragraph{Putting everything together}
Combining the bounds on $\I$, $\II$, and $\III$, we have
\[
  \Embar\brk*{\delm\indgnot}
  \leq{} (7c_1 + 14c_3)\cdot\gammat\klfg + \sqrt{14\klfg\Embar\brk*{
        \En_{\act\sim\phat}\brk*{\Dhels{M(\act)}{\Mbar(\act)}}
        }}+ \delta{}.
\]
for the second case.
\end{proof}

\begin{proof}[\pfref{prop:markov_lb}]
  By the law of total expectation,
  \begin{align*}
    \mu\leq{}\En\brk*{Z} &= \En\brk{Z\mid{}Z>\mu/2}\bbP\brk{Z>\mu/2}
          + \En\brk{Z\mid{}Z\leq{}\mu/2}\bbP\brk{Z\leq{}\mu/2}\\
        &\leq{} R\cdot\bbP\brk{Z>\mu/2}
          + \mu/2.
  \end{align*}
Rearranging yields the result.
\end{proof}

\subsubsection{Proof of \pref{thm:lower_main_expectation}}

\lowermainexpectation*

\begin{proof}[\pfref{thm:lower_main_expectation}]%
  \renewcommand{\pm}{p\subs{M}}%
    \newcommand{\Epim}{\En_{\pi\sim\pm}}%
  \newcommand{\Epimbar}{\En_{\pi\sim\pmbar}}%
  This proof follows the same structure as \pref{thm:lower_main}, with
  the main difference being that the stronger notion of localization
  allows for a stronger (as well as simpler) change of measure argument.

  Let $T\in\bbN$ be fixed, and let $\veps=c_1\frac{\gamma}{T}$, where
$c_1$ is a free parameter to be specified later. Consider a fixed algorithm
$p = \crl*{p\ind{t}(\cdot\mid\cdot)}_{t=1}^{T}$, and let
$\bbP\sups{M}(\cdot)$ and $\En\sups{M}\brk*{\cdot}$ denote the law and
expectation when the algorithm is run on model $M$. Define $\pm=\En\sups{M}\brk*{\frac{1}{T}\sum_{t=1}^{T}p\ind{t}(\cdot\mid{}\hist\ind{t-1})}\in\Delta(\Act)$.

Fix $\Mbar\in\cM$. since $\pmbar\in\Delta(\Act)$, the definition
of $\comp(\cMinf[\veps](\Mbar),\Mbar)$ guarantees that
\begin{align*}
  \sup_{M\in\cMinf(\Mbar)}
\En_{\act\sim\pmbar}\brk*{\fm(\pim) - \fm(\act)
  -\gamma\cdot\Dhels{M(\act)}{\Mbar(\act)}}
  \geq{} \comp(\cMloc[\veps](\Mbar),\Mbar).
\end{align*}
Let $M\in\cMloc[\veps](\Mbar)$ attain the supremum above, or consider
a limit sequence of the supremum is not attained. Then, rearranging, we have
\begin{align}
  \En_{\act\sim\pmbar}\brk*{\fm(\pim) - \fm(\act)}
  \geq{} \gamma\cdot \En_{\act\sim\pmbar}\brk*{\Dhels{M(\act)}{\Mbar(\act)}}
  + \comp(\cMloc[\veps](\Mbar),\Mbar).
  \label{eq:lower_basic_expectation}
\end{align}
We abbreviate $\comp{}\equiv\comp{}(\cMinf(\Mbar),\Mbar)$ for the
remainder of the proof. Recalling that $\gm(\pi)\ldef{}\fm(\pim) -
\fm(\pi)$, it follows from \pref{eq:lower_basic_expectation} that
\begin{align}
  &\En_{\pi\sim\pm}\brk*{\gm(\pi)}  +
    \En_{\pi\sim\pmbar}\brk*{\gmbar(\pi)}  \notag\\
  &= \frac{1}{3}\prn*{\En_{\pi\sim\pm}\brk*{\gm(\pi)}  +
    \En_{\pi\sim\pmbar}\brk*{\gmbar(\pi)}}\notag\\
  &~~~~+ \frac{2}{3}\prn*{
    \Epim\brk*{\gmbar(\pi)} + \Epimbar\brk*{\gm(\pi)}
    + (\Epim\brk{\gm(\pi)-\gmbar(\pi)}-\Epimbar\brk{\gm(\pi)-\gmbar(\pi)})
    }\\
      &\geq{} \frac{1}{3}\prn*{\En_{\pi\sim\pm}\brk*{\gm(\pi)}  +
    \En_{\pi\sim\pmbar}\brk*{\gmbar(\pi)} + \Epim\brk*{\gmbar(\pi)} +
        \Epimbar\brk*{\gm(\pi)}} + \frac{1}{3}\comp \notag\\
  &~~~~+ \frac{\gamma}{3}
    \En_{\act\sim\pmbar}\brk*{\Dhels{M(\act)}{\Mbar(\act)}} -
    \frac{2}{3}\abs*{\Epim\brk{\gm(\pi)-\gmbar(\pi)}-\Epimbar\brk{\gm(\pi)-\gmbar(\pi)}}.\label{eq:expectation0}
\end{align}
Recall from the definition of $\cMinf(\Mbar)$ that $\abs*{\gm(\pi)-\gmbar(\pi)}\leq\veps$ for all
$\pi$. Hence, by \pref{lem:mult_pinsker_as}, we have that
\begin{align}
  &\abs*{\Epim\brk{\gm(\pi)-\gmbar(\pi)}-\Epimbar\brk{\gm(\pi)-\gmbar(\pi)}} \notag\\
  &\leq{}
    \sqrt{8\veps\cdot{}\prn*{\Epim\brk*{\gm(\pi)+\gmbar(\pi)}+\Epimbar\brk*{\gm(\pi)+\gmbar(\pi)}}\cdot\Dhels{\bbP\sups{M}}{\bbP\sups{\Mbar}}}\notag\\
  &\leq{} 4\veps\Dhels{\bbP\sups{M}}{\bbP\sups{\Mbar}}  + \frac{1}{2}\prn*{\Epim\brk*{\gm(\pi)+\gmbar(\pi)}+\Epimbar\brk*{\gm(\pi)+\gmbar(\pi)}}.\label{eq:expectation1}
\end{align}
Furthermore, \pref{lem:hellinger_chain_rule} implies that
\[
  \Dhels{\PM}{\PMbar}
  \leq{}
  \Cn\sum_{t=1}^{T}\Embar\brk*{\Dhels{M(\act\ind{t})}{\Mbar(\act\ind{t})}}
  = \Cn{}T\cdot{}\En_{\act\sim\pmbar}\brk*{\Dhels{M(\act)}{\Mbar(\act)}},
  \]
  where $\Cn\ldef{}2^{8}(\log(2T)\wedge{}\log(\Vmbar(\cM)))$ and
$\Vmbar(\cM)\ldef{}\sup_{M\in\cM}\sup_{\act\in\Act}\sup_{A\in\Rsig\otimes\Osig}\frac{\Mbar(A\mid\act)}{M(A\mid\act)}\leq{}\abscont$. If
we set
$\veps\leq\frac{\gamma}{8\Cn{}T}$, then after combining this
inequality with \pref{eq:expectation0} and \pref{eq:expectation1} and
rearranging, we are guaranteed that
\[
  \En_{\pi\sim\pm}\brk*{\gm(\pi)}  +
  \En_{\pi\sim\pmbar}\brk*{\gmbar(\pi)}
  \geq{} \frac{1}{3}\comp.
\]
Since $\En_{\pi\sim\pm}\brk*{\gm(\pi)}=\En\sups{M}\brk*{\RegDM}$, this
completes the proof.

\end{proof}

}

\subsection{Proofs for Upper Bounds}

\subsubsection{Proof of \preft{thm:upper_main} and
  \preft{thm:upper_main_infinite}}
\newcommand{\EstProbHelT}{\wt{\mathrm{\mathbf{Est}}}_{\mathsf{H}}(T,\delta)}
\newcommand{\actt}{\act\ind{t}}
\newcommand{\Mhatt}{\Mhat\ind{t}}
\newcommand{\Event}{\mathcal{E}}
  \newcommand{\logc}{\log_{C}}
  \newcommand{\Tnot}{T_0}

We prove \pref{thm:upper_main} and \pref{thm:upper_main_infinite} as consequences of \pref{thm:upper_general} (cf. \pref{sec:algorithm}).

\uppermain*

\uppermaininfinite*

\begin{proof}[Proof of \pref{thm:upper_main} and
  \pref{thm:upper_main_infinite}]
The proof proceeds identically for \pref{thm:upper_main} and
\pref{thm:upper_main_infinite}, with the only difference being the
choice of the estimation algorithm $\AlgEst$. For
\pref{thm:upper_main}, we appeal to an algorithm from
\pref{lem:vovk_sleeping_fast_finite} (which handles finite classes),
while for \pref{thm:upper_main_infinite} we appeal to an algorithm
from \pref{lem:vovk_sleeping_fast} (which handles infinite
classes). In particular, we are guaranteed that with probability
at least $1-\delta$, for any realization of the sequence of confidence sets
  $\crl*{\cM\ind{t}}_{t=1}^{T}$,
\begin{align*}
  &\sum_{t=1}^{T}\En_{\act\sim{}p\ind{t}}\brk*{\Dhels{\Mstar(\act)}{\Mhat\ind{t}(\act)}}\indic\crl*{\Mstar\in\cM\ind{t}}\\
  &\leq{}
    R^{2}
    \ldef
    \left\{
    \begin{array}{ll}
      \log\abs{\cM} + 2\log(\delta^{-1}),&\quad\text{Finite class
                                           case (\pref{lem:vovk_sleeping_fast_finite})},\\
      2^9\log^2(2B^2T)\cdot{}\prn*{\MComp +
      \log(\delta^{-1})},&\quad\text{Infinite class case (\pref{lem:vovk_sleeping_fast})}.
    \end{array}
    \right.
\end{align*}
This is accomplished by choosing $\cI\ind{t}=\cM\ind{t}$ in each
lemma, and with this choice both estimators ensure that
  $\Mhat\ind{t}\in\conv(\cM\ind{t})$. Hence, since the conditions of
  \pref{thm:upper_general2} are satisfied, it
  follows that if we run \pref{alg:main} using \optiontwo with the radius $R^{2}$ set as
  above, we are guaranteed that with probability at least $1-\delta$,
    \begin{align*}
  \RegDM
  &\leq{}
    \sum_{t=1}^{T}\sup_{\Mbar\in\conv(\cM)}\comp(\cMloc[\veps_t](\Mbar),\Mbar)
    + \gamma\cdot{}R^{2},
\end{align*}
where $\veps_t\ldef{}6\frac{\gamma}{t}R^{2} + \sup_{\Mbar\in\conv(\cM)}\comp(\cM,\Mbar)
+ (2\gamma)^{-1}$. To simplify this result, we use to the following lemma.
\begin{lemma}
  \label{lem:localization_simple}
  Let $\psi:\bbR_{+}\to\bbR_+$ be any non-decreasing function for
  which there exists $C>1$ such that $\psi(C\veps)\leq{}C\psi(\veps)$
  for all $\veps>0$. Define $\veps_t=\frac{a}{t}+b$ for $a,b\geq{}0$. Then
  \[
    \sum_{t=1}^{T}\psi(\veps_t)\leq{}C^{2}\ceil{\log_C(T)}\cdot{}\psi(4\veps_T)T.
  \]
\end{lemma}
\begin{proof}[\pfref{lem:localization_simple}]
  Let $\Tnot$ be the largest index $t$ for which
  $b\leq\frac{a}{t}$. We break into two cases. First, for $t>T_0$, we
  have
  \[
    \veps_t = b + \frac{a}{t}\leq{} 2b\leq{} 2\veps_T,
  \]
  so that
  \[
    \sum_{t=\Tnot+1}^{T}\psi(\veps_t)\leq{}(T-\Tnot)\psi(2\veps_T).
  \]
We now handle the sum to $\Tnot$. Define
$N=\ceil{\logc(\Tnot)}$ and $\tau_n=C^{n-1}$ for $n\in\brk*{N}$. Let
$\tau_{N+1}\ldef{}\Tnot$, and note that
$\tau_{N}\leq{}\Tnot=\tau_{N+1}$. We write
\begin{align*}
  \sum_{t=1}^{\Tnot}\psi(\veps_t)
  =\sum_{n=1}^{N}\sum_{t=\tau_n}^{\tau_{n+1}-1}\psi(\veps_t)
  \leq{}
  \sum_{n=1}^{N}\sum_{t=\tau_n}^{\tau_{n+1}-1}\psi(\veps_{\tau_n})
  \leq{} \sum_{n=1}^{N}\sum_{t=\tau_n}^{\tau_{n+1}-1}\psi(2a/\tau_n),
\end{align*}
where the least inequality uses that $b\leq{}a/t$ for
$t\leq\Tnot$. Applying the growth property for $\psi$
recursively, we have
\begin{align*}
  \psi(2a/\tau_n)
  = \psi(2aC^{-(n-1)}) = \psi(C^{N-n+1}\cdot{}2aC^{-N})
  \leq{}C^{N-n+1}\psi(2aC^{-N})\leq{}C^{N-n+1}\psi(2a/\Tnot).
\end{align*}
Furthermore, we have
\[
\frac{a}{\Tnot}\leq{}2\frac{a}{\Tnot+1}\leq{}2b\leq2\veps_T.
\]
Hence
\[
  \sum_{t=1}^{\Tnot}\psi(\veps_t)
  \leq{}
  \psi(4\veps_T)\sum_{n=1}^{N}\sum_{t=\tau_n}^{\tau_{n+1}-1}C^{N-n+1}
  = \psi(4\veps_T)(C-1)\sum_{n=1}^{N}C^{N}
  \leq{}2\psi(4\veps_T)T_0C^{2}\ceil{\logc(T_0)}.
\]
The result follows by summing the two cases (using that $C>1$).

\end{proof}
Applying \pref{lem:localization_simple} and simplifying for the
respective values for $R$ yields both theorems.
\end{proof}

\subsubsection{Proof of \preft{thm:upper_main_bayes}}
\label{app:upper_main_bayes}

\newcommand{\Rho}{\mathrm{P}}
\newcommand{\rhostar}{\rho^{\star}}

  \newcommand{\fmbart}{f^{\sss{\Mbar\ind{t}}}}
  \newcommand{\pt}{p\ind{t}}
  \newcommand{\Mbartrho}{\Mbar\ind{t}_{\rho}}
    \newcommand{\Mbartrhostar}{\Mbar\ind{t}_{\rhostar}}

\begin{algorithm}[htp]
    \setstretch{1.3}
     \begin{algorithmic}[1]
       \State \textbf{parameters}:
       \Statex[1] Prior $\mu\in\Delta(\cM)$.
       \Statex[1] Exploration parameter $\gamma>0$.
       \Statex[1] Cover scale $\veps>0$.
       \State Let $\Rho\subseteq\Pim$ be a cover that witnesses $\cN(\Pim,\veps)$.
       \State Let $\cov:\Pim\to\Rho$ be any map that takes
       $\pi\in\Pim$ to a corresponding cover element, and let $\rho^\star\ldef{}\cov(\pimstar)$.
  \For{$t=1, 2, \cdots, T$}
  \State Define $\Mbar\ind{t}(\act) =
  \En_{t-1}\brk*{\Mstar(\act)}$.
  \State Define $\Mbar\ind{t}_{\rho}(\act) =
  \En_{t-1}\brk*{\Mstar(\act)\mid\rhostar=\rho}$ for each
  $\rho\in\Rho$, and set $\cM\ind{t}=\crl*{\Mbar\ind{t}_\rho}_{\rho\in\Rho}$.
\State Let
$p\ind{t}=\argmin_{p\in\Delta(\Act)}\sup_{M\in\cM\ind{t}}\gameval{\Mbar\ind{t}}(p,M)$.~~~\algcommentlight{Minimizer
  for $\comp(\cM\ind{t},\Mbar\ind{t})$; cf. Eq. \pref{eq:game_value}}.\label{line:minimax_bayes}
\State{}Sample decision $\act\ind{t}\sim{}p\ind{t}$ and update $\hist\ind{t}=\hist\ind{t-1}\cup\crl{(\act\ind{t},r\ind{t},\obs\ind{t})}$.
\EndFor
\end{algorithmic}
\caption{\mainalgB for Infinite Classes}
\label{alg:main_bayes}
\end{algorithm}

\uppermainbayes*

\begin{proof}[\pfref{thm:upper_main_bayes}]
  Per the discussion in \pref{sec:minimax_swap}, it suffices to prove
the regret bound for the Bayesian setting in which $\Mstar\sim\mu$,
where $\mu\in\Delta(\cM)$ is a known prior. We will show that the
regret bound on the right-hand side of \pref{eq:upper_main_bayes}
holds uniformly for all choices of prior:
\begin{equation}
  \label{eq:upper_main_bayes_quant}
  \En_{\Mstar\sim\mu}\En\sups{\Mstar}\brk*{\RegDM}
  \leq{}   2\cdot\min_{\gamma>0}\max\crl*{\comp(\conv(\cM))\cdot{}T,\;\;
    \inf_{\veps\geq{}0}\crl*{\gamma\cdot\log\ActCov +\veps\cdot{}T}}.
\end{equation}
This implies that the there
exists an algorithm with the same regret bound for the frequentist
setting whenever the conclusion of \pref{prop:minimax_swap} holds. Going forward, we use $\En\brk*{\cdot}$ to denote expectation
with respect to the joint law over $(\Mstar,\hist\ind{T})$ when $\Mstar\sim\mu$.

We consider the Bayesian variant of \mainalg described in
\pref{alg:main_bayes}. The algorithm begins by building with a cover $\Rho$ that
witnesses the covering number $\ActCov$. Let $\cov:\Pim\to\Rho$ be
any fixed map that takes $\pi\in\Pim$ to a corresponding
$\veps$-covering element in $\Rho$ (in the sense of
\pref{eq:action_cover}). We let $\rhostar\ldef\cov(\pimstar)$ denote the covering element for the
optimal decision $\pimstar$, which is a random variable.

At each round $t$, the algorithm computes
$\Mbar\ind{t}(\act)\ldef\En_{t-1}\brk*{\Mstar(\act)}$, where we recall
that $\En_{t}\brk*{\cdot}\ldef{}\En\brk*{\cdot\mid\filt\ind{t}}$. This
may be thought of as the Bayesian analogue of the estimator
$\Mhat\ind{t}$ used in \pref{alg:main}. The algorithm then computes a
collection of mixture models
$\cM\ind{t}=\crl*{\Mbar\ind{t}_{\rho}}_{\rho\in\Rho}$, where
\[
\Mbar\ind{t}_{\rho}(\act) \ldef
  \En_{t-1}\brk*{\Mstar(\act)\mid\rhostar=\rho}.
  \]
Note that we have $\Mbar\ind{t}\in\conv(\cM)$ and
$\cM\ind{t}\subseteq\conv(\cM)$. The algorithm proceeds to
computes the distribution
\[
  p\ind{t}=\argmin_{p\in\Delta(\Act)}\sup_{M\in\cM\ind{t}}\gameval{\Mbar\ind{t}}(p,M),
\]
which achieves the value $\comp(\cM\ind{t},\Mbar\ind{t})$, and then
samples $\act\ind{t}\sim{}p\ind{t}$ and proceeds to the next round.

We begin by expressing the expected regret as
\[
\En\brk*{\RegDM} =
\En\brk*{\sum_{t=1}^{T}\En_{t-1}\brk*{\fmstar(\pimstar) - \fmstar(\pi\ind{t})}}.
\]
We first prove an elementary bound on each conditional expectation term
above. Note that
\[
\fmbart(\act) = \En_{t-1}\brk*{\En_{t-1}\brk*{r(\act)\mid\Mstar}}=\En_{t-1}\brk*{\fmstar(\act)}.
\]
Hence, since $\fmstar$ is conditionally independent of $\actt$
given $\filt\ind{t-1}$, we have
\[
\En_{t-1}\brk*{\fmstar(\act\ind{t})} = \En_{\actt\sim\pt}\brk{\fmbart(\actt)}.
\]
Next, we write
\begin{align*}
  \En_{t-1}\brk*{\fmstar(\pimstar)}
  =   \En_{t-1}\brk*{\En_{t-1}\brk*{\fmstar(\pimstar)\mid\rhostar}}
  = \sum_{\rho\in\Rho}\bbP_{t-1}(\rhostar=\rho) \En_{t-1}\brk*{\fmstar(\pimstar)\mid\rhostar=\rho}.
\end{align*}
We bound
\begin{align*}
  \En_{t-1}\brk*{\fmstar(\pimstar)\mid\rhostar=\rho}
  &= \En_{t-1}\brk*{\fmstar(\rho)\mid\rhostar=\rho}
    + \En_{t-1}\brk*{\fmstar(\pimstar)-\fmstar(\rho)\mid\rhostar=\rho}\\
  &\leq{} \En_{t-1}\brk*{\fmstar(\rho)\mid\rhostar=\rho}
    + \veps,
\end{align*}
using that $\rhostar$ is the corresponding $\veps$-cover element for
$\pimstar$. We simplify the expectation as
\begin{align*}
  \En_{t-1}\brk*{\fmstar(\rho)\mid\rhostar=\rho}
  &=
  \En_{t-1}\brk*{\En_{t-1}\brk*{\fmstar(\rho)\mid\Mstar,\rhostar=\rho}\mid\rhostar=\rho}\\
  &=
    \En_{t-1}\brk*{\En_{t-1}\brk*{\fmstar(\rho)\mid\Mstar}\mid\rhostar=\rho}\\
  &=
    \En_{t-1}\brk*{\En_{t-1}\brk*{r(\rho)\mid\Mstar}\mid\rhostar=\rho}\\
  &=f\sups{\Mbar\ind{t}_{\rho}}(\rho).
\end{align*}
Finally, noting that
\[
\sum_{\rho\in\Rho}\bbP_{t-1}(\rhostar=\rho)
f\sups{\Mbar\ind{t}_{\rho}}(\rho)
= \En_{t-1}\brk*{f\sups{\Mbartrhostar}(\rhostar)},
\]
we have the upper bound.
\begin{align}
  \En_{t-1}\brk*{\fmstar(\pimstar) - \fmstar(\pi\ind{t})}
  \leq{} \En_{t-1}\brk*{
  \En_{\actt\sim\pt}\brk*{
  f\sups{\Mbartrhostar}(\rhostar)
  - \fmbart(\act\ind{t})
  }
  }+\veps.
\end{align}
For the next step, it follows from the definition of
$p\ind{t}$ in \pref{line:minimax_bayes} that for every possible
realization of $\rhostar$, 
\begin{align*}
    \En_{\actt\sim\pt}\brk*{
  f\sups{\Mbartrhostar}(\rhostar)
  - \fmbart(\act\ind{t})
  }
  &\leq{} \comp(\cM\ind{t},\Mbar\ind{t})
  + \gamma\cdot{}\En_{\actt\sim\pt}\brk*{
  \Dhels{\Mbar\ind{t}_{\rhostar}(\act\ind{t})}{\Mbar\ind{t}(\act\ind{t})}
    }\\
  &\leq{} \comp(\cM\ind{t},\Mbar\ind{t})
  + \gamma\cdot{}\En_{\actt\sim\pt}\brk*{
  \Dkl{\Mbar\ind{t}_{\rhostar}(\act\ind{t})}{\Mbar\ind{t}(\act\ind{t})}
  },
\end{align*}
where the second inequality is \pref{lem:pinsker}. We conclude that
\begin{align}
  \En\brk*{\RegDM}
  \leq{}
   \gamma\cdot{}\En\brk*{\sum_{t=1}^{T}\En_{t-1}\En_{\actt\sim\pt}\brk*{
  \Dkl{\Mbar\ind{t}_{\rhostar}(\act\ind{t})}{\Mbar\ind{t}(\act\ind{t})}
  }
  }
   +
  \En\brk*{\sum_{t=1}^{T}\comp(\cM\ind{t},\Mbar\ind{t})}
  + \veps{}T\notag\\
    \leq{}
  \gamma\cdot{}\En\brk*{\sum_{t=1}^{T}\En_{t-1}\brk*{
  \Dkl{\Mbar\ind{t}_{\rhostar}(\act\ind{t})}{\Mbar\ind{t}(\act\ind{t})}
  }
  }
  +\sup_{\Mbar\in\conv(\cM)}\comp(\conv(\cM),\Mbar)\cdot{}T
  + \veps{}T.\label{eq:bayes_basic}
\end{align}
It remains to bound the first sum, for which we closely follow \cite{russo2014learning}. Consider a fixed timestep $t$ and decision $\act\in\Act$. Observe that
$\Mbar\ind{t}(\act)$ is identical to the law
$\bbP_{(r\ind{t},\obs\ind{t})\mid{}\act\ind{t}=\act,\filt\ind{t-1}}$, while
$\Mbar\ind{t}_{\rho}(\act)$ is identical to the law
$\bbP_{(r\ind{t},\obs\ind{t})\mid{},\rhostar=\rho,\act\ind{t}=\act,\filt\ind{t-1}}$. Hence,
for any $\act$, we have
\[
\En_{t-1}\brk*{
  \Dkl{\Mbar\ind{t}_{\rhostar}(\act)}{\Mbar\ind{t}(\act)}
}
= I_{t-1}( \rhostar \midsem (r\ind{t},\obs\ind{t}) \mid{} \act\ind{t}=\act),
\]
where $I_{t-1}(X\midsem Y\mid{}Z)$ denotes the conditional mutual information of
$(X,Y)$ given $Z$ (under $\filt\ind{t-1}$).\footnote{For random
  variables $X,Y,Z$, we define
$I(X\midsem{}Y\mid{}Z)=\En_{Z}\brk*{\Dkl{\bbP_{(X,Y)\mid{}Z}}{\bbP_{X\mid{}Z}\otimes\bbP_{Y\mid{}Z}}}=\En_{Y,Z}\brk*{\Dkl{\bbP_{X\mid{}Y,Z}}{\bbP_{X\mid{}Z}}}$. $I_{t}(X\midsem{}Y\mid{}Z)$
denotes the same quantity, conditioned on the outcome for $\filt\ind{t}$.}
Since $\act\ind{t}$ and $\rhostar$ are
conditionally independent given $\filt\ind{t-1}$, we further have
\begin{align*}
  \En_{t-1}\brk*{
  \Dkl{\Mbar\ind{t}_{\rhostar}(\act\ind{t})}{\Mbar\ind{t}(\act\ind{t})}
  }
  &= I_{t-1}( \rhostar \midsem (r\ind{t},\obs\ind{t}) \mid{} \act\ind{t})\\
    &= I_{t-1}( \rhostar \midsem (r\ind{t},\obs\ind{t}) \mid{}
      \act\ind{t}) + I_{t-1}(\rhostar\midsem\act\ind{t})\\
  &= I_{t-1}( \rhostar \midsem (\act\ind{t},r\ind{t},\obs\ind{t})).
\end{align*}
Summing and using the chain rule for mutual information, we have
\begin{align*}
  \En\brk*{
  \sum_{t=1}^{T}I_{t-1}( \rhostar \midsem (\act\ind{t},r\ind{t},\obs\ind{t}))
  }
  = \sum_{t=1}^{T}I
  ( \rhostar \midsem
  (\act\ind{t},r\ind{t},\obs\ind{t})\mid\hist\ind{t-1})
  = I(\rhostar\midsem\hist\ind{T}).
\end{align*}
Finally, we have
\[
  I(\rhostar\midsem\hist\ind{T})
  \leq{} H(\rhostar) \leq \log\abs{\Rho}=\log\ActCov.
\]
Putting everything together gives
\begin{align*}
  \En\brk*{\RegDM}
    \leq{}
  \gamma\cdot \log\ActCov
  +\sup_{\Mbar\in\conv(\cM)}\comp(\conv(\cM),\Mbar)\cdot{}T
  + \veps{}T,
\end{align*}
and the result follows by optimizing over $\veps\geq{}0$ and $\gamma>0$.
\end{proof}

\subsection{Proofs for Learnability Results}

\learnabilitymain*

\learnabilitymainbayes*

\begin{proof}[Proof of \pref{thm:learnability_main} and \pref{thm:learnability_main_bayes}]
    The proofs for \pref{thm:learnability_main} and
  \pref{thm:learnability_main_bayes} differ only in how we derive the upper bound on regret.

  We begin with the upper bound. Assume that
  $\lim_{\gamma\to\infty}\comp(\cM)\cdot{}\gamma^{\rho}=0$ for some
  $\rho>0$. For \pref{thm:learnability_main}, we assume $\MComp=\bigoht(T^{q})$ for
  some $q<1$ and apply \pref{thm:upper_main_infinite}, which gives that for any
  $T\in\bbN$ and $\gamma>0$, \mainalg with an appropriate oracle
  has\footnote{\pref{thm:upper_main_infinite} is stated as a high-probability guarantee, but since $\cR\subseteq\brk*{0,1}$ we can deduce this in-expectation guarantee by setting $\delta=1/T$.}
  \begin{align*}
    \sup_{M\in\cM}\En\sups{M}\brk*{\RegDM}
    \leq{} \bigoht\prn*{
    \comp(\cM)\cdot{}T + \gamma\cdot\MComp
    }
     \leq{} \bigoht\prn*{
    \comp(\cM)\cdot{}T + \gamma\cdot{}T^{q}
    },
  \end{align*}
  where $\bigoht(\cdot)$ hides factors logarithmic in $T$ and $B$. For
  \pref{thm:learnability_main_bayes}, we assume
  $\ActComp=\bigoht(T^{q})$ for some $q<1$ and apply
  \pref{thm:upper_main_bayes}, which implies that there exists an
  algorithm with
    \begin{align*}
    \sup_{M\in\cM}\En\sups{M}\brk*{\RegDM}
    \leq{} \bigoh\prn*{
    \comp(\cM)\cdot{}T + \gamma\cdot\ActComp
    }
     \leq{} \bigoht\prn*{
    \comp(\cM)\cdot{}T + \gamma\cdot{}T^{q}
    }.
  \end{align*}
  For both cases, we set $\gamma_T=T^{\frac{1-q}{1+\rho}}$, where we
  recall that $1-q>0$. Since
  $\lim_{\gamma\to\infty}\comp(\cM)\cdot{}\gamma^{\rho}=0$, we have
  that for any $\veps>0$, there exists $\gamma'>0$ such that
  $\comp(\cM)\leq{}\veps/\gamma^{\rho}$ for all
  $\gamma\geq{}\gamma'$. In particular, for $T$ sufficiently large, we have
  \begin{align*}
    \sup_{M\in\cM}\En\sups{M}\brk*{\RegDM}
     \leq{} \bigoht\prn*{
    \frac{T}{\gamma_T^{\rho}}+ \gamma_T\cdot{}T^{q}
    } = \bigoht(T^{\frac{1+\rho{}q}{1+\rho}}).
  \end{align*}
  Since $p\ldef{}\frac{1+\rho{}q}{1+\rho}<1$, this establishes the result. In
  particular, defining $p'\ldef\frac{1}{2}(p+1)<1$, we have
  \[
    \lim_{T\to\infty}\frac{\sup_{M\in\cM}\En\sups{M}\brk*{\RegDM}}{T^{p'}}=0.
  \]
We now proceed with the lower bound. Suppose that
$\lim_{\gamma\to\infty}\comp(\cM)\cdot\gamma^{\rho}=\infty$ for all
$\rho>0$, and consider any fixed $\rho\in(0,1/2)$.\footnote{Assuming that
  $\lim_{\gamma\to\infty}\comp(\cM)\cdot\gamma^{\rho}=\infty$ for all
  $\rho>0$ is equivalent to assuming that $\lim_{\gamma\to\infty}\comp(\cM)\cdot\gamma^{\rho}>0$
  for all $\rho>0$.} Applying \pref{thm:lower_main}, we are guaranteed that for
any algorithm and $\delta\in(0,1)$, for any $\gamma=\omega(\sqrt{T\log(T)})$, with probability at least $\delta/2$, 
\begin{align*}
  \RegDM
  &=\bigomt\prn*{
  \min\crl*{\compbasic_{\gamma,\veps(\gamma,T)}(\cM)\cdot{}T, \gamma} - \delta{}T
    },
\end{align*}
where $\veps(\gamma,T)\ldef{}c\cdot\frac{\gamma}{T\log(T)}$ for a
numerical constant $c\leq{}1$. Hence, since \pref{ass:mild} is
satisfied, whenever $\veps(\gamma,T)\leq{}1$ we can apply the
local-to-global lemma (\pref{lem:local_to_global}) to lower bound by
\begin{align*}
  \RegDM  &\geq\bigomt\prn*{
            \min\crl*{\veps(\gamma,T)\cdot\comp(\cM)\cdot{}T, \gamma}
  -\delta{}T
  }.
\end{align*}
Let $\gamma_T=T$ (which has $\gamma_T=\omega(\sqrt{T\log(T)})$) and
$\delta_T = c_T\cdot{}T^{-\rho}$, where
$c_T\ldef{}1/\polylog(T)$ is sufficiently small. Since
$\lim_{\gamma\to\infty}\comp(\cM)\cdot\gamma^{\rho}=\infty$, we have
that for all $C>0$, there exists $\gamma'>0$ such that
$\comp(\cM)\geq{}C\gamma^{-\rho}$ for all $\gamma\geq{}\gamma'$. Hence, for $T$
sufficiently large, we have
$\comp[\gamma_T](\cM)\geq\gamma_T^{-\rho}$
and
\[
  \RegDM
  =\bigomt\prn*{
  \min\crl*{\frac{T}{\gamma_T^{\rho}}, \gamma_T}
  -\delta_TT
}
  =\bigomt\prn*{
  \min\crl*{T^{1-\rho}, T}
  -\delta_TT
}
= \bigomt(T^{1-\rho}),
\]
since $\veps(\gamma_T,T)\propto\frac{1}{\log(T)}$. Furthermore, since
$\RegDM$ is non-negative, the law of total expectation implies that
\[
\En\brk*{\RegDM} \geq{} \bigomt(T^{1-\rho}\cdot{}\delta_T) =
\bigomt(
T^{1-2\rho}
).
\]
It follows that for any $p\in(0,1)$, if we set
$\rho=\frac{1-p}{2}\in(0,1/2)$, we have
\[
  \En\brk*{\RegDM} = \bigomt(T^{p}).
\]
In particular, if we apply this argument with
$p'=\frac{1}{2}(p+1)\in(1/2,1)$, we are guaranteed that
  \[
    \lim_{T\to\infty}\frac{\sup_{M\in\cM}\En\sups{M}\brk*{\RegDM}}{T^{p}}=\infty.
  \]
\end{proof}

\section{Proofs from \preft{sec:algorithm}}
\label{app:algorithm}
\subsection{Proofs from \preft{sec:oracle}}

In this section we prove \pref{thm:upper_general}. We prove the localized regret bound \pref{eq:upper_general3} in \pref{thm:upper_general} under the following, slightly more general, version of \pref{ass:hellinger_oracle}, which will be useful for applications. 
\begin{assumption}
  \label{ass:hellinger_oracle_generalized}
  The online estimation algorithm $\AlgEst$ guarantees
  that for a given $\delta\in(0,1)$, with probability at least
  $1-\delta$,
  \begin{equation}
    \sum_{t=1}^{T}\En_{\act\ind{t}\sim{}p\ind{t}}\brk*{\Dhels{\Mstar(\actt)}{\Mhat\ind{t}(\actt)}}\indic\crl{\Mstar\in\cM\ind{t}}
    \leq{} \EstProbHelT,\label{eq:hellinger_oracle_generalized}
  \end{equation}
  where $\EstProbHelT$ is a known upper bound. We further assume that
  $\Mhatt\in\conv(\cM\ind{t})$.
\end{assumption}

\begin{thmmod}{thm:upper_general}{a}
  \label{thm:upper_general2}
  Fix $\delta\in(0,1)$ and consider \pref{alg:main} with \optiontwo and $R^{2}=\EstProbHelT$.
Suppose \pref{ass:hellinger_oracle_generalized}, that
$\Rspace\subseteq\brk*{0,1}$, and that $\AlgEst$ ensures that
$\cMhat\ind{t}\subseteq\conv(\cM\ind{t})$ for all $t$. Then for any fixed $T\in\bbN$ and
$\gamma>0$, with probability at least $1-\delta$,
  \begin{align}
  \RegDM
  &\leq{}
    \sum_{t=1}^{T}\sup_{\Mbar\in\conv(\cM)}\comp(\cMloc[\veps_t](\Mbar),\Mbar)
    + \gamma\cdot{}\EstProbHelT,
    \label{eq:upper_general4}
\end{align}
where $\veps_t\ldef{}6\frac{\gamma}{t}\EstProbHelT + \sup_{\Mbar\in\conv(\cM)}\comp(\cM,\Mbar)
+ (2\gamma)^{-1}$.
\end{thmmod}

\begin{proof}[Proof of \pref{thm:upper_general} and \pref{thm:upper_general2}]
We first prove the basic non-localized results in
\pref{eq:upper_general1} and \pref{eq:upper_general2}. We write
\begin{align*}
  \RegDM &=
           \sum_{t=1}^{T}\En_{\act\ind{t}\sim{}p\ind{t}}\brk*{\fmstar(\pimstar)-\fmstar(\act\ind{t})}\\
         &=
           \sum_{t=1}^{T}\En_{\act\ind{t}\sim{}p\ind{t}}\brk*{\fmstar(\pimstar)-\fmstar(\act\ind{t})}
           - \gamma{}\cdot
           \En_{\act\ind{t}\sim{}p\ind{t}}\brk*{\Dhels{\Mstar(\act\ind{t})}{\Mhat\ind{t}(\act\ind{t})}}
           + \gamma\cdot{}\EstHel.
\end{align*}
For each $t$, since $\Mstar\in\cM$, we have
\begin{align}
  &\En_{\act\ind{t}\sim{}p\ind{t}}\brk*{\fmstar(\pimstar)-\fmstar(\act\ind{t})}
           - \gamma{}\cdot
  \En_{\act\ind{t}\sim{}p\ind{t}}\brk*{\Dhels{\Mstar(\act\ind{t})}{\Mhat\ind{t}(\act\ind{t})}}\notag\\
&  \leq  \sup_{M\in\cM}\En_{\act\ind{t}\sim{}p\ind{t}}\brk*{\fm(\pim)-\fm(\act\ind{t})}
           - \gamma{}\cdot
                                                                                                   \En_{\act\ind{t}\sim{}p\ind{t}}\brk*{\Dhels{M(\act\ind{t})}{\Mhat\ind{t}(\act\ind{t})}}\notag\\
  &  = \inf_{p\in\Delta(\Act)}\sup_{M\in\cM}\En_{\act\sim{}p}\brk*{\fm(\pim)-\fm(\act)
           - \gamma{}\cdot
    \Dhels{M(\act)}{\Mhat\ind{t}(\act)}}\notag\\
  &  = \comp(\cM,\Mhat\ind{t}).\label{eq:minimax_regret_basic}
\end{align}
We conclude that
\[
  \RegDM \leq{} \sup_{\Mbar\in\cMhat}\comp(\cM,\Mbar)\cdot{}T + \gamma\cdot\EstHel,
\]
which establishes \pref{eq:upper_general1}. Since this bound holds almost
surely, the result in \pref{eq:upper_general2} follows from \pref{ass:hellinger_oracle}.

\paragraph{Localized upper bound (\pref{thm:upper_general2})}
We now prove the localized regret bound in
\pref{eq:upper_general4} under \pref{ass:hellinger_oracle_generalized}.
Let $\Event$ denote the high-probability event in
\pref{ass:hellinger_oracle_generalized}. Recall that we set the
confidence radius used by \pref{alg:main} to $R^{2}=\EstProbHelT$. We
first prove that under $\Event$, the confidence sets contain $\Mstar$
and the Hellinger error is controlled.
\begin{lemma}
  \label{lem:hellinger_confidence}
  Under $\Event$, we have $\Mstar\in\cM\ind{t}$ for all
  $t\leq{}T$, and consequently
  \[
\EstHel =
\sum_{t=1}^{T}\En_{\act\ind{t}\sim{}p\ind{t}}\brk*{\Dhels{\Mstar(\actt)}{\Mhat\ind{t}(\actt)}}
\leq{} \EstProbHelT.
  \]
\end{lemma}
\begin{proof}[\pfref{lem:hellinger_confidence}]
We prove the result by induction. For the base case, we have
$\Mstar\in\cM\ind{1}$ with probability $1$. Now, fix $t\geq{}2$ and suppose
$\Mstar\in\cM\ind{i}$ for all $i < t$. Then
\pref{eq:hellinger_oracle_generalized} implies that
\begin{align*}
&\sum_{i=1}^{t-1}\En_{\act\ind{i}\sim{}p\ind{i}}\brk*{\Dhels{\Mstar(\act\ind{i})}{\Mhat\ind{i}(\act\ind{i})}}
\\
  &= \sum_{i=1}^{t-1}\En_{\act\ind{i}\sim{}p\ind{i}}\brk*{\Dhels{\Mstar(\act\ind{i})}{\Mhat\ind{i}(\act\ind{i})}}\indic\crl{\Mstar\in\cM\ind{i}}
    \leq{} \EstProbHelT.
\end{align*}
We conclude from the definition of $\cM\ind{t}$ in
\pref{line:confidence_set} that $\Mstar\in\cM\ind{t}$.
\end{proof}

We next state the key technical lemma for the proof, which uses properties of the
minimax problem that defines the \CompText (i.e., \pref{eq:algorithm_argmin}) to relate the localization
radius (in the sense of $\cM_{\veps}$) to Hellinger estimation error.
\begin{lemma}
  \label{lem:localization}
  Assume that $\Rspace\subseteq\brk*{0,1}$. Consider any model $M$
  such that $M\in\cM\ind{i}$ for all $i\leq{}t$. Then for any model
  $\Mbar$ (not necessarily in $\cM$),
  \begin{align}
    &\fm(\pim) - \fmbar(\pimbar)\notag\\
    &\leq{} (2\gamma)^{-1} + \frac{1}{t}\sum_{i=1}^{t}
    \prn*{\comp(\cM\ind{i},\Mhat\ind{i}) +
                2\gamma\En_{\act\sim{}p\ind{i}}\brk*{\Dhels{M(\act)}{\Mhat\ind{i}(\act)}}
      +\gamma\En_{\act\sim{}p\ind{i}}\brk*{\Dhels{\Mbar(\act)}{\Mhat\ind{i}(\act)}}}.
  \end{align}
\end{lemma}
\begin{proof}[\pfref{lem:localization}]
  Consider $i\leq{}t$. From the definition of the minimax program
  \pref{eq:algorithm_argmin}, we have that for all $M\in\cM\ind{i}$,
  \[
    \En_{\act\sim{}p\ind{i}}\brk*{\fm(\pim) - \fm(\act)} \leq{}
    \comp(\cM\ind{i},\Mhat\ind{i}) + \gamma\En_{\act\sim{}p\ind{i}}\brk*{\Dhels{M(\act)}{\Mhat\ind{i}(\act)}}.
  \]
  By the AM-GM inequality, this implies that
    \begin{align*}
      &\En_{\act\sim{}p\ind{i}}\brk*{\fm(\pim) - \fmbar(\act)} \\&\leq{}
    \comp(\cM\ind{i},\Mhat\ind{i}) +
    \gamma\En_{\act\sim{}p\ind{i}}\brk*{\Dhels{M(\act)}{\Mhat\ind{i}(\act)}}
    + \frac{1}{2\gamma} +
                                                 \frac{\gamma}{2}\En_{\act\sim{}p\ind{i}}\brk*{(\fm(\act)-\fmbar(\act))^{2}}\\
      &\leq{}
    \comp(\cM\ind{i},\Mhat\ind{i}) +
    \gamma\En_{\act\sim{}p\ind{i}}\brk*{\Dhels{M(\act)}{\Mhat\ind{i}(\act)}}
    + \frac{1}{2\gamma} +
        \frac{\gamma}{2}\En_{\act\sim{}p\ind{i}}\brk*{\Dhels{\Mbar(\act)}{M(\act)}},
    \end{align*}
    where the second inequality uses that
    $\abs{\fm(\act)-\fmbar(\act)}\leq{}\Dtv{\Mbar(\act)}{M(\act)}\leq{}\Dhel{\Mbar(\act)}{M(\act)}$
    when rewards lie in $\brk*{0,1}$. Since $\Dhel{\cdot}{\cdot}$ satisfies the triangle
    inequality, we can use the elementary inequality
    $(x+y)^2\leq2(x^2+y^2)$ to upper bound by
    \begin{align*}
      &\En_{\act\sim{}p\ind{i}}\brk*{\fm(\pim) - \fmbar(\act)}\\
                                               &\leq{}
                \comp(\cM\ind{i},\Mhat\ind{i}) +
                2\gamma\En_{\act\sim{}p\ind{i}}\brk*{\Dhels{M(\act)}{\Mhat\ind{i}(\act)}}
                + \frac{1}{2\gamma} + \gamma\En_{\act\sim{}p\ind{i}}\brk*{\Dhels{\Mbar(\act)}{\Mhat\ind{i}(\act)}}.
    \end{align*}
    Since
    $\En_{\act\sim{}p\ind{i}}\brk{\fmbar(\act)}\leq{}\fmbar(\pimbar)$, it follows immediately that
    \[
      \fm(\pim)-\fmbar(\pimbar) \leq{}
       \comp(\cM\ind{i},\Mhat\ind{i}) +\frac{1}{2\gamma}+
                2\gamma\En_{\act\sim{}p\ind{i}}\brk*{\Dhels{M(\act)}{\Mhat\ind{i}(\act)}}
                +\gamma\En_{\act\sim{}p\ind{i}}\brk*{\Dhels{\Mbar(\act)}{\Mhat\ind{i}(\act)}}.
              \]
              We obtain the final result by averaging this inequality
              over all $i\leq{}t$.
            \end{proof}
            
We now proceed as follows. Condition on the event $\Event$. By
\pref{lem:hellinger_confidence}, this implies that
$\Mstar\in\cM\ind{t}$ for all $t\leq{}T$. Hence, proceeding as in
\pref{eq:minimax_regret_basic}, we are guaranteed that
\begin{align*}
  \RegDM
  \leq{} \sum_{t=1}^{T}\comp(\cM\ind{t},\Mhat\ind{t})
  +\gamma\cdot\EstHel
    \leq{} \sum_{t=1}^{T}\comp(\cM\ind{t},\Mhat\ind{t})
  +\gamma\cdot\EstProbHelT.
\end{align*}
We now relate the confidence sets $\cM\ind{t}$ to localized sets of
the form
$\cMloc[\veps_t](\Mhat\ind{t})$. We trivially have
$\cM\ind{1}\subseteq\cMloc[\veps_1](\Mhat\ind{1})$, so consider $t\geq{}2$. Since
$\cM\ind{t}\subseteq\cM\ind{t-1}\subseteq\cdots\subseteq\cM\ind{1}$ by
definition, \pref{lem:localization} guarantees that for all
$M\in\cM\ind{t}$,
\begin{align*}
  &\fm(\pim)- \fhat\ind{t}(\pihat\ind{t})\\
  &\leq{} (2\gamma)^{-1} + \frac{1}{t-1}\sum_{i=1}^{t-1}
    \prn*{\comp(\cM\ind{i},\Mhat\ind{i}) +
                2\gamma\En_{\act\sim{}p\ind{i}}\brk*{\Dhels{M(\act)}{\Mhat\ind{i}(\act)}}
                +\gamma\En_{\act\sim{}p\ind{i}}\brk*{\Dhels{\Mhat\ind{t}(\act)}{\Mhat\ind{i}(\act)}}},
\end{align*}
where we define $\fhat\ind{t}\ldef{}f^{\sss{\Mhat\ind{t}}}$ and $\pihat\ind{t}\ldef{}\pi_{\sss{\Mhat\ind{t}}}$.
Using the definition of $\cM\ind{t}$ from \pref{line:confidence_set}
of \pref{alg:main}, we have that for all 
$M\in\cM\ind{t}$, 
\[
\sum_{i=1}^{t-1}
\En_{\act\sim{}p\ind{i}}\brk*{\Dhels{M(\act)}{\Mhat\ind{i}(\act)}} \leq{} \EstProbHelT. %
\]
Furthermore, since $\Mhat\ind{t}\in\conv(\cM\ind{t})$, if we let $q\ind{t}\in\Delta(\cM\ind{t})$
be such that $\Mhat\ind{t}=\En_{M\sim{}q\ind{t}}\brk*{M}$, we have
\begin{align*}
  \sum_{i=1}^{t-1}\En_{\act\sim{}p\ind{i}}\brk*{\Dhels{\Mhat\ind{t}(\act)}{\Mhat\ind{i}(\act)}}
  &=
  \sum_{i=1}^{t-1}\En_{\act\sim{}p\ind{i}}\brk*{\Dhels{\En_{M\sim{}q\ind{t}}\brk*{M(\act)}}{\Mhat\ind{i}(\act)}}\\
  &\leq{}
    \En_{M\sim{}q\ind{t}}\brk*{\sum_{i=1}^{t-1}\En_{\act\sim{}p\ind{i}}\brk*{\Dhels{M(\act)}{\Mhat\ind{i}(\act)}}}\\
  &\leq{}
    \sup_{M\in\cM\ind{t}}\crl*{\sum_{i=1}^{t-1}\En_{\act\sim{}p\ind{i}}\brk*{\Dhels{M(\act)}{\Mhat\ind{i}(\act)}}}
   \leq{} \EstProbHelT. %
\end{align*}

Hence, since $\frac{1}{t-1}\leq\frac{2}{t}$, we have
\begin{align*}
  \fm(\pim)- \fhat\ind{t}(\pihat\ind{t})
  &\leq{} (2\gamma)^{-1} + \frac{2}{t}\sum_{i=1}^{t-1}
  \comp(\cM\ind{i},\Mhat\ind{i})
  + 6\frac{\gamma}{t}\EstProbHelT \\%
  &\leq{} (2\gamma)^{-1} + 2\sup_{\Mbar\in\conv(\cM)}\comp(\cM,\Mbar)
    + 6\frac{\gamma}{t}\EstProbHelT \rdef \veps_t %
\end{align*}
It follows that $\cM\ind{t}\subseteq\cM_{\veps_t}(\Mhat\ind{t})$ and
\begin{align*}
  \RegDM
  &\leq{} \sum_{t=1}^{T}\comp(\cM_{\veps_t}(\Mhat\ind{t}),\Mhat\ind{t})
    +\gamma\cdot\EstProbHelT\\
    &\leq{} \sum_{t=1}^{T}\sup_{\Mbar\in\conv(\cM)}\comp(\cM_{\veps_t}(\Mbar),\Mbar)
  +\gamma\cdot\EstProbHelT.
\end{align*}

\end{proof}

\subsection{Proofs from \preft{sec:dual}}

\newcommand{\weakstar}{weak$^{\star}$\xspace}
\newcommand{\topspace}{\cZ}
\newcommand{\radon}{\Delta}

\minimaxswapdec*
\begin{proof}[\pfref{prop:minimax_swap_dec}]
  We follow the approach in \cite{lattimore2019information}. Since
  $\compb(\cM,\Mbar) \leq  \comp(\cM,\Mbar)$ by definition, it
  suffices to prove the opposite direction of the inequality.
  
  Recall that for a topological space $\topspace{}$, let $\radon(\topspace{})$
  denotes the space of Radon probability measures over $\topspace$ when
  $\topspace{}$ is equipped with the Borel $\sigma$-algebra. In
  addition, recall that the \weakstar topology on $\radon(\topspace{})$ is the coarsest
  topology such that the function $\mu\mapsto \int f d\mu$ is continuous for all bounded, continuous functions $f:\topspace{}\to \bbR$. 
  
  Let $\Act$ be equipped with the discrete topology, and recall that
  $\cM$ is equipped with the discrete topology as well. Since $\Act$ is
  finite, it is compact with respect to the discrete
  topology. Hence, Theorem 8.9.3 of \citet{bogachev2007measure},
  $\radon(\Act)$ is \weakstar-compact; $\radon(\Act)$ is also convex.

  Let $\cQ$ be the space of finitely supported probability measures on
  $\cM$, which is a convex subset of $\radon(\cM)$ when $\cM$ is
  equipped with the discrete topology. We equip $\cQ$ with the \weakstar-topology.

  Consider any function $g: \radon(\Act) \times \cQ \to \bbR$ that is
  linear and continuous with respect to both arguments in their
  respective topologies (i.e. $p\mapsto{}g(p,\mu)$ is continuous for
  all $\mu\in\cQ$ and $\mu\mapsto{}g(p,\mu)$ is continuous for all
  $p\in\radon(\Act)$). Since $\radon(\Act)$ is \weakstar-compact,
  Sion's minimax theorem \citep{sion1958minimax} implies that
  \begin{align*}
  \min_{p\in \radon(\Act)}\sup_{\mu\in\cQ} g(p,\mu) = \sup_{\mu\in\cQ}\min_{p\in \radon(\Act)} g(p,\mu).
  \end{align*}
  We proceed to verify that the function
  \begin{align*}
  g(p,\mu) \ldef{} \En_{\act\sim{}p}\En_{M\sim\mu}\biggl[\fm(\pim)-\fm(\pi)
  -\gamma\cdot\Dhels{M(\act)}{\Mbar(\act)}
  \biggr]
  \end{align*}
  linear and continuous in the sense above. Linearity is
  immediate. For continuity, since we consider the \weakstar-topology for both spaces, we only need to show that the function
  \begin{align*}
    (\pi,M) \mapsto{} \fm(\pim)-\fm(\pi)
    -\gamma\cdot\Dhels{M(\act)}{\Mbar(\act)}
  \end{align*}
  is continuous with respect to $\act$ and $M$ individually. This is
  follows because $\Act$ and $\cM$ are equipped with the discrete
  topology and the function is bounded (since $\Rspace$ and $\Dhels{\cdot}{\cdot}$ are bounded). We conclude that
  \begin{align*}
    \comp(\cM,\Mbar)  %
                         \leq  \min_{p\in \radon(\Act)}\sup_{\mu\in\cQ} g(p,\mu) = \sup_{\mu\in\cQ}\min_{p\in \radon(\Act)} g(p,\mu) 
                         \leq
                         \compb(\cM,\Mbar).
  \end{align*} 
\end{proof}

\subsection{Proofs from \preft{sec:general_distance}}

\uppergeneraldistance*
\begin{proof}[\pfref{thm:upper_general_distance}]
  This proof closely follows that of \pref{thm:upper_general}. We have
\begin{align*}
  \RegDM &=
           \sum_{t=1}^{T}\En_{\act\ind{t}\sim{}p\ind{t}}\brk*{\fmstar(\pimstar)-\fmstar(\act\ind{t})}\\
         &=
           \sum_{t=1}^{T}\En_{\act\ind{t}\sim{}p\ind{t}}\brk*{\fmstar(\pimstar)-\fmstar(\act\ind{t})}
           - \gamma{}\cdot
           \En_{\act\ind{t}\sim{}p\ind{t}}\En_{\Mhat\ind{t}\sim\nu\ind{t}}\brk*{\Dgen{\Mstar(\act\ind{t})}{\Mhat\ind{t}(\act\ind{t})}}
           + \gamma\cdot{}\EstD.
\end{align*}
For each $t$, since $\Mstar\in\cM$, we have
\begin{align}
  &\En_{\act\ind{t}\sim{}p\ind{t}}\brk*{\fmstar(\pimstar)-\fmstar(\act\ind{t})}
           - \gamma{}\cdot
  \En_{\act\ind{t}\sim{}p\ind{t}}\En_{\Mhat\ind{t}\sim\nu\ind{t}}\brk*{\Dgen{\Mstar(\act\ind{t})}{\Mhat\ind{t}(\act\ind{t})}}\notag\\
&  \leq  \sup_{M\in\cM}\En_{\act\ind{t}\sim{}p\ind{t}}\brk*{\fm(\pim)-\fm(\act\ind{t})}
           - \gamma{}\cdot
\En_{\act\ind{t}\sim{}p\ind{t}}\En_{\Mhat\ind{t}\sim\nu\ind{t}}\brk*{\Dgen{M(\act\ind{t})}{\Mhat\ind{t}(\act\ind{t})}}\notag\\
  &  = \inf_{p\in\Delta(\Act)}\sup_{M\in\cM}\En_{\act\sim{}p}\brk*{\fm(\pim)-\fm(\act)
           - \gamma{}\cdot
    \En_{\Mhat\ind{t}\sim\nu\ind{t}}\Dgen{M(\act)}{\Mhat\ind{t}(\act)}}\notag\\
  &  = \compgen(\cM,\nu\ind{t}).\notag
\end{align}
Since
$\compgen(\cM,\nu\ind{t})\leq{}\sup_{\nu\in\Delta(\cMhat)}\compgen(\cM,\nu)$,
this establishes the result.
\end{proof}

\section{Proofs from \preft{sec:bandit}}
\label{app:bandit}

\subsection{Proofs from \preft{sec:bandits_familiar}}

\subsubsection{Linear Bandits}

\begin{proof}[\pfref{prop:bandit_lower_linear}]
Let $\Delta\in\brk{0,1}$ be a parameter. We construct a hard family of
models $\cM'=\crl{M_i}_{i\in\brk{d}}$ as follows. First,
define $\theta_i = \Delta{}\cdot{}e_i$, where $e_i$ denotes the standard
basis vector, then let $M_i(\act) = \Rad(\tri{\theta_i,\act})$; this
construction has $\fmi(\act)=\tri{\theta_i,\act}$. Then, take
$\Mbar(\act)=\Rad(\tri{\mathbf{0},\act})=\Rad(0)$ as the reference model.

We now show that $\cM'$ is a hard family of models in the sense of
\pref{lem:hard_family}. Define $\hardu_i(\act) = \act_i$ and
$\hardv_i(\act) = \act_i^2$. Then for any $i$, we have
\[
  \fmi(\pimi) - \fmi(\act) = \tri{\theta_i,e_i-\act} = \Delta(1-\hardu_i(\act)).
\]
Next, using \pref{lem:divergence_rademacher}, we have
\[
  \Dhels{M_i(\act)}{\Mbar(\act)}=\Dhels{\Rad(\tri{\theta_i,\act})}{\Rad(0)}
  \leq{} \frac{3}{4}\Delta^{2}\act_i^2=\frac{3}{4}\Delta^2\hardv_i(\act).
\]
Finally, note that since $\nrm{\act}_2\leq{}1$, we have
$\sum_{i=1}^{d}\hardu_i(\act)=\sum_{i=1}^{d}\act_i\leq\sqrt{d}\leq\frac{d}{2}$
(whenever $d\geq{}4$) and
$\sum_{i=1}^{d}\hardv_i(\act)=\sum_{i=1}^{d}\act_i^2\leq{}1$. It
follows that $\cM'$ is a $(\Delta,\frac{3}{4}\Delta^{2},0)$-family, so
\pref{lem:hard_family} implies that
\[
  \comp(\cM',\Mbar)
  \geq{} \frac{\Delta}{2} - \gamma\frac{3\Delta^2}{4d}.
\]
We choose $\Delta = \frac{d}{3\gamma}$, which gives
$\comp(\cM',\Mbar)\geq{}\frac{d}{12\gamma}$. Finally, note that
$\cM'\subseteq\cMinf[\Delta](\Mbar)$, and that
we may take $\abscontp=\bigoh(1)$ in \pref{thm:lower_main,thm:lower_main_expectation}
whenever $\Delta\leq{}1/2$; it suffices to restrict to $\gamma\geq{}\frac{2d}{3}$.
\end{proof}

\subsubsection{Nonparametric Bandits}

\begin{proof}[\pfref{prop:bandit_upper_lipschitz}]
  Let $M\in\cM$ be fixed. Let $\Act'$ be an $\veps$-cover for $\Pi$. Since $\fm$
  is $1$-Lipschitz, for all $\pi$ there exists a corresponding
  covering element $\cov(\act)\in\Act'$ such that
  $\rho(\pi,\cov(\act))\leq\veps$, and consequently for any distribution $p$,
  \begin{align*}
    \En_{\act\sim{}p}\brk*{\fm(\pim) -\fm(\act)}
    &\leq \En_{\act\sim{}p}\brk*{\fm(\cov(\pim)) -\fm(\act)} + \abs{\fm(\pim) -
    \fm(\cov(\pim))}\\
    &\leq{} \En_{\act\sim{}p}\brk*{\fm(\cov(\pim)) -\fm(\act)} +
    \rho(\pim,\cov(\pim))\\
    &\leq{} \En_{\act\sim{}p}\brk*{\fm(\cov(\pim)) -\fm(\act)} + \veps.
  \end{align*}
  At this point, since $\cov(\pim)\in\Act'$, \pref{prop:igw_mab} ensures
  that if we choose $p$ using inverse gap weighting over $\Act'$
  \[
    \En_{\act\sim{}p}\brk*{\fm(\cov(\pim))-\fm(\act)}
    \leq{} \frac{\abs{\Act'}}{\gamma} + \gamma\cdot\En_{\act\sim{}p}\brk*{(\fm(\act)-\fmbar(\act))^2}.
  \]
  From our assumption on the growth of $\Mcov(\Act,\veps)$,
  $\abs{\Act'}\leq\veps^{-d}$, so the value is at most
  \[
    \veps +  \frac{\veps^{-d}}{\gamma}.
  \]
  We choose $\veps=\gamma^{-\frac{1}{d+1}}$ to balance the terms,
  leading to the result.
\end{proof}

\begin{proof}[\pfref{prop:bandit_lower_lipschitz}]
  Let $\veps\in(0,1/2)$ be fixed, and let $N=\Mcov(\Act,2\veps)$. By the
  duality of packing and covering, there exists a packing $\act_1,\ldots,\act_N$
  such that
  \[
    \met(\act_i,\act_j)>2\veps\quad\text{for all $i\neq{}j$}.
  \]
  We define a class of models $\cM'=\crl*{M_1,\ldots,M_N}$ based on
  this packing as follows. First, define $h(x) = \max\crl*{1-x,
    0}$. Then, let
  \[
    f_i(\act) = \frac{1}{2} + \veps{}h(\met(\act,\act_i)/\veps).
  \]
We have $f_i(\act)\in[1/2,1)$, and since $\met$ is a metric, 
  \[
    \abs{f_i(\act)-f_i(\act')}
    \leq{}
    \veps\abs{h(\met(\act,\act_i)/\veps)-h(\met(\act',\act_i)/\veps)}
    \leq{} \abs{\met(\act,\act_i)-\met(\act',\act_i)}
    \leq \met(\act,\act'),
  \]
  so $f_i$ is $1$-Lipschitz. Finally, we define
  \[
    M_i(\act) = \Ber(f_i(\act)),
  \]
  and define the reference model $\Mbar$ via
  \[
    \Mbar(\act) = \Ber(\nicefrac{1}{2}).
  \]

We now prove that $\cM'$ is a hard family in the sense of
\pref{lem:hard_family}. Let
$\cI_i=\crl{\act\in\Act\mid{}\met(\act,\act_i)\leq\veps}$, and let
\[
\hardu_i(\act) = \hardv_i(\act)=\indic\crl*{\act\in\cI_i}.
\]
Since $\rho(\act_i,\act_j)>2\veps$ for $i\neq{}j$,
$\cI_i\cap\cI_j=\emptyset$, so we have $\sum_{i}\hardu_i(\act)\leq{}1$
and $\sum_{i}\hardv_i(\act)\leq{}1$ as required.

Now, observe that for all $i$ and all $\act\in\Act$,
\[
\fmi(\pimi) - \fmi(\act) \geq{}\veps\indic\crl*{\act\notin\cI_i}=\veps(1-\hardu_i(\act))
\]
and by \pref{lem:divergence_bernoulli}
\[
  \Dhels{M_i(\act)}{\Mbar(\act)}
  \leq{} 3\veps^{2}h^{2}(\rho(\act,\act_i)/\veps)
    \leq{} 3\veps^{2}\indic\crl{\act\in\cI_i}.
  \]
  Hence, $\cM'$ is a $(\veps,3\veps^2,0)$-family, and
  \pref{lem:hard_family} implies that
  \[
    \comp(\cM,\Mbar)
    \geq{} \frac{\veps}{2} - \gamma\frac{3\veps^{2}}{N}
    \geq{} \frac{\veps}{2} - \gamma\frac{3\veps^{2}}{(2\veps)^{-d}}
        \geq{} \frac{\veps}{2} - \gamma(6\veps)^{2+d}.
      \]
      We choose $\veps=(36)^{-1}\gamma^{-\frac{1}{d+1}}$, which leads
      to value
      \[
            \comp(\cM,\Mbar)  \geq (108)^{-1}\gamma^{-\frac{1}{d+1}}
      \]
      whenever $d\geq{}1$. Note that
      $\cM'\subseteq\cMinf[\veps](\Mbar)$, and that
      we may take $\abscontp=\bigoh(1)$ in \pref{thm:lower_main,thm:lower_main_expectation}
      whenever
      $\veps\leq{}1/4$; it suffices to restrict to $\gamma\geq{}1$.
\end{proof}

\subsubsection{ReLU Bandits}

\begin{proof}[\pfref{prop:bandit_lower_relu}]
  Let $\veps\in(0,1)$ be a parameter, and recall that
  $\Act=\Theta=\crl{v\in\bbR^{d}\mid{}\nrm{v}_2\leq{}1}$. For each
  $v\in\Theta$, define a mean reward function
  \[
    f_v(\act)=\relu(\tri{v,\act}-(1-\veps)),
  \]
  so that $\max_{\act\in\Act}f_v(\act)=\veps$. Define a corresponding
  model $M_v$ via
  \[
    M_v(\act) \ldef{} \Rad(f_v(\act)).
  \]
  Finally, define the reference model via
  $\Mbar(\act)\ldef{}M_{\mathbf{0}}(\act)=\Rad(0)$.  Note that if we choose
  $\cM'=\crl*{M_v\mid{}\nrm{v}_2=1}$, we have
  $\cM'\subseteq\cMinf[\veps](\Mbar)$, since $\max_{\act\in\Act}\fmbar(\act)=0$.

  We observe that this family satisfies the
  following properties:
  \begin{itemize}
    \item For all $\act\in\Act$, $\max_{\act'\in\Act}f_v(\act') -
      f_v(\act)\geq{}\veps\indic\crl{\tri{v,\act}\leq{}1-\veps}$.
    \item For all $\act\in\Act$,
      $\Dhels{M_v(\act)}{\Mbar(\act)}=\frac{3}{4}f_v^{2}(\act)\leq{}\frac{3\veps^{2}}{4}\indic\crl{\tri{v,\act}>1-\veps}$,
      where we have used \pref{lem:divergence_rademacher}.
    \end{itemize}
    Consequently, we have
    \begin{align*}
      \comp(\cM',\Mbar)
      &\geq{}
      \inf_{p\in\Delta(\Act)}\sup_{\nrm{v}_2=1}\En_{\act\sim{}p}\brk*{
      \veps\indic\crl{\tri{v,\act}\leq{}1-\veps}
      - \gamma\frac{3\veps^{2}}{4}\indic\crl{\tri{v,\act}>1-\veps}
        } \\
            &=
      \inf_{p\in\Delta(\Act)}\sup_{\nrm{v}_2=1}\En_{\act\sim{}p}\brk*{
      \veps - \veps\indic\crl{\tri{v,\act}>1-\veps}
      - \gamma\frac{3\veps^{2}}{4}\indic\crl{\tri{v,\act}>1-\veps}
              }\\
                  &\geq{}
      \inf_{p\in\Delta(\Act)}\En_{v\sim\mathrm{unif}(\bbS^{d})}\En_{\act\sim{}p}\brk*{
      \veps - \veps\indic\crl{\tri{v,\act}>1-\veps}
      - \gamma\frac{3\veps^{2}}{4}\indic\crl{\tri{v,\act}>1-\veps}
      }.
    \end{align*}
    We now appeal to the following lemma.
    \begin{lemma}[\cite{ball1997elementary}, Lecture 8]
      Let $v$ be uniform on $\bbS^{d}$. Then for any unit vector $x$
      and all $\alpha\in\brk*{0,1}$.
      \[
        \bbP(\abs{\tri{x,v}}>\alpha)\leq{}2\exp\prn*{-\frac{\alpha^2}{2}d}.
      \]
    \end{lemma}
    Applying this lemma, we have that for all $\veps\in\brk*{0,1/2}$,
    \begin{align*}
      \comp(\cM',\Mbar)
      \geq{} \veps - 2\veps\exp(-d/8) - \frac{3}{2}\gamma\veps^{2}\exp(-d/8).
    \end{align*}
    Whenever $d\geq{}16$, this is lower bounded by
    \[
      \frac{\veps}{2} - \frac{3}{2}\gamma\veps^{2}\exp(-d/8).
    \]
    We conclude by choosing $\veps =
    \frac{e^{d/8}}{6\gamma}\wedge\frac{1}{2}$, which leads to
    $\comp(\cM',\Mbar)\geq{}\frac{e^{d/8}}{24\gamma}\wedge\frac{1}{8}$,
    and has $\abscontp=\bigoh(1)$ for \pref{thm:lower_main,thm:lower_main_expectation}.

\end{proof}

\subsubsection{Gap-Dependent Lower Bounds}

\begin{proof}[\pfref{prop:bandit_gap_tabular}]
  This proof is almost the same as that of \pref{prop:mab_lower}. The
  only difference is that we change the construction so that $\Mbar$
  has a gap, which leads to worse constants.
  
  Fix $\Delta\in(0,1/8)$ and define family of models $\cM'=\crl*{M_1,\ldots,M_A}$ by setting
  \[
    M_1(\act) = \Ber(\nicefrac{1}{2}+\Delta\indic\crl{\act=1})
  \]
  and
  \[
    M_i(\act)=\Ber(\nicefrac{1}{2}+\Delta\indic\crl{\act=1}+2\Delta\indic\crl{\act=i})
  \]
  for $i>1$. Let $\Mbar=M_1$, and note that $\cM'\subseteq\cMinf[\Delta](\Mbar)$.

  We now verify that this is a hard family. Set $\hardu_i(\act)=\hardv_i(\act)=\indic\crl{\act=i}$. We have
  \[
    \fmi(\pimi)-\fmi(\act)\geq\Delta(1-\indic\crl{\act=i})
  \]
  since all of the models have gap $\Delta$. Furthermore,
  \[
    \Dhels{M_i(\act)}{\Mbar(\act)}\leq{}
    \Dhels{\Ber(\nicefrac{1}{2}+2\Delta)}{\Ber(\nicefrac{1}{2})}\indic\crl{\act=i}
    \leq{} 12\Delta^{2}\indic\crl{\act=i}
  \]
  where we have used \pref{lem:divergence_bernoulli}. It follows that this is a
  $(\Delta,5\Delta^2,0)$-family, so \pref{lem:hard_family} implies that for any $\gamma>0$,
  \[
    \comp(\cM',\Mbar)
    \geq{} \frac{\Delta}{2} - \gamma\frac{12\Delta^{2}}{A}.
  \]
  In particular, whenever $\gamma\leq{}\frac{A}{48\Delta}$, we have
  $\comp(\cM',\Mbar) \geq \frac{\Delta}{4}$. Furthermore, since
  $\Delta\leq{}1/8$,
we may take $\abscontp=\bigoh(1)$ in \pref{thm:lower_main,thm:lower_main_expectation}.
\end{proof}

\begin{proof}[\pfref{prop:bandit_gap_linear}]
    This proof is a simple modification to
    \pref{prop:bandit_lower_linear} to ensure that $\Mbar$ has gap $\Delta$.

Let $\Delta\in\prn{0,1/4}$ be a given. We set
$\Act=\crl*{e_1,\ldots,e_d}$ and construct a family
$\cM'=\crl{M_i}_{i\in\brk{d}}$ as follows. Let  $\theta_1 = \Delta{}\cdot{}e_1$, and let $\theta_i =
\Delta\cdot{}e_1 + 2\Delta\cdot{}e_i$ for all $i\geq{}2$; we have
$\nrm{\theta_i}_2\leq{}1$ whenever $\Delta\leq1/\sqrt{5}$. Define
$M_i(\act) = \Rad(\tri{\theta_i,\act})$ so that $\fmi(\act)=\tri{\theta_i,\act}$, and take $\Mbar=M_1$ as the
reference model. We have $\cM'\subseteq\cMinf[\Delta](\Mbar)$.

We now show that $\cM'$ is a hard family of models (cf.
\pref{lem:hard_family}). Define $\hardu_i(\act) =
\hardv(\act)=\indic\crl{\act=e_i}$. Then for any $i$, we have
\[
  \fmi(\pimi) - \fmi(\act) \geq{}\Delta(1-\hardu_i(\act))
\]
and, using \pref{lem:divergence_rademacher}, 
\[
  \Dhels{M_i(\act)}{\Mbar(\act)}\leq{}\Dhels{\Rad(2\Delta)}{\Rad(0)}\indic\crl{\act=e_i}
  \leq{} 3\Delta^2\hardv_i(\act).
\]
It follows that $\cM'$ is a $(\Delta,3\Delta^{2},0)$-family, so
\pref{lem:hard_family} implies that
\[
  \comp(\cM',\Mbar)
  \geq{} \frac{\Delta}{2} - \gamma\frac{3\Delta^2}{2d}.
\]
In particular, $\comp(\cM',\Mbar)\geq{}\frac{\Delta}{4}$ whenever
$\gamma\leq\frac{d}{12\Delta}$. Furthermore, since
$\Delta\leq{}1/4$,
we may take $\abscontp=\bigoh(1)$ in \pref{thm:lower_main,thm:lower_main_expectation}.
\end{proof}

\subsection{Proofs from \preft{sec:bandits_star}}

This section is organized as follows. First, we prove a generic
decoupling-type lemma which holds for a slightly more general setting
than what is described in \pref{sec:bandits_star}. We then prove
\pref{thm:eluder_star} and \pref{thm:disagreement} as consequences.

\subsubsection{Decoupling Lemma}
\newcommand{\Zset}{\cZ}%
\newcommand{\zdx}{z}%
\newcommand{\rhomu}{\rho_{\mu}}%

Consider a general setting where we are given a set $\Act$ and function class
$\cF:\Act\to\bbR$. Let $\Zset$ be an abstract set, and let $\crl{f_{\zdx}}_{\zdx\in\Zset}$
and $\crl{\act_{\zdx}}_{z\in{}\Zset}$ be subsets of $\cF$ and $\Act$
indexed by $\cZ$. Suppose we are given a distribution
$\mu\in\Delta(\cZ)$ from which a random index is drawn. The following
lemma allows one to decouple the random variables 
$f_{\zdx}$ and $\act_{\zdx}$ under $\zdx\sim\mu$.
\begin{lemma}
  \label{thm:decoupling_general}
  For any distribution $\mu\in\Delta(\cZ)$ and class
  $\cF\subseteq{}(\Act\to\brk{-R,R})$, we have
  \begin{equation}
    \label{eq:decoupling_general}
    \En_{\zdx\sim\mu}\brk*{\abs{f_{\zdx}(\act_{\zdx})}}
    \leq{} \inf_{\Delta>0}\crl*{
      2\Delta +
      6\frac{\sdis(\cF,\Delta,\gamma^{-1};\rhomu)\log^{2}(R\gamma\vee{}e)}{\gamma}}
    + \gamma\cdot \En_{z,z'\sim\mu}\brk*{f_{\zdx}^{2}(\act_{\zdx'})}.
  \end{equation}
  for all $\gamma>0$, where $\rhomu(\act)\ldef{}\mu(\crl{\act_z=\act})$.
\end{lemma}

\begin{proof}[\pfref{thm:decoupling_general}]%
  Assume $R=1$ without loss of generality. Fix a parameter
  $\Delta\in(0,1]$. We begin by defining a clipping operator $\brk*{X}_{\Delta}=\sqrt{X^{2}-\Delta^{2}}\indic\crl*{\abs*{X}\geq\Delta}$
  and upper bounding
  \begin{align*}
    \En_{\zdx\sim\mu}\brk*{\abs*{f_{\zdx}(\act_{\zdx})}}
    \leq{} \Delta + \En_{\zdx\sim\mu}\brk*{\brk*{f_{\zdx}(\act_{\zdx})}_{\Delta}}.
  \end{align*}
  Next, observe that for any $\veps\in(0,1]$,
  \begin{align*}
    \En_{\zdx\sim\mu}\brk*{\brk*{f_{\zdx}(\act_{\zdx})}_{\Delta}}
    &=
      \En_{\zdx\sim\mu}\brk*{\frac{\brk*{f_{\zdx}(\act_{\zdx})}_{\Delta}}{\prn*{
      \En_{\zdx'\sim\mu}\brk*{f_{\zdx}^{2}(\act_{\zdx'})}\vee\veps^{2}}^{1/2}}\prn*{
      \En_{\zdx'\sim\mu}\brk*{f_{\zdx}^{2}(\act_{\zdx'})}\vee\veps^{2}}^{1/2}
      }.
  \end{align*}
  By Cauchy Schwarz and the AM-GM inequality, for all $\eta>0$ this is
  bounded by
  \[
    \frac{1}{2\eta}      \En_{\zdx\sim\mu}\brk*{\frac{\brk*{f_{\zdx}(\act_{\zdx})}_{\Delta}^{2}}{
        \En_{\zdx'\sim\mu}\brk*{f_{\zdx}^{2}(\act_{\zdx'})}\vee\veps^{2}}}
+\frac{\eta}{2}\En_{\zdx\sim\mu}\En_{\zdx'\sim\mu}\brk*{f_{\zdx}^{2}(\act_{\zdx'})} + \frac{\eta}{2}\veps^{2}.
\]
The second term above matches the right-hand side of
\pref{eq:decoupling_general}, and the $\veps^2$ term will be shown to contribute a negligible error, so it remains to upper
bound the first term. To begin, we have
\begin{align*}
  \En_{\zdx\sim\mu}\brk*{\frac{\brk*{f_{\zdx}(\act_{\zdx})}_{\Delta}^{2}}{
      \En_{\zdx'\sim\mu}\brk*{f_{\zdx}^{2}(\act_{\zdx'})}\vee\veps^{2}}}
  &\leq{}\En_{\zdx\sim\mu}\brk*{\sup_{f\in\cF}\frac{\brk*{f(\act_{\zdx})}_{\Delta}^{2}}{
      \En_{\zdx'\sim\mu}\brk*{f^2(\act_{\zdx'})
    }\vee\veps^{2}}}.
\end{align*}
Observe that for any $\veps\leq{}X\leq{}1$, we have
\[
\frac{1}{X^{2}} = 2\int_{X}^{1}\frac{1}{t^{3}}dt + 1= 2\int_{\veps}^{1}\frac{1}{t^{3}}\indic\crl{t\geq{}X}dt + 1.
\]
Since $\abs*{f}\leq{}1$, this means that we can upper bound
\begin{align*}
    \En_{\zdx\sim\mu}\brk*{\sup_{f\in\cF}\frac{\brk*{f(\act_{\zdx})}_{\Delta}^{2}}{
  \En_{\zdx'\sim\mu}\brk*{f^{2}(\act_{\zdx'})}\vee\veps^{2}}}
  \leq{} 1 + 2\En_{\zdx\sim\mu}\brk*{\sup_{f\in\cF}\int_{\veps}^{1}\frac{\brk*{f(\act_{\zdx})}_{\Delta}^{2}}{
  \eps^{3}}\indic\crl*{\En_{\zdx'\sim\mu}\brk*{f^{2}(\act_{\zdx'})}\leq\eps^{2}}d\eps}.
\end{align*}
Applying a similar trick, for any $\Delta\leq{}X\leq{}1$ we have
\[
  \brk{X}_{\Delta}^{2}
  =X^{2}-\Delta^{2}=\int_{\Delta^{2}}^{1}\indic\crl{X^{2}>t}dt
  = 2\int_{\Delta}^{1}\indic\crl{X>t}tdt.
\]
This leads to the upper bound
\begin{align*}
  &2\En_{\zdx\sim\mu}\brk*{\sup_{f\in\cF}\int_{\veps}^{1}\frac{\brk*{f(\act_{\zdx})}_{\Delta}^{2}}{
    \eps^{3}}\indic\crl*{\En_{\zdx'\sim\mu}\brk*{f^{2}(\act_{\zdx'})}\leq\eps^{2}}d\eps}\\
  &\leq{} 4\En_{\zdx\sim\mu}\brk*{\sup_{f\in\cF}\int_{\veps}^{1}\int_{\Delta}^{1}\frac{\delta}{
    \eps^{3}}\indic\crl*{\abs*{f(\act_{\zdx})}>\delta\wedge\En_{\zdx'\sim\mu}\brk*{f^{2}(\act_{\zdx'})}\leq\eps^{2}}d\delta{}d\eps}\\
    &\leq{} 4\En_{\zdx\sim\mu}\brk*{\int_{\veps}^{1}\int_{\Delta}^{1}\frac{\delta}{
      \eps^{3}}\indic\crl*{\exists{}f\in\cF :
      \abs*{f(\act_{\zdx})}>\delta\wedge\En_{\zdx'\sim\mu}\brk*{f^{2}(\act_{\zdx'})}\leq\eps^{2}}d\delta{}d\eps}\\
  & \leq{}4\int_{\veps}^{1}\int_{\Delta}^{1}\frac{\delta}{
      \eps^{3}}\bbP_{\zdx\sim\mu}\prn*{\exists{}f\in\cF : \abs*{f(\act_{\zdx})}>\delta\wedge\En_{\zdx'\sim\mu}\brk*{f^{2}(\act_{\zdx'})}\leq\eps^{2}}d\delta{}d\eps.\\
\end{align*}
Now, from the definition of the disagreement coefficient, we have that
for all $\eps\geq\veps$ and $\delta\geq{}\Delta$,
\[
\bbP_{\zdx\sim\mu}\prn*{\exists{}f\in\cF :
  \abs*{f(\act_{\zdx})}>\delta\wedge\En_{\zdx'\sim\mu}\brk*{f^{2}(\act_{\zdx'})}\leq\eps^{2}}
\leq{}\frac{\eps^{2}}{\delta^{2}}\sdis(\cF,\Delta,\veps;\rhomu).
\]
It follows that
\begin{align*}
  &\int_{\veps}^{1}\int_{\Delta}^{1}\frac{\delta}{
      \eps^{3}}\bbP_{\zdx\sim\mu}\prn*{\exists{}f\in\cF :
    \abs*{f(\act_{\zdx})}>\delta\wedge\En_{\zdx'\sim\mu}\brk*{f^{2}(\act_{\zdx'})}\leq\eps^{2}}d\delta{}d\eps\\
  &~~~~\leq \sdis(\cF,\Delta,\veps;\rhomu)\cdot\int_{\veps}^{1}\int_{\Delta}^{1}\frac{1}{\delta
    \eps}d\delta{}d\eps\\
  &~~~~= \sdis(\cF,\Delta,\veps;\rhomu)\cdot\log(1/\veps)\log(1/\Delta).
\end{align*}
Altogether, we have
\begin{align*}
  \En_{\zdx\sim\mu}\brk*{\abs*{f_{\zdx}(\act_{\zdx})}}
  \leq{} \Delta + \frac{\eta\veps^{2}}{2}
    +
  \frac{1}{2\eta}(4\sdis(\cF,\Delta,\veps;\rhomu)\cdot\log(1/\veps)\log(1/\Delta)+1)
  +\frac{\eta}{2}\En_{z,z'\sim\mu}\brk*{f_{\zdx}^{2}(\act_{\zdx'})}.
\end{align*}
We conclude by tuning the parameters in this bound to derive the
theorem statement. First note that for any $R\geq{}1$, we can apply the result
above to the class $\cF/R$ to get
\begin{align*}
  &\En_{\zdx\sim\mu}\brk*{\abs*{f_{\zdx}(\act_{\zdx})}} \\
  &\leq{} \inf_{\Delta,\veps\in(0,1],\eta>0}\crl*{\Delta{}R + \frac{\eta\veps^{2}R}{2}
    +
  \frac{R}{2\eta}(4\sdis(\cF/R,\Delta,\veps;\rhomu)\cdot\log(1/\veps)\log(1/\Delta)+1)
    +\frac{\eta}{2R}\En_{z,z'\sim\mu}\brk*{f_{\zdx}^{2}(\act_{\zdx'})}}\\
  &= \inf_{\Delta,\veps\in(0,R],\eta>0}\crl*{\Delta{} + \frac{\eta\veps^{2}}{2R}
    +
  \frac{R}{2\eta}(4\sdis(\cF/R,\Delta/R,\veps/R;\rhomu)\cdot\log(R/\veps)\log(R/\Delta)+1)
    +\frac{\eta}{2R}\En_{z,z'\sim\mu}\brk*{f_{\zdx}^{2}(\act_{\zdx'})}}\\
    &= \inf_{\Delta,\veps\in(0,R],\eta>0}\crl*{\Delta{} + \frac{\eta\veps^{2}}{2R}
    +
  \frac{R}{2\eta}(4\sdis(\cF,\Delta,\veps;\rhomu)\cdot\log(R/\veps)\log(R/\Delta)+1)
  +\frac{\eta}{2R}\En_{z,z'\sim\mu}\brk*{f_{\zdx}^{2}(\act_{\zdx'})}}.
\end{align*}
We choose $\eta=2\gamma{}R$ and
$\veps=\gamma^{-1}\wedge{}R$ to get
\begin{align*}
  \inf_{\Delta\in(0,R]}\crl*{\Delta{} + 2\gamma^{-1}
    +
  \frac{\sdis(\cF,\Delta,\gamma^{-1}\wedge{}R;\rhomu)\cdot\log(R\gamma\vee{}1)\log(R/\Delta)}{\gamma}
  +\gamma\En_{z,z'\sim\mu}\brk*{f_{\zdx}^{2}(\act_{\zdx'})}}.
\end{align*}
For $\gamma^{-1}>R$, we have
$\sdis(\cF,\Delta,\gamma^{-1};\rhomu)=\sdis(\cF,\Delta,R;\rhomu)=1$,
so we can simplify to $\sdis(\cF,\Delta,\gamma^{-1};\rhomu)$. We
proceed to use that $\sdis(\cdots)\geq{}1$ to further simplify to
\begin{align*}
&  \inf_{\Delta\in(0,R]}\crl*{\Delta{} +
  3\frac{\sdis(\cF,\Delta,\gamma^{-1};\rhomu)\cdot\log(R\gamma\vee{}e)\log(R/\Delta\vee{}e)}{\gamma}
                 +\gamma\En_{z,z'\sim\mu}\brk*{f_{\zdx}^{2}(\act_{\zdx'})}}\\
  &\leq{}  \inf_{\Delta\in(\gamma^{-1},R]}\crl*{\Delta{} +
  3\frac{\sdis(\cF,\Delta,\gamma^{-1};\rhomu)\cdot\log^{2}(R\gamma\vee{}e)}{\gamma}
    +\gamma\En_{z,z'\sim\mu}\brk*{f_{\zdx}^{2}(\act_{\zdx'})}}.
\end{align*}
Let
$f(\Delta)=3\frac{\sdis(\cF,\Delta,\gamma^{-1};\rhomu)\cdot\log^{2}(R\gamma\vee{}e)}{\gamma}$,
and let us upper bound by $2\inf_{\Delta\in(\gamma^{-1},R]}\max\crl*{\Delta,
  f(\Delta)}$. Since $f(\Delta)\leq{}1/\gamma$, we have that
$\max\crl*{\Delta,f(\Delta)}=f(\Delta)$ for
$\Delta\leq{}1/\gamma$. Since $f(\Delta)$ is decreasing, the
unconstrained minimizer
must have $\Delta>1/\gamma$. Similarly, since $f(\Delta)$ is constant
for $\Delta\geq{}R$ and $\Delta$ is increasing, the unconstrained
minimizer must have $\Delta\leq{}R$ as well. We conclude that the unconstrained minimizer
over $\Delta$ is contained in the range $[\gamma^{-1},R]$, so we can
relax the infimum to the range $(0,\infty]$, which gives the final theorem statement.
\end{proof}

\subsubsection{Proof of \pref{thm:eluder_star} and \pref{thm:disagreement}}

\begin{proof}[\pfref{thm:eluder_star}]
  By the minimax theorem (cf. \pref{prop:minimax_swap_dec}), it
  suffices to bound the dual \CompText
  $\compSqdual(\cM,\Mbar)=\sup_{\mu\in\Delta(\cM)}\compSqdual(\mu,\Mbar)$. The
  bound on this quantity is an immediate consequence of \pref{thm:disagreement} and \pref{lem:disagreement_to_ratio}.
\end{proof}

\begin{proof}[\pfref{thm:disagreement}]%
Let $\mu\in\Delta(\cM)$ and $\Mbar$ be given, and recall that we define
$\rho_{\mu}(\act)\ldef{}\mu(\crl{\pim=\act})$ as the posterior
sampling distribution. Observe that for all $\pi\in\Pi$,
    \begin{align*}
      \En_{M\sim\mu}\En_{\Mtil\sim\mu}\brk*{(\fm(\pi)-\fmtil(\pi))^2}
       &\leq{} 2\En_{M\sim\mu}\brk[\Big]{(\fm(\pi)-\fmbar(\pi))^2}
          + 2\En_{\Mtil\sim\mu}\brk[\Big]{(\fmbar(\pi)-\fmtil(\pi))^2}
          \\
        &= 4 \En_{M\sim\mu}\brk[\Big]{(\fm(\pi)-\fmbar(\pi))^2}
         .
    \end{align*}
Hence, it suffices bound the quantity
    \begin{align*}
      \En_{M\sim\mu}\En_{\pi\sim{}p}\brk*{\fm(\pim) - \fm(\act)} = \En_{\Mtil\sim\mu}\En_{M\sim\mu}\En_{\pi\sim{}p}\brk*{\fm(\pim) - \jqedit{\fmtil}(\act)}
    \end{align*}
    in terms of the squared error
    $\En_{\Mtil\sim\mu}\En_{M\sim\mu}\En_{\pi\sim{}p}\brk[\big]{(\fm(\pi)-\fmtil(\pi))^2}$.
    Consider any fixed choice of $\Mtil\in\cM$. Note that, since $\act\sim\rho_{\mu}$ and $\pim$ under
$M\sim\mu$ are identical in law,
\[
  \En_{\act\sim\rhomu}\En_{\Mtil\sim\mu}\brk*{\fm(\pim)-\fmtil(\act)}
  =   \En_{M\sim\mu}\brk*{\fm(\pim)-\fmtil(\pim)}.
\]
We now apply \pref{thm:decoupling_general} with
the class $\cFm-\fmtil$ and parameter $\gamma/4$, which gives
\begin{align*}
  \En_{M\sim\mu}\brk*{\fm(\pim)-\fmtil(\pim)}
  &\leq  \inf_{\Delta>0}\crl*{
      2\Delta +
      24\frac{\sdis(\cF_{\cM}-\fmtil,\Delta,\gamma^{-1};\rhomu)\log^{2}(\gamma\vee{}e)}{\gamma}}
    + \frac{\gamma}{4}\cdot
    \En_{\pi\sim{}p,M\sim\mu}\brk*{(\fm(\pi)-\fmtil(\pi))^2}.
\end{align*}
Since this bound holds uniformly for all choices of $\Mtil\in\cM$, we
are free to take the expectation over $\Mtil\sim\mu$, which yields the result.
\end{proof}

\section{Proofs and Additional Results from \preft{sec:rl}}
\label{app:rl}

\subsection{Technical Tools}
\label{app:rl_technical}

\newcommand{\Pone}{P\ind{1}}%
\newcommand{\Ptwo}{P\ind{2}}%

\begin{lemma}[Bellman residual decomposition]
  \label{lem:bellman_residual}
  For any pair of MDPs $M=(\Pm,\Rm)$ and $\Mbar=(\Pmbar,\Rmbar)$ with
  the same initial state distribution,
  \begin{equation}
    \label{eq:15}
    \fm(\pi)- \fmbar(\pi) =  \sum_{h=1}^{H}\Enm{\Mbar}{\pi}\brk*{\Qmpi_h(s_h,a_h) - r_h -
      \Vmpi_{h+1}(s_{h+1})}
\end{equation}
for all policies $\pi\in\PiGen$.
\end{lemma}
\begin{proof}[\pfref{lem:bellman_residual}]
First, we have
  \begin{align*}
    \sum_{h=1}^{H}\Enm{\Mbar}{\pi}\brk*{\Qmpi_h(s_h,a_h) - r_h -
    \Vmpi_{h+1}(s_{h+1})}
    &= \sum_{h=1}^{H}\Enm{\Mbar}{\pi}\brk*{\Qmpi_h(s_h,a_h)  -
    \Vmpi_{h+1}(s_{h+1})}
      - \Enm{\Mbar}{\pi}\brk*{\sum_{h=1}^{H}r_h}\\
        &= \sum_{h=1}^{H}\Enm{\Mbar}{\pi}\brk*{\Qmpi_h(s_h,a_h)  -
    \Vmpi_{h+1}(s_{h+1})}
          - \fmbar(\pi).
  \end{align*}
  On the other hand, since
  $\Vmpi_h(s)=\En_{a\sim\pi_h(s)}\brk{\Qmpi_h(s,a)}$, a telescoping
  argument yields
  \begin{align*}
    \sum_{h=1}^{H}\Enm{\Mbar}{\pi}\brk*{\Qmpi_h(s_h,a_h)  -
      \Vmpi_{h+1}(s_{h+1})}
    &=\sum_{h=1}^{H}\Enm{\Mbar}{\pi}\brk*{\Vmpi_h(s_h)  -
      \Vmpi_{h+1}(s_{h+1})}\\
        &=\Enm{\Mbar}{\pi}\brk*{\Vmpi_1(s_1)}  -
          \Enm{\Mbar}{\pi}\brk*{\Vmpi_{H+1}(s_{H+1})}\\
    &=\fm(\pi),
  \end{align*}
  where we have used that $\Vmpi_{H+1}=0$, and that both MDPs have the
  same initial state distribution.
\end{proof}

\begin{lemma}[Global simulation lemma]
  \label{lem:simulation}
  Let $M$ and $M'$ be MDPs with $\sum_{h=1}^{H}r_h\in\brk*{0,1}$ almost surely, and let $\act\in\PiGen$.
Then we have
\begin{align}
  \label{eq:simulation}
  \abs*{\fm(\pi)-\fmp(\pi)}
  &\leq{} \Dtv{M(\act)}{M'(\act)} \\
  &\leq \Dhel{M(\act)}{M'(\act)}
  \leq{} \frac{1}{2\eta} + \frac{\eta}{2}\Dhels{M(\act)}{M'(\act)}\quad\forall{}\eta>0.
\end{align}
\end{lemma}
\begin{proof}[\pfref{lem:simulation}]
  Let $X=\sum_{h=1}^{H}r_h$. Since $X\in\brk*{0,1}$ almost surely, we have
  \begin{align*}
    \abs*{\fm(\pi)-\fmp(\pi)}
    =   \abs*{\Enm{M}{\act}\brk{X}-\Enm{M'}{\act}\brk{X}}
    \leq{} \Dtv{M(\act)}{M'(\act)} \leq \Dhel{M(\act)}{M'(\act)},
  \end{align*}
  where the second inequality is from \pref{lem:pinsker}. The final result now follows from the AM-GM inequality.
\end{proof}

  \begin{lemma}[Local simulation lemma]
    \label{lem:simulation_basic}
      For any pair of MDPs $M=(\Pm,\Rm)$ and $\Mbar=(\Pmbar,\Rmbar)$ with
  the same initial state distribution and $\sum_{h=1}^{H}r_h\in\brk*{0,1}$,
    \begin{align}
      \label{eq:simulation_basic1}
          \fm(\pi)- \fmbar(\pi)
      &=         \sum_{h=1}^{H}\Enm{\Mbar}{\pi}\brk*{\brk*{(\Pm_h-\Pmbar_h)\Vmpi_{h+1}}(s_h,a_h)
        }
        + \sum_{h=1}^{H}\Enm{\Mbar}{\pi}\brk*{
        \En_{r_h\sim\Rm_h(s_h,a_h)}\brk{r_h}
              - \En_{r_h\sim\Rmbar_h(s_h,a_h)}\brk{r_h}
        }\\
      &\leq{}
              \sum_{h=1}^{H}\Enm{\Mbar}{\pi}\brk*{\Dtv{\Pm_h(s_h,a_h)}{\Pmbar_h(s_h,a_h)}
              + \Dtv{\Rm_h(s_h,a_h)}{\Rmbar_h(s_h,a_h)}
        }.
        \label{eq:simulation_basic2}
    \end{align}
  \end{lemma}
  \begin{proof}[\pfref{lem:simulation_basic}]
    From \pref{lem:bellman_residual}, we have
    \begin{align*}
    \fm(\pi)- \fmbar(\pi) &=  \sum_{h=1}^{H}\Enm{\Mbar}{\pi}\brk*{\Qmpi_h(s_h,a_h) - r_h -
                            \Vmpi_{h+1}(s_{h+1})} \\
            &=
              \sum_{h=1}^{H}\Enm{\Mbar}{\pi}\brk*{\brk*{\Pm_h\Vmpi_{h+1}}(s_h,a_h)-
              \Vmpi_{h+1}(s_{h+1})
              +\En_{r_h\sim\Rm_h(s_h,a_h)}\brk{r_h}
              - \En_{r_h\sim\Rmbar_h(s_h,a_h)}\brk{r_h}
              }
      \\
      &=
        \sum_{h=1}^{H}\Enm{\Mbar}{\pi}\brk*{\brk*{(\Pm_h-\Pmbar_h)\Vmpi_{h+1}}(s_h,a_h)
        }
        + \sum_{h=1}^{H}\Enm{\Mbar}{\pi}\brk*{
        \En_{r_h\sim\Rm_h(s_h,a_h)}\brk{r_h}
              - \En_{r_h\sim\Rmbar_h(s_h,a_h)}\brk{r_h}
        }\\
            &\leq{}
              \sum_{h=1}^{H}\Enm{\Mbar}{\pi}\brk*{\Dtv{\Pm_h(s_h,a_h)}{\Pmbar_h(s_h,a_h)}
              + \Dtv{\Rm_h(s_h,a_h)}{\Rmbar_h(s_h,a_h)}
        },
    \end{align*}
where we have used that $\Vmpi_{h+1}(s)\in\brk*{0,1}$.

  \end{proof}

\changeofmeasure*

\begin{proof}[\pfref{lem:change_of_measure}]
  We first relate the left-hand side of \pref{eq:com1} to that of
  \pref{eq:com2} using \pref{lem:simulation}, which implies that for
  all $\eta>0$, 
\begin{align*}
  &\En_{M\sim\mu}\En_{\pi\sim{}p}\brk*{\fm(\pim) - \fm(\pi)}\\
  &\leq{}
    \En_{M\sim\mu}\En_{\pi\sim{}p}\brk*{\fm(\pim) - \fmbar(\pi)}
    + \eta\En_{M\sim\mu}\En_{\pi\sim{}p}\brk*{
    \Dhels{M(\act)}{\Mbar(\act)}
    } + \frac{1}{4\eta}.
  \end{align*}
We now relate the right-hand sides of \pref{eq:com1} and \pref{eq:com2}. By \pref{lem:mp_min}, since
$\Dhels{\cdot}{\cdot}\leq{}2$, we have that for any fixed draw of $M$
  and $\pi$,
  \begin{align*}
    &\Enm{\Mbar}{\pi}\brk*{\sum_{h=1}^{H}          \Dhels{\Pm(s_h,a_h)}{\Pmbar(s_h,a_h)}
    + \Dhels{\Rm(s_h,a_h)}{\Rmbar(s_h,a_h)}
    }\\
        &\leq{}
    3\Enm{M}{\pi}\brk*{\sum_{h=1}^{H}          \Dhels{\Pm(s_h,a_h)}{\Pmbar(s_h,a_h)}
    + \Dhels{\Rm(s_h,a_h)}{\Rmbar(s_h,a_h)}
    }
    + 16H\Dhels{M(\act)}{\Mbar(\act)},
  \end{align*}
Next, we recall from \pref{lem:hellinger_pair}
that for all $h$,
\[
  \Enm{M}{\act}\brk*{\Dhels{\Pm(s_h,a_h)}{\Pmbar(s_h,a_h)}}
  +   \Enm{M}{\act}\brk*{\Dhels{\Rm(s_h,a_h)}{\Rmbar(s_h,a_h)}}
  \leq{} 8\Dhels{M(\act)}{\Mbar(\act)}.
  \]
As a result, we have
  \begin{align*}
    \Enm{\Mbar}{\pi}\brk*{\sum_{h=1}^{H}          \Dhels{\Pm(s_h,a_h)}{\Pmbar(s_h,a_h)}
    + \Dhels{\Rm(s_h,a_h)}{\Rmbar(s_h,a_h)}
    }\
    &\leq 8H\Dhels{M(\act)}{\Mbar(\act)}.
  \end{align*}
  Since this holds uniformly for all $M$ and $\pi$, we conclude that
  \begin{align*}
    &\En_{M\sim\mu}\En_{\pi\sim{}p}\Enm{\Mbar}{\pi}\brk*{\sum_{h=1}^{H}    \Dtvs{\Pm(s_h,a_h)}{\Pmbar(s_h,a_h)}
    + \Dtvs{\Rm(s_h,a_h)}{\Rmbar(s_h,a_h)}}\\
    &\leq 40H\En_{M\sim\mu}\En_{\pi\sim{}p}\brk*{\Dhels{M(\act)}{\Mbar(\act)}}.
  \end{align*}
  
\end{proof}

\subsection{Proofs from \preft{sec:rl_bilinear_basic}}

In this section of the appendix we prove \pref{thm:posterior_bilinear} and \pref{thm:igw_bilinear}. Before proving the results, we state and prove a helper lemma.

\begin{lemma}
  \label{lem:bilinear_hellinger}
  Let $\cM$ be a bilinear class, and let $\pialpham$ be defined as in \pref{sec:rl_bilinear_dec}. Then for all $h$, for any $M,M',\Mbar\in\cM$,
  \begin{equation}
    \label{eq:bilinear_tv}
    \tri*{X_h(M;\Mbar), W_h(M';\Mbar)}^2
    \leq{} 4\Calpha\Lbifulls\cdot{}\Dhels{M'(\pialpham)}{\Mbar(\pialpham)},
  \end{equation}
  where $\Calpha=1$ when $\piestm=\pim$ and
  $\Calpha\leq\frac{2H}{\alpha}$ in general, as long as $\alpha\in[0,1/2]$.
\end{lemma}
\begin{proof}[\pfref{lem:bilinear_hellinger}]
  Since $W(M';M')=0$, we have
  \begin{align*}
    \abs*{\tri*{X_h(M;\Mbar), W_h(M';\Mbar)}}
    & =  \abs*{\tri*{X_h(M;\Mbar), W_h(M';\Mbar)}-\tri*{X_h(M;M'), W_h(M';M')}}\\
    & =  \abs*{\Enm{\Mbar}{\pim^{\vphantom{\mathrm{est}}}\circ_{h}\piestm}\brk*{\lestm(M';z_h)}
      -\Enm{M'}{\pim^{\vphantom{\mathrm{est}}}\circ_{h}\piestm}\brk*{\lestm(M';z_h)}
      }\\
    & =  \abs*{\Enm{\Mbar}{\pim^{\vphantom{\mathrm{est}}}\circ_{h}\piestm\circ_{h+1}\pim}\brk*{\lestm(M';z_h)}
      -\Enm{M'}{\pim^{\vphantom{\mathrm{est}}}\circ_{h}\piestm\circ_{h+1}\pim}\brk*{\lestm(M';z_h)}
      }.
  \end{align*}
  Hence, since $z_h$ is a measurable function of the learner's
  trajectory $\tau_H$, and since $\abs{\lestm(M';\cdot)}\leq{}\Lbifull$ almost surely under both
  models, we have
  \begin{align*}
    \abs*{\tri*{X_h(M;\Mbar), W_h(M';\Mbar)}}
    &\leq{}
      2\Lbifull\Dtv{M'(\pim\circ_h\piestm\circ_{h+1}\pim)}{\Mbar(\pim\circ_h\piestm\circ_{h+1}\pim)}\\
    &\leq{} 2\Lbifull\Dhel{M'(\pim\circ_h\piestm\circ_{h+1}\pim)}{\Mbar(\pim\circ_h\piestm\circ_{h+1}\pim)}.    
  \end{align*}
If $\piestm=\pim$ we are done. Otherwise, let $b_1,\ldots,b_H\in\crl*{0,1}$, and let $\pi\subs{M,h}^{b}(s)
\ldef{}  b\cdot{}\pi\subs{M,h}(s)+(1-b)\cdot\piest\subs{M,h}(s)$. For any MDP, trajectories induced by $\pialpham$ are equivalent in law
to those generated by sampling
$b_{1},\ldots,b_H\sim\Ber((1-\alpha/H))$ independently and playing
$\pim^{b}$. As a result, since Hellinger distance has
$\Dhels{\bbP_{Y\mid{}X}\tens{}\bbP_X}{\bbQ_{Y\mid{}X}\tens{}\bbP_X}=\En_{X\sim{}\bbP_X}\brk*{\Dhels{\bbP_{Y\mid{}X}}{\bbQ_{Y\mid{}X}}}$,
we have
\begin{align*}
  \Dhels{M'(\pialpham)}{\Mbar(\pialpham)}
  &= \En_{b_1,\ldots,b_H}  \brk*{\Dhels{M'(\pim^{b})}{\Mbar(\pim^{b})}}
  \\
  &\geq{}
    (\alpha/H)(1-\alpha/H)^{H-1}\cdot{}\Dhels{M'(\pim\circ_h\piestm\circ_{h+1}\pim)}{\Mbar(\pim\circ_h\piestm\circ_{h+1}\pim)}.
\end{align*}
Finally, we note that $(1-\alpha/H)^{H-1}\geq{}1-\alpha\geq{}1/2$
whenever $\alpha\leq{}1/2$.

\end{proof}

\posteriorbilinear*

\begin{proof}[\pfref{thm:posterior_bilinear}]
Throughout this proof we abbreviate $d\equiv\jqedit{\dimbi(\cM)}$. Let
a prior $\mu\in\Delta(\cM)$ \jqedit{and $\Mbar$} be fixed. We begin
with a technical lemma.
\begin{lemma}
  \label{lem:randomize}
  For any distribution $\mu\in\Delta(\cM)$ and any reference model
  $\Mbar$ (not necessarily in $\cM$), it holds that for all
  $\pi$,
  \begin{align*}
    \En_{M\sim\mu}\brk[\Big]{\Dhels{M(\pi)}{\Mbar(\pi)}}
    \geq \frac{1}{4}\En_{M\sim\mu}\En_{\Mtil\sim\mu}\brk[\Big]{\Dhels{M(\pi)}{\Mtil(\pi)}
      }
  \end{align*}
\end{lemma}
\begin{proof}[\pfref{lem:randomize}]
  The result follow by noting that
    for all $\pi$,
    \begin{align*}
      \En_{M\sim\mu}\En_{\Mtil\sim\mu}\brk[\Big]{\Dhels{M(\pi)}{\Mtil(\pi)}
      } &\leq{} 2\En_{M\sim\mu}\brk[\Big]{\Dhels{M(\pi)}{\Mbar(\pi)}
          } + 2\En_{\Mtil\sim\mu}\brk[\Big]{\Dhels{\Mbar(\pi)}{\Mtil(\pi)}
          } \\
        &= 4 \En_{M\sim\mu}\brk[\Big]{\Dhels{M(\pi)}{\Mbar(\pi)}
          }.
    \end{align*}
\end{proof}

\dfedit{By \cref{lem:randomize}, it suffices bound the quantity
    \begin{align*}
      \En_{M\sim\mu}\En_{\pi\sim{}p}\brk*{\fm(\pim) - \fm(\act)} = \En_{\Mtil\sim\mu}\En_{M\sim\mu}\En_{\pi\sim{}p}\brk*{\fm(\pim) - \jqedit{\fmtil}(\act)}
    \end{align*}
    in terms of the Hellinger error
    $\En_{\Mtil\sim\mu}\En_{M\sim\mu}\En_{\pi\sim{}p}\brk[\big]{\DhelsX{\big}{M(\pi)}{\jqedit{\Mtil}(\pi)}}$.}
To begin \jqedit{for any $\Mtil\in\cM$}, we write
  \[
    \En_{M\sim\mu}\En_{\pi\sim{}p}\brk*{\fm(\pim) - \jqedit{\fmtil}(\pi)}
    = \En_{M\sim\mu}\brk*{\fm(\pim)} -\En_{\pi\sim{}p}\brk*{\jqedit{\fmtil}(\pi)}.
  \]
  From the definition of $p$, this is equal to 
  \[
\En_{M\sim\mu}\brk*{\fm(\pim)}
-\En_{M\sim\mu}\brk*{ \jqedit{\fmtil}(\pialpham)}
\leq \En_{M\sim\mu}\brk*{\fm(\pim) - \jqedit{\fmtil}(\pim)} + \alpha,
\]
where we have used that
\[
  \Prm{\jqedit{\Mtil}}{\pialpham}\prn*{\forall{}h : \pi^{\alpha}\subs{M,h}(s_h)=\pi\subs{M,h}(s_h)}\geq{}(1-\alpha/H)^{H}\geq{}1-\alpha,
  \]
  and that $\jqedit{\fmtil}\in\brk{0,1}$.

To proceed, using the Bellman residual decomposition
(\pref{lem:bellman_residual}) and applying the first bilinear class
property, we have
\begin{align*}
  \En_{M\sim\mu}\brk*{\fm(\pim) - \jqedit{\fmtil}(\pim)}
  &=  \En_{M\sim\mu}\brk*{\sum_{h=1}^{H}\Enm{\jqedit{\Mtil}}{\pim}\brk*{\Qmstar_h(s_h,a_h) - r_h -
  \Vmstar_{h+1}(s_{h+1})}}\notag\\
  &  \leq  \sum_{h=1}^{H}
    \En_{M\sim\mu}\brk*{\abs*{\tri*{X_h(M;\jqedit{\Mtil}), W_h(M;\jqedit{\Mtil})}}
    }.
\end{align*}
Let $h\in\brk*{H}$ be fixed and abbreviate $X_h(M)\equiv{}X_h(M;\jqedit{\Mtil})$ and
$W_h(M)\equiv{}W_h(M;\jqedit{\Mtil})$. Define 
$\Sigma\ldef{}\En_{M\sim\mu}\brk*{X_h(M)
  X_h(M)^{\trn}}$. Then by Cauchy-Schwarz, 
\begin{align}
   \En_{M\sim\mu}\brk*{\abs*{\tri{W_h(M),X_h(M)}}
    }
  &=  \En_{M\sim\mu}\brk*{\abs*{\tri{\Sigma^{1/2}W_h(M), (\Sigma^{\dagger})^{1/2}X_h(M)}}
    }\notag\\
  &\leq\sqrt{\En_{M\sim\mu}\nrm[\big]{\Sigma^{1/2}W_h(M)}_2^{2}}\cdot\sqrt{\En_{M\sim\mu}\nrm[\big]{(\Sigma^{\dagger})^{1/2}X_h(M)}_2^{2}},
    \label{eq:posterior_bilinear1}
\end{align}
where the first inequality uses that $X_h(M)\in\mathrm{span}(\Sigma)$
almost surely. For the first term in \pref{eq:posterior_bilinear1}, we have
\begin{align*}
  \En_{M\sim\mu}\nrm[\big]{\Sigma^{1/2}W_h(M)}_2^{2}
  &= \En_{M\sim\mu}\brk*{\tri*{\Sigma{}W_h(M), W_h(M)}
    }\\
    &= \En_{M'\sim\mu}\brk*{\tri*{\En_{M\sim\mu}\brk*{X_h(M)X_h(M)^{\trn}}W_h(M'),W_h(M')}
    }\\
  &\leq{} \En_{M,M'\sim\mu}\brk[\Big]{\tri*{X_h(M),W_h(M')
  }^{2}
  }.
\end{align*}
We bound the second term in \pref{eq:posterior_bilinear1} by writing
\[
  \En_{M\sim\mu}\nrm[\big]{(\Sigma^{\dagger})^{1/2}X_h(M)}_2^{2}
  = \tri*{\Sigma^{\dagger},
    \En_{M\sim\mu}\brk*{X_h(M)X_h(M)^{\trn}}}\\
  = \tri{\Sigma^{\dagger},\Sigma}\leq{}d.
\]
By the AM-GM inequality, we conclude that for any $\eta>0$,
\begin{align}
    \En_{M\sim\mu}\brk*{\fm(\pim) - \jqedit{\fmtil}(\pim)}
  &\leq{} \frac{dH}{2\eta}
  + \frac{\eta{}}{2}\sum_{h=1}^{H}\En_{M,M'\sim\mu}\brk[\Big]{\tri*{X_h(M;\jqedit{\Mtil}),W_h(M';\jqedit{\Mtil})
  }^{2}}.\label{eq:posterior_bilinear_intermediate}
\end{align}
Using \pref{lem:bilinear_hellinger}, this implies that whenever
$\alpha\leq{}1/2$,
\begin{align*}
    \En_{M\sim\mu}\brk*{\fm(\pim) - \jqedit{\fmtil}(\pim)}
  &\leq{} \frac{dH}{2\eta}
    + 2\eta{}H\Calpha\Lbifulls{}\En_{M,M'\sim\mu}\brk[\Big]{\Dhels{M'(\pialpham)}{\jqedit{\Mtil}(\pialpham)}
    }\\
    &= \frac{dH}{2\eta}
    + 2\eta{}H\Calpha\Lbifulls{}\En_{M\sim\mu}\En_{\pi\sim{}p}\brk[\Big]{\Dhels{M(\pi)}{\jqedit{\Mtil}(\pi)}
    },
\end{align*}
where $\Calpha=1$ when $\piestm=\pim$ and $\Calpha=\frac{2H}{\alpha}$
otherwise.

\jqedit{Since $\Mtil$ can be chosen arbitrarily in $\cM$, we take expectation over $\Mtil\sim{}\mu$ to obtain
\begin{align*}
 \En_{M\sim\mu}\brk*{\fm(\pim) - \En_{\Mtil\sim\mu}\fmtil(\pim)} &\leq{} \frac{dH}{2\eta}
  + 2\eta{}H\Calpha\Lbifulls{}\En_{M\sim\mu}\En_{\pi\sim{}p}\En_{\Mtil\sim\mu}\brk[\Big]{\Dhels{M(\pi)}{\Mtil(\pi)}
  }.
\end{align*}
Altogether, we have shown that
\[
  \En_{M\sim\mu}\En_{\pi\sim{}p}\brk*{\fm(\pim) - \fm(\act)}
  \leq{} \frac{dH}{2\eta} 
    + 8\eta{}H\Calpha\Lbifulls{}\En_{M\sim\mu}\En_{\pi\sim{}p}\brk[\Big]{\Dhels{M(\pi)}{\Mbar(\pi)}
    } + \alpha.
  \]}

  We set \jqdelete{$\eta'$} $\eta=\frac{\gamma}{\jqedit{8}H\Calpha\Lbifulls{}}$ to conclude that
    $\compdual(\cM,\Mbar)\leq{}
    \frac{\jqedit{4}H^{2}\Calpha\Lbifulls{}\jqedit{\dimbi(\cM)}}{\gamma}+\alpha$. In
    the on-policy case we take $\alpha=0$, and in the general case,
    we set
    $\alpha=\sqrt{\frac{\jqedit{4}H^{3}\Lbifulls{}\jqedit{\dimbi(\cM)}}{\gamma}}$,
    which is admissible whenever $\gamma\geq{} \jqedit{16}H^{3}\Lbifulls{}\jqedit{\dimbi(\cM)}$.
\end{proof}

\igwbilinear*

\begin{proof}[\pfref{thm:igw_bilinear}]
    We first verify that the strategy in \pref{eq:pc_igw2} is indeed
  well-defined, in the sense that a normalizing constant
  $\lambda\in[1, 2HSA]$ always exists.
  \begin{proposition}
    There is a unique choice for $\lambda>0$ such that $\sum_{M\in\cM}p(\pialpham)=1$, and its value lies in $[1/2,1]$.\looseness=-1
  \end{proposition}
  \begin{proof}
    Let $f(\lambda)=\sum_{M\in\cM}\frac{q(M)}{\lambda +
      \eta(\fmbar(\pimbar)-\fmbar(\pim))}$. We
    observe that if $\lambda>1$, then
    $f(\lambda)\leq{}\sum_{M\in\cM}\frac{q(M)}{\lambda}=\frac{1}{\lambda}<1.$
    On the other hand for $\lambda\in(0,1/2)$,
    $f(\lambda)\geq{}\frac{q(\Mbar)}{\lambda +
      \eta(\fmbar(\pimbar)-\fmbar(\pimbar))}
    = \frac{1}{2\lambda}>1$.
    Hence, since $f(\lambda)$ is continuous and strictly decreasing over
  $(0,\infty)$, there exists unique $\lambda^{\star}\in[1/2,
  1]$ such that $f(\lambda^{\star})=1$.
  \end{proof}

  We now show that \pcigwb achieves the stated bound on the
  \CompShort. Let $M\in\cM$ and $\gamma>0$ be given. We
  focus on bounding the quantity
  \begin{align*}
    \En_{\pi\sim{}p}\brk*{\fm(\pim) - \fmbar(\act)}.
  \end{align*}
Throughout the proof we will overload notation and write $M\sim{}p$
and $\pialpham\sim{}p$ interchangeably.

As in the tabular setting (\pref{prop:igw_tabular}), we begin by writing
  \begin{align}
    \En_{\pi\sim{}p}\brk*{\fm(\pim) - \fmbar(\pi)}
    =     \underbrace{\En_{\pi\sim{}p}\brk*{\fmbar(\pimbar) -
    \fmbar(\pi)}}_{(\mathrm{I})}
    + \underbrace{\fm(\pim) - \fmbar(\pimbar)}_{(\mathrm{II})}.
    \label{eq:igw_bilinear0}
  \end{align}
  For the first term $(\mathrm{I})$, we have
  \begin{align}
    \En_{\pi\sim{}p}\brk*{\fmbar(\pimbar) - \fmbar(\pi)}
    &= \sum_{M\in\cM}p(\pialpham)\prn*{\fmbar(\pimbar) -
      \fmbar(\pialpham)}\notag\\
    &\leq{} \sum_{M\in\cM}p(\pialpham)\prn*{\fmbar(\pimbar) -
      \fmbar(\pim)}   + \alpha \notag\\
        &=\sum_{M\in\cM}q(M)\frac{\fmbar(\pimbar) -
          \fmbar(\pim)}{\lambda + \eta\prn*{\fmbar(\pimbar) -
          \fmbar(\pim)}}   + \alpha \notag\\
    &\leq{} (2\eta)^{-1}   + \alpha,
    \label{eq:igw_bilinear1}
  \end{align}
  where we have used that $\lambda\geq{}0$ and that
  $\Prm{\Mbar}{\pialpham}\prn*{\forall{}h :
    \pi^{\alpha}\subs{M,h}(s_h)=\pi\subs{M,h}(s_h)}\geq{}(1-\alpha/H)^{H}\geq{}1-\alpha$.

  We now bound the second
  term $(\mathrm{II})$. Using the Bellman residual decomposition
(\pref{lem:bellman_residual}) and applying the first bilinear class
property, we have
\begin{align}
  \fm(\pim) - \fmbar(\pimbar) &= \fm(\pim) - \fmbar(\pim)  -
                                (\fmbar(\pimbar) - \fmbar(\pim)) \notag\\ 
  &=  \sum_{h=1}^{H}\Enm{\Mbar}{\pim}\brk*{\Qmstar_h(s_h,a_h) - r_h -
  \Vmstar_{h+1}(s_{h+1})} - (\fmbar(\pimbar) - \fmbar(\pim))\notag\\
  &  \leq  \sum_{h=1}^{H}
    \abs*{\tri*{X_h(M;\Mbar), W_h(M;\Mbar)}
    } - (\fmbar(\pimbar) - \fmbar(\pim)). \label{eq:igw_bilinear2}
\end{align}
Let $h\in\brk*{H}$ be fixed and abbreviate $X_h(M)\equiv{}X_h(M;\Mbar)$ and
$W_h(M)\equiv{}W_h(M;\Mbar)$. Define 
$\Sigma\ldef{}\En_{M\sim{}p}\brk*{X_h(M)
  X_h(M)^{\trn}}$. Then by Cauchy-Schwarz and AM-GM, we have that for
any $\eta'>0$,
\begin{align}
  \abs*{\tri{W_h(M),X_h(M)}}
  &=  \abs*{\tri{\Sigma^{1/2}W_h(M), (\Sigma^{\dagger})^{1/2}X_h(M)}
    }\notag\\
  &\leq\frac{\eta'}{2}\nrm[\big]{\Sigma^{1/2}W_h(M)}_2^{2} + \frac{1}{2\eta'}\nrm[\big]{(\Sigma^{\dagger})^{1/2}X_h(M)}_2^{2},
    \label{eq:igw_bilinear3}
\end{align}
where the first inequality uses that $X_h(M)\in\mathrm{span}(\Sigma)$
almost surely. For the first term in \pref{eq:igw_bilinear3}, we have
\begin{align*}
  \nrm[\big]{\Sigma^{1/2}W_h(M)}_2^{2}
  &= \tri*{\Sigma{}W_h(M), W_h(M)}
    \\
    &= \tri*{\En_{M'\sim{}p}\brk*{X_h(M')X_h(M')^{\trn}}W_h(M),W_h(M)}
      \\
  &\leq{} \En_{M'\sim{}p}\brk[\big]{\tri*{X_h(M'),W_h(M)}^{2}}.
\end{align*}
For the second term in \pref{eq:posterior_bilinear1}, we have
  \begin{align*}
      \nrm[\big]{(\Sigma^{\dagger})^{1/2}X_h(M)}^{2}
      = X_h(M)^{\trn}\prn*{\En_{M'\sim{}p}\brk[\big]{X_h(M')X_h(M')^{\trn}}}^{\dagger}X_h(M).
  \end{align*}
We observe that
\begin{align*}
    \En_{M'\sim{}p}\brk[\big]{X_h(M')X_h(M')^{\trn}}
  &\geq{}\frac{1}{2H}\sum_{M'\in\cM}\frac{\qopt_h(M')}{\lambda+\eta(\fmbar(\pimbar)-\fmbar(\pim))}X_h(M')X_h(M')^{\trn}\\
    &\geq{}\frac{1}{2H}\sum_{M'\in\cM}\frac{\qopt_h(M')}{1+\eta(\fmbar(\pimbar)-\fmbar(\pim))}X_h(M')X_h(M')^{\trn}\\
  &=\frac{1}{2H}\sum_{M'\in\cM}\qopt_h(M')Y_h(M')Y_h(M')^{\trn},
\end{align*}
where $Y_h(M')\equiv{}Y_h(M';\Mbar)$. Hence, by the G-optimal design property for $\qopt_h$,
\begin{align*}
  &X_h(M)^{\trn}\prn*{\En_{M'\sim{}p}\brk[\big]{X_h(M')X_h(M')^{\trn}}}^{\dagger}X_h(M)\\
  &=
    Y_h(M)^{\trn}\prn*{\En_{M'\sim{}p}\brk[\big]{X_h(M')X_h(M')^{\trn}}}^{\dagger}Y_h(M)\cdot{}\prn*{1+\eta(\fmbar(\pimbar)-\fmbar(\pim))}\\
    &\leq{}
2H\cdot{}Y_h(M)^{\trn}\prn*{\En_{M'\sim{}\qopt_h}\brk[\big]{Y_h(M')Y_h(M')^{\trn}}}^{\dagger}Y_h(M)\cdot{}\prn*{1+\eta(\fmbar(\pimbar)-\fmbar(\pim))}\\
  &\leq{}  2dH\Copt\cdot{}\prn*{1+\eta(\fmbar(\pimbar)-\fmbar(\pim))}.
\end{align*}
Returning to \pref{eq:igw_bilinear2} and summing over all layers, we
conclude that
\begin{align*}
  &\fm(\pim) - \fmbar(\pimbar) \\ &\leq
                                  \frac{\eta'}{2}\sum_{h=1}^{H}\En_{M'\sim{}p}\brk[\big]{\tri*{X_h(M'),W_h(M)}^{2}}
                                  + \frac{dH^2\Copt}{\eta'}\cdot{}\prn*{1+\eta(\fmbar(\pimbar)-\fmbar(\pim))}
                                   - (\fmbar(\pimbar) - \fmbar(\pim)).
\end{align*}
Choosing $\eta'=\eta{}dH^{2}\Copt$ yields
\[
  \fm(\pim) - \fmbar(\pimbar) \leq
  \frac{\eta{}dH^{2}\Copt}{2}\sum_{h=1}^{H}\En_{M'\sim{}p}\brk[\big]{\tri*{X_h(M';\Mbar),W_h(M;\Mbar)}^{2}}
                                  + \eta^{-1}.
                                \]
                                By \pref{lem:bilinear_hellinger}, this implies that whenever
$\alpha\leq{}1/2$,
\begin{align*}
    \fm(\pim) - \fmbar(\pimbar) &=
  2\eta{}dH^{3}\Copt\Calpha\Lbifulls\cdot{}\En_{M'\sim{}p}\brk[\big]{\Dhels{M(\pialpham[M'])}{\Mbar(\pialpham[M'])}}
  + \eta^{-1}\\
  &=
  2\eta{}dH^{3}\Copt\Calpha\Lbifulls\cdot{}\En_{\pi\sim{}p}\brk[\big]{\Dhels{M(\pi)}{\Mbar(\pi)}}
                                  + \eta^{-1},
\end{align*}
where $\Calpha=1$ when $\piestm=\pim$ and $\Calpha=\frac{2H}{\alpha}$
otherwise.

Altogether, we have shown that
\begin{align*}
  \En_{\pi\sim{}p}\brk*{\fm(\pim) - \fmbar(\act)}
  \leq{}   2\eta{}dH^{3}\Copt\Calpha\Lbifulls\cdot{}\En_{\pi\sim{}p}\brk[\big]{\Dhels{M(\pi)}{\Mbar(\pi)}}
                                  + 2\eta^{-1} + \alpha
\end{align*}
which, by \pref{lem:simulation}, implies that for any $\eta'>0$,
\[
  \En_{\pi\sim{}p}\brk*{\fm(\pim) - \fm(\act)}
  \leq{}   (2\eta{}dH^{3}\Copt\Calpha\Lbifulls+\eta'/2)\cdot{}\En_{\pi\sim{}p}\brk[\big]{\Dhels{M(\pi)}{\Mbar(\pi)}}
                                  + 2\eta^{-1} + (2\eta')^{-1} + \alpha.
  \]
  Setting $\eta=\eta'=\frac{\gamma}{3dH^3\Copt\Calpha\Lbifulls{}}$ gives
    $\comp(\cM,\Mbar)\leq{}
    \frac{9H^{3}\Calpha\Copt\Lbifulls{}\dimbi(\cM,\Mbar)}{\gamma}+\alpha$. In
    the on-policy case we take $\alpha=0$, and in the general case,
    we set
    $\alpha=\sqrt{\frac{18H^{4}\Copt\Lbifulls{}\dimbi(\cM,\Mbar)}{\gamma}}$,
    which is admissible whenever $\gamma\geq{} 72H^{4}\Copt\Lbifulls{}\dimbi(\cM,\Mbar)$.
  
\end{proof}

\subsection{Proofs from \preft{sec:rl_bilinear_refined}}

\decbilinear*

  \begin{proof}[\pfref{prop:dec_bilinear}]
    This proof is a slight modification to that of \pref{thm:posterior_bilinear}. Resuming from
    \pref{eq:posterior_bilinear_intermediate} in the proof of
    \pref{thm:posterior_bilinear}, we have that for any prior
    $\mu\in\Delta(\cM)$, the posterior sampling strategy ($p(\pi)=\mu(\crl{M:\pim=\pi})$) guarantees
    that for all $\Mbar\in\cM$, and $\gamma>0$,
    \[
        \En_{M\sim\mu}\En_{\pi\sim{}p}\brk*{\fm(\pim) - \fmbar(\pi)}
        \leq{} \frac{H\cdot{}\dimbifull{}}{2\gamma}
        + \frac{\gamma{}}{2}\En_{M\sim\mu}\En_{\pi\sim{}p}\brk*{
          \Dbi{M(\pi)}{\Mbar(\pi)}
        }.
      \]
      Next, using \pref{lem:bellman_residual}, we have that for all $\gamma>0$,
      \begin{align*}
        \En_{M\sim\mu}\En_{\pi\sim{}p}\abs*{\fm(\pi) - \fmbar(\pi)}
        &\leq{} \sum_{h=1}^{H}\En_{M\sim\mu}\En_{\pi\sim{}p}\abs{
        \tri*{X_h(\pi;\Mbar),W_h(M;\Mbar)}
          } \\
        &\leq{} \frac{H}{2\gamma} + \frac{\gamma}{2}\sum_{h=1}^{H}\En_{M\sim\mu}\En_{\pi\sim{}p}
          \tri*{X_h(\pi;\Mbar),W_h(M;\Mbar)}^2\\
        &= \frac{H}{2\gamma} + \frac{\gamma}{2}\En_{M\sim\mu}\En_{\pi\sim{}p}\brk*{
          \Dbi{M(\pi)}{\Mbar(\pi)}
        }.
      \end{align*}
      Combining these results, we have that
      \begin{align*}
        \En_{M\sim\mu}\En_{\pi\sim{}p}\brk*{\fm(\pim) - \fm(\pi)}
        \leq{} \frac{H\cdot{}\dimbifull{}}{\gamma}
        + \gamma\En_{M\sim\mu}\En_{\pi\sim{}p}\brk*{
          \Dbi{M(\pi)}{\Mbar(\pi)}
        }.
      \end{align*}
      Since this result holds uniformly for all $\Mbar\in\cM$ and $p$ itself not depend on $\Mbar$, this further implies that for all
      $\nu\in\Delta(\cM)$,
\[
  \En_{M\sim\mu}\En_{\pi\sim{}p}\brk*{\fm(\pim) - \fm(\pi)}
        \leq{} \frac{H\cdot{}\dimbifull{}}{\gamma}
        + \gamma\En_{\Mbar\sim\nu}\En_{M\sim\mu}\En_{\pi\sim{}p}\brk*{
          \Dbi{M(\pi)}{\Mbar(\pi)}
        }.
      \]
\end{proof}

\bilinearrefined*

\begin{proof}[\pfref{thm:bilinear_refined}]
  We begin as in the main upper bound for \mainalgB (\pref{thm:upper_main_bayes}). Consider a fixed prior distribution
  $\mu\in\Delta(\cM)$. Let $\En\brk*{\cdot}$ denote expectation
with respect to the joint law over $(\Mstar,\hist\ind{T})$ when
$\Mstar\sim\mu$. We have
\[
\En\brk*{\RegDM} =
\En\brk*{\sum_{t=1}^{T}\En_{t-1}\brk*{\fmstar(\pimstar) - \fmstar(\pi\ind{t})}},
\]
which, by conditional independence, implies that
\[
  \sum_{t=1}^{T}\En_{t-1}\brk*{\fmstar(\pimstar) - \fmstar(\pi\ind{t})}
  =
  \sum_{t=1}^{T}\En_{\Mstar\sim\mu\ind{t}}\En_{\pi\sim{}p\ind{t}}\brk*{
    \fmstar(\pimstar) - \fmstar(\pi\ind{t})},
\]
where $\mu\ind{t}$ denotes the posterior distribution over $\cM$ given
$\hist\ind{t-1}$. In particular, when $p\ind{t}$
solves \pref{eq:comp_argmin_bilinear}, we have
\begin{align*}
  \sum_{t=1}^{T}\En_{\Mstar\sim\mu\ind{t}}\En_{\pi\sim{}p\ind{t}}\brk*{
  \fmstar(\pimstar) - \fmstar(\pi\ind{t})}
  &\leq{} \sum_{t=1}^{T}\compbidual(\cM,\mu\ind{t})
   +
  \gamma\cdot\sum_{t=1}^{T}\En_{\Mbar,\Mstar\sim\mu\ind{t}}\En_{\pi\sim{}p\ind{t}}\brk*{
  \Dbi{\Mstar(\pi)}{\Mbar(\pi)}
    }\\
  &\leq{} \frac{HT\cdot\dimbifull}{\gamma}
   +
  \gamma\cdot\sum_{t=1}^{T}\En_{\Mbar,\Mstar\sim\mu\ind{t}}\En_{\pi\sim{}p\ind{t}}\brk*{
  \Dbi{\Mstar(\pi)}{\Mbar(\pi)}
  }.
\end{align*}
We proceed to bound the estimation error term
\begin{equation}
  \label{eq:estimation_error_bilinear}
\sum_{t=1}^{T}\En_{\Mbar,\Mstar\sim\mu\ind{t}}\En_{\pi\sim{}p\ind{t}}\brk*{
  \Dbi{\Mstar(\pi)}{\Mbar(\pi)}
  },
\end{equation}
using the method of confidence sets \citep{russo2013eluder} and the elliptic potential. Define
  \[
    \bloss\ind{t}_h(f,g) = \sum_{i=1}^{t-1}\prn*{f(s_h\ind{i},a_h\ind{i})-r_h\ind{i}-\max_{a}g(s_{h+1}\ind{i},a)}^2,
  \]
  where we recall that
  $\tau\ind{t}=(s_1\ind{t},a_1\ind{t},r_1\ind{t}),\ldots(s_H\ind{t},a_H\ind{t},r_H\ind{t})$
  is the trajectory observed at episode $t$. Let $M_1,\ldots,M_N$ be
  an $\veps$-cover for $\cQm$, and abbreviate
  $\Qmistar_h\equiv\Qmstar[M_i]_h$. Fix $\Confone>0$ and define a
  confidence set
  \[
    \cI\ind{t} =
    \crl*{i\in\brk*{N}\mid{}\bloss_h\ind{t}(\Qmistar_h,\Qmistar_{h+1})
      \leq{}
      \min_{j\in\brk*{N}}\bloss_h\ind{t}(\Qmistar[j]_h,\Qmistar_{h+1})+\Confone^2\;\;\forall{}h\in\brk{H}
    }.
  \]
  Let $\cov:\cM\to\brk{N}$ be any map from a model $M$ to its
  corresponding covering element, and let $\cM\ind{t}=\crl*{M\mid{}\cov(M)\in\cI\ind{t}}$.
We have the following guarantee.
  \begin{lemma}
    \label{lem:confidence_set}
    Let $\delta\in(0,1)$, and suppose we set
    $\Confone=C\cdot{}(\veps^{2}T + \log(HTN/\delta))$ for a
    sufficiently large numerical constant $C$. Define
    $\Conftwo^2 \ldef \bigoh\prn*{\Lbifulls\prn*{\Confone^2 +
        \veps^{2}T + \log(HTN/\delta)}}$. Then with
    probability at least $1-\delta$, for all $t\in\brk*{T}$,
    \begin{enumerate}
    \item For all $M$ such that $\cov(M)\in\cI\ind{t}$,
      \begin{equation}
        \label{eq:bilinear_confidence}
        \sum_{i=1}^{t-1}\tri*{X_h(\pi\ind{i};\Mstar),W_h(M;\Mstar)}^{2}
        \leq{} \Conftwo^2 \quad\forall{}h\in\brk*{H}.
      \end{equation}
    \item $\cov(\Mstar)\in\cI\ind{t}$.
    \end{enumerate}
  \end{lemma}
  Let $\delta\in(0,1)$ be fixed, and let $\Confone$ and $\Conftwo$ be set as in
  \pref{lem:confidence_set}. Let $\cE\ind{t}(\Mstar)$ denote the event in which both
  statements in \pref{lem:confidence_set} hold at round
  $t$.\footnote{We write $\cE\ind{t}(\Mstar)$ to make
    the fact that $\cE\ind{t}$ is a measurable function of
    $\hist\ind{t-1}$ and $\Mstar$ explicit.}
  Then,
  using that $\nrm*{X_h(\cdot;\cdot)}_2,\nrm*{W_h(\cdot;\cdot)}_2\leq{}1$, we
  have
  \begin{align*}
    &\En_{\Mbar,\Mstar\sim\mu\ind{t}}\En_{\pi\sim{}p\ind{t}}\brk*{
    \Dbi{\Mstar(\pi)}{\Mbar(\pi)}
    } \\
      &=     \sum_{h=1}^{H}\En_{\Mbar,\Mstar\sim\mu\ind{t}}\En_{\pi\sim{}p\ind{t}}\brk*{
    \tri*{X_h(\pi;\Mbar),W_h(\Mstar;\Mbar)}^2
        } \\
      &\leq{}     \sum_{h=1}^{H}\En_{\Mbar,\Mstar\sim\mu\ind{t}}\En_{\pi\sim{}p\ind{t}}\brk*{
    \tri*{X_h(\pi;\Mbar),W_h(\Mstar;\Mbar)}^2\indic\crl{\cE\ind{t}(\Mstar)}
        }  + \En_{\Mstar\sim\mu\ind{t}}\brk{\indic\crl{\neg\cE\ind{t}(\Mstar)}}H\\
    &\leq{}     \sum_{h=1}^{H}\En_{\Mbar\sim\mu\ind{t}}\En_{\pi\sim{}p\ind{t}}\brk*{\sup_{M\in\cM\ind{t}}
    \tri*{X_h(\pi;\Mbar),W_h(M;\Mbar)}^2
      }  + \En_{\Mstar\sim\mu\ind{t}}\brk{\indic\crl{\neg\cE\ind{t}(\Mstar)}}H\\
    &=    \sum_{h=1}^{H}\En_{\Mstar\sim\mu\ind{t}}\En_{\pi\sim{}p\ind{t}}\brk*{\sup_{M\in\cM\ind{t}}
    \tri*{X_h(\pi;\Mstar),W_h(M;\Mstar)}^2}
      + \En_{\Mstar\sim\mu\ind{t}}\brk{\indic\crl{\neg\cE\ind{t}(\Mstar)}}H,
  \end{align*}
  where we have used that $\cM\ind{t}$ is
  $\hist\ind{t-1}$-measurable. We further bound
  \begin{align*}
    &\En_{\Mstar\sim\mu\ind{t}}\En_{\pi\sim{}p\ind{t}}\brk*{\sup_{M\in\cM\ind{t}}
    \tri*{X_h(\pi;\Mstar),W_h(M;\Mstar)}^2}
       \\
    &\leq
      \En_{\Mstar\sim\mu\ind{t}}\En_{\pi\sim{}p\ind{t}}\brk*{\sup_{M\in\cM\ind{t}}
    \tri*{X_h(\pi;\Mstar),W_h(M;\Mstar)}^2\indic\crl*{\cE\ind{t}(\Mstar)}
    } +\En_{\Mstar\sim\mu\ind{t}}\brk{\indic\crl{\neg\cE\ind{t}(\Mstar)}}.
  \end{align*}

  Next, using the definition of $\cE\ind{t}$,
\begin{align*}
&\sup_{M\in\cM\ind{t}}
\tri*{X_h(\pi,\Mstar),W_h(M;\Mstar)}^{2}\indic\crl*{\cE\ind{t}(\Mstar)}\\
&\leq{} \sup_{M}\crl*{\tri*{X_h(\pi;\Mstar),W_h(M;\Mstar)}^{2}\mid{}
  \sum_{i=1}^{t-1}\tri*{X_h(\pi\ind{i};\Mstar),W_h(M;\Mstar)}^{2}\leq{}
                                                                           \Conftwo^2}\\
  &\leq{} \sup_{W}\crl*{\tri*{X_h(\pi;\Mstar),W}^{2}\mid{}
  \sum_{i=1}^{t-1}\tri*{X_h(\pi\ind{i};\Mstar),W}^{2}\leq{}
    \Conftwo^2,\; \nrm*{W}_2^2\leq{}1}\\
  &\leq{} \sup_{W}\crl*{\tri*{X_h(\pi;\Mstar),W}^{2}\mid{}
  \sum_{i=1}^{t-1}\tri*{X_h(\pi\ind{i};\Mstar),W}^{2}+\nrm*{W}_2^2\leq{}
    \Conftwo^2+1}\\
  & \leq \nrm*{X_h(\pi;\Mstar)}^{2}_{(\Sigma\ind{t-1}_{h})^{-1}}(\Conftwo^2+1),
\end{align*}
where $\Sigma\ind{t}_h\ldef\sum_{i=1}^{t-1}X_h(\pi\ind{i};\Mstar)
X_h(\pi\ind{i};\Mstar)^{\trn}+I$. Summing across all rounds and
marginalizing, we have
\begin{align*}
  \En\brk*{\RegDM}
  &\leq{} \frac{HT\cdot\dimbifull}{\gamma}
  + \gamma(\Conftwo^2+1)\En\brk*{\sum_{t=1}^{T}\sum_{h=1}^{H}\nrm*{X_h(\pi\ind{t};\Mstar)}^{2}_{(\Sigma\ind{t-1}_{h})^{-1}}}
  +2H\sum_{t=1}^{T}\bbP(\neg\cE\ind{t}(\Mstar))\\
  &\leq{} \frac{HT\cdot\dimbifull}{\gamma}
  + \gamma(\Conftwo^2+1)\En\brk*{\sum_{t=1}^{T}\sum_{h=1}^{H}\nrm*{X_h(\pi\ind{t};\Mstar)}^{2}_{(\Sigma\ind{t-1}_{h})^{-1}}}
  +2\delta{}HT.
\end{align*}
Finally, using the standard elliptic potential lemma (e.g.,
\cite{lattimore2020bandit}, Lemma 19.4), we are guaranteed that for
all $h$, with probability $1$,
\[
  \sum_{t=1}^{T}\nrm*{X_h(\pi\ind{t};\Mstar)}^{2}_{(\Sigma\ind{t-1}_{h})^{-1}}
\leq{} 2\dimbifull\log(1+T/\dimbifull) \leq{} 2\dimbifull\log(2T).
\]
Altogether, after recalling the definition of $\Conftwo$, we have that
\[
  \En\brk*{\RegDM}
  \leq{} \bigoht\prn*{\frac{HT\cdot\dimbifull}{\gamma}
  + \gamma\cdot{}H\cdot{}\dimbifull\cdot{}\Lbifulls(\veps^{2}T + \log(\cN(\cQm,\veps)/\delta))
  +\delta{}HT
  }.
\]
The result now follows by setting $\delta=1/TH$
and tuning $\gamma$ and $\veps$.%
\end{proof}

\begin{proof}[\pfref{lem:confidence_set}]
  We first prove Property 1. For each $M\in\cM$, let
  \[
    Z_h\ind{t}(M)
    \ldef \prn*{\Qmstar_h(s_h\ind{t},a_h\ind{t}) - r_h\ind{t} -
    \Vmstar_{h+1}(s_{h+1}\ind{t})}^2
    - \prn*{\brk{\cTmstar_h\Vmstar_{h+1}}(s_h\ind{t},a_h\ind{t})-
      r_h\ind{t} - \Vmstar_{h+1}(s_{h+1}\ind{t})}^2.
  \]
  Define a filtration
  $\filt\ind{t}=\sigma(\hist\ind{t-1},\pi\ind{t})$ and let
  $\En_t\brk{\cdot}\ldef\En\brk*{\cdot\mid{}\filt\ind{t}}$. Then we
  have $\abs{Z_h\ind{t}}\leq{}4$, 
  \[
    \En_{t-1}\brk*{Z_h\ind{t}(M)}
    = \Enm{\Mstar}{\pi\ind{t}}\brk*{
      \prn*{\Qmstar_h(s_h,a_h) - \brk{\cTmstar_h\Vmstar_{h+1}}(s_h,a_h)}^2
    }
  \]
  and
  \[
    \En_{t-1}\brk*{\prn*{Z_h\ind{t}(M)}^2}
    \leq{} 16\En_{t-1}\brk*{Z_h\ind{t}(M)},
\]
where we have used that $\sum_{h=1}^{H}r_h\in\brk*{0,1}$. As a result,
applying \pref{lem:freedman} and taking a union bound, we have that with probability at least
$1-\delta$, for all $t\in\brk{T}$, $h\in\brk{H}$, and $i\in\brk{N}$,
\begin{equation}
  \sum_{k=1}^{t-1}\En_{k-1}\brk{Z_h\ind{k}(M_i)}
  \leq{} 2  \sum_{k=1}^{t-1}Z_h\ind{k}(M_i)
   + 64\log(HTN/\delta).\label{eq:bilinear_conf0}
 \end{equation}
 Consider $i\in\cI\ind{t}$, and let $j$ be any covering element that is $\veps$-close to
 $\brk{\cTmstar_h\Vmistar_{h+1}}$; such an element is guaranteed to
 exist by \pref{ass:completeness}. From the definition of $\cI\ind{t}$, we have
 \begin{align}
   \sum_{k=1}^{t-1}Z_h\ind{k}(M_i)
   &= \bloss_h\ind{t}(\Qmistar_h,\Qmistar_{h+1})
     - \bloss_h\ind{t}(\brk{\cTmstar_h\Vmistar_{h+1}},\Qmistar_{h+1})\notag\\
   &= \bloss_h\ind{t}(\Qmistar_h,\Qmistar_{h+1})
     - \bloss_h\ind{t}(\Qmistar[j]_h,\Qmistar_{h+1})
     + \bloss_h\ind{t}(\Qmistar[j]_h,\Qmistar_{h+1})
   -
     \bloss_h\ind{t}(\brk{\cTmstar_h\Vmistar_{h+1}},\Qmistar_{h+1})\notag\\
      &\leq \bloss_h\ind{t}(\Qmistar_h,\Qmistar_{h+1})
     - \min_{j'}\bloss_h\ind{t}(\Qmistar[j']_h,\Qmistar_{h+1})
     + \bloss_h\ind{t}(\Qmistar[j]_h,\Qmistar_{h+1})
   - \bloss_h\ind{t}(\brk{\cTmstar_h\Vmistar_{h+1}},\Qmistar_{h+1})\notag\\
   &\leq{} \Confone^{2} + \bloss_h\ind{t}(\Qmistar[j]_h,\Qmistar_{h+1})
     - \bloss_h\ind{t}(\brk{\cTmstar_h\Vmistar_{h+1}},\Qmistar_{h+1}).
     \label{eq:bilinear_conf1}
 \end{align}
 Next, we observe that
 \begin{align*}
   &\bloss_h\ind{t}(\Qmistar[j]_h,\Qmistar_{h+1})
     - \bloss_h\ind{t}(\brk{\cTmstar_h\Vmistar_{h+1}},\Qmistar_{h+1})\\
   &=
\sum_{k=1}^{t-1} \prn*{\Qmstar[j]_h(s_h\ind{k},a_h\ind{k})-\brk{\cTmstar_h\Vmistar_{h+1}}(s_h\ind{k},a_h\ind{k})}^2\\ &~~~~+2\sum_{k=1}^{t-1}\underbrace{\prn*{\Qmstar[j]_h(s_h\ind{k},a_h\ind{k})-\brk{\cTmstar_h\Vmistar_{h+1}}(s_h\ind{k},a_h\ind{k})}\prn*{\brk{\cTmstar_h\Vmistar_{h+1}}(s_h\ind{k},a_h\ind{k})
   - r_h\ind{k} +
                                                                                                                        \Vmistar_{h+1}(s_{h+1}\ind{k})}}_{\rdef{}\Delta_h\ind{k}(i,j)}\\
   &\leq{} \veps^{2}T + 2\sum_{k=1}^{t-1}\Delta_h\ind{k}(i,j),
 \end{align*}
 where the inequality uses the covering property for $M_j$. We now
 appeal to the following result.
 \begin{lemma}
   \label{lem:bilinear_conf2}
   With probability at least $1-\delta$, for all $i,j\in\brk*{N}$,
   $h\in\brk*{H}$, and $t\in\brk*{T}$, 
   \[
     \sum_{k=1}^{t-1}\Delta_h\ind{k}(i,j)
     \leq{} 4
          \sum_{k=1}^{t-1}\En_{k-1}\brk*{
\prn*{\Qmstar[j]_h(s_h\ind{k},a_h\ind{k})-\brk{\cTmstar_h\Vmistar_{h+1}}(s_h\ind{k},a_h\ind{k})}^2
          }
    + 2\log(HTN^2\delta^{-1}).
   \]
 \end{lemma}
 Conditioning on the event in \pref{lem:bilinear_conf2}, the covering
 property guarantees that
 \[
  \sum_{k=1}^{t-1}\Delta_h\ind{k}(i,j)
     \leq{} 4\veps^{2}T
    + 2\log(HTN^2\delta^{-1}).
  \]
  Combining this with \pref{eq:bilinear_conf0} and
  \pref{eq:bilinear_conf1}, we have that with probability at least
  $1-2\delta$, for all $h\in\brk*{H}$ and $t\in\brk*{T}$, all
  $i\in\cI\ind{t}$ satisfy
  \[
    \sum_{k=1}^{t-1}\Enm{\Mstar}{\pi\ind{k}}\brk*{  \prn*{\Qmistar_h(s_h,a_h) - \brk{\cTmstar_h\Vmistar_{h+1}}(s_h,a_h)}^2 } =     \sum_{k=1}^{t-1}\En_{k-1}\brk{Z_h\ind{k}(M_i)}
    \leq{} 2\Confone^2 + 18\veps^{2}T + 72\log(HTN/\delta).
  \]
  It follows immediately from the covering property that any $M$ for
  which $\cov(M)\in\cI\ind{t}$ must have
  \[
    \sum_{k=1}^{t-1}\Enm{\Mstar}{\pi\ind{k}}\brk*{  \prn*{\Qmstar_h(s_h,a_h) - \brk{\cTmstar_h\Vmstar_{h+1}}(s_h,a_h)}^2 }   
    \leq \bigoh(\Confone^2 + \veps^{2}T + \log(HTN/\delta)).
  \]
  Property 1 now follows from \pref{eq:bilinear_complete}.

We now prove Property 2. Let $i$ be an $\veps$-covering element for
$\Mstar$, and define
\[
  E_h\ind{t}(M) = \prn*{\Qmstar(s_h\ind{t},a_h\ind{t}) -
    r_h\ind{t}-\Vmistar_{h+1}(s_{h+1}\ind{t})}^2
  - \prn*{\Qmistar(s_h\ind{t},a_h\ind{t}) - r_h\ind{t}-\Vmistar_{h+1}(s_{h+1}\ind{t})}^2.
\]
Note that $\abs{E_h\ind{t}(M)}\leq{}4$. Defining $y_h\ind{i}=r_h\ind{t}-\Vmistar_{h+1}(s_{h+1}\ind{t})$, we
can write
\begin{align*}
  E_h\ind{t}(M)
  &= \prn*{\Qmstar_h(s_h\ind{t},a_h\ind{t}) -
  \brk{\cTmstar_h\Vmistar_{h+1}}(s_{h}\ind{t},a_h\ind{t})}^2
  - \prn*{\Qmistar_h(s_h\ind{t},a_h\ind{t}) -
    \brk{\cTmstar_h\Vmistar_{h+1}}(s_{h}\ind{t},a_h\ind{t})}^2\\
  &~~~~+\prn*{\Qmstar_h(s_h\ind{t},a_h\ind{t}) -
    \brk{\cTmstar_h\Vmistar_{h+1}}(s_{h}\ind{t},a_h\ind{t})}\prn*{\brk{\cTmstar_h\Vmistar_{h+1}}(s_{h}\ind{t},a_h\ind{t})-y_h\ind{t}}\\
  &~~~~-\prn*{\Qmstar_h(s_h\ind{t},a_h\ind{t}) -
    \brk{\cTmstar_h\Vmistar_{h+1}}(s_{h}\ind{t},a_h\ind{t})}\prn*{\brk{\cTmstar_h\Vmistar_{h+1}}(s_{h}\ind{t},a_h\ind{t})-y_h\ind{t}}.
\end{align*}
It follows that
\[
  \En_{t-1}\brk*{E\ind{t}_h(M)}
  = \En_{t-1}\brk*{\prn*{\Qmstar_h(s_h\ind{t},a_h\ind{t}) -
  \brk{\cTmstar_h\Vmistar_{h+1}}(s_{h}\ind{t},a_h\ind{t})}^2}
  - \En_{t-1}\brk*{\prn*{\Qmistar_h(s_h\ind{t},a_h\ind{t}) -
    \brk{\cTmstar_h\Vmistar_{h+1}}(s_{h}\ind{t},a_h\ind{t})}^2}
\]
and
\begin{align*}
  &\En_{t-1}\brk*{\prn*{E\ind{t}_h(M)}^2}\\ 
  &\leq{} 
20\En_{t-1}\brk*{\prn*{\Qmstar_h(s_h\ind{t},a_h\ind{t}) -
  \brk{\cTmstar_h\Vmistar_{h+1}}(s_{h}\ind{t},a_h\ind{t})}^2}
+ 20\En_{t-1}\brk*{\prn*{\Qmistar_h(s_h\ind{t},a_h\ind{t}) -
    \brk{\cTmstar_h\Vmistar_{h+1}}(s_{h}\ind{t},a_h\ind{t})}^2}.
\end{align*}
Applying \pref{lem:freedman} to the sequence
$\En_{t-1}\brk*{E_h\ind{t}(M_i)}-E_h\ind{t}(M_i)$ and taking a union
bound, we have that for any $\eta\leq{}1/8$, with probability at least
$1-\delta$, for all $t\in\brk*{H}$, $h\in\brk*{H}$, and
$j\in\brk*{N}$,
\begin{align*}
  -\sum_{k=1}^{t-1}E_h\ind{k}(M_j)
  &\leq{} -\sum_{k=1}^{t-1}\En_{k-1}\brk*{E_h\ind{k}(M_j)}
  + \eta \sum_{k=1}^{t-1}\En_{k-1}\brk*{(E_h\ind{k}(M_j))^2}
    + \eta^{-1}\log(HTN/\delta).
\end{align*}
We choose $\eta\leq{}1/20$ and observe that
\begin{align*}
  &-\sum_{k=1}^{t-1}\En_{k-1}\brk*{E_h\ind{k}(M_j)}
  + \eta \sum_{k=1}^{t-1}\En_{k-1}\brk*{(E_h\ind{k}(M_j))^2}\\
  &\leq (1+20\eta)\sum_{k=1}^{t-1}
\En_{k-1}\brk*{\prn*{\Qmistar_h(s_h\ind{k},a_h\ind{k}) -
    \brk{\cTmstar_h\Vmistar_{h+1}}(s_{h}\ind{k},a_h\ind{k})}^2}\\
  &~~~~ -(1-20\eta)\sum_{k=1}^{t-1}
\En_{k-1}\brk*{\prn*{\Qmistar[j]_h(s_h\ind{k},a_h\ind{k}) -
    \brk{\cTmstar_h\Vmistar_{h+1}}(s_{h}\ind{k},a_h\ind{k})}^2}\\
  &\leq 4(1+20\eta)\veps^{2}T,
\end{align*}
where the last inequality uses the covering property of $M_i$. We
conclude that for all $j\in\brk{N}$,
\begin{align*}
  \bloss_h\ind{t}(\Qmistar_h,\Qmistar_{h+1})
  -   \bloss_h\ind{t}(\Qmistar[j]_h,\Qmistar[i]_{h+1})
  =   -\sum_{k=1}^{t-1}E_h\ind{k}(M_j)
  \leq{} \bigoh(\veps^2{}T + \log(HTN/\delta)).
\end{align*}
Since this bound holds uniformly for all $j\in\brk*{N}$, it follows
that by setting $\Confone=C\cdot{}(\veps^2{}T + \log(HTN/\delta))$ for
a sufficiently large numerical constant $C$, we have
$i\in\cI\ind{t}$ for all $h\in\brk*{H}$ and $t\in\brk*{T}$ whenever
the concentration event above holds.

\end{proof}

\begin{proof}[\pfref{lem:bilinear_conf2}]
  Let $i,j\in\brk*{N}$, $h\in\brk*{H}$, and $t\in\brk*{T}$ be
  fixed. Since $\abs{\Delta_h\ind{t}(i,j)}\leq{}2$ and
  $\En_{t-1}\brk*{\Delta_h\ind{t}(i,j)}=0$, \pref{lem:freedman}
  implies that with probability at least $1-\delta$,
  \begin{align*}
    \sum_{k=1}^{t-1}\Delta_h\ind{k}(i,j)
    &\leq{} \frac{1}{2}
    \sum_{k=1}^{t-1}\En_{k-1}\brk*{(\Delta_h\ind{k}(i,j))^2}
      + 2\log(\delta^{-1})\\
        &\leq{} 4
          \sum_{k=1}^{t-1}\En_{k-1}\brk*{
\prn*{\Qmstar[j]_h(s_h\ind{k},a_h\ind{k})-\brk{\cTmstar_h\Vmistar_{h+1}}(s_h\ind{k},a_h\ind{k})}^2
          }
    + 2\log(\delta^{-1}).
  \end{align*}
  The result now follows from a union bound.
  
\end{proof}

\subsection{Proofs from \preft{sec:rl_lower}}

\subsubsection{Proof of \preft{prop:linear_qstar}}

  \newcommand{\term}{\mathsf{term}}%
  \newcommand{\Mastar}{M_{\astar}}%
  \newcommand{\Pastar}{P\ind{\astar}}%
  \newcommand{\Rastar}{R\ind{\astar}}%
  \newcommand{\Qastar}{Q\ind{\astar}}%
  \newcommand{\Vastar}{V\ind{\astar}}%
  \newcommand{\piastar}{\pi_{\astar}}%
  \newcommand{\Qmbara}{Q\sups{\Mbar}}%
  \newcommand{\fmastar}{f\sups{\Mastar}}%

\linearlyrealizable*

\begin{proof}[\pfref{prop:linear_qstar}]%
  We follow the construction of \cite{wang2021exponential}. Let
  $d\in\bbN$ be given and fix $\Delta\in\brk{0,1/6}$. By the Johnson-Lindenstrauss theorem, there exist
  $m\ldef\exp\prn*{\frac{1}{8}\Delta^{2}d}$ unit vectors
  $\crl*{v_1,\ldots,v_m}$ such that $\abs*{\tri{v_i,v_j}}\leq{}\Delta$
  for all $i\neq{}j$. We set $\cA=\crl*{1,\ldots,m}\cup\crl{\term}$ and
  $\cS=\crl*{1,\ldots,m}\cup\crl*{\term}$, where $\term$ denotes both
  i) a special
  terminal state, and ii) a corresponding action that always transitions to
  said state. For each state $s\in\cS\setminus\crl{\term}$, the available actions are
  $\cA\setminus\crl*{s}$. For the terminal state, all actions are
  available.

  We begin with a feature map $\phi(s,a)\in\bbR^{d}$ defined via
  \begin{align*}
    &\phi(s,a) =
    (\tri{v_{s},v_{a}}+2\Delta)v_{a}\quad\forall{}s\in\brk{m}, a\notin\crl{s,\term},\\
    &\phi(\term,\cdot)=\mb{0},\\
    &\phi(\cdot,\term)=\mb{0}.
  \end{align*}
  We define a family of MDPs
  $\crl*{\Mastar}_{\astar\in\brk*{m}}$ that are linearly
  realizable with respect to $\phi$. Let $\astar\in\brk{m}$
  be fixed. Each MDP $\Mastar$ has a stationary probability transition kernel
  $\Pastar\equiv{}\Pm[\Mastar]$ and initial state distribution $d_1$
  defined as follows:
\begin{itemize}
\item The initial state distribution $d_1$ is uniform over $\brk*{m}$.
\item $P\ind{\astar}(\term\mid{}\term,\cdot)=1$.
\item $P\ind{\astar}(\term\mid{}\cdot,\term)=1$.
\item $P\ind{\astar}(\term\mid{}s,\astar) = 1$ for all $s\notin\crl{\astar,\term}$.
\item For all $s\neq{}\term$, $a\notin\crl{s,\astar,\term}$:
  \begin{align*}
    P\ind{\astar}(\cdot\mid{}s,a)=\left\{
    \begin{array}{ll}
      a : \tri{v_{s},v_{a}} + 2\Delta\\
      \term : 1-(\tri{v_{s},v_{a}} + 2\Delta)
    \end{array}
    \right..
  \end{align*}
  \end{itemize}
  This is a valid transition distribution whenever $\Delta\leq1/3$, since
  $\Delta\leq{}\tri{v_{s},v_{a}} + 2\Delta\leq{}3\Delta$.
  Next, we define a nonstationary reward function
  $\Rastar\equiv\Rm[\Mastar]$. For $1\leq{}h<H$, we have:
  \begin{align*}
    &R\ind{\astar}_h(\term,\cdot)=\Rad(0),\\
    &R\ind{\astar}_h(\cdot,\term)=\Rad(0),\\
    &R_h\ind{\astar}(s,\astar) = \Rad\prn*{\tri{v_s,v_{\astar}} + 2\Delta}\quad\forall{}s\notin\crl{\astar,\term},\\
    &R\ind{\astar}_h(s,a) = \Rad(-2\Delta(\tri{v_s,v_a}+2\Delta))\quad\forall{}s\neq{}\term,a\notin\crl{s,\astar,\term}.
  \end{align*}
  We define $R\ind{\astar}_H(s,a) =
  \Rad(\tri*{\phi(s,a),v_{\astar}})$.
  
To finish the construction, we define a reference model $\Mbar$. As
before, we take
$d_1=\mathrm{unif}(\brk*{m})$. The stationary transition kernel
$\Pbar\equiv\Pmbar$ is given by
\begin{align*}
  &\Pbar(\term\mid{}\term,\cdot)=1,\\
    &\Pbar(\term\mid{}\cdot,\term)=1,\\
  &\Pbar(\cdot\mid{}s,a)=\left\{
    \begin{array}{ll}
      a : \tri{v_{s},v_{a}} + 2\Delta\\
      \term : 1-(\tri{v_{s},v_{a}} + 2\Delta)
    \end{array}
    \right.\quad\forall{}s\neq\term, a\notin{}\crl{s,\term}.
\end{align*}
We define the nonstationary reward function $\Rbar\equiv\Rmbar$ for
the reference model as follows. For all $h<H$, set
\begin{align*}
  &\Rbar_h(\term,\cdot)=\Rad(0),\\
  &\Rbar_h(\cdot,\term)=\Rad(0),\\
    &\Rbar_h(s,a) = \Rad(-2\Delta(\tri{v_s,v_a}+2\Delta))\quad\forall{}s\neq{}\term,a\notin\crl{s,\term},
  \end{align*}
and take $\Rbar_H(\cdot,\cdot)=\Rad(0)$.
  
Our lower bound is based on the family of models
$\cM'\ldef\crl*{M_{a}}_{a\in\brk{m}}\cup\crl*{\Mbar}$. We first show
that this class is indeed linearly realizable; proofs for this and all
subsequent lemmas are deferred to the end of the main proof. We abbreviate
$Q\ind{a}\equiv\Qmstar[M_a]$, $V\ind{a}\equiv\Vmstar[M_a]$,
$\Qmbara\equiv\Qmstar[\Mbar]$, and $\pi_{a}\equiv\pi_{M_a}$ throughout
\begin{lemma}%
  \label{lem:linear_family}
If $\Delta<1/3$, then for all $\astar\in\brk{m}$, $Q\ind{\astar}_h(s,a) =
\tri*{\phi(s,a),v_{\astar}}$ for all $h\in\brk{H}$.
\end{lemma}
This lemma shows that $\crl{M_a}_{a\in\brk{m}}$ is linearly realizable
with respect to $\phi$. To show ensure $\Mbar$ is linearly realizable,
we consider the expanded feature map feature map $\phi'(s,a)\in\bbR^{d+1}$ given by
\[\phi'(s,a) = \prn{\phi(s,a),
    -2\Delta(\tri{v_s,v_a}+2\Delta)\indic\crl{s\neq\term,a\notin\crl{s,\term}}}.\]
\vspace{-15pt}
  \begin{lemma}
  \label{lem:qbar_linear}
  We have $\Qmbara_h(s,a)=\tri{\phi'(s,a),(\mb{0},1)}=0$ for all
$h<H$ and $\Qmbara_H(s,a)=\tri{\phi'(s,a),\mb{0}}$.
\end{lemma}
This establishes that $\cM'$ is linearly realizable with the
feature map $\phi'$. Note that $\nrm*{\phi'(s,a)}_2\leq{}1$ whenever
$\Delta\leq{}1/6$, and that all of the weight parameters for the
$Q$-functions defined above have norm at most $1$ as well.

Next, we show that $\cM'$ is localized around $\Mbar$.
\begin{lemma}[Local property]
  \label{prop:linq_local}
  For all $\astar\in\brk{m}$,
  $\fmbar(\pimbar) \geq{} f\sups{\Mastar}(\piastar) -
  3\Delta^{2}$, so that $\cM'\subseteq\cM_{3\Delta^{2}}(\Mbar)$.
\end{lemma}

To proceed, we state three lemmas which will be used to prove that $\cM'$ is a
hard sub-family of models.%
\begin{lemma}
  \label{prop:terminal_probability}
For all $M\in\cM'$, and $h\in\brk{H}$, $\sup_{\pi\in\PiRNS}\Prm{M}{\pi}(s_{h}\neq\term)\leq{}(3\Delta)^{h-1}$.
\end{lemma}

\begin{lemma}
  \label{prop:linq_gap}
  Whenever $\Delta\leq{}1/6$, we have that for all $\pi\in\PiRNS$,
  \begin{equation}
    f\sups{\Mastar}(\piastar) - f\sups{\Mastar}(\pi)
    \geq{}
    \frac{\Delta}{2}\prn*{1- \Prm{\Mbar}{\pi}(s_1=\astar\vee{}a_1=\astar)}.
    \label{eq:linq_gap}
  \end{equation}
\end{lemma}

\begin{lemma}
  \label{prop:linq_kl}
Let $\Delta\leq{}1/6$. Then for all $\astar\in\brk*{m}$ and all $\pi\in\PiRNS$,
  \begin{align} 
    \Dhels{\Mastar(\act)}{\Mbar(\act)}
    \leq{}
    11\Delta\sum_{h=1}^{H}\Prm{\Mbar}{\pi}(s_h\neq\term,a_h=\astar) +     (3\Delta)^{H+1}.
    \label{eq:linq_gap}
\end{align}
\end{lemma}
We now complete the proof. We claim that $\crl{\Mastar}_{\astar\in\brk{m}}$ is a hard family of
models in the sense of \pref{lem:hard_family}. Define
$u_{\astar}(\act)=\Prm{\Mbar}{\pi}(s_1=\astar\vee{}a_1=\astar)$. Then
whenever $m\geq{}4$, for all $\pi\in\PiRNS$,
\begin{align*}
  \sum_{\astar\in\brk{m}}u_{\astar}(\pi)
  \leq{}   \sum_{\astar\in\brk{m}}\Prm{\Mbar}{\pi}(s_1=\astar) +
  \Prm{\Mbar}{\pi}(a_1=\astar)\leq{}2 \leq{}\frac{m}{2}.
\end{align*}
Next, let $v_{\astar}(\pi) =
\frac{1}{2}\sum_{h=1}^{H}\Prm{\Mbar}{\pi}(s_h\neq\term,a_h=\astar)$. Then by
\pref{prop:terminal_probability}, whenever $\Delta\leq{}1/6$, we have
that for all $\pi\in\PiRNS$,
\begin{align*}
  \sum_{\astar\in\brk{m}}v_{\astar}(\pi)
  \leq{}
  \frac{1}{2}\sum_{\astar\in\brk{m}}\sum_{h=1}^{H}(3\Delta)^{h-1}\Prm{\Mbar}{\pi}\prn*{a_h=\astar\mid
  s_h\neq\term}
  \leq{} \frac{1}{2}\sum_{h=1}^{H}(3\Delta)^{h-1}\leq{}1.
\end{align*}
As a result, \pref{prop:linq_gap} and \pref{prop:linq_kl} imply that
$\crl{\Mastar}_{\astar\in\brk{m}}$ is a $(\Delta/2, 22\Delta,
(3\Delta)^{H+1})$-family, and \pref{lem:hard_family} gives the lower bound
\[
\comp(\cM,\Mbar)
\geq{} \frac{\Delta}{4} - \gamma\prn*{
  \frac{22\Delta}{m} + (3\Delta)^{H+1}
}.
\]
In particular, whenever $\Delta\leq{}1/6$, we see that as long as
$\gamma\cdot\prn*{22\cdot{}m^{-1}+3\cdot{}2^{-H}}\leq{}1/8$, \[\comp(\cM,\Mbar) \geq{} \Delta/8.\]

We set $\Delta=1/6$. Recalling that $m=\exp(\frac{1}{8}\Delta^{2}d)$,
it suffices to take
\[
\gamma{}\leq{} 2^{-9}\min\crl*{m, 2^{H}}=2^{-9}\min\crl*{\exp(288^{-1}d), 2^{H}}.
\]
Finally, we recall that the dimension of the construction is
$d+1$ (as opposed to $d$), and simplify further to
\[
\gamma{}\leq{} 2^{-9}\min\crl*{\exp(2^{-10}d), 2^{H}}.
\]
To ensure that $m\geq{}4$ as required, it suffices to take $d\geq{}2^{9}$.

\end{proof}

\subsubsection{Proofs for Auxiliary Lemmas (\pref{prop:linear_qstar})}

\begin{proof}[\pfref{lem:linear_family}]
Let $\astar$ be fixed, and let us abbreviate
$Q\equiv\Qastar$ and
$V\equiv\Vastar$.
We begin by noting that in the terminal state,
$Q_h(\term,\cdot)=V_h(\term,\cdot)=0$, so that
$Q_h(\term,a)=\tri{\phi(\term,a),v_{\astar}}$. Similarly, for any
state $s$, $Q_h(s,\term)=0=\tri{\phi(s,\term),v_{\astar}}$. We proceed
to prove the result for non-terminal states and actions. We prove inductively that for all $s\neq{}\term$,
$a\notin\crl{s,\term}$,
\begin{align}
  &Q_h(s,a) = (\tri{v_s,v_a}+2\Delta)\tri{v_{\astar},v_a}\indic\crl{a\neq\term} =
    \tri{\phi(s,a), v_{\astar}},\label{eq:linq1}
    \intertext{and that for all $s\notin\crl{\astar,\term}$,}
  &V_h(s) = Q_h(s,\astar) = \tri{v_s,v_{\astar}} + 2\Delta.\label{eq:linq2}
\end{align}
Condition \pref{eq:linq1} holds at layer $H$ by definition. We now show
that \pref{eq:linq2} follows from \pref{eq:linq1} for all $h$. Indeed,
for $s\neq\term$, $a\notin\crl{s,\astar}$,
\begin{align*}
  Q_h(s,a)=(\tri{v_s,v_a}+2\Delta)\tri{v_{\astar},v_a}\indic\crl{a\neq\term}\leq{}3\Delta^{2},
\end{align*}
while for $s\neq\astar$, $Q_h(s,\astar)=(\tri{v_s,v_{\astar}}+2\Delta)\geq{}\Delta$. This
implies that the optimal action is $\astar$ whenever $\Delta<1/3$, so that $V_h(s)=Q_h(s,\astar)$.

We now complete the induction. For any $s\notin\crl{\astar,\term}$, we
trivially have
\[
  Q_h(s,\astar)= (\tri{v_{s},v_{\astar}}+2\Delta)=(\tri{v_{s},v_{\astar}}+2\Delta)\tri{v_{\astar},v_{\astar}},
\]
from the definition of the reward distribution. For any
$s\neq\term$ and $a\notin\crl{s,\astar,\term}$, by Bellman optimality,
\begin{align*}
  Q_h(s,a) &= \En\sups{\Mastar}\brk*{r_h\mid{}s_h=s,a_h=a}  +
             \En\sups{\Mastar}\brk*{V_{h+1}(s_{h+1})\mid{}s_h=s,a_h=a}\\
  &= -2\Delta(\tri{v_s,v_a}+2\Delta) +
    (\tri{v_s,v_a}+2\Delta)V_{h+1}(a)
    + (1- (\tri{v_s,v_a}+2\Delta))V_{h+1}(\term)\\
  &= -2\Delta(\tri{v_s,v_a}+2\Delta) +
    (\tri{v_s,v_a}+2\Delta)(\tri{v_a,v_{\astar}}+2\Delta)\\
  &= (\tri{v_s,v_a}+2\Delta)(\tri{v_a,v_{\astar}}.
\end{align*}
This establishes the result.
\end{proof}

\begin{proof}[\pfref{lem:qbar_linear}]
  Let us abbreviate $Q\equiv{}\Qmstar[\Mbar]$ and
  $V\equiv{}\Vmstar[\Mbar]$. The claim that $Q_H(s,a)=0$
  follows immediately from the definition of $\Rbar_H$. For $h<H$, we
  prove the result under the inductive hypothesis that $V_{h+1}(s)=0$,
  which is clearly satisfied at layer $H$.

  Fix $h<H$. Then by Bellman optimality,
  \begin{align*}
    Q_h(s,a) &= \En\sups{\Mbar}\brk*{r_h\mid{}s_h=,a_h=a} + \En\sups{\Mbar}\brk*{V_{h+1}(s_{h+1})\mid{}s_h=s,a_h=a}\\
             &= \En\sups{\Mbar}\brk*{r_h\mid{}s_h=s,a_h=a}\\
             &=    -2\Delta(\tri{v_s,v_a}+2\Delta)\indic\crl{s\neq\term,a\notin\crl{s,\term}}\\
             &= \tri{\phi'(s,a),(\mb{0},1)}.
  \end{align*}
  We now prove that $V_h(s)=0$ as a consequence. If $s=\term$ this is trivial, so
  suppose otherwise. Then we have
  \[
-2\Delta(\tri{v_s,v_a}+2\Delta) \leq{} -2\Delta^{2} < 0,
\]
so that $Q_h(s,a)<Q_h(s,\term)=0$ for all
$a\notin\crl{s,\term}$. It follows that the optimal action is $\term$,
and that $V_h(s)=0$.

\end{proof}

\begin{proof}[\pfref{prop:terminal_probability}]
  For all models in the family, $P(\term\mid{}s,a)\geq{}1-3\Delta$ for all admissible $(s,a)$
  pairs. The result follows immediately.
\end{proof}

\begin{proof}[\pfref{prop:linq_gap}]%
  Recall that $\piastar\equiv\pi\subs{\Mastar}$. By the performance difference
  lemma \citep{kakade2003sample}, we have that for any $\pi\in\PiRNS$,
  \[
    \fmastar(\piastar)
    - \fmastar(\pi)
    = \Enm{\Mastar}{\pi}\brk*{
      \sum_{h=1}^{H}\Qastar_{h}(s_h,\piastar(s_h))
      - \Qastar_{h}(s_h,\pi(s_h))
    }
  \]
\pref{lem:linear_family} shows that for all
$s\notin\crl{\astar,\term}$, we have $\piastar(s)=\astar$, and for $a\neq\piastar(s)$,
\[
\Qastar_{h}(s,\astar)
      - \Qastar_{h}(s,a) =\tri{\phi(s,\astar)-\phi(s,a),v_{\astar}}\geq{} \Delta-3\Delta^{2} \geq{} \frac{\Delta}{2},
    \]
    whenever $\Delta\leq{}1/6$. Hence, since $\astar$ is only
    reachable at $h=1$, we have
    \begin{align*}
      \fmastar(\piastar)
      - \fmastar(\pi)
      &\geq{}
      \frac{\Delta}{2}\sum_{h=2}^{H}\Prm{\Mastar}{\pi}(s_h\neq\term,a\neq\astar)
      +
      \frac{\Delta}{2}\Prm{\Mastar}{\pi}(s_1\notin\crl{\astar,\term},a_1\neq\astar)\\
      &\geq{}
        \frac{\Delta}{2}\Prm{\Mastar}{\pi}(s_1\notin\crl{\astar,\term},a_1\neq\astar).
    \end{align*}
    Since $s_1\sim\mathrm{unif}(\brk{m})$ for all models, and $\term$ is not reachable at $h=1$, this implies
    \[
      \fmastar(\piastar)
      - \fmastar(\pi)
      \geq{}
      \frac{\Delta}{2}\Prm{\Mbar}{\pi}(s_1\neq\astar,a_1\neq\astar)
      = \frac{\Delta}{2}\prn*{1- \Prm{\Mbar}{\pi}(s_1=\astar\vee{}a_1=\astar)}.
    \]
\end{proof}

\begin{proof}[\pfref{prop:linq_kl}]
  First, observe that
  $P\ind{\astar}(\cdot\mid{}\term,\cdot)=\Pbar(\cdot\mid\term,\cdot)$,
  and for all $s\neq\term$ and $a\neq{}\astar$, we have
  \[
    P\ind{\astar}(\cdot\mid{}s,a) = \Pbar(\cdot\mid{}s,a),
  \]
  while for $s\neq{}\astar$, we have
  \[
    \P\ind{\astar}(\term\mid{}s,\astar)=1,\mathand \Pbar(\cdot\mid{}s,\astar)=\left\{
    \begin{array}{ll}
      \astar : \tri{v_{s},v_{\astar}} + 2\Delta\\
      \term : 1-(\tri{v_{s},v_{\astar
      }} + 2\Delta)
    \end{array}
    \right.\quad\forall{}s\neq\term, a\neq{}s.
  \]
Similarly, we have $R\ind{\astar}_h(\term,\cdot)=\Rbar_h(\term,\cdot)$
for all $h\in\brk*{H}$, and for all $h<H$ and $s\neq{}\term$,
$a\neq\astar$ we have $R_h(s,a)=\Rbar(s,a)$. Next,
\[
  R_h\ind{\astar}(s,\astar)=\Rad(\tri{v_s,v_{\astar}}+2\Delta)\neq{}\Rbar_h(s,\astar)=\Rad(-2\Delta(\tri{v_s,v_{\astar}}+2\Delta)).
\]
Lastly, for $s\neq\term$, we have
$R_H\ind{\astar}(s,\astar)=\Rad(\tri{\phi(s,a),v_{\astar}})$, while $\Rbar_H(s,a)=0$.

Using that $\Delta\leq{}1/3$, we have $\tri{v_s,v_{\astar}}+2\Delta\in\brk*{\Delta,3\Delta}$,
$-2\Delta(\tri{v_s,v_{\astar}}+2\Delta)\in\brk*{-6\Delta^{2},6\Delta^{2}}$,
and
\[
\abs*{\tri{v_s,v_{\astar}}+2\Delta -
  \prn*{-2\Delta(\tri{v_s,v_{\astar}}+2\Delta)}} \leq 5\Delta.
\]
Likewise, we have
\[
\abs{\tri{\phi(s,a),v_{\astar}}}\leq{}3\Delta.
\]
Using these observations, we appeal to the following result.
\begin{proposition}
  \label{prop:kl_rademacher}
  For any $\mu_1, \mu_2\in(-1,+1)$,
  $\Dkl{\Rad(\mu_1)}{\Rad(\mu_2)}\leq\frac{(\mu_1-\mu_2)^{2}}{2\min\crl*{1+\mu_1,1+\mu_2,1-\mu_1,1-\mu_2}}$.
\end{proposition}
This implies that whenever
$\Delta\leq{}1/6$, for all $h<H$,
\[
\Dkl{\Rbar_h(s,a)}{R\ind{\astar}_h(s,a)}\leq{}25\Delta^{2}\indic\crl*{s\neq\term,a=\astar},
\]
and that
\[
\Dkl{\Rbar_H(s,a)}{R_H\ind{\astar}(s,a)}\leq{}9\Delta^{2}\indic\crl*{s\neq\term}.
\]
Next, we calculate that whenever $\Delta\leq{}1/6$,
\begin{align*}
  \Dkl{\Pbar(s,a)}{P\ind{\astar}(s,a)}
  &=
  \Dkl{\mathrm{Mult}((0,1))}{\mathrm{Mult}((\tri{v_s,v_{\astar}}+2\Delta,
    1-\tri{v_s,v_{\astar}}-2\Delta))}\indic\crl{s\neq\term,a=\astar}\\
  &\leq \log\prn*{\frac{1}{1-3\Delta}}\indic\crl{s\neq\term,a=\astar},\\
  &\leq 6\Delta\indic\crl{s\neq\term,a=\astar}.
\end{align*}
Combining these bounds, we have
\begin{align*} 
  \Dhels{\Mbar(\pi)}{\Mastar(\pi)}
  &\leq{}   \Dkl{\Mbar(\pi)}{\Mastar(\pi)}\\
  &=\Enm{\Mbar}{\pi}\brk*{
  \sum_{h=1}^{H}\Dkl{\Pbar(s_h,a_h)}{P\ind{\astar}(s_h,a_h)}
  +\Dkl{\Rbar(s_h,a_h)}{R\ind{\astar}(s_h,a_h)}
    }\\
  & \leq{}
    9\Delta^{2}\Prm{\Mbar}{\pi}(s_H\neq\term) + 
    \sum_{h=1}^{H}6\Delta\Prm{\Mbar}{\pi}(s_h\neq\term,a_h=\astar)
    + 25\Delta^{2}\Prm{\Mbar}{\pi}(s_h\neq\term,a_h=\astar)\\
      & \leq{}
    9\Delta^{2}\Prm{\Mbar}{\pi}(s_H\neq\term) + 
        11\Delta\sum_{h=1}^{H}\Prm{\Mbar}{\pi}(s_h\neq\term,a_h=\astar),
\end{align*}
where we have used that $\Delta\leq{}1/6$ to simplify. Finally, using
\pref{prop:terminal_probability}, we have $9\Delta^{2}\bbP_{\Mbar,\pi}(s_H\neq\term)\leq{}9\Delta^{2}(3\Delta)^{H-1}=(3\Delta)^{H+1}$.
\end{proof}

\begin{proof}[\pfref{prop:linq_local}]
  \pref{lem:linear_family} proves that for all
  $\astar\in\brk{m}$, $Q\ind{\astar}_h(s,a) =
  \tri{\phi(s,a),v_{\astar}}\leq{}3\Delta^{2}$ for all $h$. It follows
  that $\max_{\pi}f\sups{\Mastar}(\pi)\leq{}3\Delta^{2}$. On the other
  hand, \pref{lem:qbar_linear} shows that
  $\max_{\pi}\fmbar(\pi)=0$.
\end{proof}

\subsubsection{Proof of \pref{prop:mdp_gap_linear}}

\deterministiclinear*

\begin{proof}[\pfref{prop:mdp_gap_linear}]
As with \pref{prop:mdp_lower_linear}, we consider a bandit problem
with no state, which is a special case of the reinforcement learning
setting with $H=1$. The proof is a small modification to \pref{prop:bandit_gap_linear}, with the only difference being that rewards are no longer stochastic.

We set $\Act=\cA=\crl*{e_1,\ldots,e_d}$ and construct a family
$\cM'=\crl{M_i}_{i\in\brk{d}}$ as follows. Let  $\theta_1 = \tfrac{1}{3}\cdot{}e_1$, and let $\theta_i =
\tfrac{1}{3}\cdot{}e_1 + \tfrac{2}{3}\cdot{}e_i$ for all $i\geq{}2$; we have
$\nrm{\theta_i}_2\leq{}1$. Define
$M_i(\act) = \delta_{\tri{\theta_i,\act}}$ so that 
$\fmi(\act)=\tri{\theta_i,\act}$. Take $\Mbar=M_1$ as the
reference model. We have $\cM'\subseteq\cMinf[1/3](\Mbar)$.

We now show that $\cM'$ is a hard family of models in the sense of
\pref{lem:hard_family}. Define $\hardu_i(\act) =
\hardv(\act)=\indic\crl{\act=e_i}$. Then for any $i$, we have
\[
  \fmi(\pimi) - \fmi(\act) \geq{}\tfrac{1}{3}(1-\hardu_i(\act)),
\]
and
\[
  \Dhels{M_i(\act)}{\Mbar(\act)}\leq{}2\indic\crl{\act=e_i}
  \leq{} 2\hardv_i(\act).
\]
It follows that $\cM'$ is a $(\tfrac{1}{3},2,0)$-family, so
\pref{lem:hard_family} implies that
\[
  \comp(\cM',\Mbar)
  \geq{} \frac{1}{6} - \gamma\frac{4}{d}.
\]
In particular, $\comp(\cM',\Mbar)\geq{}\frac{1}{12}$ whenever $\gamma\leq\frac{d}{48}$.
\end{proof}

\subsection{Additional Results: Efficient Implementation and
  Estimation for \pcigwb}
\label{sec:igw_bilinear}
This section of the appendix contains additional results concerning
the \pcigwb algorithm
(\cref{alg:igw_bilinear}). \cref{sec:rl_bilinear_computational} gives
an approach to efficient implementation, and
\cref{sec:rl_bilinear_estimation} gives an approach to online
estimation that can applied with \cref{thm:igw_bilinear}.

\subsubsection{Efficient Implementation of \pcigwb}
\label{sec:rl_bilinear_computational}
In this section, we show that it is possible to efficiently implement the optimal design step
required by \pcigw (\pref{alg:igw_bilinear}) whenever we have access to an
efficient oracle for planning with a fixed model. For
the results in this subsection, we assume that for all $\Mbar\in\cM$,
each set $\crl*{X_h(M;\Mbar)}_{M\in\cM}$ is compact and has full dimension. We adopt the convention that for a collection of vectors
$x_1,\ldots,x_d\in\bbR^{d}$, $\det(x_1,\ldots,x_d)$ denotes the determinant of
the matrix $(x_1,\ldots,x_d)\in\bbR^{d\times{}d}$.

\paragraph{Oracle and assumptions}
Our main computational primitive is as follows.
\begin{definition}[Bilinear planning oracle]
  \label{def:bilinear_planning}
  For a bilinear class $\cM$, a bilinear planning oracle takes as input a model $\Mbar\in\cM$ with
  bilinear dimension $d$ and vector $\theta\in\bbR^{dH}$ with
  $\nrm{\theta}_{2}\leq{}1$ and returns
  \begin{equation}
    \label{eq:bilinear_oracle}
    \oracle(\theta;\Mbar) = \argmax_{M\in\cM}\tri*{X(M;\Mbar), \theta}.
  \end{equation}
\end{definition}
Informally, the planning oracle \pref{def:bilinear_planning} asserts
that we can efficiently perform
linear optimization when the underlying model $\Mbar$ is held
fixed. For intuition recall that for most bilinear classes, we have $X_h(M;\Mbar) =
X_h(\pim;\Mbar)$, and that $X_h(\pim;\Mbar)$ typically represents a sufficient statistic
for the roll-in distribution induced by running $\pim$ in $\Mbar$. In
this case, the bilinear planning oracle precisely corresponds to
planning (i.e., maximizing a known reward function)
with a known model.
\begin{fact}
  The following models admit efficient bilinear planning oracles:
  \begin{itemize}
  \item For Linear MDPs, Linear Bellman Complete MDPs, Block MDPs,
    FLAMBE/Feature Selection, and MDPs with Linear
    $Q^{\star}/V^{\star}$, the oracle $\oracle$ can be implemented efficiently
    whenever we can efficiently plan with known dynamics in $\cM$ and any linear
    reward function.
  \item For the Low Occupancy Complexity setting, $\oracle$ can be implemented efficiently
    whenever we can efficiently plan with known dynamics in $\cM$ and
    an arbitrary reward function.\footnote{For low occupancy complexity, we require
      that the feature map has
      $\mathrm{dim}(\crl{\phi(s,a)}_{s\in\cS,a\in\cA})=d$.}
  \end{itemize}
\end{fact}

Beyond a planning oracle, we require a (mild) additional assumption on
the bilinear representation for $\cM$.

\begin{assumption}
  \label{ass:bilinear_reward}
  For all $\Mbar\in\cM$, there exists $\thetambar\in\bbR^{dH}$ such
  that for all $M\in\cM$,
  \begin{equation}
    \label{eq:36}
    \fmbar(\pim) = \tri*{X(M;\Mbar),\thetambar}
  \end{equation}
\end{assumption}
This assumption implies that the bilinear planning oracle can be used
to find an optimal policy when the model is known to the
learner, which is a fairly minimal requirement if one wishes to develop efficient
algorithms when the dynamics are \emph{unknown}.
\begin{fact}
  The following models satisfy \pref{ass:bilinear_reward}: Linear/Linear Bellman Complete MDPs,\footnote{For the
      Linear Bellman Complete setting, we require that the all-zero
      function is contained in the value function class of interest.} 
    Block MDPs, FLAMBE/Feature Selection, Linear
    $Q^{\star}$/$V^{\star}$, Low Occupancy Complexity, and Linear Mixture MDPs.\footnote{For linear mixture MDPs, we must
      artificially expand the bilinear class construction in
      \cite{du2021bilinear} from dimension $d$ to dimension $2d$.}
  \end{fact}

  \paragraph{Algorithm} The algorithm we consider, \igwspanner
  (\pref{alg:igw_spanner}), is an adaptation of an algorithm for
  contextual bandits with linearly structured actions from 
  concurrent work by a subset of the authors \citep{zhu2021making}. The idea behind the algorithm
  is to replace the notion of optimal design with that of a
  \emph{barycentric spanner} \citep{awerbuch2008online}.
\begin{definition}[Barycentric spanner]
  For a set $\cX\subseteq\bbR^{d}$, a subset
  $\cC=\crl{x_1,\ldots,x_d}\subseteq\cX$ is said to be a
  \emph{$C$-approximate barycentric spanner} (for $C\geq{}1$) if every
  element $x\in\cX$ can be expressed as a linear combination of
  elements in $\cC$ with coefficients in $\brk*{-C,C}$.
\end{definition}
The following well-known result shows that any barycentric spanner
yields a $d$-approximate optimal design.
\begin{fact}[\cite{awerbuch2008online,dani2007price}]
  Let $\cC=\crl*{x_1,\ldots,x_d}$ be a $C$-approximate barycentric
  spanner for $\cX\subseteq\bbR^{d}$. Then $p\ldef{}\unif(\cC)$ is a
  $C\cdot{}d$-approximate optimal design.
\end{fact}
\cite{awerbuch2008online} show that for any compact set $\cX\subseteq\bbR^{d}$, one can
efficiently find a $2$-approximate barycentric spanner using
$\bigoh(d^2\log(d))$ calls to a linear optimization oracle capable of
solving $\argmax_{x\in\cX}\tri*{x,\theta}$ for any
$\nrm*{\theta}_2\leq{}1$. \pref{alg:igw_spanner} is an adaptation of
their algorithm. The key challenge in applying this technique is that
the bilinear planning oracle (\pref{def:bilinear_planning}) can solve
$\argmax_{M\in\cM}\tri*{X_h(M;\Mbar),\theta}$, but our aim is to find
a bayrcentric spanner for the reweighted factors
$\crl*{Y_h(M;\Mbar)}_{M\in\cM}$. Implementing a linear optimization
oracle for the reweighted factors entails solving
\begin{equation}
  \label{eq:reweighted_argmax}
\argmax_{M\in\cM}\frac{\tri*{X_h(M;\Mbar),\theta}}{\sqrt{1+\eta(\fmbar(\pimbar)
    - \fmbar(\pim))}},
\end{equation}
which is not a linear function and hence cannot be
solved by $\oracle$ directly. To overcome this
challenge, \pref{alg:igw_spanner} appeals to a subroutine, \igwargmax
(\pref{alg:igw_argmax}), which uses binary search to reduce the optimization problem in
\pref{eq:reweighted_argmax} to a series of linear subproblems which
can be solved by the bilinear planning oracle. In total, $\bigoht(d)$
oracle calls are required to implement \pref{eq:reweighted_argmax}. The final guarantee for the algorithm is as follows.

\algblockdefx{MyRepeat}{EndRepeat}{\textbf{repeat}}{}
\algnotext{EndRepeat}

\begin{algorithm}[t]
    \setstretch{1.3}
    \begin{algorithmic}[1]
      \Statex[0] \algcomment{Find barycentric spanner for collection
        $\crl*{\frac{X_h(M;\Mbar)}{\sqrt{1+\eta(\fmbar(\pimbar)-\fmbar(\pim))}}}_{M\in\cM}$ using bilinear planning oracle.}
       \State \textbf{parameters}:
       \Statex[1] Class $\cM$ and reference model $\Mbar\in\cM$ with
       bilinear dimension $d$.
       \Statex[1] Layer $h\in\brk{H}$.
       \Statex[1] Learning rate $\eta>0$.
       \Statex[1] Initial collection
       $M_{1},\ldots{}M_d\in\cM$ with $\abs*{\det\prn*{X_h(M_1;\Mbar),\ldots,
           X_h(M_d;\Mbar)}}\geq{}r^d$ for $r>0$.
       \State Set
       $\cC=\crl{Y_h(M_1;\Mbar),\ldots,Y_h(M_d;\Mbar)}$.\hfill\algcommentlight{$Y_h(M;\Mbar)$
         is defined in \pref{eq:bilinear_reweighted}.}
       \MyRepeat \label{line:igw_repeat}
       \For{$i=1,\ldots,d$}
       \State Let $\theta\in\bbR^{d}$ represent the linear function
       $Y\mapsto{}\det(Y,\cC_{-i})$. \hfill\algcommentlight{$\det(Y,\cC_{-i})=\tri*{\theta,Y}$ for all $Y$.}
       \State Let $M\gets\textsf{IGW-ArgMax}(\theta; \Mbar, h, \eta, r)$.\hfill\algcommentlight{\pref{alg:igw_argmax}.}
       \If
       {$\abs*{\det(Y_h(M;\Mbar),\cC_{-i})}\geq{}2^{1/2}\abs*{\det(\cC)}$}
       \State Replace $M_i$ with $M$ in
       $\cC$. \hfill\algcommentlight{That is, $\cC\gets{}\crl{Y_h(M;\Mbar)}\cup\cC_{-i}$.}
       \State \textbf{continue} to \pref{line:igw_repeat}.
       \EndIf
       \EndFor
       \State \textbf{break}
       \EndRepeat
       \State \textbf{return} $\unif(M_1,\ldots{}M_d)$
       \hfill\algcommentlight{$2d$-approximate optimal design.}
\end{algorithmic}
\caption{\igwspanner \citep{zhu2021making}}
\label{alg:igw_spanner}
\end{algorithm}

\begin{algorithm}[t]
    \setstretch{1.3}
    \begin{algorithmic}[1]
      \Statex[0] \algcomment{Approximately solve
        $\argmax_{M\in\cM}\frac{\abs{\tri{X_h(M;\Mbar),\theta}}}{\sqrt{1+\eta(\fmbar(\pimbar)-\fmbar(\pim))}}$
        with bilinear planning oracle (by grid search).}
       \State \textbf{parameters}:
       \Statex[1] Bilinear planning oracle $\oracle$ for
       $\cM$. \hfill\algcommentlight{See \pref{eq:bilinear_oracle}.}
       \Statex[1] Parameter $\theta\in\bbR^{d}$.
       \Statex[1] Reference model $\Mbar\in\cM$.
       \Statex[1] Layer $h\in\brk{H}$, learning rate $\eta>0$, scale
       parameter $r\in(0,1)$.
       \State Let $N=\ceil{\log_{\frac{4}{3}}(\tfrac{4}{3}r^{-d})}$ and
       define $\cE=\crl*{\prn{\frac{3}{4}}^{i}}_{i=1}^{N}\cup \crl*{-\prn{\frac{3}{4}}^{i}}_{i=1}^{N}$.
       \State
       $\cMhat=\crl*{\emptyset}$.\hfill\algcommentlight{Candidate
         argmax models.}
       \For{each $\veps\in\cE$}
       \State Define $\tilde{\theta}\in\bbR^{dH}$ via $\tilde{\theta}=\veps{}\cdot{}(\theta\otimes{}e_h) +
         \eta\veps^{2}\cdot{}\thetambar$.
       \State $M \gets{} \oracle(\tilde{\theta};
       \Mbar)$.\hfill\algcommentlight{Call bilinear planning oracle \pref{eq:bilinear_oracle}
         with $\tilde{\theta}$.}
       \State Add $M$ to $\cMhat$.
       \EndFor
       \State \textbf{return}
       $\argmax_{M\in\cMhat}\frac{\abs{\tri{X_h(M;\Mbar),\theta}}}{\sqrt{1+\eta(\fmbar(\pimbar)-\fmbar(\pim))}}$.
       \hfill\algcommentlight{Enumeration over $\bigoht(d)$ candidates.}
\end{algorithmic}
\caption{\igwargmax \citep{zhu2021making}}
\label{alg:igw_argmax}
\end{algorithm}

\begin{theorem}[Efficiency of \igwspanner \citep{zhu2021making}]
  \label{thm:igw_spanner}
  Let $\cM$ have bilinear dimension $d$ relative to $\Mbar\in\cM$, and
  let
  \pref{ass:bilinear_reward} hold. Consider \pref{alg:igw_spanner}
  with $\eta\geq{}1$, and suppose the initial
  collection $M_1,\ldots,M_d$ has
  \[
\abs*{\det\prn*{X(M_1;\Mbar),\ldots,
           X(M_d;\Mbar)}}\geq{}r^d
     \]
     for some $r\in(0,1)$, and that
     $\sup_{M\in\cM}\nrm*{X_h(M;\Mbar)}_{2}\leq{}1$.
     Then \pref{alg:igw_spanner} computes a $2$-approximate barycentric
  spanner (and consequently a $2d$-approximate optimal design) for the
  collection $\crl*{Y_h(M;\Mbar)}_{M\in\cM}$ using
  $\bigoh(d^{3}\log(d/r)\log(\eta/r))$ calls to the bilinear planning
  oracle $\oracle$.
\end{theorem}

Applying this result, we can implement \pcigwb using
$\bigoht(\dimbifulls[3])$ calls to the bilinear planning oracle, at
the cost an extra $\dimbifull$ factor on the \CompShort bound due to
the approximation factor paid by the barycentric spanner. For example, the final bound on the \CompText is
\[
\comp(\cM,\Mbar) \approxleq{}\frac{\dimbifulls}{\gamma}
\]
in the on-policy case.

\subsubsection{Online Estimation for \pcigwb}
\label{sec:rl_bilinear_estimation}
In order to apply the
\mainalg algorithm (via \pref{thm:upper_general}), we need to bound
$\comp(\cM,\Mhat\ind{t})$ for the sequence of estimators
$\Mhat\ind{1},\ldots,\Mhat\ind{T}$ produced by the online estimation
oracle $\AlgEst$.  The bound on the \CompText attained through posterior sampling in
\pref{thm:posterior_bilinear} holds for arbitrary reference models
$\Mbar$, and hence can be applied with any estimator, but the bounds
on the \CompShort attained by \pcigw in \cref{thm:igw_bilinear} are proven under the
assumption that $\Mbar$ belongs to the class $\cM$. This limits the
immediate application of \pcigw, because
most online estimation
algorithms are improper and produce estimates in
$\conv(\cM)$.\footnote{Recall that an estimation algorithm is said to be proper if
  it produces estimates that lie inside $\cM$.}  
To address this issue, and derive efficient end-to-end algorithms
based on \pcigw, we sketch an approach based on layer-wise
estimators.\footnote{This issue can be addressed more directly under various technical
  conditions, but we leave this for a future version of this work.} Suppose for simplicity that rewards are known, and let
$\cP_h=\crl*{\Pm_h\mid{}M\in\cM}$ be the class of transition kernels
for layer $h$. We make the following assumption.
\begin{assumption}
  \label{ass:layerwise}
  The class $\cM$ has product structure
  $\cM_1\times\cdots\times\cM_H$. Moreover, each layer-wise class
  $\cP_h$ is convex.
\end{assumption}
Instead of directly working with an estimator for the entire model $M$, we assume
access to layer-wise estimators
$\AlgEsth[1],\ldots,\AlgEsth[H]$. At each round $t$, given the history $\crl*{(\act\ind{i},
r\ind{i},\obs\ind{i})}_{i=1}^{t-1}$, the layer-$h$ estimator $\AlgEsth[h]$ produces an estimate
$\Phat\ind{t}_h$ for the true transition kernel $\Pmstar_h$. We
measure performance of the estimator via layer-wise Hellinger error:
\begin{equation}
  \label{eq:layerwise_hellinger}
  \EstHelh \ldef{}\sum_{t=1}^{T}\En_{\act\ind{t}\sim{}p\ind{t}}\Enm{\Mstar}{\pi\ind{t}}\brk*{\Dhels{\Pmstar_h(s_h,a_h)}{\Phat\ind{t}_h(s_h,a_h)}}.
\end{equation}
We obtain an estimation algorithm $\AlgEst$ for the full model by
taking $\Mhat\ind{t}$ as the MDP that has $\Phat_h\ind{t}$ as the
transition kernel for each layer $h$. This
algorithm has the following guarantee.
\begin{proposition}
  \label{prop:layerwise_estimator}
  The estimator $\AlgEst$ described above has
  \[
    \EstHel \leq \bigoh(\log(H))\cdot{}\sum_{h=1}^{H}\EstHelh[h].
  \]
Moreover, if $\Phat\ind{t}\in\conv(\cP_h)$ for all $h$ and
\pref{ass:layerwise} is satisfied, then $\Mhat\ind{t}\in\cM$.  
\end{proposition}
\begin{proof}
  Immediate consequence of \pref{lem:hellinger_chain_rule}.
\end{proof}

For example, whenever \pref{ass:layerwise} is satisfied, we can use
this reduction with the tools in \pref{app:online} to generically provide
proper estimators with 
\[
\EstHel\leq\bigoht(H\cdot{}\max_{h}\PComph),
\]
 where $\PComph$ is the covering
number-based complexity measure introduced in \pref{def:g_cover}.

\begin{example}[Linear MDP]
  In the linear MDP setting \citep{jin2020provably}, all $M\in\cM$ have
  $\Pm_h(s'\mid{}s,a)=\tri*{\phi_h(s,a),\mum_h(s')}$, where $\phi_h(s,a)\in\bbR^{d}$ is a
  known feature map and $\mum_h(s)\in\bbR^{d}$ is a model-dependent
  feature map; we assume $\nrm*{\phi_h(s,a)}_2,\nrm*{\mum(s')}_2\leq{}1$. In addition, the
  reward distribution $\Rm_h(s,a)$ is assumed to have mean $\tri*{\phi(s,a),\wm_h}$. This setting has
  $\dimbifull=d$, $\Lbifull=1$, and $\piestm=\pim$. One can verify that the bilinear
  planning oracle is equivalent to planning in a linear MDP with fixed
  dynamics and rewards, which is computationally efficient, and that
  \pref{ass:bilinear_reward} is satisfied. As a result, \pcigwb with
  the \igwspanner subroutine is efficient, and certifies that
  \[
    \comp(\cM,\Mbar) \leq \bigoh\prn*{\frac{H^3d^2}{\gamma}},
  \]
  and---when used within the \mainalg algorithm---ensures that
  \[
    \RegDM \leq \bigoh\prn[\big]{\sqrt{H^{3}d^2T\cdot\EstHel}}.
  \]
Since this class has the layer-wise structure in \pref{ass:layerwise}, we can obtain $\EstHel\leq{}\bigoht(H\max_h\PComph)$.
\end{example}

\subsection{Additional Results: Bellman Representability and Bellman-Eluder Dimension}
\label{app:rl_extensions}
\label{sec:bedim}
In this section we provide a nonlinear generalization of the bilinear
class property which we refer to as \emph{Bellman representability}, and give a bound on the \CompText based on this
property. The results here recover the Bellman-eluder dimension
\citep{jin2021bellman} as a special case. Throughout the section, we
assume that $\sum_{h=1}^{H}r_h\in\brk*{0,1}$.
\begin{definition}[Bellman representability]
  \label{def:bellman_representable}
  Let $\cM$ and $\Mbar$ be given. Let $\cGmbar=\cGmbar_1,\ldots,\cGmbar_H$ be a collection of a function
  classes of the form $\cGmbar_h=\crl*{\gmmbar_h:\cM\to\brk*{-1,+1}}_{M\in\cM}$. We say
  that $\cGmbar$ is a Bellman representation for $\cM$ (relative to
  $\Mbar$) if:
      \begin{enumerate}
  \item Fo all $h$ and $M\in\cM$:
    \begin{equation}
      \label{eq:bilinear_residual}
\abs*{\Enm{\Mbar}{\pim}\brk*{
          \Qmstar_h(s_h, a_h) - r_h - \Vmstar_h(s_{h+1})
        }}
\leq\abs{\gmmbar_h(M)}.
    \end{equation}
    We assume that $\gmmbar[\Mbar]_h(M)=0$ for all $M\in\cM$.
  \item Let $z_h = (s_h, a_h, r_h, s_{h+1})$. There exists a collection of estimation policies
    $\crl*{\piestm}_{M\in\cM}$ and estimation functions $\crl{\lestm(\cdot;\cdot)}_{M\in\cM}$
    such that for all $M, M'\in\cM$, $h\in\brk{H}$,
    \begin{equation}
      \label{eq:bilinear_tv}
\gmmbar[M']_h(M)
      = \Enm{\Mbar}{\pi^{\vphantom{\mathrm{est}}}_{M}\circ_{h}\piest_{M}}\brk*{
        \lestm(M';z_h)
        }.
      \end{equation}
      If $\piestm=\pim$, we say that estimation is on-policy.
    \end{enumerate}
    We let $\Lbr(\cM;\Mbar)\geq{}1$ denote any almost sure upper bound on
    $\abs{\lestm(M';z_h)}$ under $\Mbar$, and let $\Lbrfull=\sup_{\Mbar\in\cM}\Lbr(\cM;\Mbar)$.
   \end{definition}

   The following result generalizes
   \pref{thm:posterior_bilinear}. Recall that $\pialpham$ denotes a
   randomized policy that, for each $h$, plays $\pi_{\sss{M,h}}$ with probability
$1-\alpha/H$ and $\piest_{\sss{M,h}}$ with probability
$\alpha/H$.
\begin{theorem}
  \label{thm:posterior_bilinear_representable}
  Let $\cM$ be a bilinear class and let $\Mbar$ be an arbitrary
  reference model (not necessarily in $\cM$). Abbreviate
  \[
    \El(\cM,\Delta)=\max_{h}\sup_{M\in\cM}\El(\cG\sups{M}_h,\Delta),\mathand
    \Star(\cM,\Delta)=\max_{h}\sup_{M\in\cM}\Star(\cG\sups{M}_h,\Delta).
  \]
  Let
  $\mu\in\Delta(\cM)$ be given, and consider the modified posterior
  sampling strategy that samples $M\sim\mu$ and plays $\pialpham$ for $\alpha\in[0,1]$.
  \begin{itemize}
    \item If $\piestm = \pim$ (i.e., estimation is on-policy), this strategy with $\alpha=0$ certifies that
      \[
    \compdual(\cM,\Mbar) \leq{}
    \bigoh(H^{2}\Lbrfulls)\cdot{}\inf_{\Delta>0}\crl*{
      \Delta +
      \frac{\min\crl{\El(\cM,\Delta),\Star^{2}(\cM,\Delta)}\log^{2}(\gamma)}{\gamma}}.
  \]
    for all $\gamma\geq{}e$.
  \item For general estimation policies, this strategy
    (with an appropriate choice of $\alpha$) certifies that
      \[
    \comp(\cM,\Mbar) \leq{}
    \bigoh(H^{3/2}\Lbrfull)\cdot{}\inf_{\Delta>0}\prn*{
      \Delta +
      \frac{\min\crl{\El(\cM,\Delta),\Star^{2}(\cM,\Delta)}\log^{2}(\gamma)}{\gamma}}^{1/2},
  \]
  whenever $\gamma\geq{}e$ is sufficiently large.
  \end{itemize}
\end{theorem}

  \begin{example}[Bellman-eluder dimension]
  By taking
  \[
    \gmmbar_h(M')=\Enm{\Mbar}{\pi\subs{M'}}\brk*{\Qmstar_h(s_h,a_h)  - r_h - \Vmstar_{h+1}(s_{h+1})},
  \]
and $\piestm=\pim$, \pref{thm:posterior_bilinear_representable} recovers the Q-type Bellman-eluder dimension of
  \cite{jin2021bellman}. Taking
    \[
    \gmmbar_h(M')=\Enm{\Mbar}{\pi\subs{M'}\circ_{h}\pim}\brk*{\Qmstar_h(s_h,a_h)  - r_h - \Vmstar_{h+1}(s_{h+1})},
  \]
and $\piestm=\mathrm{unif}(\cA)$ recovers the V-type Bellman-eluder dimension.

On the other hand, the notion of Bellman representability goes well
beyond the Bellman-eluder dimension, since it recovers the usual Bilinear class definition as
a special case. Interestingly
\pref{thm:posterior_bilinear_representable} shows that the Bellman-eluder
dimension can be replaced by a weaker parameter we call the \emph{Bellman-star number},
which, in general, can be arbitrarily small compared to the former quantity.
\end{example}

\begin{proof}[\pfref{thm:posterior_bilinear_representable}]
  We only sketch the proof, as it closely follows
  \pref{thm:posterior_bilinear}.
  Let $\mu\in\Delta(\cM)$,
  $\Mbar$, and $\gamma\geq{}1$ be fixed. \dfedit{By \cref{lem:randomize}, it suffices bound the quantity
    \begin{align*}
      \En_{M\sim\mu}\En_{\pi\sim{}p}\brk*{\fm(\pim) - \fm(\act)} = \En_{\Mtil\sim\mu}\En_{M\sim\mu}\En_{\pi\sim{}p}\brk*{\fm(\pim) - \jqedit{\fmtil}(\act)}
    \end{align*}
    in terms of the Hellinger error
    $\En_{\Mtil\sim\mu}\En_{M\sim\mu}\En_{\pi\sim{}p}\brk[\big]{\DhelsX{\big}{M(\pi)}{\jqedit{\Mtil}(\pi)}}$. Fix
  $\Mtil\in\cM$.} Following the same steps as \pref{thm:posterior_bilinear} and using
the first property of \pref{def:bellman_representable}, we have
\begin{align*}
  \En_{M\sim\mu}\En_{\pi\sim{}p}\brk*{
  \fm(\pim) - \fmtil(\pi)
  }
  &\leq{} \En_{M\sim\mu}\brk*{\sum_{h=1}^{H}\Enm{\Mtil}{\pim}\brk*{\Qmstar_h(s_h,a_h) - r_h -
    \Vmstar_{h+1}(s_{h+1})}} + \alpha \\
    &\leq{} \sum_{h=1}^{H}\En_{M\sim\mu}\abs{\gmmtil_h(M)} + \alpha. 
\end{align*}
Using \pref{thm:decoupling_general}, we have that for each $h$, for
all $\eta>e$,
\begin{align*}
  \En_{M\sim\mu}\abs{\gmmtil_h(M)}
  \leq{} \inf_{\Delta>0}\crl*{
      2\Delta +
      6\frac{\sdis(\cGmtil_h,\Delta,\eta^{-1};\rhomu)\log^{2}(\eta)}{\eta}}
    + \eta\cdot \En_{M,M'\sim\mu}\brk*{(\gmmtil_h(M'))^2},
\end{align*}
where $\sdis$ is the disagreement coefficient
(\pref{def:disagreement}) and $\rho_{\mu}$ is the induced distribution
over policies. \pref{lem:disagreement_to_ratio} implies
that
\[
  \sdis(\cGmbar_h,\Delta,\eta^{-1};\rhomu)
  \leq{} 4\min\crl*{\El(\cGmtil_h,\Delta),\Star^2(\cGmtil_h,\Delta)}
  \leq{} 4\min\crl*{\El(\cM,\Delta),\Star^2(\cM,\Delta)},
\]
and \pref{lem:bilinear_hellinger} (via the second property of
\pref{def:bellman_representable}) implies that for all $\alpha\leq{}1/2$,
\begin{align*}
  \En_{M,M'\sim\mu}\brk*{(\gmmtil_h(M'))^2}
  &\leq{} 4\Calpha\Lbrfulls\cdot\En_{M,M'\sim\mu}\brk*{
    \Dhels{M(\pialpham[M'])}{\Mtil(\pialpham[M'])}}\\
  &= 4\Calpha\Lbrfulls\cdot\En_{M\sim\mu}\En_{\pi\sim{}p}\brk*{
    \Dhels{M(\pi)}{\Mtil(\pi)}},
\end{align*}
where $\Calpha=1$ in the on-policy case and $\Calpha\leq{}2H/\alpha$
in the general case.

Altogether, we have that
\begin{align*}
  &\En_{M\sim\mu}\En_{\pi\sim{}p}\brk*{
  \fm(\pim) - \fmtil(\pi)
  }\\
  &\leq{} \bigoh(H)\cdot{}\inf_{\Delta>0}\crl*{
  \Delta +
  \frac{\min\crl*{\El(\cM,\Delta),\Star^2(\cM,\Delta)}\log^{2}(\eta)}{\eta}}
    + \eta 4H\Calpha\Lbrfulls\cdot\En_{M\sim\mu}\En_{\pi\sim{}p}\brk*{
    \Dhels{M(\pi)}{\Mtil(\pi)}} + \alpha.
\end{align*}
Since this holds uniformly for all $\Mtil\in\cM$, taking the
expectation under $\Mtil\sim\mu$ yields
\begin{align*}
  &\En_{M,\Mtil\sim\mu}\En_{\pi\sim{}p}\brk*{
  \fm(\pim) - \fmtil(\pi)
  }\\
  &\leq{} \bigoh(H)\cdot{}\inf_{\Delta>0}\crl*{
  \Delta +
  \frac{\min\crl*{\El(\cM,\Delta),\Star^2(\cM,\Delta)}\log^{2}(\eta)}{\eta}}
    + \eta 4H\Calpha\Lbrfulls\cdot\En_{M,\Mtil\sim\mu}\En_{\pi\sim{}p}\brk*{
    \Dhels{M(\pi)}{\Mtil(\pi)}} + \alpha.
\end{align*}
From here, following the same steps as in
\pref{thm:posterior_bilinear} and tuning $\eta$ and $\alpha$
appropriately leads to
the result.
\end{proof}

\section{Additional Proofs}
\label{app:additional}

\subsection{Proofs from \preft{sec:intro}}
\label{app:intro}
\newcommand{\Algo}{A}%
\newcommand{\algo}{\Algo}%
\newcommand{\Aspace}{\cA}%
\newcommand{\histIns}{h}%
\newcommand{\Ena}[2]{\En^{\sss{#1,#2}}}
\newcommand{\Pra}[2]{\bbP^{\sss{#1,#2}}}

\begin{proof}[\pfref{prop:minimax_swap}]%
We first state a minimax theorem for \emph{finitely supported}
models, which is a straightforward adaptation of Theorem 1 from \cite{lattimore2019information}.
\begin{lemma}
\label{lem:finite_support_minimax_swap}
Suppose $\cM$ is finitely supported in the sense that $\abs*{\bigcup_{\act\in \Act} \supp(M(\act)) }
< \infty$ for all $M\in\cM$. In addition, assume that $\Act$ is finite and
$\Rspace$ is bounded. Then we have
\begin{equation}
    \MinimaxReg = \MinimaxRegBayes.
\end{equation}
\end{lemma}
We prove \pref{prop:minimax_swap} by arguing that any countably
supported model class can be approximated by a finitely supported class, then
applying \pref{lem:finite_support_minimax_swap} to the approximating class.

Fix $\veps>0$. For each $M\in\cM$, since $M(\act)$ has countable
support and $\Act$ is finite, there exists a model $M_{\veps}$ with
finite support such that
\[
\max_{\act\in\Act}\Dtv{M(\act)}{\Meps(\act)}\leq\veps.
\]
Let $\cM_{\veps}\ldef\crl*{\Meps\mid{}M\in\cM}$
be a class of approximating models. We use the following lemma to relate the
regret under each model in $\cM$ to its approximate counterpart. %
\begin{lemma}
  \label{lem:tv_chain}
    Let $(\cX_1,\filt_1),\ldots,(\cX_n,\filt_n)$ be a sequence of
  measurable spaces, and let $\cX\ind{i}=\prod_{i=t}^{i}\cX_t$ and
  $\filt\ind{i}=\bigotimes_{t=1}^{i}\filt_t$. For each $i$, let
  $\bbP\ind{i}(\cdot\mid{}\cdot)$ and $\bbQ\ind{i}(\cdot\mid{}\cdot)$ be probability kernels from
  $(\cX\ind{i-1},\filt\ind{i-1})$ to $(\cX_i,\filt_i)$. Let $\bbP$ and
  $\bbQ$ be
  the laws of $X_1,\ldots,X_n$ under
  $X_i\sim{}\bbP\ind{i}(\cdot\mid{}X_{1:i-1})$ and
  $X_i\sim{}\bbQ\ind{i}(\cdot\mid{}X_{1:i-1})$ respectively, and
  supposed that
  \[
    \Dtv{\bbP\ind{i}(\cdot\mid{}X_{1:i-1})}{\bbQ\ind{i}(\cdot\mid{}X_{1:i-1})}\leq{}\veps
  \]
  almost surely under $X\sim\bbP$. Then it holds that
  \[
    \Dtvs{\bbP}{\bbQ}
    \leq{}\bigoh(\veps{}\cdot{}n\log{}n).
  \]
\end{lemma}
Using \pref{lem:tv_chain}, we have that for any $M\in\cM$ and any
algorithm $p$,
\[
  \abs*{\Enm{M}{p}\brk{\RegDM}
    -\Enm{\Meps}{p}\brk{\RegDM}}
  \leq{} \bigoht(T^{3/2}\veps^{1/2}).
\]

As a result, we have
\begin{align*}
\MinimaxReg \leq \mathfrak{M}(\cM_\veps,T) + \bigoht(T^{3/2}\veps^{1/2})  = \underbar{\mathfrak{M}}(\cM_{\veps},T) + \bigoht(T^{3/2}\veps^{1/2}) \leq \MinimaxRegBayes + \bigoht(T^{3/2}\veps^{1/2}),
\end{align*}
where the equality holds due to \pref{lem:finite_support_minimax_swap}.
Finally, since neither $\MinimaxReg$ nor $\MinimaxRegBayes$ depends
on $\veps$, we can take $\veps\to 0$ to finish the proof.

\end{proof}

\draft{
\begin{proof}[\pfref{lem:finite_support_minimax_swap}]%
\newcommand{\algdist}{\nu}
This proof follows \cite{lattimore2019information}. Since $
\MinimaxRegBayes \leq  \MinimaxReg$ by definition, it suffices to
prove the opposite direction of the inequality.

Recall
(cf. \pref{sec:prelims}) that we define
$\Hspace\ind{t}\ldef{}\prod_{i=1}^{t}(\Act\times\Rspace\times\Obs)$, and let
us adopt the convention that $\Hspace\ind{0} = \crl*{\emptyset}$ and $\Hspace=\Act\times\Rspace\times\Obs$.

Per the discussion in \cite{lattimore2019information}, rather than
representing the learner's algorithm as a sequence of mappings, $p\ind{1},\ldots,p\ind{T}$, where $p\ind{t}(\cdot\mid\cdot)$ is a
probability kernel from $(\Hspace\ind{t-1},\Hsig\ind{t-1})$ to
$(\Act,\Asig)$, it suffices to instead represent the learner's
algorithm as a single distribution over deterministic mappings. In
more detail, a deterministic algorithm $\algo=(\algo\ind{1},\ldots,\algo\ind{T})$
is a sequence of mappings $\algo\ind{t}:
\Hspace\ind{t-1}\to{}\Act$. We let $\Aspace$ denote the set of all
such deterministic algorithms, and represent the learner's algorithm
as a distribution $\algdist\in\Delta(\Aspace)$.

Let $\Act$ and $\Hspace$. Since $\Act$
is finite, it is compact with respect to the discrete topology, and by Tychonoff's Theorem, we have $\Aspace$ is
compact in product topology. Thus by Theorem 8.9.3 of
\citet{bogachev2007measure}, the space $\Delta(\Aspace)$ is
\weakstar-compact. Note that $\Delta(\Aspace)$ is also convex.

Next, let $\cQ$ be the space of finitely supported probability
measures on $\cM$, which is a convex subset of $\Delta(\cM)$; recall that $\cM$ has the discrete topology. Let $\cQ$ be equipped with the \weakstar-topology.

Consider any function $g: \Delta(\Aspace) \times \cQ \to \bbR$ that is linear and continuous with respect to both arguments (in their respective topologies). Since $\Delta(\Aspace)$ is \weakstar-compact, Sion's minimax theorem \citep{sion1958minimax} implies that,
\begin{align*}
\min_{\algdist\in \Delta(\Aspace)}\sup_{\mu\in\cQ} g(\algdist,\mu) = \sup_{\mu\in\cQ}\min_{\algdist\in \Delta(\Aspace)} g(\algdist,\mu).
\end{align*}
We proceed to verify that the function
\begin{align*}
g(\algdist,\mu) \ldef \En_{M\sim\mu}\En_{\Algo\sim{}\algdist}\brk*{\sum_{t=1}^{T} \Ena{M}{\Algo}  \brk*{ \fm(\pim)-\fm(\Algo\ind{t}(\hist\ind{t-1})) }}
\end{align*}
is linear and continuous the sense describe above, where we recall
that $\hist\ind{t}=(\act\ind{1},r\ind{1},\obs\ind{1}),\ldots,(\act\ind{t},r\ind{t},\obs\ind{t})$. Linearity is immediate. For continuity, since we consider the \weakstar-topology for both spaces, we only need to show that the function
\begin{align*}
   (\Algo,M) \mapsto \Ena{M}{\Algo} \brk*{\fm(\pim)-\fm(\Algo\ind{t}(\hist\ind{t-1}))  } = \fm(\pim)  - \Ena{M}{\Algo} \brk*{  \fm(\Algo\ind{t}(\hist\ind{t-1}))   }
\end{align*}
is continuous with respect to $\Algo$ and $M$ individually. The continuity of $M$ follows because $\cM$ is equipped with discrete topology. 

We proceed to establish continuity of the function $\Algo\mapsto{}\Ena{M}{\Algo}
\brk*{\fm(\Algo\ind{t}(\hist\ind{t-1}))}$. Let $M$ be
fixed. From the definition of the product
topology, for any history $\hist\ind{t}$,  $\Algo \mapsto
\Algo\ind{t}(\hist\ind{t})$ is continuous. Since the decision space
$\Pi$ is finite, we have that for any history $\hist\ind{t}$ and any
decision $\act$, the function $\Algo\mapsto \indic
\crl*{\Algo\ind{t}(\hist\ind{t}) = \act} $ is continuous. We show
inductively that for any history $\hist\ind{t}$, the map
$\Algo\mapsto\Pra{M}{\Algo}_t(\hist\ind{t})$ is continuous, where
$\Pra{M}{\Algo}_t(\cdot)$ denotes the law of $\hist\ind{t}$ under $M$
and $\Algo$. For the base case, $\Algo \mapsto
\Pra{M}{\Algo}_0(\emptyset) = 1  $ is a constant map, and thus
continuous. Next, for any $ \hist\ind{t} = \hist\ind{t-1}\cup(\act\ind{t}, r\ind{t} , \obs\ind{t})$, we have
\begin{align*}
\Pra{M}{\Algo}_t(\hist\ind{t}) = \Pra{M}{\Algo}_{t-1}(\hist\ind{t-1} )\cdot{}\indic \crl*{ \Algo(\hist\ind{t-1}) = \act\ind{t} } \cdot\densm(r\ind{t} ,\obs\ind{t}\mid{} \act\ind{t}),
\end{align*}
which---by the inductive hypothesis---is a product of continuous maps,
and thus continuous.  This establishes that
$\Algo\mapsto\Pra{M}{\Algo}_t(\hist\ind{t})$ is continuous. Now,
observe since $M$ is finitely supported, there can be only finite many
possible trajectories $\hist\ind{t-1}$. Hence the function under
consideration is actually a finite sum of continuous maps:
\begin{align*}
\Ena{M}{\Algo} \brk*{  \fm(\Algo\ind{t}(\hist\ind{t-1}))   } = \sum_{\hist\ind{t-1}:
  (r\ind{i}, \obs\ind{i})\in  \bigcup_{\act}
  \supp(M(\act)) \;\forall{}i<t} 
  \fm(\Algo\ind{t}(\hist\ind{t-1}))\cdot{}
\Pra{M}{\Algo}_{t-1}(\hist\ind{t-1}).
\end{align*}
This establishes continuity for $\Algo$.
Following the reasoning in \cite{lattimore2019information}, we
conclude that
\begin{align*}
    \MinimaxReg \leq   \min_{\algdist\in \Delta(\Aspace) }\sup_{\mu\in\cQ} g(\algdist,\mu) = \sup_{\mu\in\cQ}\min_{\algdist\in \Delta(\Aspace) } g(\algdist,\mu) \leq  \MinimaxRegBayes.
\end{align*}

\end{proof}
}

\begin{proof}[\pfref{lem:tv_chain}]
  Using \pref{lem:hellinger_chain_rule}, we have
  \begin{align*}
\Dtvs{\bbP}{\bbQ} \leq{}    \Dhels{\bbP}{\bbQ}
  &\leq{}
\bigoh(\log(n))\cdot\En_{\bbP}\brk*{\sum_{i=1}^{n}\Dhels{\bbP\ind{i}(\cdot\mid{}X_{1:i-1})}{\bbQ\ind{i}(\cdot\mid{}X_{1:i-1})}}.
  \end{align*}
Noting that
$\Dhels{\bbP\ind{i}(\cdot\mid{}X_{1:i-1})}{\bbQ\ind{i}(\cdot\mid{}X_{1:i-1})}\leq{}2
\Dtv{\bbP\ind{i}(\cdot\mid{}X_{1:i-1})}{\bbQ\ind{i}(\cdot\mid{}X_{1:i-1})}$
concludes the proof.
\end{proof}

\subsection{Proofs from \preft{sec:examples}}
\label{app:examples}

\subsubsection{Proof of \preft{prop:efficient_tabular}}

\newcommand{\lpsolve}{\mathsf{LPSolve}\xspace}
\newcommand{\opt}{\mathrm{OPT}}
\newcommand{\optlfp}{\mathrm{OPT}_{\mathrm{LFP}}}
\newcommand{\optlp}{\mathrm{OPT}_{\mathrm{LP}}}
\newcommand{\set}[1]{\{#1\}}
\newcommand{\Set}[1]{\left\{#1\right\}}
\newcommand{\vd}{\mb{d}}
\newcommand{\vr}{\mb{r}}
\newcommand{\vw}{\mb{w}}
\newcommand{\norm}[1]{\left\|{#1}\right\|}
  \newcommand{\sbar}{\bar{s}}
  \newcommand{\abar}{\bar{a}}  
  
In this section we prove the following quantitative version of \pref{prop:efficient_tabular}.

\begin{propmod}{prop:efficient_tabular}{a}
  \label{prop:efficient_tabular2}
  Let $\Mbar\in\cM$ and $\eta\geq{}e$ be given. Suppose there exists
  $\delta\in(0,1)$ such that i) the initial distribution has
  $d_1(s)\geq \delta $ for all $s$, and ii) for all $(s,a,s')$ and
  $h\in\brk{H}$, $\Pmbar_h(s'|s,a) \geq \delta$. Then an approximate version \pcigw
  algorithm (\pref{alg:policy_cover_igw}) which satisfies
  \begin{align*}
\sup_{M\in\cM}\En_{\act\sim{}p}\brk*{\fm(\pim)-\fm(\pi)
  -\gamma\cdot\Dhels{M(\act)}{\Mbar(\act)}}
\leq{} 175\frac{H^3SA}{\gamma}.
  \end{align*}
  can be implemented in
  $\poly(H,S,A,\log(\eta/\delta))$ time with high probability via linear programming.
\end{propmod}
\pref{prop:efficient_tabular2} requires the additional assumption that
the initial state probabilities and transition probabilities are lower
bounded by a constant $\delta$. This assumption is fairly mild because
the runtime scales with $\log(\delta^{-1})$. In fact, when
\pref{alg:policy_cover_igw} is invoked within \mainalg, where $\Mbar$
represents an estimator, one can always modify $\Mbar$ such that
$\delta=1/\poly(T)$, without worsening the regret by more than a
constant factor; we omit the details.

\begin{proof}[\pfref{prop:efficient_tabular2}] It suffices to provide an algorithm that, for any
fixed $\bar{s}\in\cS$, $\bar{a}\in\cA$, and $\bar{h}\in\brk{H}$, solves the
optimization problem.
\begin{equation}
  \label{eq:efficient_tabular0}
\pi_{\bar h, \bar s, \bar a} =
\argmax_{\pi\in\PiRNS}\frac{\dm{\Mbar}{\pi}_{\bar h}(\bar s, \bar a)}{2HSA +
    \eta(\fmbar(\pimbar)-\fmbar(\pi))}.
\end{equation}
The following proposition---proven at the end---shows that an
approximate solution suffices.
\begin{proposition}
  \label{prop:approximate_igw_tabular}
  Consider the setting of \pref{prop:efficient_tabular2}, with parameter
  $\delta>0$.  If we run the \pcigw strategy in
  \pref{alg:policy_cover_igw} with a collection of policies $\PiCov=\crl{\pi_{h,s,a}}_{h\in\brk{H},s\in\brk{S},a\in\brk{A}}$ such that
  \begin{equation}
    \label{eq:pcigw_approx}
    \frac{\dm{\Mbar}{\pi_{h,s,a}}_{h}(s, a)}{2HSA +
    \eta(\fmbar(\pimbar)-\fmbar(\pi_{h,s,a}))} \geq 
    \argmax_{\pi\in\PiRNS}\frac{\dm{\Mbar}{\pi}_h(s,a)}{2HSA +
      \eta(\fmbar(\pimbar)-\fmbar(\pi))} - \frac{\delta}{4HSA + 2\eta},
  \end{equation}
  then by setting $\eta=\frac{\gamma}{41H^2}$, we have
  \begin{align*}
\sup_{M\in\cM}\En_{\act\sim{}p}\brk*{\fm(\pim)-\fm(\pi)
  -\gamma\cdot\Dhels{M(\act)}{\Mbar(\act)}}
\leq{} 175\frac{H^3SA}{\gamma}.
  \end{align*}
\end{proposition}

To proceed, we formulate \pref{eq:efficient_tabular0} as a
linear-fractional program and solve it.

\paragraph{Formulating the problem as a linear-fractional program}
Let $(\hbar,\sbar,\abar)$ be fixed. We adopt a standard dual approach
(e.g., \citet{neu2020unifying}) and---rather than optimizing over
policies directly---optimize over occupancy measures, from which an
optimal policy $\pi_{\bar h,\bar s,\bar a}$ can extracted. Let $\vr \ldef
\prn[\big]{\En_{r_h\sim{}\Rmbar_h(s,a)}\brk{r_h}}_{(h,s,a)\in\brk{H}\times\cS\times\cA}$ be the vector of average rewards for $\Mbar$. We
consider a following linear-fractional program with decision
variables $\vd = (d_{h,s,a}) _{(h,s,a)\in\brk{H}\times\cS\times\cA}$ representing feasible occupancy measures for $\Mbar$:
\begin{equation}
    \label{eq:linear-fractional-program}
    \begin{aligned}
        \maximize_{\vd}\quad& \frac{ d_{\bar h,\bar s,\bar a}}{2HSA + \eta\prn*{  \fmbar(\pimbar)   -     \tri*{ \vd, \vr }  }} , \\
        \textrm{subject to}\quad&d_{h,s,a} \geq 0, & \forall h\in\brk{H},s\in\cS,a\in\cA\\
        &\sum\limits_{a} d_{1,s,a} = d_1(s), & \forall s\in  \cS,\\
        &\sum\limits_{s,a} d_{h,s,a} \Pmbar_h(s'|s,a) =
        \sum\limits_{a} d_{h+1,s',a}, & \forall h\in\brk{H-1}, s'\in\cS.
    \end{aligned}
  \end{equation}
This is a linear-fractional program of $HSA$ decision variables with
$HSA + HS $ constraints.  We let $\optlfp$ denote the value of the program.

\paragraph{Applying the Charnes-Cooper transformation}
Next, we apply the Charnes-Cooper transformation
\citep{charnes1962programming} to transform the linear-fractional
program above into a linear pogram wth $HSA+1$ decision variables and
$HSA + HS +2$ constraints. The new program has decision variables $\vw =
(w_{h,s,a})_{h\in\brk{H},s\in\cS,a\in\cA}$ and $t\in \bbR$, with the
(implicit) correspondence
\begin{align*}
\vw &= \frac{\vd}{2HSA + \eta\prn*{  \fmbar(\pimbar)   -   \tri*{\vd, \vr}  }  },\\
t &= \frac{1}{2HSA + \eta\prn*{  \fmbar(\pimbar)   -   \tri*{\vd, \vr}  }  }\;,
\end{align*}
and is defined as follows:
\begin{equation}
    \label{eq:linear-program}
    \begin{aligned}
    \maximize_{\vw, t}\quad&  w_{\bar h, \bar s,\bar a}, & \\
    \textrm{subject to}\quad
    &2HSAt + \eta\fmbar(\pimbar) t  -     \eta\tri*{\vw, \vr}  = 1, & \\
    &\sum\limits_{s,a} w_{h,s,a} \Pmbar_h(s'|s,a) = \sum\limits_{a}
    w_{h+1,s',a}, & \forall h \in [H-1], s'\in\cS,\\
    &\sum\limits_{a} w_{1,s,a} = t d_1(s), & \forall s \in  \cS,\\
    &1\geq w_{h,s,a} \geq 0, & \forall h\in\brk{H}, s\in\cS, a\in\cA, \\
    &1\geq t \geq 0.
    \end{aligned}
\end{equation}
Let $\optlp$ denote the value of this program, which has the following properties:
\begin{itemize}
  \item Value is preserved: $\optlfp = \optlp$.
  \item For feasible point $(\vw, t)$ for \pref{eq:linear-program}, the corresponding solution $\vd
    = \vw/t$ is feasible to the linear-fractional program
    \pref{eq:linear-fractional-program}.

\item If $\vd$ is feasible for \pref{eq:linear-fractional-program},
  then the variables $(\vw,t)$  given by
  \begin{align*}
w_{h,s,a} &= \frac{d_{h,s,a}}{2HSA + \eta\prn*{  \fmbar(\pimbar)   -   \tri*{\vd, \vr}  }  } \leq \frac{1}{HSA} \leq 1, \\
t &= \frac{1}{2HSA + \eta\prn*{  \fmbar(\pimbar)   -   \tri*{\vd, \vr}  }  }\leq \frac{1}{HSA} \leq 1.
\end{align*}
are feasible for \pref{eq:linear-program}.
\end{itemize}

\paragraph{Solving the linear program}

To deduce the final runtime bound, we appeal to a standard linear
program solver for \pref{eq:linear-program}.

\begin{proposition}[\cite{lee2019solving}, Theorem 1\protect\footnote{Theorem 1 in \cite{lee2019solving} is stated
      with constant probability, but the high probability result here
      trivially follows via confidence boosting.}]
\label{prop:lpsolve}
Let $\opt$ denote the value of the linear program
\begin{equation}
  \begin{aligned}
    \maximize_{x} \quad& \tri{c,x} , &\\
    \textrm{\textnormal{subject to}}\quad & G^\trn x= b, & \\
    & l_i \leq x_i \leq u_i, & \forall i\in [m],
  \end{aligned}
\end{equation}
where $G \in \bbR^{m\times n}$, $b\in \bbR^n$, $c\in \bbR^m$ and
$l_i\in \bbR\cup\set{-\infty}, u_i\in \bbR\cup\set{\infty}$ are given
$x\in\bbR^{m}$ is the decision variable. Suppose the following technical conditions hold:
\begin{itemize}
    \item $G^\top$ has full row-rank, i.e., no constraints are
      linearly dependent (this implies that $m\geq n$).
    \item  For each $i\in\brk{m}$, $\mathrm{dom}(x_i) \ldef \set{x :
        l_i < x <u_i}$ is neither the empty set nor the entire real line.
    \item  The interior $\Omega \ldef \set{x\in R^m :G^\top x = b, l_i < x_i < u_i }$ is not empty.
    \end{itemize}
    There exists a randomized algorithm $\lpsolve$ which, for any
    $\veps>0$, $\alpha>0$, and initial
    point $x^0\in \Omega$, outputs a point $x\in \Omega$ for which $c^\top x
    \geq \opt - \veps $ with probability at least
    $1-\alpha$, and does so using at most
    \begin{align*}
        \bigoh(n^{1.5}m^2  \log^{13}m \cdot \log(mU/\veps)\log(1/\alpha))
    \end{align*}
    steps, where $U \ldef{} \max \set{1/\nrm{u-x^0}_\infty, 1/\nrm{x^0-l}_\infty, \nrm{u-l}_\infty, \nrm{c}_\infty }$.
\end{proposition}

We now address the technical conditions required to apply $\lpsolve$
to \pref{eq:linear-program}.
\begin{itemize}
    \item \emph{No constraints are linearly dependent.} This can be
      achieved by applying Gaussian elimination to remove dependent
      constraints, which takes no more than $\bigoh(H^3S^3A^3)$ steps.
    \item \emph{The domain of each decision variable is
    neither the empty set or the entire real line.} This is trivially satisfied, since $1\geq t,\vw\geq 0$.
\item \emph{The interior of the polytope is not empty.} Below we
  construct an initial point $x^{0}\in\Omega$, certifies that the
  interior is non-empty.

\end{itemize}
\emph{Finding an initial point.}
We construct an interior point $x^{0} = (\vw^0,t^0)\in\Omega$ by first
finding an interior point $\vd^0$ of the linear-fractional program
\pref{eq:linear-fractional-program}, and then applying the
Charnes-Cooper transformation to obtain  $(\vw^0,t^0)$.
First, note that under
the assumption in \pref{prop:efficient_tabular2} we have that for any
policy $\pi$, $\sum_{a} \dm{\Mbar}{\pi}_h(s,a) \geq \delta$ for all
$h\in\brk{H}$, $s\in\cS$. Thus, the uniform policy
$\piunif(s)\ldef{}\unif(\cA)$ induces an interior point $\vd^0 =
(d_{h,s,a}^0)_{h\in\brk{H},s\in\cS,a\in\cA}$ for the linear-fractional program
\pref{eq:linear-fractional-program} by taking
\begin{align*}
d_{h,s,a}^0 \ldef \dm{\Mbar}{\piunif}_h(s,a) = \frac{1}{A}\dm{\Mbar}{\piunif}_h(s) \geq \frac{\delta}{A}.
\end{align*}
Hence, if we obtain $(\vw^0,t^0)$ by applying the
Charnes-Cooper transformation to $\vd^0$, we have
\begin{gather*}
\frac{1}{HSA} \geq w^0_{h,s,a}  \ldef  \frac{d_{h,s,a}^0}{2HSA +
  \eta\prn*{  \fmbar(\pimbar)   -     \tri*{\vd^0, \vr}  } }   \geq
\frac{\delta}{2A(HSA + \eta)}\,,
\intertext{and}
\frac{1}{HSA}\geq t^0 \ldef  \frac{1}{2HSA + \eta\prn*{  \fmbar(\pimbar)   -     \tri*{\vd^0, \vr}  } } \geq \frac{1}{2(HSA + \eta)}\,.
\end{gather*}
It follows that the interior of \pref{eq:linear-program} is non-empty,
and that by initializing with $x^0=(\vw^0, t^0)$, we have
\begin{align*}
U \leq{} \max \crl*{ \prn*{1-\tfrac{1}{HSA}}^{-1}, \prn*{\tfrac{\delta}{2A(HSA + \eta)} - 0}^{-1}, 1, 1} \leq 2A(HSA+\eta)/\delta.
\end{align*}

\paragraph{Final guarantee}
 Applying $\lpsolve$ from $(\vw^0, t^0)$, in
 $\bigoht\prn*{(HSA)^{3.5}\log\prn*{\eta/(\delta\veps)}\log(1\alpha)}$ steps we obtain an interior point
 $(\wt{\vw}, \wt{t})$ for which
\begin{align*}
\wt{w}_{\bar h, \bar s, \bar a}  \geq \optlp - \veps,
\end{align*}
with probability at least $1-\alpha$. To obtain a policy, we define
an occupancy measure $\wt{\vd} = \wt{\vw}/\wt{t}$, then take
\[
  \wt{\pi}_h(s,a) = \wt{d}_{h,s,a}/ \sum_{a\in \cA} 
  \wt{d}_{h,s,a}.
\]
This policy has the property that
$\dm{\Mbar}{\wt{\pi}}_h(s,a)=\wt{d}_h(s,a)$. As a result, the policy satisfies
\begin{align*}
    \frac{\dm{\Mbar}{\wt{\pi}}_{\bar h}(\bar s, \bar a)}{2HSA +
    \eta(\fmbar(\pimbar)-\fmbar(\wt{\pi}))} 
    &=  \frac{ \wt{d}_{\bar h, \bar s,\bar a}}{2HSA + \eta\prn[\big]{  \fmbar(\pimbar)   -     \tri[\big]{\wt{\vd}, \vr}  }} 
  = \wt{w}_{\bar h, \bar s, \bar a},
\end{align*}
and hence
\begin{align*}
    \frac{\dm{\Mbar}{\wt{\pi}}_{\bar h}(\bar s, \bar a)}{2HSA +
    \eta(\fmbar(\pimbar)-\fmbar(\wt{\pi}))}       \geq \optlp - \veps 
    = \optlfp - \veps 
    = \max_{\pi\in\PiRNS}\frac{\dm{\Mbar}{\pi}_{\bar h}(\bar s, \bar a)}{2HSA + 
    \eta(\fmbar(\pimbar)-\fmbar(\pi))}  -\veps.
\end{align*}
To conclude, we set $\veps$ so that the condition of
\pref{prop:approximate_igw_tabular} holds.

\end{proof}

\begin{proof}[Proof sketch for \pref{prop:approximate_igw_tabular}]
    We show that the proof of \pref{prop:igw_tabular} goes through
    essentially as-is under the condition in \pref{eq:pcigw_approx}. Observe that each policy satisfies
    \begin{align*}
        \frac{\dm{\Mbar}{\pi_{h,s,a}}_{h}(s, a)}{2HSA +
        \eta(\fmbar(\pimbar)-\fmbar(\pi_{h,s,a}))} 
        &\geq 
        \argmax_{\pi\in\PiRNS}\frac{\dm{\Mbar}{\pi}_h(s,a)}{2HSA +
          \eta(\fmbar(\pimbar)-\fmbar(\pi))} - \frac{\delta}{4HSA + 2\eta}\\
        &\geq  \frac{1}{2} \argmax_{\pi\in\PiRNS}\frac{\dm{\Mbar}{\pi}_h(s,a)}{2HSA +
        \eta(\fmbar(\pimbar)-\fmbar(\pi))}.
    \end{align*}
    As a result, Eq. \pref{eq:pcigw_multiplicative} in the proof of \pref{prop:igw_tabular} continues to hold up to a factor of $2$, and changing the parameters $\eta$ and $\eta'$ accordingly yields the result.
\end{proof}

\subsection{Proofs from \preft{sec:contextual}}
\label{app:contextual}

\uppercontextual*
\begin{proof}[\pfref{thm:upper_contextual}]%
  \newcommand{\cont}{\con\ind{t}}%
  This proof follows the same template as \pref{thm:upper_general,thm:upper_general_distance}. We have
\begin{align*}
  \RegCDM &=
           \sum_{t=1}^{T}\En_{\act\ind{t}\sim{}p\ind{t}}\brk*{\fmstar(\cont,\cpolstar(\cont))-\fmstar(\cont,\act\ind{t})}\\
         &=
           \sum_{t=1}^{T}\En_{\act\ind{t}\sim{}p\ind{t}}\brk*{\fmstar(\cont,\cpolstar(\cont))-\fmstar(\cont,\act\ind{t})}
           - \gamma{}\cdot
           \En_{\act\ind{t}\sim{}p\ind{t}}\brk*{\Dgen{\Mstar(\cont,\act\ind{t})}{\Mhat\ind{t}(\cont,\act\ind{t})}}\\
  &\qquad+ \gamma\cdot{}\EstCD.
\end{align*}
For each $t$, since $\Mstar\in\cM$, we have
\begin{align}
  &\En_{\act\ind{t}\sim{}p\ind{t}}\brk*{\fmstar(\cont,\cpolstar(\cont))-\fmstar(\cont,\act\ind{t})}
           - \gamma{}\cdot
  \En_{\act\ind{t}\sim{}p\ind{t}}\brk*{\Dgen{\Mstar(\cont,\act\ind{t})}{\Mhat\ind{t}(\cont,\act\ind{t})}}\notag\\
&  \leq  \sup_{M\in\cM}\En_{\act\ind{t}\sim{}p\ind{t}}\brk*{\fm(\cont,\cpolm(\cont))-\fm(\cont,\act\ind{t})}
           - \gamma{}\cdot
\En_{\act\ind{t}\sim{}p\ind{t}}\brk*{\Dgen{M(\cont,\act\ind{t})}{\Mhat\ind{t}(\cont,\act\ind{t})}}\notag\\
  &  = \inf_{p\in\Delta(\Act)}\sup_{M\in\cM}\En_{\act\sim{}p}\brk*{\fm(\cont,\cpolm(\cont))-\fm(\cont,\act)}
           - \gamma{}\cdot
    \En_{\act\sim{}p}\brk*{\Dgen{M(\cont,\act)}{\Mhat\ind{t}(\cont,\act)}}\notag\\
  &  = \compd(\cMx[\cont],\Mhat\ind{t}(\cont,\cdot)).\notag
\end{align}
We conclude that
\[
  \RegCDM \leq{} \sup_{\con\in\cX}\sup_{\Mbar\in\cMhat}\compd(\cMx[\con],\Mbar(\con,\cdot))\cdot{}T + \gamma\cdot\EstCD.
\]
\end{proof}

\subsection{Proofs from \preft{sec:related}}
\label{app:related}

\informationratiobayes*

\begin{proof}[\pfref{prop:information_ratio_bayes_lower}]%
  For the upper bound on the \CompShort, refer to \pref{sec:bandit}. We
  focus on proving the lower bound on the information ratio. We
  consider the case $d=1$, since this immediately implies a lower
  bound for higher dimensions.

  We define a subclass of models as follows. Let $\veps\in(0,1/2)$ be fixed, and let
$N=\floor*{\frac{1}{\veps}}\geq{}\frac{1}{2\veps}$. For each
$i\in\brk*{N}$, define $\act_i=\veps\cdot{}i-\veps/2$. Let
$h(\act)=\max\crl*{1-\abs{\act},0}$, which is $1$-Lipschitz. For
each $i$, define $M_i$ via
\[
f\sups{M_i}(\act)=f_i(\act) \ldef \frac{1}{2} + \veps\cdot{}h((\act-\act_i)/\veps),
\]
which is also $1$-Lipschitz. Finally, define $\Mbar$ via
$f\sups{\Mbar}(\act)=\fbar(\act)\ldef\frac{1}{2}$.
Let $\cI_i=\brk*{\act_i-\veps/2, \act_i+\veps/2}$, and observe that
\begin{equation}
  f_i(\pimi) - f_i(\act)\geq{}\veps\cdot\indic\crl*{\act\notin\cI_i}
  \label{eq:ratio_fi}
\end{equation}
and
\begin{equation}
  \Dkl{M_i(\act)}{\Mbar(\act)} = \frac{1}{2}(f_i(\act)-\fbar(\act))^2\leq{}\frac{\veps^2}{2}\indic\crl*{\act\in\cI_i}.\label{eq:ratio_fi_kl}
\end{equation}
    Observe that for any $x\geq{}0,y\geq{}0$, 
\begin{equation}
  \frac{x}{y}=\sup_{\eta>0}\crl*{2\eta{}x^{1/2} - \eta^{2}y},
  \label{eq:ratio_variational}
\end{equation}
with the convention that $0/0=0$. As a result, we can lower bound
\begin{align*}
\InfB(\cM,\Mbar) &=\sup_{\mu\in\Delta(\cM)}\inf_{p\in\Delta(\Act)}\frac{\prn*{\En_{\act\sim{}p}\En_{M\sim\mu}\brk*{\fm(\act)-\fm(\pim)}}^2}{
\En_{\act\sim{}p}\En_{M\sim\mu}\brk*{\Dkl{M(\act)}{\Mbar(\act)}}} \\
  &\geq
\sup_{\mu\in\Delta(\cM)}\inf_{p\in\Delta(\Act)}\sup_{\eta>0}\crl*{2\eta\cdot{}\En_{\act\sim{}p}\En_{M\sim\mu}\brk*{\fm(\act)-\fm(\pim)}
-\eta^{2}\cdot{}\En_{\act\sim{}p}\En_{M\sim\mu}\brk*{\Dkl{M(\act)}{\Mbar(\act)}}}.
\end{align*}
Let $\mu=\unif(\crl{M_i}_{i=1}^{N})$. Then, using the expressions in
\pref{eq:ratio_fi} and \pref{eq:ratio_fi_kl}, we have that for all
distributions $p\in\Delta(\Act)$,
\begin{align*}
  &\sup_{\eta>0}\crl*{2\eta\cdot{}\En_{\act\sim{}p}\En_{M\sim\mu}\brk*{\fm(\act)-\fm(\pim)}
    -\eta^{2}\cdot{}\En_{\act\sim{}p}\En_{M\sim\mu}\brk*{\Dkl{M(\act)}{\Mbar(\act)}}}\\
  &\geq{}\sup_{\eta>0}\crl*{2\eta\cdot{}(\veps-1/N)
    -\eta^{2}\frac{\veps^2}{2N}} \\
  &\geq{}\sup_{\eta>0}\crl*{\eta\veps
-\eta^{2}\frac{\veps^2}{2N}}.
\end{align*}
This expression is optimized by taking $\eta=\frac{N}{\veps}$, which
gives
\[
\InfB(\cM,\Mbar) \geq{} \frac{N}{2}\geq{}\frac{1}{4\veps}.
\]
Since this argument holds uniformly for all $\veps\in(0,1/2)$, we take $\veps\to{}0$ and conclude that $\InfB(\cM,\Mbar)=+\infty$.

\end{proof}

\informationratiofreq*

\begin{proof}[\pfref{prop:information_ratio_separation}]\newcommand{\afbar}{\pimbar}%
Let $\Mbar$ be given. Using \pref{eq:ratio_variational}, we can lower bound
\begin{align*}
\InfF(\cM,\Mbar) &=\inf_{p\in\Delta(\Act)}\sup_{\mu\in\Delta(\cM)}\frac{\prn*{\En_{\act\sim{}p}\En_{M\sim\mu}\brk*{\fm(\act)-\fm(\pim)}}^2}{
\En_{\act\sim{}p}\En_{M\sim\mu}\brk*{\Dkl{M(\act)}{\Mbar(\act)}}} \\
  &\geq
\inf_{p\in\Delta(\Act)}\sup_{\mu\in\Delta(\cM)}\sup_{\eta>0}\crl*{2\eta\cdot{}\En_{\act\sim{}p}\En_{M\sim\mu}\brk*{\fm(\act)-\fm(\pim)}
-\eta^{2}\cdot{}\En_{\act\sim{}p}\En_{M\sim\mu}\brk*{\Dkl{M(\act)}{\Mbar(\act)}}}\\
&\geq
\inf_{p\in\Delta(\Act)}\sup_{M\in\cM}\sup_{\eta>0}\crl*{2\eta\cdot{}\En_{\act\sim{}p}\brk*{\fm(\act)-\fm(\pim)}
-\eta^{2}\cdot{}\En_{\act\sim{}p}\brk*{\Dkl{M(\act)}{\Mbar(\act)}}}.
\end{align*}
We further rewrite this as
\begin{align}
  &\inf_{p\in\Delta(\Act)}\sup_{\eta>0}\sup_{\actstar\in\Act}\sup_{M\in\cM}\crl*{2\eta\cdot{}\En_{\act\sim{}p}\brk*{\fm(\act)-\fm(\actstar)}
    -\eta^{2}\cdot{}\En_{\act\sim{}p}\brk*{\Dkl{M(\act)}{\Mbar(\act)}}}\notag\\
  &=\inf_{p\in\Delta(\Act)}\sup_{\eta>0}\sup_{\actstar\in\Act}\sup_{f\in\brk{0,1}^{A}}\En_{\act\sim{}p}\brk*{2\eta\cdot{}(f(\act)-f(\actstar))
-\eta^{2}\cdot{}(f(\act)-\fbar(\act))^{2}},\label{eq:info_ratio_freq_lb}
\end{align}
where $\fbar\ldef{}f_{\Mbar}$. 

Consider any $p\in\Delta(\Act)$. Recall that
$\Delta\ldef{}\min_{\act\neq\pimbar}\crl{\fmbar(\act)-\fmbar(\pimbar)}>0$ and $\min\fbar>0$ (since $\fbar\in\mathrm{int}(\brk{0,1}^{A})$). We
consider two cases. First, if some $\actstar\in\Act$ has $p_{\actstar}=0$, we can
choose the vector $f$ in the supremum in \pref{eq:info_ratio_freq_lb} to have $f(\act)=\fbar(\act)$ for $\act\neq\actstar$ and $f(\actstar)=0$. In this
case, for any fixed $\eta$, the value in \pref{eq:info_ratio_freq_lb} is
\[
  2\eta\sum_{\act\neq\actstar}p_{\act}\fbar(\act)\geq{}2\eta\cdot{}\min_{\act}\fbar(\act).
\]
Hence, if $\min\fbar>0$, the value is $+\infty$, since we can take
$\eta$ to be arbitrarily large.

For the second case, suppose that $p\in\mathrm{int}(\Delta(\Act))$ and let $\eta$ and $\actstar$ be fixed. The first-order
conditions for optimality imply that the maximizer for $f$ in \pref{eq:info_ratio_freq_lb} is given by
$f(\act) = \fbar(\act) + \frac{1}{\eta}$ for $\act\neq\actstar$ and
$f(\actstar)=\fbar(\actstar)-
\frac{(1-p_{\actstar})}{\eta{}p_{\actstar}}$; the feasibility of this
choice will be verified shortly. The resulting value is
\begin{align*}
  &2\eta\sum_{\act}p_{\act}(\fbar(\act)-\fbar(\actstar)) + \frac{2}{p_{\actstar}}
    - (1-p_{\actstar}) - \frac{(1-p_{\actstar})^{2}}{p_{\actstar}}\\
    &\geq{} 2\eta\sum_{\act}p_{\act}(\fbar(\act)-\fbar(\actstar)) + \frac{1}{p_{\actstar}}
  - 1.
\end{align*}
If we choose $\actstar=\afbar$, the expression in \pref{eq:info_ratio_freq_lb} is lower bounded by
\[
2\eta\Delta(1-p_{\afbar}).
\]
Since $p\in\mathrm{int}(\Delta(\Act))$, we have $(1-p_{\afbar})>0$, and hence
we can drive the value to $+\infty$ by choosing $\eta$
arbitrarily large. Furthermore, since
$\fbar\in\mathrm{int}(\brk{0,1}^{A})$, we have  $f\in\brk*{0,1}^{A}$
as required once $\eta$ is sufficiently large.
  
\end{proof}

\end{document}